%% file: main.tex
\pdfoutput=1
\documentclass{article}

\usepackage{CustomCommands}

\begin{document}
\title{Learning payoffs while routing in skill-based queues}
\author{Sanne van Kempen, Jaron Sanders, Fiona Sloothaak, Maarten G. Wolf}
\date{}
\maketitle

\begin{abstract}
\input{Abstract}

\end{abstract}

\section{Introduction}
\label{sec:introduction}
\input{Introduction}

\section{Model description}
\label{sec:problem_setting}
\input{ProblemSetting}

\section{Regret lower bound}
\label{sec:regret}
\input{RegretResults}

\section{Adaptive queue routing policy}
\label{sec:learning_alg}

\input{LearningAlgorithm}

\section{Numerical results}
\label{sec:simulation}
\input{NumericalResults}

\newpage
\appendix
\section{Proof of Theorem~\ref{thm:regret_lb}}
\label{app:proof_lb}
\input{AppendixLB}

\newpage
\section{Proof of Theorem~\ref{thm:regret_ub}}
\label{app:proof_ub}
\input{AppendixUB}

\newpage
\section{Remaining proofs}
\input{AppendixOther}

\section*{Acknowledgements}
The authors would like to thank Thomas van Vuren for his insightful comments and fruitful discussions.\\

This work is part of \emph{Valuable AI}, a research collaboration between the Eindhoven University of Technology and the Koninklijke KPN N.V.
Parts of this research have been funded by the EAISI's IMPULS program, and by Holland High Tech | TKI HSTM via the PPS allowance scheme for public-private partnerships.

\bibliographystyle{unsrt}
\bibliography{Literature}
\end{document}

%% file: Abstract.tex
Motivated by applications in service systems, we consider queueing systems where each customer must be handled by a server with the right skill set. 
We focus on optimizing the routing of customers to servers in order to maximize the total payoff of customer--server matches.
In addition, customer--server dependent payoff parameters are assumed to be unknown \emph{a priori}. 
We construct a machine learning algorithm that adaptively learns the payoff parameters while maximizing the total payoff and prove that it achieves polylogarithmic regret.
Moreover, we show that the algorithm is asymptotically optimal up to logarithmic terms by deriving a regret lower bound. 
The algorithm leverages the basic feasible solutions of a static linear program as the action space. 
The regret analysis overcomes the complex interplay between queueing and learning by analyzing the convergence of the queue length process to its stationary behavior. 
We also demonstrate the performance of the algorithm numerically, and have included an experiment with time-varying parameters highlighting the potential of the algorithm in non-static environments. 

%% file: Introduction.tex
Service systems such as contact centers, computer networks, and manufacturing systems are widely used in practice \citep{Harchol2013,Chen2020,Koole2002}. 
Achieving the highest possible quality of service in such systems is consequently of general importance. 
However, actually doing so is typically quite challenging because of complex interactions within the system, between e.g.\ customers and servers.
Moreover, service provisioning is an intrinsically uncertain process.

In this paper, we develop a machine learning algorithm that can attain the highest possible performance in one such service system.
Specifically, we focus on a finite skill-based queueing system with customer--server dependent random payoffs.
One such system is illustrated in Figure~\ref{fig:intro_qsystem_example}.
We consider different customers to have different needs, and different servers to be better at helping certain types of customer over others.
In fact, some servers might even be unable to help some types of customers. 
We model these aspects by assuming that there are compatibility relations between customers and servers, and by assuming that whenever a customer is served by a server, a random payoff is generated that depends on the specific customer--server pairing.
Since the different servers are shared between the different customer types, and since they are limited in their number, the highest possible average reward can only be achieved by optimizing how one matches customers to servers.

\begin{figure}[h]
    \begin{center}
        \begin{tikzpicture}[
            server/.style={circle, minimum size = .8cm, thick,draw},
            scale=.7
            ]
            \scriptsize
            \foreach \x in {1,2,3,4,5}{ 
                \node[draw = none] at (4*\x,0) (queue\x) {};
                \node[above of = queue\x, yshift = -.5cm] (labelqueue\x) {};
                \draw[thick] (queue\x.center) --++(-.5,0) --++(0,1.2); 
                \draw[thick] (queue\x.center) --++(.5,0) --++(0,1.2); 
                \node[draw = none, above of = queue\x, yshift = .2cm] {$\downarrow \lambda_\x$};
            }

            \fill[black] ($(queue1.center) + (-.4,.4)$) rectangle ++(.8,-0.3);
            \fill[black] ($(queue1.center) + (-.4,.8)$) rectangle ++(.8,-0.3);
            \fill[black] ($(queue1.center) + (-.4,1.2)$) rectangle ++(.8,-0.3);
            \fill[black] ($(queue2.center) + (-.4,.4)$) rectangle ++(.8,-0.3);
            \fill[black] ($(queue3.center) + (-.4,.4)$) rectangle ++(.8,-0.3);
            \fill[black] ($(queue5.center) + (-.4,.4)$) rectangle ++(.8,-0.3);
            \fill[black] ($(queue5.center) + (-.4,.8)$) rectangle ++(.8,-0.3);

            \foreach \x in {1,2,3,4,5}{
                \node[server] at (4*\x,-4) (server\x) {};
                \node at (server\x)  {$\mu_\x$};
            }

            \draw[very thick, black] (queue1.center) -- (server1.north) node[pos=0.3, left]
            {$\theta_{11}$};
            \draw[very thick, black] (queue1.center) -- (server2.north) node[pos=0.25, left,yshift=-.05cm]
            {$\theta_{12}$};
            \draw[very thick, black] (queue1.center) -- (server3.north) node[pos=0.1, right, yshift = .1cm]
            {$\theta_{13}$};

            \draw[very thick, black] (queue2.center) -- (server1.north);
            \draw[very thick, black] (queue2.center) -- (server2.north);
            \draw[very thick, black] (queue2.center) -- (server3.north);
            \draw[very thick, black] (queue2.center) -- (server5.north);
            \draw[very thick, black] (queue3.center) -- (server4.north);
            \draw[very thick, black] (queue3.center) -- (server5.north);
            \draw[very thick, black] (queue4.center) -- (server5.north);
            \draw[very thick, black] (queue5.center) -- (server3.north);
            \draw[very thick, black] (queue5.center) -- (server4.north);
            \draw[very thick, black] (queue5.center) -- (server5.north);
        \end{tikzpicture}
    \end{center}
    \caption{A skill-based queueing system with compatibility lines and customer--server dependent payoffs. Here, the $\lambda_i$ denote arrival rates, the $\mu_j$ denote service rates, and the $\theta_{ij}$ indicate the average payoff generated upon service completion of a type-$i$ customer at server $j$.}
    \label{fig:intro_qsystem_example}
\end{figure}
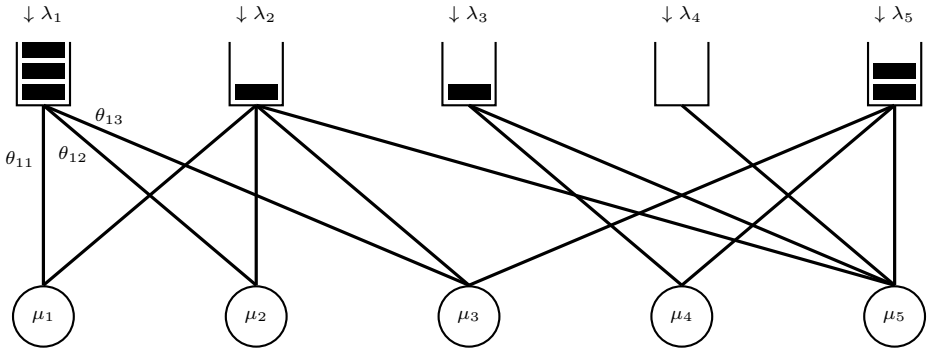

Within queueing literature, a canonical tool for doing server allocation is routing policies.
Routing policies decide which customer is served by which server, and at what time.
However, routing policies typically do not take into account (i) compatibility relations or (ii) payoffs generated by different customer--server pairs. 
Furthermore, they usually do not consider (iii) uncertainty in and/or a lack of knowledge on  model parameters (such as the average payoffs $\theta_{ij}$).

Our machine learning algorithm, Algorithm~\ref{alg:learning_alg} in Section~\ref{sec:learning_alg}, does take into account aspects (i)--(iii). 
Letting $D_{ij}(t)$ refer to the number of departures of type-$i$ customers at server $j$ up to time $t$, and $\calL$ to the set of compatibility lines, Algorithm~\ref{alg:learning_alg} maximizes the long-term expected reward rate

\begin{align}
    \frac 1t \sum_{(ij)\in\calL} \theta_{ij} \E(D_{ij}(t)).
    \label{eq:intro_reward}
\end{align}

Algorithm~\ref{alg:learning_alg} does so while 
(a) honoring the queueing dynamics and compatibility relations;
(b) guaranteeing the stability of all queues; 
(c) being able to hone in on the true system parameters $\theta_{ij}$ when either misspecified or unknown \emph{a priori}; 
and 
(d) dealing with the exploration--exploitation dilemma efficiently. 
These aspects (a)-(d) will be discussed in more detail in Section~\ref{sec:Main-contributions} below, but 
may be summarized in one sentence as follows: we establish in Theorem~\ref{thm:regret_ub} that Algorithm~\ref{alg:learning_alg} achieves a polylogarithmic regret, and in Theorem~\ref{thm:regret_lb} that this is (asymptotically), up to logarithmic terms, the best possible  within a certain class of stable policies.

The nature of the optimization problem that we study, required us to combine techniques from different fields to develop Algorithm~\ref{alg:learning_alg} and prove Theorems~\ref{thm:regret_lb} and \ref{thm:regret_ub}.
Firstly, we were inspired by techniques for \gls{MAB} problems from \emph{machine learning} for the learning aspect.
Secondly, in the regret analysis of Algorithm~\ref{alg:learning_alg}, we had to rely on \emph{queueing theory} to establish the typical behavior of the queue length process under the dynamic routing policy. 
Finally, a key idea was to leverage \emph{combinatorial optimization} to come up with `good' actions for Algorithm~\ref{alg:learning_alg}.

Let us briefly discuss these three facets in more detail:

\emph{
How queueing dynamics complicate decision making and how MAB techniques can be adapted to cope. }\\
Contrary to the classical \gls{MAB} setting \citep{Lai1985}, 
we have to deal with queueing dynamics and stability constraints that complicate the problem.
Specifically, both the decision moments and the set of available routing decisions are subject to the underlying random queueing dynamics:
we can only route a customer if there are a compatible customer and server available. 
Moreover, simply routing each customer to the server with the highest payoff might not result in a stable system in the long run. 

To overcome this, we define episodes and only make a decision on the routing strategy at the start of each episode. 
The estimated reward of the chosen action is updated after each episode based on observed payoff samples. 
This way, similar to the \gls{MAB} setting, the algorithm learns the average reward of each action adaptively. 
To face the exploration--exploitation tradeoff efficiently, we use \gls{UCB} for the payoff parameters \citep{Auer2002a}.

\emph{How queueing theory is used to analyze Algorithm~\ref{alg:learning_alg}.}\\
Each episode, Algorithm~\ref{alg:learning_alg} chooses an action that determines a fixed routing policy for the duration of that episode. 
This means that the queue length process will show different behavior in each episode. 
Worse yet, at the moment the routing policy changes, the queueing system will likely be far from its new equilibrium behavior. 
It then takes some time to adjust to the new environment. 
Still, in order to prove convergence results of Algorithm~\ref{alg:learning_alg}, we must show that sufficient payoff samples of all the different lines are collected. 
To this end, we utilize known results of the stationary behavior and analyze the typical time of convergence of the queue length process to stationarity.

\emph{How combinatorial optimization allows one to identify `good' actions.}\\
The optimal routing problem inherently poses combinatorial challenges.
In particular, the optimal routing problem can be considered as a stochastic variant of the optimal transport problem \citep{Bertsimas1997} with a finite set of possible solutions. 
We formulate an optimal transport \gls{LP} and exploit its structural properties in the construction of our algorithm:
first, we characterize the basic feasible solutions of the \gls{LP} in terms of routing rates.
We then consider the different sets of routing rates as candidate routing policies in an \gls{MAB} setting. 
In the analysis, we exploit the dual of the \gls{LP} to decompose the regret. 

\subsection{Main contributions}
\label{sec:Main-contributions}

Our main contributions can be summarized as follows:

{\bf Regret lower bound.}
We present an asymptotic regret lower bound for a class of routing policies satisfying certain stability constraints in Theorem~\ref{thm:regret_lb}. The stability constraints ensure that the queue backlog does not grow infinitely large. 
In particular, we prove that any policy of this class must suffer $\Omega(\ln(t))$ regret to learn the payoff parameters. 

In an \gls{MAB} setting, a regret lower bound is often described in terms of a \emph{suboptimality gap} for each suboptimal action. The gap measures the instant loss in reward by choosing this action over the optimal action \citep{Burnetas1996,Auer2002a,Lai1985}. 
A straightforward implementation of this approach proves to be difficult since in our model, 
some suboptimal routing decision might be necessary to maintain stability. 
Therefore, it is not trivial to identify the optimal action at each decision moment, nor to quantify the suboptimality gaps. 

Our solution to this problem is as follows: we split the lines into an `optimal' and `suboptimal' set, based on the optimal solution of the optimal transport \gls{LP}, and consequently quantify the suboptimality gaps using the dual. 
\citep{Burnetas2017,Burnetas2012}~use a similar method to obtain a regret lower bound in an \gls{MAB} setting with cost constraints, although their work is limited to programs with only two equality constraints. 
The proof of Theorem~\ref{thm:regret_lb} also uses a change--of--measure argument, a technique that is frequently used in the \gls{MAB} setting~\citep{Lai1985, Burnetas1996, Burnetas1997, Lattimore2020}.

{\bf Adaptive learning \& routing algorithm.}
We present a machine learning algorithm, that utilizes the basic feasible solutions of the optimal transport \gls{LP} as actions in Algorithm~\ref{alg:learning_alg}---a method that, to the best of our knowledge, has not been explored before in this setting.
We consider the related optimal transport \gls{LP}, where the objective function (that includes the payoff parameter  $\theta$) is unknown. 
However, by the properties of the \gls{LP}, the set of basic feasible solutions is finite and the optimal solution is attained at one of these solutions \citep{Bertsimas1997}. 
Hence, we identify the set of basic feasible solutions as actions. 
Each solution uniquely defines a set of routing rates. 
Moreover, the objective function is approximated using \gls{UCB} indices of the payoff parameters. 

The learning is schematically depicted in Figure~\ref{fig:architecture} and works as follows.
At the start of each episode, the algorithm chooses the action with maximal (estimated) reward.
During the episode, customers are routed according to the rates of the chosen action.
In particular, customers are assigned a label corresponding to one of the compatible servers upon arrival (illustrated by colors in Figure~\ref{fig:architecture}). 
Each server serves customers with matching label in \gls{FCFS} order, and service completions generate payoffs. 
If, at the start of an episode, the chosen action differs from the action of the previous episode, the algorithm assigns new labels for all queueing customers sampled according to the routing rates of the new action.

For each line, we maintain a \gls{UCB} index of its payoff parameter. 
At the end of the episode, the payoff samples are used to update the \gls{UCB} indices.
The reward of an action is then computed as the sum over all lines of the \gls{UCB} index weighted by the routing rate of the action.
We assume Bernoulli payoffs for simplicity. 

\begin{figure}[H]
    \centering
    \begin{subfigure}{1\linewidth}
        \centering
        \begin{tikzpicture}[
            node distance=1.5cm and 2cm,
            server/.style={circle, minimum size = .2cm, thick,draw},
                scale=.7]
                
                \foreach \x/\y in {1/black,2/black,3/black}{
                    \node[draw = none] at (-10+1*\x,-4.3) (queue\x) {};
                    \node[above of = queue\x, yshift = 0cm] (labelqueue\x) {};
                    \draw[thick] (queue\x.center) --++(-.2,0) --++(0,.6); 
                    \draw[thick] (queue\x.center) --++(.2,0) --++(0,.6);
                    \node[draw = none, above of = queue\x, yshift = -.8cm] {\scriptsize $ \lambda_\x$};
                    \node[server, below of=queue\x, yshift=.5cm]  (server\x) {};
                    \node at (server\x) {\scriptsize $\mu_\x$};
                }

                \foreach \x in {1,2,3} {
                    \foreach \y in {1,2,3} {
                        \draw[thick] (queue\x.center) -- (server\y.north);
                    }
                }
                \node[draw=none, right of = queue3, yshift=-.1cm, xshift=-.4cm] (arrow1) {\Huge $\Rightarrow$};
            
            \foreach \action/\label in {1/1,2/2,3/3,4/4,6/n} {
                \foreach \i in {1,2,3} {
                    \node[rectangle, draw, minimum size =.27cm, inner sep=0pt] (r\action\i) at (-7+.55*\i+2.6*\action,-4) {};
                    \node[circle, draw, minimum size =.27cm, inner sep=0pt] (c\action\i) at (-7+.55*\i+2.6*\action,-5) {}; 
                }
                \node[above of = r\action2, yshift = -1cm] (labelaction\action) {\small action $\label$};
            }
            \node[draw=none, right of = r43, yshift=-.45cm] (dots) {\huge $\dots$};
            
            \draw[thick] (r11.south) -- (c11.north);
            \draw[thick] (r12.south) -- (c12.north);
            \draw[thick] (r13.south) -- (c13.north);
        
            \draw[thick] (r21.south) -- (c21.north);
            \draw[thick] (r21.south) -- (c22.north);
            \draw[thick] (r22.south) -- (c22.north);
            \draw[thick] (r22.south) -- (c23.north);
            \draw[thick] (r23.south) -- (c23.north);

            \draw[thick] (r31.south) -- (c32.north);
            \draw[thick] (r32.south) -- (c33.north);
            \draw[thick] (r33.south) -- (c33.north);
            \draw[thick] (r33.south) -- (c31.north);

            \draw[thick] (r61.south) -- (c61.north);
            \draw[thick] (r61.south) -- (c62.north);
            \draw[thick] (r62.south) -- (c62.north);
            \draw[thick] (r63.south) -- (c62.north);
            \draw[thick] (r63.south) -- (c63.north);

            \draw[thick] (r41.south) -- (c41.north);
            \draw[thick] (r42.south) -- (c42.north);
            \draw[thick] (r43.south) -- (c42.north);
            \draw[thick] (r43.south) -- (c43.north);

        \end{tikzpicture}
    \caption{Identify actions.}
    \label{fig:architecture_actions}
    \end{subfigure}

    \begin{subfigure}{1\linewidth}
        \centering
        \begin{tikzpicture}
            \definecolor{darkgray}{RGB}{150,150,150}
            \definecolor{lightgray}{RGB}{220,220,220}
            \definecolor{lightblue}{RGB}{172,217,255}
        \begin{scope}[shift={(1,-7)}, scale=.7, transform shape]
            \node[circle, draw, align=center,  minimum size = 1cm] at (-10,-2) (MAB) {};
            \node[below=1.5ex of MAB] {\small \bf MAB}; 
            \draw ($(MAB.south) + (1,-1)$) arc[start angle=0, end angle=180, radius=1cm];

            \foreach \i/\c in {1/black,2/red,3/lightblue} { 
                \node[draw = none] (q_bottom\i) at ($(MAB.south) + (-2+\i,-3.5)$){};
                \draw[thick] (q_bottom\i.center) --++(-.35,0) --++(0,1.2); 
                \draw[thick] (q_bottom\i.center) --++(.35,0) --++(0,1.2); 
                \node[draw, circle, very thick, color=\c, minimum size = .7cm, fill] (c_bottom\i) at ($(MAB.south) + (-2+\i,-4.3)$) {};
            }
            \node[draw = none, below of = c_bottom2] (label_episode_routing) {\small \shortstack{route during \\ episode $k$}};
            \fill[black] ($(q_bottom1.center) + (-.25,.3)$) rectangle ++(.5,-0.2);
            \fill[red] ($(q_bottom1.center) + (-.25,.6)$) rectangle ++(.5,-0.2);
            \fill[black] ($(q_bottom1.center) + (-.25,.9)$) rectangle ++(.5,-0.2);
            \fill[red] ($(q_bottom2.center) + (-.25,.3)$) rectangle ++(.5,-0.2);
            \fill[lightblue] ($(q_bottom2.center) + (-.25,.6)$) rectangle ++(.5,-0.2);
            \fill[lightblue] ($(q_bottom3.center) + (-.25,.3)$) rectangle ++(.5,-0.2);
            \fill[lightblue] ($(q_bottom3.center) + (-.25,.6)$) rectangle ++(.5,-0.2);

            \draw[->, thick] ($(MAB.west) + (-1.3,-.7)$) to [out=180,in=180] node[pos=0.5, left,align=center] {\shortstack{choose \\ action}}  ($(q_bottom1.center) + (-.35,.7) + (-.5,0)$) ;
            \draw[->, thick] ($(q_bottom3.center) + (.35,.7) + (.5,0)$) to [out=0,in=0] node[pos=0.5, right] {\shortstack{update \\ belief}} ($(MAB.east) + (1.3,-.7)$) ;
        \end{scope}
        \begin{scope}[shift={(-1.7,-11.5)}, xscale=.65, scale=.4]
            \draw[ thick, black, ->] (0,0) -- (0,6.5) node[pos=0.9,left,rotate=90,yshift=.5cm,xshift=.15cm,align=center] {\scriptsize \shortstack{queue length \\ customer type 1}};
            \draw[ thick, black, ->] (0,0) -- (35.2,0);
            \draw[ thick] (0,0) --++ (-.7,0);
            \draw[ thick] (0,0) --++ (0,-.4);
            \draw[ thick] (13,-.7) -- (13,6.5);
            \draw[ thick] (30,-.7) -- (30,6.5);
            \draw[fill=green,opacity=0.25] (0,0) rectangle ++(5,6.5);
            \draw[fill=yellow,opacity=0.25] (5,0) rectangle ++(8,6.5);
            \draw[fill=green,opacity=0.25] (13,0) rectangle ++(6,6.5);
            \draw[fill=yellow,opacity=0.25] (19,0) rectangle ++(11,6.5);
            \draw[fill=green,opacity=0.25] (30,0) rectangle ++(5.2,6.5);
            \node[draw=none]  at (6.5,-1) (labelepi1) {\scriptsize episode $k$};
            \node[draw=none]  at (21.5,-1) (labelepi2) {\scriptsize episode $k+1$};
            \node[draw=none]  at (35,-1) (labeltime) {\scriptsize time};
            \draw[ thick] (0,6) -- (0.2,6) -- (0.2,5.5) -- (0.4,5.5) -- (0.4,6) -- (0.6,6) -- (0.6,5.5) -- (0.8,5.5) -- (0.8,5) -- (1,5) -- (1,4.5) -- (1.2,4.5) -- (1.2,4) -- (1.8,4) -- (1.8,4.5) -- (2,4.5) -- (2,4) -- (2.4,4) -- (2.4,3.5) -- (2.6,3.5) -- (2.6,3) -- (3,3) -- (3,2.5) -- (3.2,2.5) -- (3.2,2) -- (3.6,2) -- (3.6,1.5) -- (3.8,1.5) -- (3.8,2) -- (4.2,2) -- (4.2,1.5) -- (4.6,1.5) -- (4.6,1) --  (5,1);
            \draw[ thick] (5,1) --++ (0,.5) --++ (.4,0) --++ (0,-.5) --++ (.4,0) --++ (0,.5) --++ (.4,0) --++ (0,.5) --++ (.6,0) --++ (0,-.5) --++ (.4,0) --++ (0,-.5) --++ (.2,0) --++ (0,-.5) --++ (.4,0) --++ (0,-.5) --++ (.4,0) --++ (0,.5) --++ (.2,0) --++ (0,-.5)  --++ (.4,0) --++ (.4,0) --++ (0,.5) --++ (.6,0) --++ (0,.5) --++ (.4,0) --++ (.2,0)  --++ (.4,0) --++ (0,-.5) --++ (.2,0) --++ (0,-.5) --++ (.4,0)--++ (0,.5) --++ (.2,0) --++ (0,.5) --++ (.6,0);
            \draw[ thick] (12,1) --++ (.2,0) --++ (.4,0) --++ (0,.5) --++ (.2,0) --++ (0,-.5) --++ (.4,0) --++ (0,.5) --++ (.2,0) --++ (0,.5) --++ (.6,0) --++ (0,.5) --++ (.2,0) --++ (0,.5) --++ (.2,0) --++ (0,.5) --++ (.4,0) --++ (0,-.5) --++ (.2,0) --++ (0,.5) --++ (.4,0) --++ (0,.5) --++ (.6,0) --++ (0,.5) --++ (.2,0) --++ (0,-.5) --++ (.4,0) --++ (0,.5) --++ (.2,0) --++ (0,.5) --++ (.2,0) --++ (0,.5) --++ (.6,0)--++ (0,-.5) --++ (.6,0) --++ (0,-.5) --++ (.4,0) --++ (0,-.5) --++ (.2,0) --++ (0,.5) --++ (.6,0) --++ (0,.5) --++ (.2,0) --++ (0,-.5) --++ (.4,0) --++ (0,.5) --++ (.2,0) --++ (0,.5) --++ (.6,0) --++ (0,.5) --++ (.4,0) --++ (0,-.5) --++ (.2,0) --++ (0,-.5) --++ (.6,0) --++ (0,-.5) --++ (.6,0) --++ (0,-.5) --++ (.2,0) --++ (0,.5) --++ (.4,0) --++ (0,-.5) --++ (.4,0) --++ (0,.5) --++ (.2,0) --++ (0,.5) --++ (.4,0) --++ (0,.5) --++ (.6,0) --++ (.2,0) --++ (0,-.5) --++ (.2,0) --++ (0,-.5) --++ (.4,0) --++ (0,.5) --++ (.4,0) --++ (0,.5) --++ (.6,0) --++ (0,.5) --++ (.4,0) --++ (0,-.5) --++ (.4,0) --++ (0,.5) --++ (.4,0) --++ (0,.5) --++ (.4,0) --++ (0,-.5) --++ (.2,0) --++ (0,.5) --++ (.6,0) --++ (0,-.5) --++ (.4,0) --++ (0,-.5) --++ (.4,0) --++ (0,.5) --++ (.2,0);
            \draw[ thick] (30,6) --++ (.2,0) --++ (0,-.5) --++ (.4,0) --++ (0,-.5) --++ (.2,0) --++ (0,.5) --++ (.2,0) --++ (0,-.5) --++ (.4,0)  --++ (0,-.5) --++ (.2,0) --++ (0,-.5) --++ (.6,0) --++ (0,.5)  --++ (.4,0) --++ (0,-.5) --++ (.2,0) --++ (0,.5) --++ (.2,0) --++ (0,-.5) --++ (.2,0) --++ (0,-.5)  --++ (.4,0)  --++ (0,.5)  --++ (.6,0) --++ (0,.5)  --++ (.2,0)  --++ (0,-.5) --++ (.4,0) --++ (0,-.5)  --++ (.2,0)  --++ (0,.5) --++ (.2,0);
        \end{scope}
        \end{tikzpicture}
    \caption{Routing dynamics and possible queue length behavior.}
    \label{fig:architecture_routing}
    \end{subfigure}

    \caption{Schematic overview of Algorithm~\ref{alg:learning_alg}.}
    \label{fig:architecture}
\end{figure}
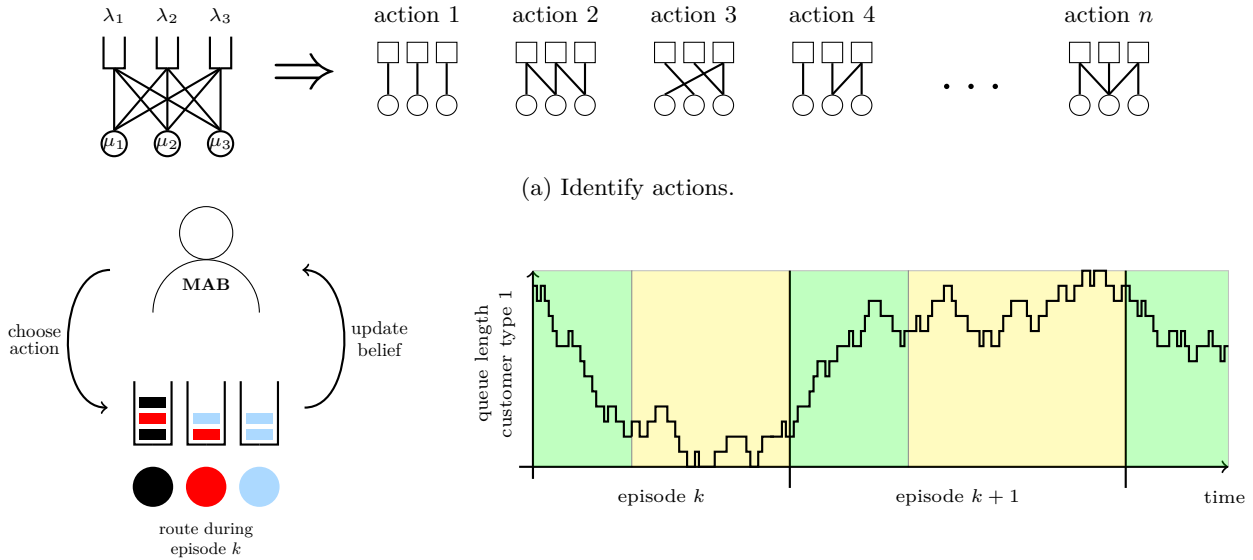

{\bf Regret upper bound.}
We show that the asymptotic regret of our policy is $\calO(\ln^{2\beta}(t))$ for any $\beta > 1$ in Theorem~\ref{thm:regret_ub}. 
This regret is only slightly worse than the $\calO(\ln(t))$ regret of the benchmark UCB policy~\citep{Auer2002a} in the classical \gls{MAB} setting, even though our learning problem is more challenging.
The regret analysis relies on concentration inequalities that bound the probability that the empirical average of payoff samples deviates far from its mean. 
However, it differs from the standard \gls{MAB} techniques since the number of routing decision and therefore the number of samples is subject to queueing dynamics. 
We deal with this issue by splitting each episode into a transient and a stationary phase and analyzing the convergence time of the queue length process to its stationary measure.
We let the episodes grow in length such that the probability that the queue length process does not reach stationarity decreases sufficiently fast. 

{\bf Numerical experiments.}
Lastly, we analyze Algorithm~\ref{alg:learning_alg} and its performance in several numerical experiments. 
We find that in our experiments, the average reward of Algorithm~\ref{alg:learning_alg} converges quickly to the optimal reward and chooses the optimal action most of the time, even in a larger system with many different actions. 
The algorithm outperforms a benchmark greedy policy that myopically optimizes the instantaneous payoff without considering long-term effects. 
Even though the theoretical analysis of Algorithm~\ref{alg:learning_alg} is for a static setting, we test its robustness against the change of a parameter value and find that the algorithm adapts adequately. 
This highlights the potential applicability of our algorithm beyond static environments. 

\subsection{Related literature}
\label{sec:Related-literature}
The majority of literature on optimal routing in skill-based queues does not account for customer--server dependent payoffs, and focuses on waiting time minimization~\citep{Koole2002,Gurvich2010,Visschers2012,Adan2012,Adan2014,Yang2022,Krishnasamy2018,Lee2021,Choudhury2021,Zhong2022}.
In this scenario, the $c\mu$ (or $c\mu/\theta$ in case of abandonments) routing policy that prioritizes customer types based on a ranking depending on the holding costs, is known to be asymptotically optimal \citep{Zhong2022}.
Therefore, an intuitive approach for an adaptive learning policy is to route according to a variant of the $c\mu$ policy where the unknown parameters are replaced by empirical estimators. 
In our case however, the objective in~\eqref{eq:intro_reward} is difficult to evaluate in full generality, since the departure process depends on the availability of servers and the state of the queues which are random and fairly intractable. 
Even for simple routing policies like \gls{FCFS} \gls{ALIS}, the expected number of type-$i$ departures at server $j$ is, to the best of our knowledge, unknown for finite time.
Only a select few convergence results are known for specific queueing systems \citep{Adan2012,Adan2014}.

On top of the classical exploration--exploitation tradeoff, learning in queueing systems requires one to deal with limited capacity and the interplay between queueing and learning (queueing degrades learning and \emph{vice versa}). 
Some works decouple learning and queueing by splitting the time horizon into distinct exploration and exploitation phases \citep{Johari2021,Zhong2022,Jia2022}, while others directly implement a reinforcement learning algorithm \citep{Liu2019,Choudhury2021}, or apply Bayesian inference \citep{Shah2020,Bimpikis2019}. 
Our Algorithm~\ref{alg:learning_alg}, however, is an integrated learning and routing policy. 

The focus in our work lies on online payoff maximization in skill-based queues. 
Payoff maximization problems with unknown utility functions are well studied in literature \citep{Tan2012,Tan2020,Vera2021,Agarwal2013,Agrawal2014,Yu2017}.
In the setting of queueing systems, there exists a variety of algorithms based on different approaches. 
For example, \cite{Sun2023,Hsu2022,Kim2021a} propose utility guided algorithms based on a static \gls{LP}. This method is an extension of the Lyapunov drift plus penalty reward, where the Lyapunov drift assures stability and the penalty is used for payoff maximization.
\cite{Sun2023} obtain moment bounds for the maximal queue length in the system and an instance independent regret upper bound of order $\calO(\sqrt{t\ln(t)})$.
However, the algorithm suffers a linear loss in reward due to stability constraints.
In our work, we instead split the regret into a queueing and a learning component. 

Other algorithms for online payoff maximization in queueing systems are considered by \cite{Liu2020,Steiger2022}, who use Lyapunov drift analysis.
A different method is used by \cite{Fu2022}, who present a routing algorithm that combines the Join--the--Shortest--Queue (JSQ) routing scheme with a confidence ball learning algorithm (introduced by \cite{Dani2008}) in the setting of optimization under bandit feedback.
It is shown that, given a fixed horizon, the threshold parameter $K$ in the JSQ-$K$ algorithm can be tuned in such a way that polylogarithmic regret can be attained, but no guarantees are provided on the convergence of the routing rates to their the optimal values. 
In our work, we consider an asymptotic result and therefore our algorithm does not require the knowledge of the time horizon. 
Lastly, in \citep{Tan2012}, a primal--dual method is used for reward maximization in an online advertisement setting.
An algorithm is presented  where the current queue length is deducted from the estimated server--customer payoffs to penalize congestion and near optimality with respect to an oracle reward is proven.

The analysis of our learning algorithm relies on establishing convergence properties of an episodic queue length process to a stationary probability measure. 
Similar convergence properties have been used in~\citep{Jia2022,Chen2019,Besbes2015}.
We let the episodes grow in duration so that convergence is achieved with increasing probability. 
\cite{Sanders2016,Comte2023,Jiang2010,Liu2019} use similar concepts in different settings of optimal control in queueing networks.

The methodology of our learning algorithm bears most resemblance with the achievable region approach in \citep{Bertsimas1995,Dacre1999}, where the goal is to solve an optimization problem with an unknown utility function using the feasible region spanned by a set of constraints. 
\cite{Bertsimas1995,Dacre1999} consider the optimal routing problem in a static setting with known parameters, while our policy integrates learning and optimization in an online fashion. \\

\subsection{Outline}
The remainder of this work is organized as follows. 
In Section~\ref{sec:problem_setting}, we  discuss the model, optimal routing problem, and regret formulation.
Next, in Section~\ref{sec:regret}, we study the regret of routing policies and present an asymptotic regret lower bound for a class of routing policies. 
In Section~\ref{sec:learning_alg}, we present the adaptive learning algorithm and show that its asymptotic regret is of polylogarithmic order. 
Lastly, in Section~\ref{sec:simulation}, we present a numerical implementation of the algorithm.

%% file: ProblemSetting.tex
In this section we describe the queueing system and routing policies in more detail. 
We analyze a related deterministic optimal transport \gls{LP} and discuss its properties. 
Lastly, we present a definition of regret of a routing policy, which is based on the optimal transport \gls{LP}.

\subsection{Arrival and service process}
\label{sec:arr_serv_process}
We consider a continuous-time queueing system with a fixed set of customer types $\calI = \{1,\dots,I\}$, servers $\calJ = \{1,\dots,J\}$, and a set of lines connecting compatible customer types and servers $\calL = \{(ij): i\in\calI,j\in\calJ\}$.
Here, $I,J\in\N_{> 0}$. 
Let $L = |\calL|\in\N_{> 0}$ denote the number of compatible lines.
To exclude trivial cases, we assume that $L > I+J-1$. 
We denote the set of servers that are compatible with customer type $i$ by $\calS_i$, and similarly, we denote the set of customer types that are compatible with server $j$ by $\calC_j$, i.e.,
\begin{align}
    \calS_i &= \{j\in\calJ: (ij)\in\calL\}, 
    \label{eq:def_set_servers_customer_i}
    \\
    \calC_j &= \{i\in\calI: (ij)\in\calL\}.
    \label{eq:def_set_customers_server_j}
\end{align}
An example of such a queueing network is given in Figure~\ref{fig:qsys_example_3}.
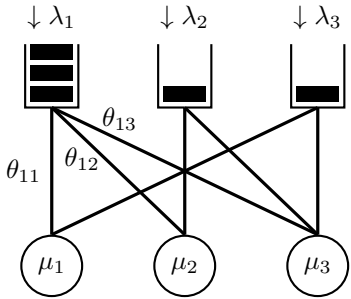
\begin{figure}[H]
    \begin{center}
        \begin{tikzpicture}[
            server/.style={circle, minimum size = .8cm, thick,draw},
            scale=.7
            ]
            \foreach \x/\y in {1/black,2/black,3/black}{ 
                \node[draw = none] at (2.5*\x,0) (queue\x) {};
                \node[above of = queue\x, yshift = -.5cm] (labelqueue\x) {};
                \draw[thick] (queue\x.center) --++(-.5,0) --++(0,1.2); 
                \draw[thick] (queue\x.center) --++(.5,0) --++(0,1.2); 
                \node[draw = none, above of = queue\x, yshift = .2cm] {$\downarrow \lambda_\x$};
            }
            \fill[black] ($(queue1.center) + (-.4,.4)$) rectangle ++(.8,-0.3);
            \fill[black] ($(queue1.center) + (-.4,.8)$) rectangle ++(.8,-0.3);
            \fill[black] ($(queue1.center) + (-.4,1.2)$) rectangle ++(.8,-0.3);

            \fill[black] ($(queue2.center) + (-.4,.4)$) rectangle ++(.8,-0.3);

            \fill[black] ($(queue3.center) + (-.4,.4)$) rectangle ++(.8,-0.3);

            \foreach \x in {1,2,3}{ 
                \node[server] at (2.5*\x,-3) (server\x) {};
                \node at (server\x) {$\mu_\x$};
            }

            \draw[very thick, black] (queue1.center) -- (server1.north) node[pos=0.5, left]
            {$\theta_{11}$};
            \draw[very thick, black] (queue1.center) -- (server2.north) node[pos=0.4, left, xshift=.07cm]
            {$\theta_{12}$};
            \draw[very thick, black] (queue1.center) -- (server3.north) node[pos=0.15, right, yshift = .1cm]
            {$\theta_{13}$};

            \draw[very thick, black] (queue2.center) -- (server2.north);
            \draw[very thick, black] (queue2.center) -- (server3.north);

            \draw[very thick, black] (queue3.center) -- (server1.north);
            \draw[very thick, black] (queue3.center) -- (server3.north);
        \end{tikzpicture}
    \end{center}
    \caption{A queueing model with $I=3$ types of customers and $J = 3$ servers with compatibility lines. 
    Here, e.g., $\calS_1 = \{1,2,3\}$ and $\calC_1 = \{1,3\}$. }
    \label{fig:qsys_example_3}
\end{figure}

For $i\in\calI$, type-$i$ customers are assumed to arrive according to a Poisson process with rate $\lambda_i > 0$.
They wait in queue $i$ until they are allocated to a compatible (and idle) server.
We assume that there are no customers in the system at time $t=0$. 
For $j\in\calJ$, service times of customers at server $j$ are assumed exponentially distributed with rate $\mu_j > 0$.
We also assume that service times are independent between servers and across customers. 
Note that the service time distribution only depends on the server and not the customer type. 
For convenience, we write $\lambda = (\lambda_1,\dots,\lambda_I)$ and similarly $\mu = (\mu_1,\dots,\mu_J$).
For stability, we assume that for any subset of customer types $\calA\subseteq \calI$,
\begin{align}\label{eq:ass_stab}
     \sum_{i \in \calA} \lambda_i < \sum_{j\in\calS_{\calA}} \mu_j.
\end{align}
Here, $\calS_{\calA} = \bigcup_{i\in\calA} \calS_i$ denotes the set of servers that are compatible with any of the customer types $i\in\calA$. 
By \mbox{\cite[Theorem 2.1]{Adan2014}}, there exists a routing policy discipline such that the joint queue length process under that discipline is ergodic.
\\

Upon the service completion of a type-$i$ customer at server $j$, a random payoff $Y_{ij}$ is obtained, sampled from a payoff distribution $P_{\theta_{ij}}$ with support $\calX \subseteq \N_{\geq 0}$ and unknown and finite mean $\theta_{ij} \in \R_{\geq 0}$.
Note that while all lines share the same payoff distribution, the mean payoff differs per line.
We denote the vector of unknown mean payoffs as  $\theta=(\theta_{ij})_{(ij)\in\calL} \in \R_{\geq 0}^L$.
As quantification of the difference between the payoff distribution for two lines $(ij)$ and $(k\ell)\in\calL$, we use the Kullback--Leibler (KL) divergence, given by
\begin{align}\label{eq:def_KL}
    I(\theta_{ij},\theta_{k\ell}) &:= \sum_{x\in\calX} P_{\theta_{ij}}(x) \ln\Bigl(\frac{P_{\theta_{ij}(x)}}{P_{\theta_{k\ell}}(x)}\Bigr).
\end{align}
We assume that the payoff distribution is such that
\begin{align}\label{eq:ass_KL_finite}
    0 < I(\theta_{ij}, \theta_{k\ell}) < \infty \ \ \textrm{when} \ \ \theta_{ij} < \theta_{k\ell}.
\end{align}

\subsection{Routing policies}
A \emph{routing policy} $\pi$ determines which customer is served by which server at what time.
We denote the probability measure and expectation with respect to routing policy $\pi$ and payoff vector $\theta\in\R_{\geq 0}^L$ by $\P_\pi^\theta$ and $\E_\pi^\theta$, respectively.
Moreover, we denote the total number of service completions, i.e., departures, of type-$i$ customers at server $j$ up to time  $t \geq 0$ by $D_{ij}(t)$.

We consider a class of routing policies that satisfy for $\eps \geq 0$ and for any payoff vector $\theta\in\R_{\geq 0}$ the following conditions: 
\begin{align}
    \lim_{t\to\infty} t\lambda_i - \sum_{j\in\calS_i} \E_\pi^\theta(D_{ij}(t)) &< \infty, \ \ \ \forall i\in\calI,
    \label{eq:def_stability}
    \\
    t(\mu_j-\eps) - \sum_{i\in\calC_j} \E_\pi^\theta(D_{ij}(t)) &\geq 0, \ \ \ \forall j\in\calJ, \ \ \ \forall t\geq 0.
    \label{eq:def_eps_restrictive}
\end{align}
Constraint~\eqref{eq:def_stability} is a notion of stability: we require that the difference between the expected number of arrivals and  departures, i.e., the expected queue length, of any customer type does not diverge.
Note that the set of policies satisfying~\eqref{eq:def_stability} is nonempty by assumption~\eqref{eq:ass_stab}.
Constraint~\eqref{eq:def_eps_restrictive} requires that the expected workload of server $j$ is at most $(\mu_j-\eps)/\mu_j$. 
If $\eps>0$, \eqref{eq:def_eps_restrictive} prevents critically loaded servers. 

We consider the reward of a policy $\pi$ with respect to payoff vector $\theta \in \R_{\geq 0}^L$ up to time $t\geq 0$ as the total average payoff, i.e.,
\begin{align}
    \sum_{(ij)\in\calL} \theta_{ij} \E_\pi^\theta(D_{ij}(t)).
    \label{eq:def_reward_policy}
\end{align}
As discussed previously, \eqref{eq:def_reward_policy} can be difficult to compute in full generality. 
Instead, we consider a static version of the optimal control problem which is discussed next.

\subsection{Optimal transport LP}
For a payoff vector $\theta\in\R_{\geq 0}^L$ and $\eps \geq 0$, let
\begin{subequations}
    \begin{align}
        \LP{\theta,\eps}: \ \ \max_x \ \ &\sum_{(ij)\in\calL} \theta_{ij} x_{ij}, \\
        \textrm{s.t.} \ \ &\sum_{j\in\calS_i} x_{ij} = \lambda_i, \ \ \ \forall i\in\calI, \label{eq:LP_eps_constr_x}\\
        &\sum_{i\in\calC_j} x_{ij} \leq \mu_j - \eps, \ \ \ \forall j\in\calJ, \label{eq:LP_eps_constr_mu} \\
        &x_{ij} \geq 0, \ \ \ \forall (ij)\in\calL. \label{eq:LP_eps_constr_nonneg} 
    \end{align}
    \label{eq:LP_eps}
\end{subequations}
\noindent The objective of \LP{\theta,\eps} is to maximize the average payoff per time unit. Each optimization variable $x_{ij}\in[0,\lambda_i]$ can be interpreted as the long-term  rate of type-$i$ customers that is routed to server $j$ per time unit.
Constraint~\eqref{eq:LP_eps_constr_x} requires that all customers are routed to some server and constraint~\eqref{eq:LP_eps_constr_mu} states that the capacity of each server must not be exceeded.
Here, $\eps \geq 0$ is a fixed amount of slack at each server.
The feasibility of \LP{\theta,\eps} depends on the value of $\eps$, which will be discussed later. 
\LP{\theta,\eps} can be regarded as a variant of the assignment problem initially presented in \citep{Shapley1971}, where the resources are divisible.
Similar \gls{LP} formulations of the static planning problem in the context of reward maximization in queueing models are used in \citep{Hsu2022, Fu2021, Fu2022, Tan2012, Tan2020, Vera2021}. 

We briefly discuss some fundamental properties of \gls{LP}s. 
Consider an \gls{LP} in standard form: $\max_x c' x$ s.t.\ $Ax = b$, $x \geq 0$ with $A \in \R^{m\times n}$.
For $B\subseteq \{1,\dots,n\}$, we denote by $\mathbf{B} = [A_{B(1)}, \dots,A_{B(m)}]\in\R^{n\times n}$  the matrix formed by the columns $B(1),\dots,B(m)$.
We call $B\subseteq \{1,\dots,n\}$ a \emph{basis} if $|B| = m$ and $\mathbf{B}$ is invertible (or equivalently, has full rank); see \citep{Bertsimas1997}.
In this case, $x_B=\mathbf{B}^{-1}b\in\R^n$ is a \emph{basic solution}. 
A basis (or equivalently a basic solution) is \emph{feasible} if $\mathbf{B}^{-1} b \geq 0$, and \emph{nondegenerate} if $(\mathbf{B}^{-1} b)_k \neq 0$ for all $k\in B$.
Lastly, recall that if an \gls{LP} has an optimal solution, then it has an optimal basic feasible solution \cite[Theorem 2.7]{Bertsimas1997}. 

In the remainder of this work, we assume nondegeneracy, i.e., we assume that all basic feasible solutions of \LP{\theta,\eps} and its dual are nondegenerate. 
We note that this assumption implies that each basic feasible solution of the primal (dual) implies a unique basic feasible solution of the dual (primal) (see e.g.\ \cite[Exercise 4.12]{Bertsimas1997}).

\subsubsection{Basic feasible solutions}
\label{sec:BFS}
We will now characterize the basic feasible solutions of \LP{\theta,\eps} in \eqref{eq:LP_eps}.
Since \LP{\theta,\eps} has a finite number of linear inequality constraints, the set of basic feasible solutions is finite~\cite[Corollary 2.1]{Bertsimas1997}. 

It is a known result that a basic feasible solution of \LP{\theta,\eps} induces a spanning forest on the bipartite graph $(\calI\cup\calJ,\calL)$ \mbox{\cite[Proposition 2]{Fu2022}}.
We extend the result of \citep{Fu2022} by providing the explicit form of any basis (see Lemma~\ref{lem:LP_basis} below) and any basic solution (see Lemma~\ref{lem:basic_solutions} below). The proofs are presented in Appendix~\ref{app:lem_LP_basis} and Appendix~\ref{app:lem_basic_solutions}, respectively. 

The introduction of Lemma~\ref{lem:LP_basis} and~\ref{lem:basic_solutions} do require some more notation. 
For $\scrL\subseteq\calL$ and $\scrJ\subseteq\calJ$ we denote by $\calG(\scrL,\scrJ)$ the subgraph of $(\calI\cup\calJ,\calL)$ induced by $\scrL$ and $\scrJ$.
We say that $\calG(\scrL,\scrJ)$ is a \emph{spanning forest} of $(\calI\cup\calJ,\calL)$ if 
(i) it is a union of trees, i.e., it contains no cycles,
(ii) each $v\in\calI\cup\calJ$ is contained in $\calG(\scrL,\scrJ)$, 
and (iii) each tree of $\calG(\scrL,\scrJ)$ contains a unique root node $j\in\scrJ$. 
In such a tree, parent and child nodes of servers are customer types and \emph{vice versa}. 
We denote the subtree rooted in $v$, including $v$ itself, by $\subtree(v)$.
We let
$\sfC(\subtree(v))$ denote the set of customer types contained in the subtree rooted in $v$, i.e., 
$\sfC(\subtree(v)) = \subtree(v)\cap\calI$. 
Similarly, $\sfS(\subtree(v)) = \subtree(v)\cap\calJ$. 
An example is shown in Figure~\ref{fig:example_SF}.

\begin{figure}[H]
    \centering
    \begin{subfigure}{0.4\linewidth}
        \centering
        \begin{tikzpicture}[
            server/.style={circle, minimum size = .6cm, thick,draw},
            q/.style={rectangle, minimum size = .6cm, thick,draw},
            scale=.7
            ]
            \foreach \x in {1,2,3}{ 
                \node[q] at (2*\x,0) (queue\x) {\x};
                \node[server] at (2*\x,-3) (server\x) {};
                \node at (server\x) {\x};
            }
            \draw[very thick, black] (queue1.south) -- (server1.north);
            \draw[very thick, black] (queue1.south) -- (server2.north);
            \draw[very thick, black] (queue1.south) -- (server3.north);
            \draw[very thick, black] (queue2.south) -- (server2.north);
            \draw[very thick, black] (queue2.south) -- (server3.north);
            \draw[very thick, black] (queue3.south) -- (server1.north);
            \draw[very thick, black] (queue3.south) -- (server3.north);
        \end{tikzpicture}
        \caption{Bipartite graph $(\calI\cup\calJ,\calJ)$.}
        \label{}
    \end{subfigure}
    \begin{subfigure}{0.4\linewidth}
        \centering
        \begin{tikzpicture}[
            scale = .7,
                server/.style={circle, minimum size = .6cm, thick,draw},
                q/.style={rectangle, minimum size = .6cm, thick,draw}
                ]
                \node[server] at (0,0) (server1) {};
                \node at (server1) {$1$};
                \node[q, below of = server1, yshift=0cm] (queue1) {1};
                \node[server, below of = queue1, yshift=0cm] (server3) {};
                \node[q, below of = server3, xshift = -.7cm, yshift=0cm] (queue2) {2};
                \node[q, below of = server3, xshift = .7cm, yshift=0cm] (queue3) {3};
                \node at (server3) {$3$};
                \node[server] at (2,0) (server2) {};
                \node at (server2) {$2$};
                \draw[thick] (server1.south) -- (queue1.north);
                \draw[thick] (queue1.south) -- (server3.north);
                \draw[thick] (server3.south) -- (queue2.north);
                \draw[thick] (server3.south) -- (queue3.north);
        \end{tikzpicture}
        \caption{The induced subgraph $\calG(\scrL,\scrJ$).}
        \label{}
    \end{subfigure}
    \caption{
    Consider the queueing system in Figure~\ref{fig:qsys_example_3} and let $\scrL = \{(11),(13),(23),(33)\}$ and $\scrJ = \{1,2\}$.
    $\calG(\scrL,\scrJ)$ is then a spanning forest of $(\calI\cup\calJ,\calL)$ consisting of two trees.
    }
    \label{fig:example_SF}
\end{figure}
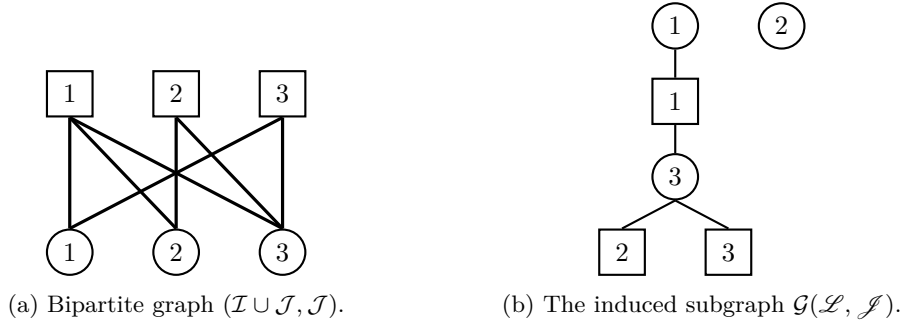

We are now in position to state Lemma~\ref{lem:LP_basis} and~\ref{lem:basic_solutions}:
\begin{lemma}
    \label{lem:LP_basis}
    Let $\eps \geq 0$,  $\scrL\subseteq\calL$, and $\scrJ\subseteq\calJ$.
    $B=\scrL\cup\scrJ$ is a basis of \LP{\theta,\eps} if and only if $\calG(\scrL,\scrJ)$ is a spanning forest of $(\calI\cup\calJ,\calL)$. 
\end{lemma}

\begin{lemma}
    \label{lem:basic_solutions}
    Let $\eps \geq 0$,  $\scrL\subseteq\calL$, and $\scrJ\subseteq\calJ$.
    If $B=\scrL\cup\scrJ$ is a basis of \LP{\theta,\eps}, then its corresponding basic solution $x\in\R^L$ satisfies
    \begin{align}
        x_{ij} =
        \begin{cases}
            \sum\limits_{k\in\sfC(\subtree(i))} \lambda_k - \sum\limits_{\ell\in\sfS(\subtree(i))} (\mu_\ell-\eps), 
            &\textrm{if} \ i \ \textrm{is a child of} \ j \ \textrm{in} \ \calG(\scrL,\scrJ),
            \\
            \sum\limits_{\ell\in\sfS(\subtree(j))} (\mu_\ell-\eps) - \sum\limits_{k\in\sfC(\subtree(j))} \lambda_k, 
            &\textrm{if} \ i \ \textrm{is the parent of} \ j \ \textrm{in} \ \calG(\scrL,\scrJ), 
            \\
            0 
            &\textrm{if} \ (ij)\in\calL\setminus \scrL.
        \end{cases}
        \label{eq:basic_sol}
    \end{align}
\end{lemma}

\subsubsection{Feasibility}
Note that by~\eqref{eq:ass_stab}, \LP{\theta,0} is feasible, i.e, there exists at least one basic feasible solution.
The feasibility region of \LP{\theta,\eps} in~\eqref{eq:LP_eps} depends on $\eps$:
the larger~$\eps$, the smaller the feasible region.
Lemma~\ref{lem:insensitive} characterizes the range of $\eps$ such that the \LP{\theta,\eps} has the same set of feasible bases as \LP{\theta,0}. The lemma is proven in Appendix~\ref{app:proof_lem_insensitive}.

\begin{lemma}
    \label{lem:insensitive}
    Let $B=\scrL\cup\scrJ$ with $\scrL\subseteq\calL$ and $\scrJ\subseteq\calJ$ be a nondegenerate feasible basis of \LP{\theta,0}. 
    If 
    \begin{align}
        0 \leq \eps &< \
        \min_{j\in\calJ} \frac{\sum_{\ell\in\sfS(\subtree(j))} \mu_\ell  - \sum_{k\in\sfS(\subtree(j))} \lambda_k}{|\sfS(\subtree(j))|},
        \label{eq:def_eps_insensitive}
    \end{align}
    then $B$ also is a nondegenerate feasible basis of \LP{\theta,\eps}. 
\end{lemma}

\subsection{Regret formulation}
\label{sec:regret_formulation}

We use the value of \LP{\theta,\eps} in~\eqref{eq:LP_eps} as the oracle reward to define the regret of a learning policy. 
Let $\eps > 0$ satisfy~\eqref{eq:def_eps_insensitive},  $\theta\in\R_{\geq 0}^L$ be a payoff vector, and  $x^\theta\in\R_{\geq 0}^L$ be the optimal basic feasible solution of \LP{\theta,\eps} in~\eqref{eq:LP_eps}.
We define the (expected) \emph{regret} of a routing policy $\pi$ at time $t\geq 0$ as
\begin{align}
    \label{eq:regret}
    R^\theta_\pi(t)
    &:= \sum_{(ij)\in\calL} \theta_{ij} \bigl(t x_{ij}^\theta - \E_\pi^{\theta}(D_{ij}(t))\bigr).
\end{align}
Lemma~\ref{lem:oracle_reward} states that  the asymptotic regret is nonnegative and is proven in Appendix~\ref{app:lem_oracle_reward}.
\begin{lemma}
    \label{lem:oracle_reward}
    Let $\eps \geq 0$ and let $\pi$ satisfy~\eqref{eq:def_stability} and~\eqref{eq:def_eps_restrictive}. Then $\lim_{t\to\infty} R_\pi^\theta(t)/t \geq 0$.
\end{lemma}

%% file: RegretResults.tex
We provide an asymptotic lower bound for the regret for routing policies satisfying constraints~\eqref{eq:def_stability} and~\eqref{eq:def_eps_restrictive} in Theorem~\ref{thm:regret_lb}. To this end, we first provide a regret decomposition in Lemma~\ref{lem:regret_decomp}.

\subsection{Regret decomposition}
Let $\eps > 0$ satisfy~\eqref{eq:def_eps_insensitive}, $\theta\in\R_{\geq 0}^L$ be a payoff vector, and $x^\theta\in\R_{\geq 0}^L$ be the unique optimal solution of \LP{\theta,\eps} in~\eqref{eq:LP_eps}.
We define the set of optimal lines $O(\theta)$ as the lines with non-zero load in the optimal solution $x^\theta\in\R^L$ and the set of suboptimal lines $O^c(\theta)$ as its complement:
\begin{align}
    O(\theta) &= \{(ij)\in\calL: \ x_{ij}^\theta > 0\},
    \label{eq:def_optimal_lines_theta}
    \\
    O^c(\theta) &= \calL\setminus O(\theta).
    \label{eq:def_suboptimal_lines_theta}
\end{align}

The dual problem \citep{Bertsimas1997} of \LP{\theta,\eps}  is given by
\begin{subequations}
    \label{eq:LP_eps_dual}
    \begin{align}
        \dual{\theta,\eps}: \ \ \min_{v,w} \ \ &\sum_{i\in\calI} \lambda_i v_i + \sum_{j\in\calJ} (\mu_j-\eps) w_j \\
            \textrm{s.t.} \ \ & v_i + w_j - \theta_{ij} \geq 0, \ \ \ \forall (ij)\in\calL, \label{eq:LP_eps_dual_constr_phi} \\
            &v_i \in \R, \ \ \ \forall i\in\calI, \\
            &w_j \geq 0, \ \ \ \forall j\in\calJ. \label{eq:LP_eps_dual_constr_w}
    \end{align}
\end{subequations}
By strong duality, we have that \dual{\theta,\eps} is feasible since \LP{\theta,\eps} is feasible. 
Moreover, $x^\theta$ and $v^\theta$, $w^\theta$ are an optimal primal--dual pair if and only if they satisfy the complementary slackness conditions \mbox{\cite[Theorem 4.5]{Bertsimas1997}}, which are given by
\begin{align}
    x_{ij}^\theta (v_i^\theta + w_j^\theta - \theta_{ij}) &= 0, \ \ \ \forall (ij)\in\calL, \label{eq:slack_cond_phi_eps} \\
    \Bigl(\mu_j - \eps - \sum_{i\in\calC_j} x_{ij}^\theta \Bigr) w_j^\theta &= 0, \ \  \ \forall j\in\calJ.  \label{eq:slack_cond_w_eps}
\end{align}
We note that the optimal solutions of \LP{\theta,\eps} and \dual{\theta,\eps} are unique due to the nondegeneracy assumption. 
Moreover, the nondegeneracy assumption implies strict complementarity (see e.g. Exercises 4.12 and 4.20c in \citep{Bertsimas1997}). 

For a line $(ij)\in \calL$, we define 
\begin{align}
    \phi^\theta_{ij} &= v^\theta_i + w^\theta_j - \theta_{ij}
    \label{eq:phi_def}
\end{align}
as its suboptimality gap. 
Here, $v^\theta\in\R^I$, $w^\theta\in\R_{\geq 0}^J$ is the optimal solution of \dual{\theta,\eps} in \eqref{eq:LP_eps_dual}.
Note that $\phi_{ij}^\theta = 0$ for any $(ij)\in O(\theta)$ by~\eqref{eq:slack_cond_phi_eps} and $\phi_{k\ell} > 0$ for any $(k\ell)\in O^c(\theta)$ by strict complementarity. 

Lemma~\ref{lem:regret_decomp} presents a regret decomposition based on the line classification. The proof can be found in Appendix~\ref{app:proof_regret_decomposition}.

\begin{lemma}\label{lem:regret_decomp}
    Let $v^\theta\in\R^I$, $w^\theta\in\R_{\geq 0}^J$ be the optimal solution of \dual{\theta,\eps} in~\eqref{eq:LP_eps_dual}.
    Let $\pi$ satisfy~\eqref{eq:def_stability} and~\eqref{eq:def_eps_restrictive}. Then the regret in \eqref{eq:regret} can be expressed as
    \small
    \begin{align}
        R^\theta_\pi(t) &= 
        \underbrace{\sum_{(k\ell)\in O^c(\theta)} \phi_{k\ell}^\theta \E_\pi^\theta(D_{k\ell}(t))}_{\mathrm{I}}
        + \underbrace{\sum_{j\in\calJ}  w_j^\theta \Bigl(t (\mu_j-\eps) - \sum_{i\in\calC_j}  \E_\pi^\theta(D_{ij}(t))\Bigr)}_{\mathrm{II}}
        + \underbrace{\sum_{i\in\calI}  v_i^\theta \Bigl(t \lambda_i - \sum_{j\in\calS_i} \E_\pi^\theta(D_{ij}(t))\Bigr)}_{\mathrm{III}}.
        \label{eq:regret_decomp}
    \end{align}
    \normalsize
\end{lemma}
In decomposition~\eqref{eq:regret_decomp}, term I represents the loss in reward due to customers being routed over suboptimal lines. 
The other two summations represent regret accumulated by deviating from the oracle routing rates.
In particular, II measures regret from not utilizing the allowed capacity from servers and III measures  regret accumulated from customers that are still waiting in the queue or are in service at time $t$, since we do not obtain any payoff until a customer finishes service.
Note that since $v_i^\theta\in\R$, the queue length can actually contribute to negative regret, since it might be beneficial to let customers wait until a high pay-off server is available.

\subsection{Regret lower bound}
Before we can state the asymptotic regret lower bound, we need to introduce some further notation. Let $\theta\in\R_{\geq 0}^L$. 
We define $A(\theta,ij,a)\in\R_{\geq 0}^L$ element-wise by
\begin{align}
    \label{eq:def_A_vector}
    A(\theta,ij,a)_{k\ell} =
    \begin{cases}
        a &\textrm{if} \ (k\ell) = (ij), \\
        \theta_{k\ell} &\textrm{otherwise}.
    \end{cases}
\end{align}
In particular, $A(\theta,ij,a)$ represents a payoff vector that is (almost) equal to $\theta$ except for index $(ij)$. 
For $(ij)\in O^c(\theta)$, we define 
\begin{align}
    \label{eq:def_Delta_set}
    \Delta(\theta,ij) &= \bigl\{a\geq 0: \ (ij)\in O(A(\theta,ij,a)) \bigr\}
\end{align}
as the set of parameter values that make $(ij)$ an optimal line without changing the payoff values of the other lines.
Note that by construction, either $\Delta(\theta,ij)=\emptyset$, or $\Delta(\theta,ij)=(a_0,\infty)$ for some $a_0\in\R_{\geq 0}$. 
We define
\begin{align}
    K(\theta,ij) &= \inf\bigl\{ I(\theta_{ij},a): \ a\in\Delta(\theta,ij) \bigr\}
    \label{eq:def_min_KL_dist}
\end{align}
as the minimum distance in terms of KL divergence between the parameter $\theta_{ij}$ and the set of parameter values that make $(ij)$ an optimal line.
Here, the KL divergence $I$ is defined in \eqref{eq:def_KL}.

Lemma~\ref{lem:suboptgap_nonempty} states that any suboptimal line could become optimal by changing its payoff value. It is proven in Appendix~\ref{app:suboptgap_nonempty}.

\begin{lemma}\label{lem:suboptgap_nonempty}
    For any $(ij)\in O^c(\theta)$, we have $\Delta(\theta,ij) \neq \emptyset$.
\end{lemma}

Theorem~\ref{thm:regret_lb} provides an asymptotic regret lower bound and is proven in Appendix~\ref{app:proof_lb}. 

\begin{theorem}
    \label{thm:regret_lb}
    Let $\pi$ satisfy~\eqref{eq:def_stability} and~\eqref{eq:def_eps_restrictive}. 
    Moreover, we assume that $\pi$ is consistent, i.e., that for any deterministic payoff vector $\xi\in\R_{\geq 0}^L$,
    \begin{align}
        \lim_{t\to\infty} \P^\xi_\pi(D_{ij}(t) < b \ln(t)) 
        &= 0, \qquad \forall b > 0, \qquad \forall (ij)\in O(\xi).
        \label{eq:consistency}
    \end{align}
    Then, for any $\theta\in\R_{\geq 0}^L$,
    \begin{align}
        \label{eq:regret_lb}
        \lim_{t\to\infty} \frac{R^\theta_\pi(t)}{\ln(t)}
        \geq
        \sum_{(k\ell)\in O^c(\theta)} \frac{\phi_{k\ell}^\theta }{K(\theta,k\ell)}
        > 
        0.
    \end{align}
\end{theorem}

The lower bound in \eqref{eq:regret_lb} is a summation over all suboptimal lines of the suboptimality gap divided by the infimum KL distance as defined in \eqref{eq:def_min_KL_dist}.
This means that a suboptimal line that is nearly optimal, in the sense that its suboptimality gap is small, leads to a large contribution to the regret, since it is difficult for a policy to distinguish between optimal and nearly optimal lines.
Note that the lower bound is not redundant as strict positivity is guaranteed. 

Assumption~\eqref{eq:consistency} is a notion of consistency which is a common assumption in regret analysis in \gls{MAB} literature \citep{Lattimore2020} as it excludes policies that are specialized on a subset of problem instances.
Intuitively speaking, it states that optimal lines are sufficiently explored. 
In particular,~\eqref{eq:consistency} requires that the number of departures of optimal lines eventually grows larger than logarithmic with probability one. 
In classical bandit literature, an assumption on the concentration of the expectation of the number of optimal decisions is often sufficient, since there is exactly one decision at each time step. 
In our case however, the number of departures is random even for stationary policies. 
It remains an open question whether a similar regret lower bound can be proven for a class of policies satisfying a weaker constraint such as $\lim_{t\to\infty} \E_\pi^\theta(tx_{ij}^\theta - D_{ij}(t))/t^b = 0$ for any $b> 0$ and $(ij)\in\calL$, in combination with assumptions~\eqref{eq:def_stability} and~\eqref{eq:def_eps_restrictive}.

Theorem~\ref{thm:regret_lb} states that the regret is $\Omega(\ln(t))$, which is the same order as the lower bound for classical \gls{MAB}s \cite[Theorem 2]{Lai1985}. 
The proof of Theorem \ref{thm:regret_lb} in Appendix~\ref{app:proof_lb} is structured in a similar way as those presented in~\citep{Lai1985, Burnetas1996, Burnetas1997, Lattimore2020} and uses a change--of--measure argument.
Since payoffs are random samples, each line is associated with a level of uncertainty that decreases with the number of obtained samples, which in our case corresponds to departures.
In order to prove the theorem, we construct a second, `confusing' system which has the same parameters as the original model with the exception of the average payoff $\theta_{k\ell}$ for some line $(k\ell)\in\calL$, which we set to a high value such that line $(k\ell)$ is an optimal line in the second model.
We then prove that a consistent policy needs $\Omega(\ln(t))$ samples of payoffs of line $(k\ell)$ in order to distinguish between the two models.

Recall that Theorem~\ref{thm:regret_lb} assumes that the payoff distribution is discrete (see the end of Section~\ref{sec:arr_serv_process}).
We note that the proof can be adapted to include payoff distributions defined by a probability density function $f(\cdot,\theta_{ij})$: 
in this case, the KL divergence in~\eqref{eq:def_KL} changes to
\begin{align}
    I(\theta_{ij},\theta_{k\ell}) := \int_0^\infty f(x;\theta_{ij})\ln\Bigl(\frac{f(x;\theta_{ij})}{f(x;\theta_{k\ell})}\Bigr) \textrm{d}x.
    \label{eq:def_KL_ct}
\end{align}
Similarly, the \gls{RN} derivative in the proof of Theorem~\ref{thm:regret_lb} in Appendix~\ref{app:proof_lb} changes to $\textrm{d}\P_\pi^\theta / \textrm{d}\P_\pi^{\theta'} = \prod_{i=1}^m f(X_i;\theta_{k\ell})/f(X_i;\theta_{kl}')$ and equations~\eqref{eq:LB_proof_KL_exp}, \eqref{eq:LB_proof_B_int_prod}, and~\eqref{eq:LB_proof_meas_change} change accordingly. 
Extension of our result to general probability measures requires careful construction of the \gls{RN} derivative. 
This can be done in line with the discussion following \cite[Proposition 4.8]{Lattimore2020}.

%% file: LearningAlgorithm.tex
In this section, we introduce a routing algorithm called Adaptive UCB Queue Routing Algorithm (UCB QR) and show that its asymptotic regret is close to the lower bound in Theorem~\ref{thm:regret_lb}.
Specifically, Theorem~\ref{thm:regret_ub} proves that its regret is polylogarithmic as time grows large.

\subsection{Description of the algorithm}
We consider a fixed network structure $\calI,\calJ,\calL,\lambda,\mu$ and $\eps > 0$ satisfying \eqref{eq:def_eps_insensitive}. 
We assume that the payoff distribution of line $(ij)\in\calL$ is Bernoulli distributed with parameter $\theta_{ij}\in[0,1]$.\\

We consider the basic feasible solutions of \LP{\theta,\eps} in~\eqref{eq:LP_eps} as the different actions in an \gls{MAB} setting.
Since the set of basic feasible solutions of \LP{\theta,\eps} is finite (see Section~\ref{sec:BFS}), the set of actions is finite and will be denoted by $\calA$.
Concretely, each action $a\in\calA$ is associated to a basic feasible solution $x^a \in\R_{\geq 0}^L$ \LP{\theta,\eps}.
We call the action corresponding to the optimal basic feasible solution the optimal action, which is unique by the nondegeneracy assumption. 

For $a\in\calA$, we denote the set of lines with positive load by $\calL^a\subseteq \calL$, i.e.,
\begin{align}
    \calL^a &:= \{(ij)\in\calL: \ x_{ij}^a >0\}.
    \label{eq:def_set_L_a}
\end{align}
The long-term average reward $r^a\in\R_{\geq 0}$ will be denoted by
\begin{align}
    r^a := \sum_{(ij)\in\calL^a} x_{ij}^a \theta_{ij},
    \label{eq:def_reward_action_a}
\end{align}
and its associated suboptimality gap by
\begin{align}
    d^a := \max_{\tilde{a}} r^{\tilde{a}} - r^a.
    \label{eq:def_subopt_gap_action}
\end{align}
The UCB QR algorithm is given in Algorithm~\ref{alg:learning_alg}.

\begin{algorithm}
    \caption{Adaptive UCB Queue Routing Algorithm (UCB QR).}
    \label{alg:learning_alg}
    \begin{algorithmic}[1]
        \State Initialize $k=1$, for all $(ij)\in\calL$ initialize $T_{ij}(0) = 0$, $\hat{\theta}_{ij}(0) = 0$, and $U_{ij}(0) = \infty$ and for all $a\in\calA$ initialize $U^a(0) = \infty$.
        \For{each episode $k=1,2,\dots$ of length $H_k$}
            \State Choose action $A_k = \argmax_{a} U^{a}(k-1)$, with ties broken arbitrarily.  \label{line:maxucb} 
            \If{$A_k \neq A_{k-1}$}
                \State Reallocate customers to virtual queues: for each waiting type-$i$ customer, assign it to the virtual queue of server $j$ with probability $x_{ij}^{A_k}/\lambda_i$.  \label{line:reallocation1} 
                \State For each $j\in\calJ$, order the customers in virtual queue $j$ by their arrival time.
                \label{line:reallocation2}
            \EndIf
            \For{time $t = 0,\dots, H_k$}
                \State Serve customers according to Algorithm~\ref{alg:routing}: $\FCFSRR{A_k}$. 
            \EndFor
            \State For each line $(ij)\in\calL^{A_k}$, observe $N_{ij}^k\in\N_{\geq 0}$ payoff samples $(Y_{ij}^\ell)_{\ell=1,\dots,N_{ij}^k}$. \label{line:observe_payoffs}
            \State For each line $(ij)\in\calL^{A_k}$, update \label{line:update_ucb}
            \begin{align}
                T_{ij}(k)
                &= T_{ij}(k-1) + N_{ij}^k, 
                \label{eq:UCB_update_algo_T}
                \\
                \hat{\theta}_{ij}(k)
                &= \frac{\hat{\theta}_{ij}(k-1) T_{ij}(k-1) + \sum_{\ell=1}^{N_{ij}^k} Y_{ij}^\ell}{T_{ij}(k)}, 
                \label{eq:UCB_update_algo_theta_hat}
                \\
                U_{ij}(k) &=
                    \begin{cases}
                        \hat{\theta}_{ij}(k) + \sqrt{\frac{ \ln(k)}{T_{ij} (k)}} &\textrm{if} \  T_{ij}(k) > 0, \\
                        \infty &\textrm{if} \ T_{ij}(k) = 0.
                    \end{cases}
                \label{eq:UCB_update_algo_U_per_line}
            \end{align}
            \State Update the UCB index of action $A_k$,
            \begin{align}
                \label{eq:UCB_def_action}
                U^{A_k}(k) &= \sum_{(ij) \in \calL^{A_k}} x_{ij}^{A_k} U_{ij}(k).
            \end{align}
        
        \EndFor
    \end{algorithmic}
\end{algorithm}

\begin{algorithm}
    \caption{$\FCFSRR{a}$.}
    \label{alg:routing}
    \begin{algorithmic}[1]
        \For{each arrival of a type-$i$ customer}
            \State Assign customer to the virtual queue of server $j$ with probability $x_{ij}^a/\lambda_i$.
            \State If the server is idle, start service. 
        \EndFor
        \For{each service completion at server $j$}
            \State Obtain a payoff $Y_{ij}\sim\bern{\theta_{ij}}$, where $i$ is the departing customer's type.
            \State If the virtual queue of server $j$ is nonempty, start service of the customer at the head of the queue.
        \EndFor
    \end{algorithmic}
\end{algorithm}

The learning algorithm works as follows.
We maintain UCB indices for the average payoffs of lines~$(ij)\in\calL$ and use these to compute UCB indices for each action~$a\in\calA$.
At the start of each episode, we choose the action~$a\in\calA$ with the highest UCB index. During the episode, we route customers according to the rates~$x^a\in\R_{\geq 0}^L$ using the \gls{FCFSRR} scheme in Algorithm~\ref{alg:routing}.
We use virtual queues to implement the \gls{FCFSRR} scheme. 
For each departure of a type-$i$ customer at server~$j$, we obtain a payoff sample from a Bernoulli distribution with parameter~$\theta_{ij} \in [0,1]$.
At the end of each episode, we use the obtained payoff samples to update the UCB indices. 
Note that customers may be reallocated to a different virtual queue when a change in action occurs at episode transitions. 
Episode $k$ has a predetermined length $H_k\in\R_{> 0}$ where $H_k$ is a fixed constant that may depend on $k$. \\

For each line $(ij)\in\calL$, we keep track of $T_{ij}(k)$, the total number of departures of type-$i$ customers at server $j$ up to episode number~$k\in\N_{\geq 1}$. 
Similarly, we keep track of the empirical mean~$\hat{\theta}_{ij}(k)$ of type-$(ij)$ payoffs and a UCB estimator $U_{ij}(k)$. 
The counters and empirical means are both initialized at zero while the UCB estimator is initialized at infinity as an incentive for the algorithm to obtain at least one sample from each line.
The number of type-$(ij)$  departures in episode $k$, $N_{ij}^k$, in line \ref{line:observe_payoffs} is random due to the queueing dynamics.

\subsection{Regret upper bound}
Theorem~\ref{thm:regret_ub} gives an asymptotic upper bound on the regret as defined in~\eqref{eq:regret} of Algorithm~\ref{alg:learning_alg}.

We denote the minimal positive rate on a line across any action as $\xmin := \min_{a\in\calA} \min_{(ij)\in\calL^a} x_{ij}^a$. 
Moreover, $\tilde{\lambda}_j^a := \sum_{i\in\calC_j} x_{ij}^a$ denotes the total arrival rate of customers at server $j$ under action $a$ and 
\begin{align}
    \rho_j^a := \tilde{\lambda}^a_j / \mu_j \in (0,1)
    \label{eq:def_rho_j}
\end{align} 
denotes the load of server $j$ under action $a$, for each $j\in\calJ$ and $a\in\calA$. 

\begin{theorem}
    \label{thm:regret_ub}
    Let $\beta \in\R_{> 1}$ and let $\alpha \in\R_{\geq 1}$ satisfy 
    \begin{align}
        \label{eq:c0_choice}
        \alpha &\geq \max_{j\in\calJ}\max_{a\in\calA} \Bigl\{\frac{3(\tilde{\lambda}_j^a+\mu_j) \ln(\rho_j^a) - 2\sqrt{(\mu_j-\tilde{\lambda}_j^a)^2 + 9\tilde{\lambda}_j^a\mu_j \ln^2(\rho_j^a)}}{2(\mu_j-\tilde{\lambda}_j^a)^2 \ln(\rho_j^a)}, 1\Bigr\}.
    \end{align}
    Let $H_0\in\R_{\geq 1}$ and the episode lengths $H_k$, $k=1,2,\dots$, satisfy 
    \begin{align}
        \label{eq:ucbqr_episode_length}
        H_k = \tau_k + H_0 \quad \textrm{with} \quad  
        \tau_k = \alpha \ln^\beta(2Jk) \quad
        \textrm{and} \quad 
        H_0 \geq \max\Bigl\{\frac{4}{\xmin},1\Bigr\}.
    \end{align}
    Lastly, let $\theta\in\R_{\geq 0}^L$ be a payoff vector.
    Then, Algorithm~\ref{alg:learning_alg} satisfies~\eqref{eq:consistency}.
    Moreover, the regret $R^\theta(t)$ of Algorithm~\ref{alg:learning_alg} satisfies
    \begin{align}
        \lim_{t\to\infty} \frac{R^\theta(t)}{\alpha \ln^{2\beta}(t)} 
        &\leq \sum_{(k\ell)\in O^c(\theta)}\phi_{k\ell}^\theta (\mu_\ell-\eps)|\calA| + \sum_{j\in\calJ} w_j^\theta(\mu_j-\eps)|\calA|,
        \label{eq:ucbqr_regret_ub}
    \end{align}
    where  $w^\theta\in\R_{\geq 0}^J$ is the optimal dual variable of \dual{\theta,\eps} in~\eqref{eq:LP_eps_dual} and $\phi^\theta \in \R_{\geq 0}^L$ is defined in~\eqref{eq:phi_def}. 
\end{theorem}

The regret upper bound in Theorem~\ref{thm:regret_ub} provides a theoretical guarantee on the performance of Algorithm~\ref{alg:learning_alg}. 
It shows that the regret scales with rate at most $\alpha \ln^{2\beta}(t)$ as the time horizon $t$ grows large. 
Note that this matches with the regret lower bound in Theorem~\ref{thm:regret_lb} up to logarithmic terms. 
From the right side of \eqref{eq:ucbqr_regret_ub}, we observe that the upper bound increases with the cardinality of the action set $|\calA|$, this will be discussed in Section~\ref{sec:disc_transfer_learning}. 
Moreover, the regret for each suboptimal action is upper bounded by the instant regret of suboptimal lines in the set $O^c(\theta)$ (recall \eqref{eq:def_suboptimal_lines_theta}) multiplied by the maximal expected departure rate at server $j$, which is $\mu_j-\eps$ by construction. 
For a suboptimal line $(ij)\in O^c(\theta)$, the instant regret is expressed by the suboptimality gap $\phi_{ij} \in \R_{> 0}$  in \eqref{eq:phi_def}.
The last term in~\eqref{eq:ucbqr_regret_ub} is related to the queueing regret: it represents the regret obtained at the start of optimal episodes when the queue length process is not necessarily close to its stationary behavior. 
We outline the proof of Theorem~\ref{thm:regret_ub} in Section~\ref{sec:ucbqr_proof_outline}.
The full proof is given in Appendix~\ref{app:proof_ub}.

\subsubsection{Episode length}
In order to analyze the regret of the algorithm, we provide lower bounds on the number of departures for each line.
The analysis is challenging since the routing scheme changes at the start of each episode, hence the starting state of the queueing system is not necessarily close to its stationary behavior, as illustrated in Figure~\ref{fig:stationary_episodes}.
This motivates the introduction of a \emph{warmup time} $\tau_k$ in the construction of the episode length $H_k$ in \eqref{eq:ucbqr_episode_length}. The warmup time is followed by a period $H_0$ of fixed length.

In the analysis of the algorithm, we provide a lower bound on the probability of convergence of the queue length process to its stationary behavior within the warmup time.
Constraint \eqref{eq:c0_choice} guarantees that this bound is sharp enough. 
The warmup time $\tau_k$ is increasing in $k$, so that this probability converges to 1 at the right speed in terms of the episode number~$k$.

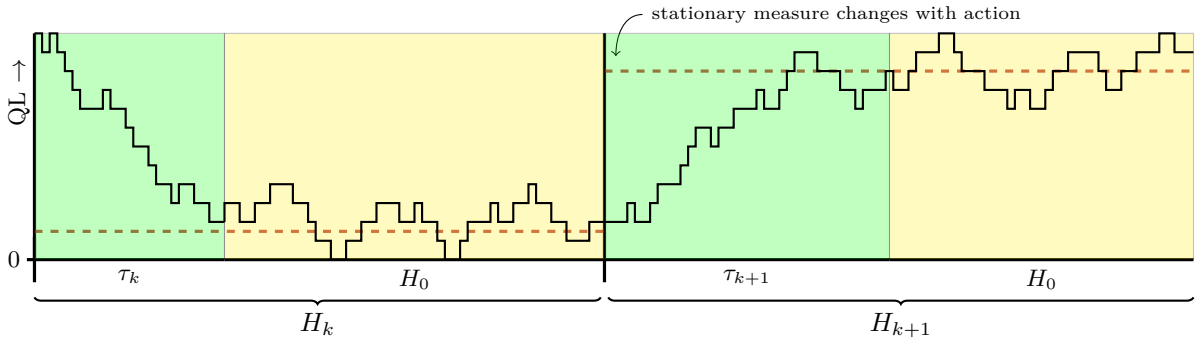
\begin{figure}[H]
    \begin{center}
    \begin{tikzpicture}[scale=.5]
        \draw[very thick, black] (0,0) -- (0,6) node[pos=0.9,left,rotate=90,yshift=.2cm,xshift=.1cm] {\small QL $\rightarrow$};
        \draw[very thick, black] (0,0) -- (5,0) node[pos = 0.5, below] {\small $\tau_k$};
        \draw[very thick, black] (5,0) -- (15,0) node[pos = 0.5, below] {\small $H_0$};
        \draw[very thick, black] (15,0) -- (22.5,0) node[pos = 0.5, below] {\small $\tau_{k+1}$};
        \draw[very thick, black] (22.5,0) -- (30.5,0) node[pos = 0.5, below] {\small $H_0$};
        \draw[very thick] (0,-.6) -- (0,.6);
        \draw[very thick] (-.2,0) -- node[left] {\small 0} (0,0);
        \draw[very thick] (15,-.6) -- (15,6);
        \draw[darkred,opacity=.7,very thick,dashed] (0,.75) --++ (15,0);
        \draw[darkred,opacity=.7,very thick,dashed] (15,5) --++ (15.5,0);
        \node[draw = none,right] at (16,6.5) (labelaction) {\scriptsize stationary measure changes with action};
        \draw[<-,black] (15.3,5.4) to [in=180,out=90] (labelaction);
        \draw[fill=green,opacity=0.25] (0,0) rectangle ++(5,6);
        \draw[fill=yellow,opacity=0.25] (5,0) rectangle ++(10,6);
        \draw[fill=green,opacity=0.25] (15,0) rectangle ++(7.5,6);
        \draw[fill=yellow,opacity=0.25] (22.5,0) rectangle ++(8,6);
        \draw[thick] (0,6) -- (0.2,6) -- (0.2,5.5) -- (0.4,5.5) -- (0.4,6) -- (0.6,6) -- (0.6,5.5) -- (0.8,5.5) -- (0.8,5) -- (1,5) -- (1,4.5) -- (1.2,4.5) -- (1.2,4) -- (1.8,4) -- (1.8,4.5) -- (2,4.5) -- (2,4) -- (2.4,4) -- (2.4,3.5) -- (2.6,3.5) -- (2.6,3) -- (3,3) -- (3,2.5) -- (3.2,2.5) -- (3.2,2) -- (3.6,2) -- (3.6,1.5) -- (3.8,1.5) -- (3.8,2) -- (4.2,2) -- (4.2,1.5) -- (4.6,1.5) -- (4.6,1) --  (5,1);
        \draw[thick] (5,1) --++ (0,.5) --++ (.4,0) --++ (0,-.5) --++ (.4,0) --++ (0,.5) --++ (.4,0) --++ (0,.5) --++ (.6,0) --++ (0,-.5) --++ (.4,0) --++ (0,-.5) --++ (.2,0) --++ (0,-.5) --++ (.4,0) --++ (0,-.5) --++ (.4,0) --++ (0,.5) --++ (.4,0) --++ (0,.5) --++ (.4,0) --++ (0,.5) --++ (.6,0) --++ (0,-.5) --++ (.4,0) --++ (0,.5) --++ (.2,0) --++ (0,-.5) --++ (.4,0) --++ (0,-.5) --++ (.2,0) --++ (0,-.5) --++ (.4,0)--++ (0,.5) --++ (.2,0) --++ (0,.5) --++ (.6,0) --++ (0,.5) --++ (.2,0) --++ (0,-.5) --++ (.4,0) --++ (0,.5) --++ (.4,0) --++ (0,.5) --++ (.2,0) --++ (0,-.5) --++ (.4,0) --++ (0,-.5) --++ (.4,0) --++ (0,-.5) --++ (.6,0) --++ (0,.5) --++ (.4,0);
        \draw[thick] (15,1) --++ (.2,0) --++ (.4,0) --++ (0,.5) --++ (.2,0) --++ (0,-.5) --++ (.4,0) --++ (0,.5) --++ (.2,0) --++ (0,.5) --++ (.6,0) --++ (0,.5) --++ (.2,0) --++ (0,.5) --++ (.2,0) --++ (0,.5) --++ (.4,0) --++ (0,-.5) --++ (.2,0) --++ (0,.5) --++ (.4,0) --++ (0,.5) --++ (.6,0) --++ (0,.5) --++ (.2,0) --++ (0,-.5) --++ (.4,0) --++ (0,.5) --++ (.2,0) --++ (0,.5) --++ (.2,0) --++ (0,.5) --++ (.6,0)--++ (0,-.5) --++ (.6,0) --++ (0,-.5) --++ (.4,0) --++ (0,-.5) --++ (.2,0) --++ (0,.5) --++ (.6,0) --++ (0,.5) --++ (.2,0) --++ (0,-.5) --++ (.4,0) --++ (0,.5) --++ (.2,0) --++ (0,.5) --++ (.6,0) --++ (0,.5) --++ (.4,0) --++ (0,-.5) --++ (.2,0) --++ (0,-.5) --++ (.6,0) --++ (0,-.5) --++ (.6,0) --++ (0,-.5) --++ (.2,0) --++ (0,.5) --++ (.4,0) --++ (0,-.5) --++ (.4,0) --++ (0,.5) --++ (.2,0) --++ (0,.5) --++ (.4,0) --++ (0,.5) --++ (.6,0) --++ (.2,0) --++ (0,-.5) --++ (.2,0) --++ (0,-.5) --++ (.4,0) --++ (0,.5) --++ (.4,0) --++ (0,.5) --++ (.6,0) --++ (0,.5) --++ (.4,0) --++ (0,-.5) --++ (.5,0);
    
        \draw [thick,decoration={brace,mirror,raise=0.5cm},decorate] (0,0) -- (14.9,0) node [midway,anchor=north,yshift=-.57cm] {$H_k$};
        \draw [thick,decoration={brace,mirror,raise=0.5cm},decorate] (15.1,0) -- (30.5,0) node [midway,anchor=north,yshift=-.57cm] {$H_{k+1}$};
    \end{tikzpicture}
    
	\end{center}
    \caption{Possible realization of the queue length process of the virtual queue of server $j\in\calJ$ under Algorithm~\ref{alg:learning_alg} for two episodes. Episode $k$ consists of a warmup time $\tau_k$ followed by a period of fixed length $H_0$.  Every episode, the algorithm can choose a different action with its own stationary measure.}
    \label{fig:stationary_episodes}
\end{figure}

\subsection{Proof outline}
\label{sec:ucbqr_proof_outline}
We sketch the proof of Theorem~\ref{thm:regret_ub}, highlighting the key contributions. 
The proof can be split into four steps. 
We shortly explain these steps, followed by a more in-depth elaboration per step. 
Without loss of generality, we label the first action as the unique optimal action, i.e., $\argmax_{a\in\calA} r^a = 1$. 
\begin{enumerate}
    \item[Step 1.] {\it Analyzing the number of departures in \gls{FCFSRR}.} 
    For the proof of Theorem~\ref{thm:regret_ub}, we need to show that Algorithm~\ref{alg:learning_alg} learns the payoff parameters sufficiently fast. 
    Since payoff samples are collected upon departures, we need to prove a lower bound on the number of departures in the queueing system. 
    We do so via an intermediate step: 
    we consider the scenario where the system is initially empty at the start of an episode, which represents a `worst case' scenario. 
    Next, we show that this worst case queue length process reaches stationarity within warmup time $\tau_m$ with sufficiently high probability (Lemma~\ref{lem:episode_not_mixed}).
    \item[Step 2.] {\it Bounding the probability of choosing a suboptimal action.}
    We note that Algorithm~\ref{alg:learning_alg} chooses a suboptimal action in Line \ref{line:maxucb}  only if the UCB index of the suboptimal action in \eqref{eq:UCB_def_action} is at least as high as the UCB index of the optimal action. 
    Broadly speaking, this can happen if
    (A) the index of the suboptimal action \emph{overestimates} its true mean, or
    (B) the index of the optimal action \emph{underestimates} its true mean.
    Compared to the classical \gls{MAB} analysis in \citep{Lattimore2020}, event (A) is more complicated to analyze in our model, since the number of payoff samples we obtain within one episode is stochastic rather than deterministic as in the classical \gls{MAB} setting. 
    We show that event (A) is unlikely by showing that on the one hand, overestimation based on 'sufficient' samples is unlikely and on the other hand, it is unlikely to obtain 'insufficient' samples. Here, the term 'sufficient' is carefully constructed, as discussed below. 
    Event (B) is unlikely by construction of the UCB estimators.
    We bound the probability of event (A) in Lemma~\ref{lem:prob_insuf_samples} and~\ref{lem:bnd_overestimation_suf_samples} and event (B) in Lemma~\ref{lem:prob_underestimation}.
    \item[Step 3.] {\it Bounding the number of suboptimal episodes.} 
    We use the bounds from Step 2 to bound the number of episodes where Algorithm~\ref{alg:learning_alg} chooses a suboptimal action in Lemma~\ref{lem:n_bad_episodes}.
    \item[Step 4.] {\it Bounding the regret of Algorithm~\ref{alg:learning_alg}.} 
    Finally, we prove Theorem~\ref{thm:regret_ub} in Appendix~\ref{app:proof_ub} using the regret decomposition~\eqref{eq:regret_decomp}. 
    Term I is the main contributing factor since it measures the regret accumulated by using suboptimal lines. We bound this term using Lemma~\ref{lem:n_bad_episodes}.
    Term II is bounded by analyzing the queue length behavior in episodes where the algorithm chooses the optimal action.
    Lastly, we bound term III using the properties of \gls{FCFSRR} and the constraints in~\LP{\theta,\eps}. 
\end{enumerate}

Let us next describe Steps 1-4 in more detail. 

\subsubsection{Step 1. Analyzing the number of departures in FCFS RR}
\label{sec:step1}
We consider the queue length process $Q_j^{mA_m}(t)$, $t\in [0,H_m)$ of the virtual queue of server $j\in\calJ$ during episode $m\in\N_{\geq 1}$, where $A_m$ denotes the action chosen by Algorithm~\ref{alg:learning_alg} in Line~\ref{line:maxucb}.
Within an episode, Algorithm~\ref{alg:learning_alg} routes customers according to the fixed \gls{FCFSRR} policy in Algorithm~\ref{alg:routing}. 
Since arrival processes are independent across customer types, we have by the Poisson split and merge properties \citep{Cinlar1968} that the arrival process of $Q_j^{mA_m}(t)$ is a Poisson process with rate $\sum_{i\in\calI} x_{ij}^{A_m}$.
This implies that $Q_j^{mA_m}(t)$ behaves as an $M/M/1$ queueing system independently from other servers.
Recall that $x^{A_m}$ is a basic feasible solution of $\LP{\theta,\eps}$ in \eqref{eq:LP_eps}, so by constraint \eqref{eq:LP_eps_constr_mu}, we have $\sum_{i\in\calC_j} x_{ij}^{A_m} \leq \mu_j - \eps < \mu_j$.
Therefore, $\rho_j^{A_m} = \sum_{i\in\calI} x_{ij}^{A_m} / \mu_j < 1$ and hence  $Q_j^{mA_m}(t)$ is positive recurrent \citep{Cohen1981}, and its stationary distribution is given by 
\begin{align}
    p_j^{A_m}(n) := \lim_{t\to\infty} \P(Q_j^{mA_m}(t) = n) = (1-\rho_j^{A_m})(\rho_j^{A_m})^n.
    \label{eq:queue_statdistr}
\end{align}

For each server, we associate two new queue length processes $\hat{Q}_j^{mA_m}(t)$ and $\underline{Q}_j^{mA_m}(t)$
where the initial value $\hat{Q}_j^{mA_m}(0)$ is sampled from the stationary measure $p_j^{A_m}$
and $\underline{Q}_j^{mA_m}(0) = 0$.
We couple the processes $Q_j^{mA_m}(t)$, $\hat{Q}_j^{mA_m}(t)$, and $\underline{Q}_j^{mA_m}$ in the sense that the sampled arrival times and service completions (if possible) are the same for all processes, as illustrated in Figure~\ref{fig:coupling}.

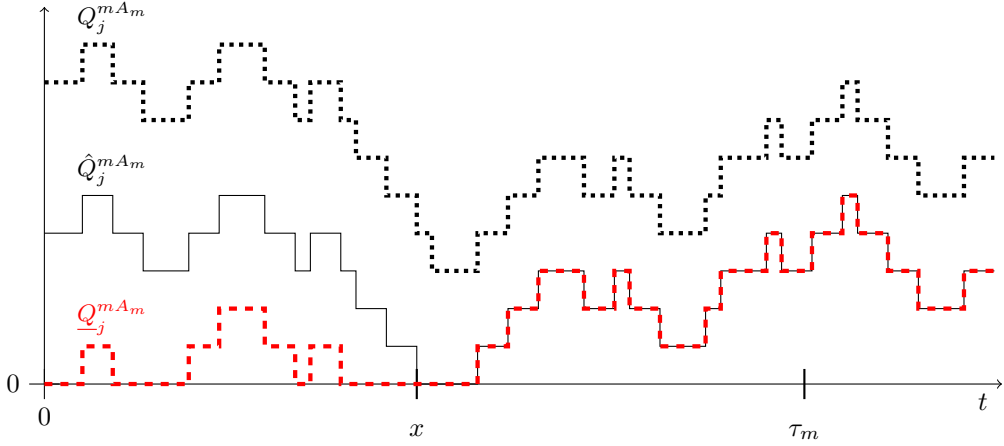
\begin{figure}[H]
    \centering
    \begin{tikzpicture}[scale=1]
        \draw[black, ->] (0,0) -- (0,5);
        \draw[black, ->] (0,0) -- (12.6,0) node[pos = 0.98, below] {$t$};
        \draw (0,-.2) -- node[below,yshift=-.2cm] {0} (0,.2);
        \draw (-.2,0) -- node[left,xshift=-.2cm] {0} (.2,0);
        \draw[thick] (10,-.2)  node[below,yshift=-.2cm]  {$\tau_m$} --++ (0,.4);
        \draw[thick] (4.9,-.2) node[below,yshift=-.2cm]  {$x$}  --++(0,.4);
        \draw[ultra thick,dotted]  (0,4)  --++ (.5,0) --++ (0,.5) --++ (.4,0) node[above] {\small $Q_j^{mA_m}$} --++ (0,-.5) --++ (.4,0) --++(0,-.5) --++ (.6,0) --++ (0,.5) --++ (.4,0) --++ (0,.5) --++ (.6,0) --++ (0,-.5) --++ (.4,0) --++ (0,-.5) --++ (.2,0) --++ (0,.5) --++ (.4,0) --++ (0,-.5) --++ (0.2,0) --++ (0,-.5) --++ (.4,0) --++ (0,-.5) --++ (.4,0) --++ (0,-.5) --++ (0.2,0) --++ (0,-.5) --++ (.6,0) --++ (0,.5) --++ (.4,0) --++ (0,.5) --++ (.4,0) --++ (0,.5) --++ (.6,0) --++ (0,-.5) --++ (.4,0) --++ (0,.5) --++ (.2,0) --++ (0,-.5) --++ (.4,0) --++ (0,-.5) --++ (.2,0)  --++ (.4,0)--++ (0,.5) --++ (.2,0) --++ (0,.5) --++ (.6,0) --++ (0,.5) --++ (.2,0) --++ (0,-.5) --++ (.4,0) --++ (0,.5) --++ (.4,0) --++ (0,.5) --++ (.2,0) --++ (0,-.5) --++ (.4,0) --++ (0,-.5) --++ (.4,0) --++ (0,-.5) --++ (.6,0) --++ (0,.5) --++ (.4,0);
        \draw  (0,2)  --++ (.5,0) --++ (0,.5) --++ (.4,0) node[above] {\small $\hat{Q}_j^{mA_m}$} --++ (0,-.5) --++ (.4,0) --++(0,-.5) --++ (.6,0) --++ (0,.5) --++ (.4,0) --++ (0,.5) --++ (.6,0) --++ (0,-.5) --++ (.4,0) --++ (0,-.5) --++ (.2,0) --++ (0,.5) --++ (.4,0) --++ (0,-.5) --++ (0.2,0) --++ (0,-.5) --++ (.4,0) --++ (0,-.5) --++ (.4,0) --++ (0,-.5) --++ (0.2,0)  --++ (.6,0) --++ (0,.5) --++ (.4,0) --++ (0,.5) --++ (.4,0) --++ (0,.5) --++ (.6,0) --++ (0,-.5) --++ (.4,0) --++ (0,.5) --++ (.2,0) --++ (0,-.5) --++ (.4,0) --++ (0,-.5) --++ (.2,0)  --++ (.4,0)--++ (0,.5) --++ (.2,0) --++ (0,.5) --++ (.6,0) --++ (0,.5) --++ (.2,0) --++ (0,-.5) --++ (.4,0) --++ (0,.5) --++ (.4,0) --++ (0,.5) --++ (.2,0) --++ (0,-.5) --++ (.4,0) --++ (0,-.5) --++ (.4,0) --++ (0,-.5) --++ (.6,0) --++ (0,.5) --++ (.4,0);
        \draw[ultra thick, red, dashed]  (0,0)  --++ (.5,0) --++ (0,.5)   --++ (.4,0) node[above] {\small $\underline{Q}_j^{mA_m}$} --++ (0,-.5) --++ (.4,0) --++ (.6,0) --++ (0,.5) --++ (.4,0) --++ (0,.5) --++ (.6,0) --++ (0,-.5) --++ (.4,0) --++ (0,-.5) --++ (.2,0) --++ (0,.5) --++ (.4,0) --++ (0,-.5) --++ (0.2,0) --++ (1.6,0)--++ (0,.5) --++ (.4,0) --++ (0,.5) --++ (.4,0) --++ (0,.5) --++ (.6,0) --++ (0,-.5) --++ (.4,0) --++ (0,.5) --++ (.2,0) --++ (0,-.5) --++ (.4,0) --++ (0,-.5) --++ (.2,0)  --++ (.4,0)--++ (0,.5) --++ (.2,0) --++ (0,.5) --++ (.6,0) --++ (0,.5) --++ (.2,0) --++ (0,-.5) --++ (.4,0) --++ (0,.5) --++ (.4,0) --++ (0,.5) --++ (.2,0) --++ (0,-.5) --++ (.4,0) --++ (0,-.5) --++ (.4,0) --++ (0,-.5) --++ (.6,0) --++ (0,.5) --++ (.4,0);
    \end{tikzpicture}
    \caption{A possible realization of the queue length processes $Q_j^{mA_m}(t)$, $\hat{Q}_j^{mA_m}(t)$, and $\underline{Q}_j^{mA_m}(t)$. 
    Customer arrivals lead to a unit increase and service completions to a unit decrease (unless the queue is empty).
    At hitting time $x$ the processes $\hat{Q}_j^{mA_m}(t)$ and $\underline{Q}_j^{mA_m}(t)$ collide and evolve identically afterwards. }
    \label{fig:coupling}
\end{figure}

Note that if $\hat{Q}_j^{mA_m}(t)$ and $\underline{Q}_j^{mA_m}(t)$ collide, then the departure process of $\underline{Q}_j^{mA_m}(t)$ is stationary afterwards. 
Lemma~\ref{lem:episode_not_mixed} provides a bound on the hitting time of $\hat{Q}_j^{mA_m}(t)$ and $\underline{Q}_j^{mA_m}(t)$. 
The proof in Appendix~\ref{app:episodes_not_mixed} relies on hitting time analysis of an $M/M/1$ queueing system. 
Lemma~\ref{lem:episode_not_mixed} extends the result of \cite[Proposition 4]{Jia2022}, in the sense that we provide a sharper bound by exploiting the known formula for the moment generating function of hitting times in an $M/M/1$ system. We also provide exact values for the constants in our result, which are not provided in \citep{Jia2022}.

\begin{lemma}
    \label{lem:episode_not_mixed}
    Let  $\beta \in\R_{> 1}$ and let $m\in\N_{\geq 1}$ satisfy $m\geq C_\beta$ with 
    \begin{align}
        C_\beta := \frac{1}{2J} \exp\bigl(\beta^{\frac{1}{\beta-1}}\bigr).
        \label{eq:def_Cbeta}
    \end{align}
    For any $t\geq \tau_m$, we have
    $
        \P(\underline{Q}_j^{mA_m}(t) = \hat{Q}_j^{mA_m}(t), \ \forall j\in\calJ) \geq 1-1/m^\beta.
    $
\end{lemma}

\subsubsection{Step 2. Bounding the probability of choosing a suboptimal action}\hspace{2em}
\begin{flushleft}
{\bf Bound on event (A).}
Note that the confidence bound in \eqref{eq:UCB_def_action} decreases with the number of obtained samples (departures).
This means that overestimation of the true mean is likely when the number of samples is small, but becomes less likely as the number of obtained samples increases.
We split the analysis into two parts. 
First, Lemma~\ref{lem:prob_insuf_samples} bounds the probability that Algorithm~\ref{alg:learning_alg} has not obtained sufficient samples after $C_\beta + \ln^\beta(k)$ episodes.
Here, we quantify `sufficient' by $\sigma_{ijk}^a$ in \eqref{eq:def_seq_u_eta_sig_lem} below.\\
\end{flushleft}

To provide the lemma, we introduce some notation. 
Let
\begin{align}
    \overline{\theta}_{ij}[s]&:= \frac 1s \sum_{\ell=1}^s Y_{ij}^\ell
\end{align}
denote the empirical average payoff of type-$(ij)$ departures based on $s$ samples.
Here, the $(Y_{ij}^\ell)_{\ell=1}^s$ denote i.i.d.\ random variables with distribution  $\bern{\theta_{ij}}$.
We let $D_{ij}^{mA_m}$ denote the number of type-$(ij)$ departures within episode $m\in\N_{\geq 1}$ \emph{but after warmup time $\tau_m$}, where $A_m$ denotes the action chosen by Algorithm~\ref{alg:learning_alg} in Line~\ref{line:maxucb}.
Moreover, for $a\in\calA$ let
\begin{align}
    \overline{T}_{ij}^a[s] := \sum_{m=1}^s D_{ij}^{ma}
    \label{eq:def_overline_T}
\end{align}
denote the total number of type-$(ij)$ departures (after warmup periods) after completing $s\in\N_{\geq 1}$ episodes where action $a$ was chosen.
We let $\overline{U}^a[s,m]$ denote the UCB index of action $a\in\calA$ at episode $m\in\N_{\geq 1}$ if action $a$ was chosen in $s\leq m$ episodes, 
\begin{align}\label{eq:Uestimator_episodes}
    \overline{U}^a[s,m] &:= \sum_{(ij)\in\calL^a} x_{ij}^a \Bigl(\overline{\theta}_{ij}\bigl[\overline{T}_{ij}^a[s]\bigr]  + \sqrt{\frac{ \ln(m)}{\overline{T}_{ij}^a[s]}}\Bigr) .
\end{align}
We  define for $k\in\N_{\geq 1}$,  $a\in\calA$, and $(ij)\in\calL^a$ the quantities
\begin{align}
    \label{eq:def_seq_u_eta_sig_lem}
    u_k &= \ln^\beta(k), \qquad
    \eta_k = \ln^{\frac{\beta+1}{2}}(k), \qquad
    \sigma_{ijk}^a = x_{ij}^a H_0 (u_k - 2\eta_k).
\end{align} 
We are interested in the event that the number of type-$(ij)$ departures under action $a$ after $\lceil C_\beta + u_k\rceil$ episodes is at least $\lceil\sigma_{ijk}^a\rceil$, i.e., 
\begin{align}
    \label{eq:def_event_E_lem}
    E_k^a &= \Bigl\{ \overline{T}_{ij}^a\bigl[\lceil C_\beta + u_k\rceil\bigr] \geq  \lceil\sigma_{ijk}^a\rceil, \ \forall (ij)\in\calL^a \Bigr\}.
\end{align}

\begin{lemma}
    \label{lem:prob_insuf_samples}
    Let $\beta\in\R_{> 1}$ and $a\in\calA$. Then $\lim_{k\to\infty} k \P((E_k^a)^c) = 0$.
\end{lemma}
The proof of Lemma~\ref{lem:prob_insuf_samples} in Appendix \ref{app:prob_insuf_samples} relies on Lemma~\ref{lem:episode_not_mixed}.
We use that the number of departures is at least as high as the number of departures in a coupled system that starts from the empty state at the start of the episode (the `worst case' scenario). 
We then apply Lemma~\ref{lem:episode_not_mixed} to show that the worst case process reaches stationarity with high probability. 
Next, we use Poisson merging and splitting properties (see e.g. \citep{Cinlar1968}) and Burke's Theorem~\cite[II.2.4 Theorem 2.1]{Cohen1981} to obtain that the \emph{stationary} departure process of the $\FCFSRR{a}$ policy in Algorithm~\ref{alg:routing} is a Poisson process with rate $x_{ij}^a$. 
Lastly, we use a Poisson tail bound (see Lemma~\ref{lem:pois_tail}) to obtain a lower bound on the number of departures of the stationary process. \\

For the second part of event (A), Lemma~\ref{lem:bnd_overestimation_suf_samples} shows that the probability of overestimating the true reward is small if the number of obtained samples (departures) is sufficiently large.
The proof in Appendix~\ref{app:bnd_overestimation_suf_samples} relies on Hoeffding's inequality \cite[Theorem 2]{Hoeffding1963}.

\begin{lemma}
    \label{lem:bnd_overestimation_suf_samples}
    Let $\beta\in\R_{> 1}$,  $a\in\calA$, and $\xi_k := \ln^{-\frac{\beta}{4}}(k)$, then $
        \lim_{k\to\infty} \sum_{s= \lceil C_\beta + u_k \rceil}^k \P\bigl(\overline{U}^a[s,k] - r^a  \geq d^a -\xi_k, \ E_k^a\bigr) = 0$.
\end{lemma}

The value $\sigma_{ijk}^a$ in~\eqref{eq:def_seq_u_eta_sig_lem} is constructed precisely such that Lemma~\ref{lem:prob_insuf_samples} and Lemma~\ref{lem:bnd_overestimation_suf_samples} hold simultaneously. 
The intuition behind this value is as follows: if we have observed $u_k$ episodes with action $a$, 
then for approximately $u_k-\eta_k$ of those episodes, 
the number of departures can be bounded from below by the number of departures of the  stationary process, which is Poisson distributed. 
From these episodes, we lose approximately $\eta_k$ episodes where the Poisson tail bound fails (see Lemma~\ref{lem:pois_tail}).
Hence, we collect payoff samples (departures) at a rate of approximately $u_k-2\eta_k$, as reflected in the definition of $\sigma_{ijk}^a$.  \\

\noindent {\bf Bound on event (B).}
For event (B), recall that the true mean of action $a\in\calA$ is $r^a$ as defined in \eqref{eq:def_reward_action_a}.
By construction, the UCB index in \eqref{eq:UCB_def_action} includes a confidence bound (exploration bonus) which implies that the probability of event (B) is small. 
This is made precise in Lemma~\ref{lem:prob_underestimation}, which is proven in Appendix \ref{app:prob_underestimation}.
We let $\|\cdot\|_2$ denote the $L^2$-norm. 

\begin{lemma}
    \label{lem:prob_underestimation}
    Let $\beta\in\R_{> 1}$, $a\in\calA$, $k\in\N_{\geq 1}$, and $0 <\xi <r^a$.
    Then, 
    $
        \sum_{m=1}^k \P(U^a(m) < r^a - \xi) 
        \leq
        \pi^2 \|\lambda\|_2^2/(12\xi^2).
    $
\end{lemma}

\subsubsection{Step 3. Bounding the number of suboptimal episodes}
We define $S^a(k)$ as the number of episodes where Algorithm~\ref{alg:learning_alg} chooses action $a\in\calA$ in Line~\ref{line:maxucb} up to and including episode number~$k$, i.e.,
\begin{align}
    S^a(k) &:= \sum_{m=1}^k \ind{A_m = a}.
    \label{eq:Sak_def}
\end{align}
From Lemma~\ref{lem:n_bad_episodes} it follows that the expected number of suboptimal episodes grows asymptotically with rate at most $\ln^\beta(k)$ with the number of episodes $k$. 
The proof of Lemma~\ref{lem:n_bad_episodes} is given in Appendix~\ref{app:n_bad_episodes}, and relies on Lemma~\ref{lem:prob_insuf_samples},~\ref{lem:bnd_overestimation_suf_samples}, and~\ref{lem:prob_underestimation}.

\begin{lemma}
    \label{lem:n_bad_episodes}
    Let $\beta\in\R_{> 1}$. 
    For any suboptimal action $a\in\calA\setminus\{1\}$ we have
    \begin{align}
        \lim_{k \to \infty} \frac{\E(S^a(k))}{\ln^\beta(k)}
        &\leq 
        1. 
    \end{align}
\end{lemma}

\subsubsection{Step 4. Bounding the regret of Algorithm~\ref{alg:learning_alg}}
We bound the terms in the regret decomposition~\eqref{eq:regret_decomp} individually. 
For term~I, we use that suboptimal lines are only used in suboptimal episodes by construction. 
Hence, for $(k\ell)\in O^c(\theta)$, the expected number of type-$(k\ell)$ departures can be upper bounded by the maximal departure rate of server~$\ell$, which is $\mu_\ell-\eps$ by construction, multiplied by the total number of suboptimal episodes. This, we bound using Lemma~\ref{lem:n_bad_episodes}. 

For term~II, we show that the queue length process in consecutive optimal episodes converges sufficiently fast to its stationary behavior. 
Specifically, we use the bound in Lemma~\ref{lem:episode_not_mixed}: if \mbox{$A_k=A_{k+1}=1$} and the queue length processes reaches stationarity within the warmup period of episode $k$, then the process is stationary at the start of episode $k+1$, since the stationary measure does not change in this case. 

Lastly, for term~III we use the fact that $x^{A_m}$ satisfies~\eqref{eq:LP_eps_constr_mu} to show that the virtual queues of all servers are positive recurrent under the \FCFSRR{A_m} policy. In view of~\eqref{eq:LP_eps_constr_x}, this implies that the queue of any customer type is positive recurrent as well. This means that the expected queue lengths remain finite, and hence the contribution of term~III to the regret vanishes on a logarithmic scale.

\subsection{Discussion}
\subsubsection{Avoiding carryover effect between episodes}
The reallocation in Lines~\ref{line:reallocation1}-\ref{line:reallocation2} avoids a delay in observations when a new episode starts:  if the chosen action in episode $k$ differs from the action that was chosen in the previous episode $k-1$, we reallocate all waiting customers to the virtual queues according to new rates of the action chosen in episode $k$. 
After the new reallocation, we order the customers in each virtual queue in order of their arrival times, so that the service order within the virtual queue is \gls{FCFS}.
This reallocation improves the learning efficiency since there is no `carryover' effect between episodes, unlike other learning algorithms in literature \citep{Fu2022,Jia2022}. 

\subsubsection{Transfer learning} 
\label{sec:disc_transfer_learning}
The rewards of the actions are related by the common unknown payoff parameter $\theta$ via~\eqref{eq:def_reward_action_a}.
Algorithm~\ref{alg:learning_alg} utilizes transfer learning in the sense that knowledge about the payoff parameter is shared between actions:
specifically, the \gls{UCB} index of a reward is computed using the \gls{UCB} indices of the lines.
This effect also becomes evident in the numerical analysis in Section~\ref{sec:numerical_small}.

Our regret bound in Theorem~\ref{alg:learning_alg} is pessimistic: in analyzing the convergence of the \gls{UCB} index $U^a$ in \eqref{eq:UCB_def_action}, we only  account for departures that are obtained in episodes where action $a$ was chosen. 
However, we possibly observe type-$(ij)$ departures as well in episodes where another action $\tilde{a}$ with $x_{ij}^{\tilde{a}} > 0$ was chosen. 
This insight can be used in future work to improve the bound in Lemma~\ref{lem:prob_insuf_samples} and in turn improve the regret bound in Theorem~\ref{thm:regret_ub}. 
In particular, it might be possible to replace the cardinality of the action  $|\calA|$ in~\eqref{eq:ucbqr_regret_ub} by the number of lines $L$, which is typically smaller. 
Such analysis methods are often used in the setting of linear bandits \citep{Lattimore2020}.

\subsubsection{Parameter values}
The required lower bound on $\alpha$ in \eqref{eq:c0_choice} can be  large when the value of the slack parameter $\eps >0$ is small.
As a consequence, the episode length $H_k$ in \eqref{eq:ucbqr_episode_length} increases quickly in the number of episodes $k$. 
Since Algorithm~\ref{alg:learning_alg} only chooses a new action at the start of episodes, this can in turn lead to a slow decrease in the rate of regret accumulation.
\eqref{eq:c0_choice} is constructed in such a way to bound a term in the proof of Lemma~\ref{lem:episode_not_mixed}. However, this bound is quite loose. More careful analysis of the individual terms can potentially give a lower value than \eqref{eq:c0_choice}. 
We show in Section \ref{sec:numerical_big} that Algorithm~\ref{alg:learning_alg} still performs well in finite time, even when $\alpha$ is chosen smaller than the lower bound in \eqref{eq:c0_choice}.

\subsubsection{Number of actions} 
The number of actions grows exponentially with the number of queues $I$, servers $J$, and lines $L$ in the queueing system.
Concretely, the number of bases of \LP{\theta,\eps} in \eqref{eq:LP_eps} is $\binom{L+J}{I+J}$,
hence the number of basic feasible solutions grows exponentially with the system's complexity. 
Reducing the number of actions using dimension reduction techniques is appealing for future research. 
For example, we could apply a column generation technique on the utility function that uses the UCB indices as proxy for the true payoff parameters  \citep{Desrosiers2005}, or use a branch--and--bound technique like Bender's decomposition to split the optimization problem into a master problem and subproblems \citep{Garcia-Munoz2023}.

%% file: NumericalResults.tex
In this section we analyze the properties of Algorithm~\ref{alg:learning_alg} (UCB QR) numerically.
In Section~\ref{sec:numerical_small}, we consider a small skill-based queueing system.
For this system, we analyze the regret of the UCB QR algorithm and compare its long-term average payoff rate with several benchmark routing policies.  
In Section~\ref{sec:numerical_changing_theta}, we analyze the robustness of the learning policy against changes in the true payoff parameter.
In Section~\ref{sec:numerical_big}, we consider a larger queueing system and compare the performance of UCB QR against benchmark policies in several  scenarios.\\

We consider four different benchmark policies which are described as follows. 
\begin{itemize}
    \item {\bf Oracle:} Almost the same as the UCB QR algorithm, but with Line~\ref{line:maxucb} replaced by $A_k = \argmax_{\tilde{a}} r^{\tilde{a}}$. Note that this policy relies on the true payoff parameter $\theta$. This policy represents the optimal upper bound for the long-term reward rate of any stabilizing policy. 
    \item {\bf FCFS ALIS:}  If there are nonempty compatible queues at a service completion, assign the customer with maximal waiting time. If there are compatible servers idle upon a customer arrival, assign it to the server with maximal idle time.
    This policy is widely used in practice since both customers and servers experience a sense of fairness \citep{Adan2014}. 
    \item {\bf Greedy:} If there are nonempty compatible queues at a service completion of server $j$, assign the first--in--line customer from the queue $\argmax_{i\in\calC_j} \theta_{ij}$. If there are compatible servers idle upon a type-$i$ customer arrival, assign it to the server $\argmax_{j\in\calS_i} \theta_{ij}$. This policy myopically optimizes the instantaneous payoff without considering long-term effects or stability constraints. 
    \item {\bf Random:} If there are nonempty compatible queues at a service completion, choose one of these queues uniformly at random and assign the first--in--line customer. If there are compatible servers idle at a customer arrival, assign it to one of these servers uniformly at random. This policy serves as a most naive baseline by making completely uninformed decisions. 
\end{itemize}

\subsection{Small example}
\label{sec:numerical_small}
We consider the queueing system illustrated on the left in Figure~\ref{fig:qsmall_actions} with $\eps = 0.5$.
It can be verified that~\eqref{eq:def_eps_insensitive} is satisfied. 
\LP{\theta,\eps} in \eqref{eq:LP_eps} has six nondegenerate basic feasible solutions which we label as actions in Figure~\ref{fig:qsmall_actions}. 
The optimality of the actions depends on the payoff vector $\theta = \{\theta_{11},\theta_{12},\theta_{21},\theta_{22}\}$ as illustrated in Figure~\ref{fig:qsmall_opt_sol_LP}.

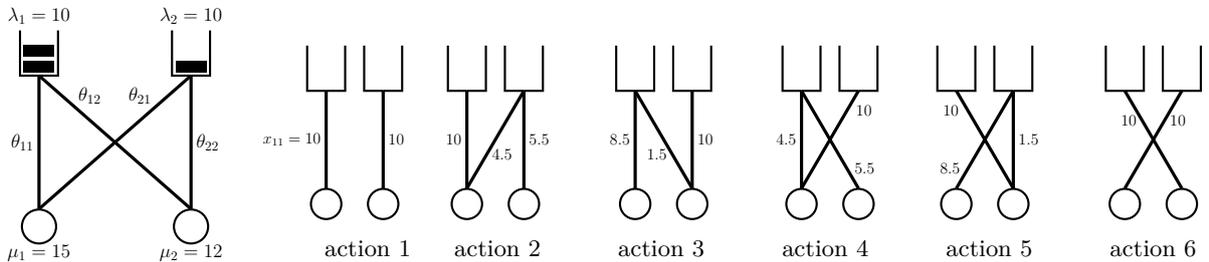
\begin{figure}[h]
    \centering
        \begin{subfigure}{.12\linewidth}
            \centering
            \begin{tikzpicture}[
                server/.style={circle, minimum size = .9cm, thick,draw},
                scale=.5, transform shape
                ]
                \foreach \x/\z in {1/10,2/10}{ 
                    \node[draw = none] at (4*\x,0) (queue\x) {};
                    \node[above of = queue\x, yshift = -.5cm] (labelqueue\x) {};
                    \draw[thick] (queue\x.center) --++(-.5,0) --++(0,1.2); 
                    \draw[thick] (queue\x.center) --++(.5,0) --++(0,1.2); 
                    \node[draw = none, above of = queue\x, yshift = .6cm] {\Large $\lambda_\x=\z$};
                }
                \fill[black] ($(queue1.center) + (-.4,.4)$) rectangle ++(.8,-0.3);
                \fill[black] ($(queue1.center) + (-.4,.8)$) rectangle ++(.8,-0.3);
                \fill[black] ($(queue2.center) + (-.4,.4)$) rectangle ++(.8,-0.3);
    
                \foreach \x/\y in {1/15,2/12}{ 
                    \node[server] at (4*\x,-4) (server\x) {};
                    \node at (server\x) [yshift=-.7cm] {\Large $\mu_\x=\y$};
                }
    
                \draw[very thick, black] (queue1.center) -- (server1.north) node[pos=0.5, left] {\Large $\theta_{11}$};
                \draw[very thick, black] (queue1.center) -- (server2.north) node[pos=0.15, right, xshift = .3cm] {\Large $\theta_{12}$};
    
                \draw[very thick, black] (queue2.center) -- (server1.north)  node[pos=0.15, left, xshift = -.3cm] {\Large $\theta_{21}$};
                \draw[very thick, black] (queue2.center) -- (server2.north) node[pos=0.5, right] {\Large $\theta_{22}$};
            \end{tikzpicture}
        \end{subfigure}
        \hspace{3em}
        \begin{subfigure}{0.12\linewidth}
            \centering
            \begin{tikzpicture}[
                server/.style={circle, minimum size = .8cm, thick,draw},
                scale=.5, transform shape
                ]
                \foreach \x/\y/\z in {1/black/10,2/black/10}{
                    \node[draw = none] at (1.5*\x,0) (queue\x) {};
                    \node[above of = queue\x, yshift = -.5cm] (labelqueue\x) {};
                    \draw[thick] (queue\x.center) --++(-.5,0) --++(0,1.2);
                    \draw[thick] (queue\x.center) --++(.5,0) --++(0,1.2); 
                }
                \foreach \x/\y in {1/15,2/12}{ 
                    \node[server] at (1.5*\x,-3) (server\x) {};
                }
                \draw[very thick, black] (queue1.center) -- (server1.north) node[pos=0.5, left] {\large $x_{11} = 10$};
                \draw[very thick, black] (queue2.center) -- (server2.north) node[pos=0.5, right] {\large 10};
            \end{tikzpicture}
            \caption*{\hfill action 1}
            \label{fig:qsmall_action1}
        \end{subfigure}
        \begin{subfigure}{0.12\linewidth}
            \centering
            \begin{tikzpicture}[
                server/.style={circle, minimum size = .8cm, thick,draw},
                scale=.5, transform shape
                ]
                \foreach \x/\y/\z in {1/black/10,2/black/10}{
                    \node[draw = none] at (1.5*\x,0) (queue\x) {};
                    \node[above of = queue\x, yshift = -.5cm] (labelqueue\x) {};
                    \draw[thick] (queue\x.center) --++(-.5,0) --++(0,1.2); 
                    \draw[thick] (queue\x.center) --++(.5,0) --++(0,1.2); 
                }
                \foreach \x/\y in {1/15,2/12}{ 
                    \node[server] at (1.5*\x,-3) (server\x) {};
                }
                \draw[very thick, black] (queue1.center) -- (server1.north) node[pos=0.5, left] {\large 10};
                \draw[very thick, black] (queue2.center) -- (server1.north) node[pos=0.65, right] {\large 4.5};
                \draw[very thick, black] (queue2.center) -- (server2.north) node[pos=0.5, right] {\large 5.5};
            \end{tikzpicture}
            \caption*{action 2}
            \label{fig:qsmall_action2}
        \end{subfigure}
        \begin{subfigure}{0.12\linewidth}
            \centering
            \begin{tikzpicture}[
                server/.style={circle, minimum size = .8cm, thick,draw},
                scale=.5, transform shape
                ]
                \foreach \x/\y/\z in {1/black/10,2/black/10}{ 
                    \node[draw = none] at (1.5*\x,0) (queue\x) {};
                    \node[above of = queue\x, yshift = -.5cm] (labelqueue\x) {};
                    \draw[thick] (queue\x.center) --++(-.5,0) --++(0,1.2); 
                    \draw[thick] (queue\x.center) --++(.5,0) --++(0,1.2); 
                }
                \foreach \x/\y in {1/15,2/12}{ 
                    \node[server] at (1.5*\x,-3) (server\x) {};
                }
                \draw[very thick, black] (queue1.center) -- (server1.north) node[pos=0.5, left] {\large 8.5};
                \draw[very thick, black] (queue1.center) -- (server2.north) node[pos=0.65, left] {\large 1.5};
                \draw[very thick, black] (queue2.center) -- (server2.north) node[pos=0.5, right] {\large 10};
            \end{tikzpicture}
            \caption*{action 3}
            \label{fig:qsmall_action3}
        \end{subfigure}
        \begin{subfigure}{0.12\linewidth}
            \centering
            \begin{tikzpicture}[
                server/.style={circle, minimum size = .8cm, thick,draw},
                scale=.5, transform shape
                ]
                \foreach \x/\y/\z in {1/black/10,2/black/10}{ 
                    \node[draw = none] at (1.5*\x,0) (queue\x) {};
                    \node[above of = queue\x, yshift = -.5cm] (labelqueue\x) {};
                    \draw[thick] (queue\x.center) --++(-.5,0) --++(0,1.2); 
                    \draw[thick] (queue\x.center) --++(.5,0) --++(0,1.2); 
                }
                \foreach \x/\y in {1/15,2/12}{ 
                    \node[server] at (1.5*\x,-3) (server\x) {};
                }
                \draw[very thick, black] (queue1.center) -- (server1.north) node[pos=0.5, left] {\large 4.5};
                \draw[very thick, black] (queue1.center) -- (server2.north) node[pos=0.8, right, xshift=.06cm] {\large 5.5};
                \draw[very thick, black] (queue2.center) -- (server1.north) node[pos=0.2, right, xshift=.1cm] {\large 10};
            \end{tikzpicture}
            \caption*{action 4}
            \label{fig:qsmall_action4}
        \end{subfigure}
        \begin{subfigure}{0.12\linewidth}
            \centering
            \begin{tikzpicture}[
                server/.style={circle, minimum size = .8cm, thick,draw},
                scale=.5, transform shape
                ]
                \foreach \x/\y/\z in {1/black/10,2/black/10}{
                    \node[draw = none] at (1.5*\x,0) (queue\x) {};
                    \node[above of = queue\x, yshift = -.5cm] (labelqueue\x) {};
                    \draw[thick] (queue\x.center) --++(-.5,0) --++(0,1.2); 
                    \draw[thick] (queue\x.center) --++(.5,0) --++(0,1.2); 
                }
                \foreach \x/\y in {1/15,2/12}{ 
                    \node[server] at (1.5*\x,-3) (server\x) {};
                }
                \draw[very thick, black] (queue1.center) -- (server2.north) node[pos=0.2, left, xshift=-.1cm] {\large 10};
                \draw[very thick, black] (queue2.center) -- (server1.north) node[pos=0.8, left, xshift=-.06cm] {\large 8.5};
                \draw[very thick, black] (queue2.center) -- (server2.north) node[pos=0.5, right] {\large 1.5};
            \end{tikzpicture}
            \caption*{action 5}
            \label{fig:qsmall_action5}
        \end{subfigure}
        \begin{subfigure}{0.12\linewidth}
            \centering
            \begin{tikzpicture}[
                server/.style={circle, minimum size = .8cm, thick,draw},
                scale=.5, transform shape
                ]
                \foreach \x/\y/\z in {1/black/10,2/black/10}{ 
                    \node[draw = none] at (1.5*\x,0) (queue\x) {};
                    \node[above of = queue\x, yshift = -.5cm] (labelqueue\x) {};
                    \draw[thick] (queue\x.center) --++(-.5,0) --++(0,1.2);
                    \draw[thick] (queue\x.center) --++(.5,0) --++(0,1.2);
                }
                \foreach \x/\y in {1/15,2/12}{
                    \node[server] at (1.5*\x,-3) (server\x) {};
                }
                \draw[very thick, black] (queue1.center) -- (server2.north) node[pos=0.3, left] {\large 10};
                \draw[very thick, black] (queue2.center) -- (server1.north) node[pos=0.3, right] {\large 10};
            \end{tikzpicture}
            \caption*{action 6}
            \label{fig:qsmall_action6}
        \end{subfigure}
    \caption{
        A skill-based queueing system with $I=2$ customer classes, $J=2$ servers, and compatibility lines $\calL=\{(11),(12),(21),(22)\}$, along with the six different actions. 
        For each action, the routing rates $\{x_{ij}\}_{(ij)\in\calL}$ are illustrated next to the corresponding lines. }
    \label{fig:qsmall_actions}
\end{figure}

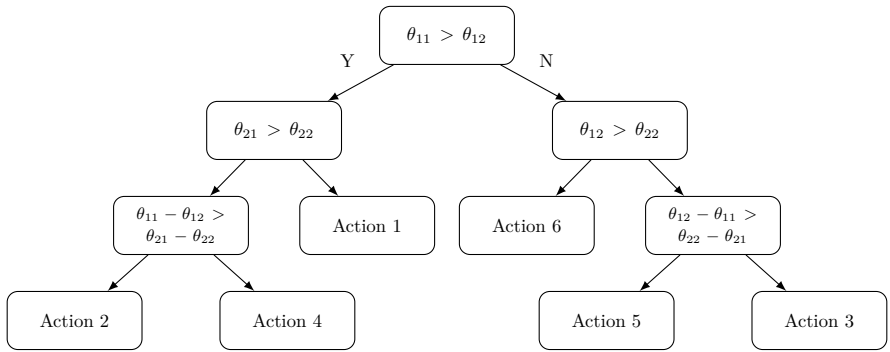
\begin{figure}[H]
    \centering
    \begin{tikzpicture}[
        level distance=18mm, 
        level 1/.style={sibling distance=65mm}, 
        level 2/.style={sibling distance=35mm}, 
        level 3/.style={sibling distance=40mm}, 
        edge from parent/.style={draw,-latex},
        every node/.style={draw, rounded corners, text width=23mm, align=center, minimum height=11mm},
        scale=.7, transform shape
        ]

        \node {$\theta_{11} > \theta_{12}$}
            child { node {$\theta_{21} > \theta_{22}$} 
                child { node {\small $\theta_{11} - \theta_{12} > \theta_{21} - \theta_{22}$}
                      child { node {Action 2}}
                      child { node {Action 4}}
                    }
                child { node {Action 1}} 
                edge from parent node[draw=none, pos=.7, above] {Y}  
                }
            child { node {$\theta_{12} > \theta_{22}$}
                child { node {Action 6}} 
                child { node {\small $\theta_{12} - \theta_{11} > \theta_{22} - \theta_{21}$}
                      child { node {Action 5}}
                      child { node {Action 3}}
                    }
                edge from parent node[draw=none, pos=.7, above] {N} 
            };
    \end{tikzpicture}
    \caption{Optimality conditions for each of the six feasible solutions of \LP{\theta,\eps}. }
    \label{fig:qsmall_opt_sol_LP}
\end{figure}

We set the true payoff vector as $\theta =\{0.4,0.1,0.3,0.01\}$. 
This implies that action 2 is optimal and that the optimal long-term payoff rate is approximately 5.4. 
Moreover, line $(12)$ is suboptimal since it has zero load in the optimal solution.
The suboptimality gap of line $(12)$ (recall \eqref{eq:phi_def}) is $\phi_{12} = 0.01$.
The suboptimality gaps for each action in~\eqref{eq:def_subopt_gap_action} are given by
\begin{align}
    d^1 = 1.305, \ d^2 = 0, \ d^3 = 1.755, \ d^4=0.055, \ d^5=1.84, \ d^6 = 1.405.
\end{align}

Recall that the UCB QR algorithm requires configuration parameters $\alpha,\beta$, and $H_0$ which define the episode length according to~\eqref{eq:ucbqr_episode_length}.
We set $\alpha = 364$, $\beta = 1.01$, and $H_0 = 10$. 
These parameters satisfy  \eqref{eq:c0_choice} and \eqref{eq:ucbqr_episode_length}. 
All of our numerical results are based on 50 independent replications and all plots show 95\% confidence areas, although these intervals are too small to be visible in Figure~\ref{fig:qsmall_running_reward}. \\

We observe in Figure~\ref{fig:qsmall_running_reward} that, as expected, the payoff rate of the UCB QR algorithm converges to the optimal payoff rate of the Oracle policy.
In the initial period up to approximately time 2000,  the UCB QR algorithm has not yet collected sufficient samples to distinguish the actions. Therefore, it chooses actions uniformly at random, which corresponds to the first few episodes in Figure~\ref{fig:qsmall_action_freq}. 
In turn, the regret accumulation is high since actions with large suboptimality gaps are chosen. 
After the initial period, we observe in Figure~\ref{fig:qsmall_action_freq} that the algorithm primarily struggles to differentiate the optimal action 2 and action 4. 
This leads to a sharp decrease in regret accumulation, since action 4 has a small suboptimality gap.

Moreover, we observe in Figure~\ref{fig:qsmall_action_freq} that actions $a\in\{1,3,5,6\}$ are hardly ever chosen by the UCB QR algorithm. 
As opposed to a classical MAB, the rewards of the actions are related by the common unknown payoff parameter $\theta$. 
Therefore, it suffices to only sample a subset of actions, as long as we observe sufficient payoff samples from all lines. 
By only using actions 2 and 4, the algorithm infers sufficient information about all four payoff parameters.
In particular, in accordance with Line~\ref{line:update_ucb} in Algorithm~\ref{alg:learning_alg}, the indices $U_{11}$, $U_{21}$, and $U_{22}$ are updated after episodes with action 2, while $U_{11}$, $U_{12}$, and $U_{21}$ are updated after episodes with action 4. 
We note that our regret analysis in Section~\ref{sec:learning_alg} does not exploit this feature of transfer learning in the construction of the regret upper bound. This remains an interesting direction for future research.

The FCFS ALIS, Greedy and Random policies maintain a consistent gap to the optimal payoff of the Oracle policy. 
The value of this gap depends on the true payoff parameter $\theta$. 
Since the total load in the system is far from critical, the long-term routing rates under the FCFS ALIS policy are close to those of the Random policy. This leads to the minor difference in total payoff rate with respect to the Random policy, as shown in Figure~\ref{fig:qsmall_running_reward}. In particular, the FCFS ALIS policy is indistinguishable from the Random policy in terms of reward. 
The Greedy policy obtains a slightly higher reward making a greedy decision whenever possible.

To summarize, we have shown for a small queueing system that the payoff rate of the UCB QR algorithm converges to the optimal rate and that it chooses the optimal action most of the time. Thereby, it outperforms our benchmark policies that (a) are agnostic to the payoff parameters, or (b) rely on the payoff parameters but make suboptimal routing decisions. 

\begin{figure}[H]
    \centering
    \begin{minipage}[c]{.3\linewidth}
        \centering
        \input{Figs/TikzPlot_Small_RunningAvgTotalReward.tex}
        \caption{Running average payoff rate over time for all policies.}
        \label{fig:qsmall_running_reward}
    \end{minipage}
    \hspace{0.02\linewidth}
    \begin{minipage}[c]{.3\linewidth}
        \centering
        \vspace{.7em}
        \input{Figs/TikzPlot_Small_Regret.tex}
        \caption{Regret accumulation over time of the UCB QR algorithm. }
        \label{fig:qsmall_regret}
    \end{minipage}
    \hspace{0.04\linewidth}
    \begin{minipage}[c]{.3\linewidth}
        \centering
        \vspace{1.9em}
        \input{Figs/TikzPlot_Small_ActionFrequency.tex}
        \caption{Cumulative count of actions over the first 20 episodes of the UCB QR algorithm. }
        \label{fig:qsmall_action_freq}
    \end{minipage}
\end{figure}

\subsection{Changing parameters}
\label{sec:numerical_changing_theta}
Although our theoretical analysis of the UCB QR algorithm in Section~\ref{sec:learning_alg} is for a static environment, we analyze its robustness against changes in the payoff parameter $\theta$. 
In this experiment, we use the same payoff vector as in Section~\ref{sec:numerical_small} initially.
At one-third of the total runtime, we update the payoff parameter of line $(12)$ to $\theta_{12}' =0.5$. 
As a consequence, the optimal action under the new payoff parameter is action 6. 
We set $\alpha = H_0 = 10$, and $\beta = 1.01$. 
Our numerical results are based on 100 independent replications.

Observe in Figure~\ref{fig:qsmall_action_frequency_change} that upon the parameter change (at episode number 60, depicted by a vertical line), the UCB QR algorithm at first prefers action 4.
This makes sense since action 4 is closest to the initially optimal action 2, while it also includes line $(12)$ which is an optimal line under the updated parameter $\theta'$. 
From approximately episode number 100 onwards, the learning algorithm switches more to the optimal action 6. 
Hence, we see that the UCB QR algorithm suffers from a switch-over period, but eventually changes to the optimal action. 
The convergence after the parameter change is more slow than the initial convergence. We see this effect since the empirical estimators take into account the complete history. Therefore,  after the parameter update, learning is hindered by samples obtained before the update. 
We expect that robustness against changing parameter values can be improved by decreasing the episode length, or altering the empirical estimators in~\eqref{eq:UCB_update_algo_theta_hat} by decreasing the weight of observations from the past. This remains open for future work. 

Hence, the numerical results suggest that the UCB QR algorithm can correctly identify changes in the true system parameters.

\begin{figure}[H]
    \centering
    \input{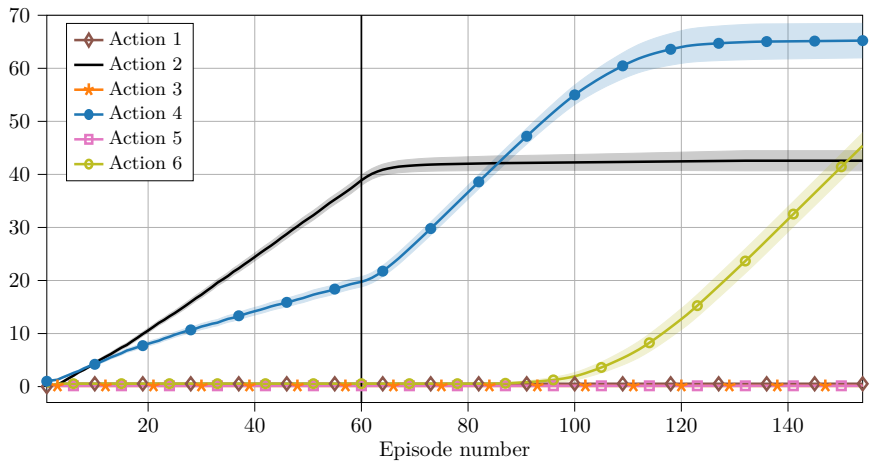}
    \caption{Cumulative count of actions per episode of the UCB QR algorithm when changing the true payoff parameter of line $(21)$ to $\theta_{21}'=0.5$ at one-third of the total runtime. }
    \label{fig:qsmall_action_frequency_change}
\end{figure}

\subsection{Complex queueing system}
\label{sec:numerical_big}
We consider the queueing system illustrated in Figure~\ref{fig:qbig_layout}.
This system is inspired by the operation of the call center of a real-world telecommunications company. 
We let $\eps = 0.05$ such that \eqref{eq:def_eps_insensitive} is satisfied. 
The number of basic feasible solutions of \LP{\theta,\eps} in \eqref{eq:LP_eps}, i.e., the number of actions, is $88$. 
The lower bound in \eqref{eq:c0_choice} is $165088$. 
However, since a large $\alpha$ slows down learning, we set $\alpha=10$.
We set $\beta=1.01$, and $H_0 = 10$. 
Our numerical results are based on 100 independent replications.

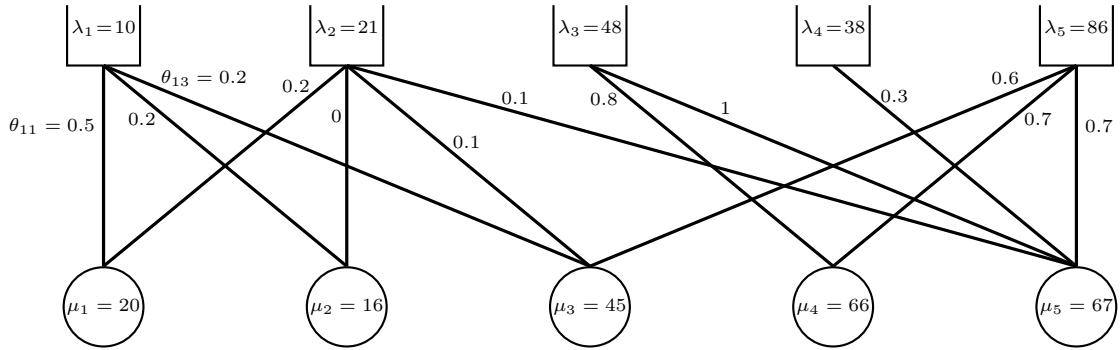
\begin{figure}[H]
    \begin{center}
        \begin{tikzpicture}[
            server/.style={circle, minimum size = 1.05cm, thick,draw},
            scale=.8
            ]
            \scriptsize
            \foreach \x/\y in {1/10,2/21,3/48,4/38,5/86}{ 
                \node[draw = none] at (4*\x,0) (queue\x) {};
                \node[above of = queue\x, yshift = -.5cm] (labelqueue\x) {\scriptsize $\lambda_{\x} \hspace{-.2em} = \hspace{-.2em}\y$} ;
                \draw[thick] (queue\x.center) --++(-.6,0) --++(0,1);
                \draw[thick] (queue\x.center) --++(.6,0) --++(0,1); 
            }

            \foreach \x/\y in {1/20,2/16,3/45,4/66,5/67}{ 
                \node[server] at (4*\x,-4) (server\x) {};
                \node at (server\x)  {\scriptsize $\mu_\x = \y$};
            }

            \draw[very thick, black] (queue1.center) -- (server1.north) node[pos=0.3, left]
            {$\theta_{11} = 0.5$};
            \draw[very thick, black] (queue1.center) -- (server2.north) node[pos=0.25, left,yshift=-.05cm]
            {$0.2$};
            \draw[very thick, black] (queue1.center) -- (server3.north) node[pos=0.1, right, yshift = .1cm]
            {$\theta_{13} = 0.2$};

            \draw[very thick, black] (queue2.center) -- (server1.north)node[pos=0.1, left, xshift=-.05cm]
            {$0.2$};
            \draw[very thick, black] (queue2.center) -- (server2.north)node[pos=0.25, left,xshift=.05cm]
            {$0$};
            \draw[very thick, black] (queue2.center) -- (server3.north)node[pos=0.4, right, yshift = .05cm]
            {$0.1$};
            \draw[very thick, black] (queue2.center) -- (server5.north)node[pos=0.2, right, yshift=.1cm]
            {$0.1$};

            \draw[very thick, black] (queue3.center) -- (server4.north) node[pos=0.15, left, yshift=-.05cm]
            {$0.8$};
            \draw[very thick, black] (queue3.center) -- (server5.north) node[pos=0.25, right,yshift=0.1cm]
            {$1$};

            \draw[very thick, black] (queue4.center) -- (server5.north) node[pos=0.15, right, xshift=.03cm]
            {$0.3$};

            \draw[very thick, black] (queue5.center) -- (server3.north) node[pos=0.1, left,yshift=.1cm]
            {$0.6$};
            \draw[very thick, black] (queue5.center) -- (server4.north) node[pos=0.25, right,yshift=-.05cm]
            {$0.7$};
            \draw[very thick, black] (queue5.center) -- (server5.north) node[pos=0.3, right]
            {$0.7$};
        \end{tikzpicture}
    \end{center}
    \caption{Big queueing system with $I=5$ customer classes and $J=5$ servers. The parameter values are shown in the figure.}
    \label{fig:qbig_layout}
\end{figure}

We consider three different scenarios with different levels of difficulty for adaptive learning:
\begin{itemize}
    \item {\bf Initial parameters:} 
    The first scenario is the system with initial parameters as shown in Figure~\ref{fig:qbig_layout}. 
    \item {\bf Minimal payoff discrepancy:} 
    In this scenario, all--but--one payoff parameters are equal. 
    In particular,  $\theta_{ij} = 0.5$ for all $(ij)\in\calL$ except for $\theta_{55} = 0.6$.
    The learning policy must identify this minimal payoff gap. 
    \item {\bf Balanced arrival rates:} In the third scenario, the arrival rates are equal for all customers, namely $\lambda_i = 42$ for all $i\in\calI$. This value is chosen so that the total load of the system is high, namely approximately 0.98. 
\end{itemize}

We compare the UCB QR algorithm with the other benchmark policies as described above.
The average number of customers in the system is provided in Table~\ref{tab:qbig_QL} and the payoff rate is illustrated in Figure~\ref{fig:qbig_payoff}.

Observe in Table~\ref{tab:qbig_QL} that the UCB QR and Oracle policies have a similar average number of customers in the system in all scenarios. 
The average queue lengths of the FCFS ALIS policy is much smaller than those of Oracle and the UCB QR algorithm. 
This is to be expected, since this policy inherently aims to decrease the queue length.
However, the nonidling nature of this policy leads to a gap in payoff rate with respect to the Oracle policy, as can be seen in Figure~\ref{fig:qbig_payoff}.
For the Greedy and Random policies, the average number of customers in the system is much higher than for the other policies. 
Despite the high average queue length, the average payoff of Greedy and Random in Figure~\ref{fig:qbig_payoff} is very comparable to those of FCFS ALIS. 
This is explained as follows: the policies obtain high reward by routing customer types with high average payoffs, while other customer types with overall lower average payoffs are not served at all.  
We find that for the initial parameters and minimal payoff discrepancy scenarios, especially the queue length of customer type 4 is large for these policies, since this customer type is only compatible with server 5 while the service capacity of server 5 is for a large part used by customer types 3 and 5. For this reason, the Greedy policy obtains an even lower average reward than the Random policy. 

Observe also in Figure~\ref{fig:qbig_payoff} that the convergence of the payoff rate of the UCB QR algorithm to the optimal payoff of the Oracle policy is much faster than in the small queueing system considered in Section~\ref{sec:numerical_small} (see Figure~\ref{fig:qsmall_running_reward}), i.e., the regret is smaller. 
This can be explained by the choice of $\alpha$ (10 vs.\ 364).
A smaller value of $\alpha$ implies that the episodes are shorter. 
Hence, per time interval there are more decision moments where the policy can switch between actions, which speeds up learning.

Comparing between the different scenarios, we find that in the minimal payoff discrepancy scenario, the UCB QR algorithm chooses suboptimal actions more frequently than in the other scenarios. However, since the suboptimality gaps of the suboptimal actions are small by construction, the average reward is still close to the optimal, as shown in Figure~\ref{fig:qbig_payoff_minimalpayoffdiscrepancy}. 
Out of the three scenarios, the convergence is slowest in Figure~\ref{fig:qbig_payoff_balanced}.
Our belief is that in this case, the algorithm needs approximately the same number of suboptimal episodes to learn the payoff parameters as in the initial scenario, while the suboptimality gaps are larger, i.e., the cost of exploration is higher. 

Lastly, we note that Algorithm~\ref{alg:learning_alg}, although it maintains system stability by construction, it is not incentivized to minimize or balance the queue lengths or server loads.
Possible extensions of the algorithm where customer waiting times, server loads or fairness constraints are taken into account are interesting for future research. 
For example, \LP{\theta,\eps} in~\eqref{eq:LP_eps} can be altered to include either waiting time constraints, or a penalty factor proportional to the average queue length in the objective function. 
This would result in a different action space for the algorithm. 

To summarize the findings of this section, we have shown that even in a complex queueing system with 88 different actions, the UCB QR algorithm converges quickly to the oracle reward, while maintaining reasonable queue lengths in different scenarios. 

\begin{table}[H]
    \centering
    \begin{tabular}{c|cc|ccc|cc}
        \diagbox{\small{Policy}}{\small{Scenario}} & \multicolumn{2}{c|}{\small{Initial parameters}} & \multicolumn{3}{c|}{\small{Min. payoff discrepancy}} & \multicolumn{2}{c}{\small{Balanced arrival rates}} \\\hline
         & $\E$ & $\sigma$ & & $\E$ & $\sigma$ & $\E$ & $\sigma$ \\\hline
         Oracle & 329 & 106 & & 269 & 80 & 277 & 90 \\\hline  
         UCB QR & 315 & 101 & & 147 & 31 & 248 & 71 \\\hline 
         FCFS ALIS & 9  & 0.3 & & 12 & 0.8 & 12 & 0.8 \\\hline 
         Greedy & 14069 & 142 & & 17372 & 148 & 14193 & 207 \\\hline 
         Random & 8491 & 138 & & 8483 & 130 & 9447 & 130 \\\hline 
    \end{tabular}
    \caption{Mean ($\E$) and standard deviation ($\sigma$) of the number of customers in the system over the entire simulation time for all policies and scenarios.}
    \label{tab:qbig_QL}
\end{table}

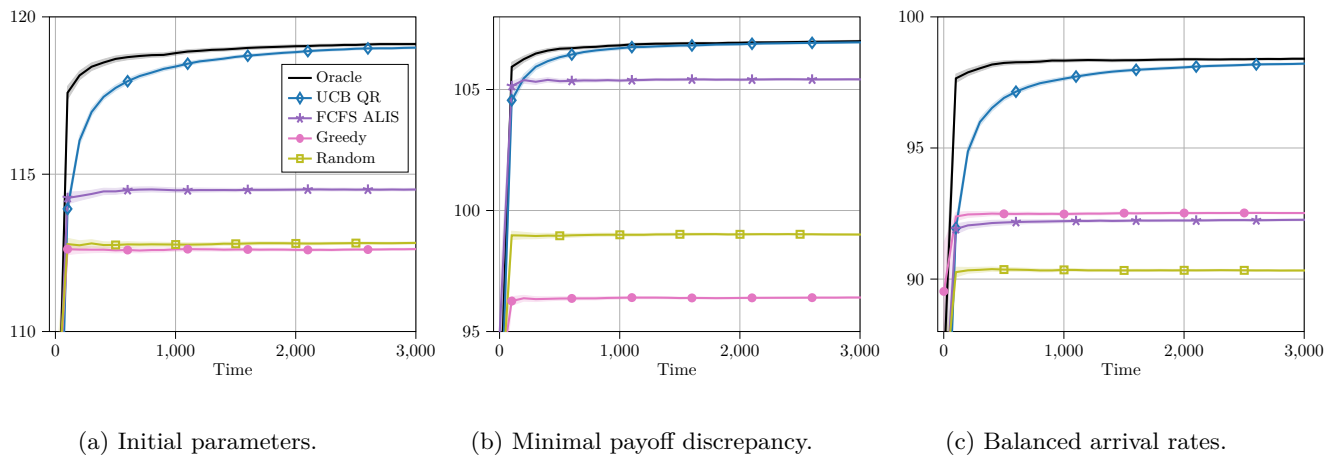
\begin{figure}[H]
    \centering
    \begin{subfigure}{0.3\linewidth}
        \centering
        \input{Figs/TikzPlot_Big_RunningAvgTotalReward.tex}
        \caption{Initial parameters.}
        \label{fig:qbig_payoff_initial}
    \end{subfigure}
    \hspace{0.03\linewidth}
    \begin{subfigure}{0.3\linewidth}
        \centering
        \input{Figs/TikzPlot_Minimal_RunningAvgTotalReward.tex}
        \caption{Minimal payoff discrepancy.}
        \label{fig:qbig_payoff_minimalpayoffdiscrepancy}
    \end{subfigure}
    \hspace{0.03\linewidth}
    \begin{subfigure}{0.3\linewidth}
        \centering
        \input{Figs/TikzPlot_Balanced_RunningAvgTotalReward.tex}
        \caption{Balanced arrival rates.}
        \label{fig:qbig_payoff_balanced}
    \end{subfigure}
    \caption{Running average payoff rate over time for all policies and scenarios. The legend is the same for all subplots. }
    \label{fig:qbig_payoff}
    \hspace{2em}
\end{figure}

%% file: Figs/TikzPlot_Small_RunningAvgTotalReward.tex
\begin{tikzpicture}[scale=.65,transform shape]

\definecolor{darkgray176}{RGB}{176,176,176}
\definecolor{forestgreen4416044}{RGB}{44,160,44}
\definecolor{goldenrod18818934}{RGB}{188,189,34}
\definecolor{mediumpurple148103189}{RGB}{148,103,189}
\definecolor{orchid227119194}{RGB}{227,119,194}
\definecolor{steelblue31119180}{RGB}{31,119,180}

\definecolor{black}{RGB}{0,0,0}
\definecolor{darkgray}{RGB}{160,160,160}
\definecolor{lightgray}{RGB}{210,210,210}
    
\begin{axis}[
tick align=outside,
tick pos=left,
x grid style={darkgray176},
xlabel={Time},
xmin=-4949, xmax=103951,
xtick style={color=black},
ytick={4,4.5,...,5.5},
y grid style={darkgray176},
ymin=4, ymax=5.5,
ytick style={color=black},
xmajorgrids,
ymajorgrids,
legend image post style={scale=1,xscale=0.8},
legend cell align=left,
legend style={at={(.95,.8)}, anchor=north east, font = \small, mark size=1pt},  
width=9cm,  
height=8cm,  
]
\path [draw=black, fill=black, opacity=0.2]
(axis cs:1,3.65827102145025)
--(axis cs:1,4.66172897854975)
--(axis cs:1001,5.41254631242022)
--(axis cs:2001,5.41855298209897)
--(axis cs:3001,5.41328000875017)
--(axis cs:4001,5.41268168393649)
--(axis cs:5001,5.40888309029706)
--(axis cs:6001,5.40889901821762)
--(axis cs:7001,5.40778785366196)
--(axis cs:8001,5.407774296823)
--(axis cs:9001,5.4077040115523)
--(axis cs:10001,5.4061892496925)
--(axis cs:11001,5.40629612542655)
--(axis cs:12001,5.40554940534616)
--(axis cs:13001,5.40629337281786)
--(axis cs:14001,5.40591821598886)
--(axis cs:15001,5.40576329692886)
--(axis cs:16001,5.40585156734933)
--(axis cs:17001,5.40649075898874)
--(axis cs:18001,5.40738140348857)
--(axis cs:19001,5.40775023738757)
--(axis cs:20001,5.40777777275705)
--(axis cs:21001,5.40831907919759)
--(axis cs:22001,5.40830431057107)
--(axis cs:23001,5.40806209823655)
--(axis cs:24001,5.40820304837142)
--(axis cs:25001,5.40851376726688)
--(axis cs:26001,5.40783374095834)
--(axis cs:27001,5.4078823728229)
--(axis cs:28001,5.40749539189027)
--(axis cs:29001,5.40687755104743)
--(axis cs:30001,5.40709481180854)
--(axis cs:31001,5.40732703004326)
--(axis cs:32001,5.40661537480271)
--(axis cs:33001,5.40627112478785)
--(axis cs:34001,5.40645937687959)
--(axis cs:35001,5.40622367096323)
--(axis cs:36001,5.40602990794912)
--(axis cs:37001,5.40584607894306)
--(axis cs:38001,5.40578958564715)
--(axis cs:39001,5.40552252806821)
--(axis cs:40001,5.40498619106604)
--(axis cs:41001,5.40524303769256)
--(axis cs:42001,5.40512191846467)
--(axis cs:43001,5.40503189459416)
--(axis cs:44001,5.4047374320763)
--(axis cs:45001,5.4046982860983)
--(axis cs:46001,5.40487286251451)
--(axis cs:47001,5.40511968081145)
--(axis cs:48001,5.40504531164593)
--(axis cs:49001,5.40465303689977)
--(axis cs:50001,5.40466143052572)
--(axis cs:51001,5.40471940278271)
--(axis cs:52001,5.40460409460363)
--(axis cs:53001,5.40466887438348)
--(axis cs:54001,5.40447937741637)
--(axis cs:55001,5.40467180209289)
--(axis cs:56001,5.40483888553614)
--(axis cs:57001,5.40491570145624)
--(axis cs:58001,5.40464717655866)
--(axis cs:59001,5.40464649218261)
--(axis cs:60001,5.40470220529915)
--(axis cs:61001,5.40477261312776)
--(axis cs:62001,5.40463867553202)
--(axis cs:63001,5.4049034200126)
--(axis cs:64001,5.40492667649469)
--(axis cs:65001,5.40487783402787)
--(axis cs:66001,5.40475648417582)
--(axis cs:67001,5.40469246026657)
--(axis cs:68001,5.40479602127972)
--(axis cs:69001,5.40449237958126)
--(axis cs:70001,5.40452227839373)
--(axis cs:71001,5.40461428123526)
--(axis cs:72001,5.40448089295934)
--(axis cs:73001,5.40446223871483)
--(axis cs:74001,5.40431162770889)
--(axis cs:75001,5.40436254405688)
--(axis cs:76001,5.40436452996157)
--(axis cs:77001,5.40420091455736)
--(axis cs:78001,5.40419738537715)
--(axis cs:79001,5.40426920490156)
--(axis cs:80001,5.40455285752435)
--(axis cs:81001,5.40422582546277)
--(axis cs:82001,5.40427621088669)
--(axis cs:83001,5.40452205309095)
--(axis cs:84001,5.40448465480389)
--(axis cs:85001,5.40444124830281)
--(axis cs:86001,5.40448981796865)
--(axis cs:87001,5.40470682953626)
--(axis cs:88001,5.40476209078951)
--(axis cs:89001,5.40446339643185)
--(axis cs:90001,5.40456982155717)
--(axis cs:91001,5.40457978027105)
--(axis cs:92001,5.40454213152084)
--(axis cs:93001,5.40462987337924)
--(axis cs:94001,5.40489927259718)
--(axis cs:95001,5.40480575084071)
--(axis cs:96001,5.40473977163695)
--(axis cs:97001,5.40471752349759)
--(axis cs:98001,5.40460182634142)
--(axis cs:99001,5.40453165290644)
--(axis cs:99001,5.40068001162221)
--(axis cs:99001,5.40068001162221)
--(axis cs:98001,5.40071118066871)
--(axis cs:97001,5.40071541018351)
--(axis cs:96001,5.40063808901033)
--(axis cs:95001,5.4006602968851)
--(axis cs:94001,5.40066917880227)
--(axis cs:93001,5.40048534043567)
--(axis cs:92001,5.40035824999675)
--(axis cs:91001,5.40043027456352)
--(axis cs:90001,5.40043812279901)
--(axis cs:89001,5.40026464033178)
--(axis cs:88001,5.40038603252727)
--(axis cs:87001,5.40036897419013)
--(axis cs:86001,5.40017012784593)
--(axis cs:85001,5.40014552126461)
--(axis cs:84001,5.4001576708827)
--(axis cs:83001,5.40020608271464)
--(axis cs:82001,5.39985788503897)
--(axis cs:81001,5.39979338417661)
--(axis cs:80001,5.40001808533886)
--(axis cs:79001,5.39987605908245)
--(axis cs:78001,5.39979076862087)
--(axis cs:77001,5.39970656716366)
--(axis cs:76001,5.39978857328707)
--(axis cs:75001,5.39979633382474)
--(axis cs:74001,5.39981804621444)
--(axis cs:73001,5.3997486625057)
--(axis cs:72001,5.39981460293654)
--(axis cs:71001,5.39997607664702)
--(axis cs:70001,5.39975080341938)
--(axis cs:69001,5.39975567479476)
--(axis cs:68001,5.40006096611752)
--(axis cs:67001,5.40015104953179)
--(axis cs:66001,5.40025980345618)
--(axis cs:65001,5.40019654944315)
--(axis cs:64001,5.40023761781321)
--(axis cs:63001,5.40027427557953)
--(axis cs:62001,5.4000341523256)
--(axis cs:61001,5.40000566919549)
--(axis cs:60001,5.39989371810213)
--(axis cs:59001,5.39987004143546)
--(axis cs:58001,5.39987481444149)
--(axis cs:57001,5.40009122824675)
--(axis cs:56001,5.39996602866182)
--(axis cs:55001,5.39954339399445)
--(axis cs:54001,5.39935536638467)
--(axis cs:53001,5.39944878377391)
--(axis cs:52001,5.39918967859304)
--(axis cs:51001,5.39932757668828)
--(axis cs:50001,5.39924169141184)
--(axis cs:49001,5.39914443662117)
--(axis cs:48001,5.39956709226232)
--(axis cs:47001,5.39970574843473)
--(axis cs:46001,5.39952182455751)
--(axis cs:45001,5.39914873952335)
--(axis cs:44001,5.39922793234724)
--(axis cs:43001,5.39916893794462)
--(axis cs:42001,5.39932845176461)
--(axis cs:41001,5.39932270460641)
--(axis cs:40001,5.3990677075865)
--(axis cs:39001,5.39953477815471)
--(axis cs:38001,5.39973237430126)
--(axis cs:37001,5.39965323188638)
--(axis cs:36001,5.39964437887625)
--(axis cs:35001,5.40006414938476)
--(axis cs:34001,5.40010866523682)
--(axis cs:33001,5.40020140634757)
--(axis cs:32001,5.40031190871968)
--(axis cs:31001,5.40084948039189)
--(axis cs:30001,5.40034094033306)
--(axis cs:29001,5.39990566332449)
--(axis cs:28001,5.40031004362989)
--(axis cs:27001,5.40032621204433)
--(axis cs:26001,5.40003211035507)
--(axis cs:25001,5.40053387082759)
--(axis cs:24001,5.4000532742818)
--(axis cs:23001,5.39970104249646)
--(axis cs:22001,5.39964623713131)
--(axis cs:21001,5.39946531202187)
--(axis cs:20001,5.39891189275968)
--(axis cs:19001,5.39904940473653)
--(axis cs:18001,5.39811384677531)
--(axis cs:17001,5.39795133265293)
--(axis cs:16001,5.3968307650049)
--(axis cs:15001,5.39635389525832)
--(axis cs:14001,5.39627305606314)
--(axis cs:13001,5.3963679609257)
--(axis cs:12001,5.39494388687949)
--(axis cs:11001,5.3952201003711)
--(axis cs:10001,5.39429870151238)
--(axis cs:9001,5.39465572625461)
--(axis cs:8001,5.39361552944871)
--(axis cs:7001,5.39347482309851)
--(axis cs:6001,5.39238076848459)
--(axis cs:5001,5.38982116885111)
--(axis cs:4001,5.39257700139218)
--(axis cs:3001,5.3908919339356)
--(axis cs:2001,5.38778384948524)
--(axis cs:1001,5.37426687439296)
--(axis cs:1,3.65827102145025)
--cycle;

\path [draw=steelblue31119180, fill=steelblue31119180, opacity=0.2]
(axis cs:1,3.32736867468859)
--(axis cs:1,4.43263132531141)
--(axis cs:1001,4.63006635206991)
--(axis cs:2001,4.65966224458156)
--(axis cs:3001,4.90335419586464)
--(axis cs:4001,5.02837985000741)
--(axis cs:5001,5.1045394891789)
--(axis cs:6001,5.15053028538008)
--(axis cs:7001,5.18523249605437)
--(axis cs:8001,5.20958608623705)
--(axis cs:9001,5.23033574252917)
--(axis cs:10001,5.24506335607953)
--(axis cs:11001,5.25808974318217)
--(axis cs:12001,5.26976117651496)
--(axis cs:13001,5.27898557193896)
--(axis cs:14001,5.287059892584)
--(axis cs:15001,5.29422031679953)
--(axis cs:16001,5.30077091917117)
--(axis cs:17001,5.30649081575471)
--(axis cs:18001,5.3119282044644)
--(axis cs:19001,5.31717003966925)
--(axis cs:20001,5.32150534045979)
--(axis cs:21001,5.32556101410775)
--(axis cs:22001,5.32965531492883)
--(axis cs:23001,5.3324331022214)
--(axis cs:24001,5.33520097710056)
--(axis cs:25001,5.33752160334958)
--(axis cs:26001,5.33972119771551)
--(axis cs:27001,5.34202714108702)
--(axis cs:28001,5.34335000359934)
--(axis cs:29001,5.34470697342016)
--(axis cs:30001,5.34655381898119)
--(axis cs:31001,5.34821169455529)
--(axis cs:32001,5.34947806878833)
--(axis cs:33001,5.35063854192147)
--(axis cs:34001,5.35220142971075)
--(axis cs:35001,5.3535475125971)
--(axis cs:36001,5.35554821903957)
--(axis cs:37001,5.35727151588867)
--(axis cs:38001,5.35869183635626)
--(axis cs:39001,5.36009649375218)
--(axis cs:40001,5.36074668256095)
--(axis cs:41001,5.36170097807015)
--(axis cs:42001,5.36282463222099)
--(axis cs:43001,5.36382915018121)
--(axis cs:44001,5.36490338989049)
--(axis cs:45001,5.36572764264724)
--(axis cs:46001,5.36666132233461)
--(axis cs:47001,5.36736315638887)
--(axis cs:48001,5.36794339632231)
--(axis cs:49001,5.36900395947018)
--(axis cs:50001,5.36991404809901)
--(axis cs:51001,5.37029608238882)
--(axis cs:52001,5.37095118474652)
--(axis cs:53001,5.37172263008813)
--(axis cs:54001,5.3725938588151)
--(axis cs:55001,5.3733755648317)
--(axis cs:56001,5.37387675153051)
--(axis cs:57001,5.37408878877733)
--(axis cs:58001,5.37465607426015)
--(axis cs:59001,5.37501356919583)
--(axis cs:60001,5.37537807246272)
--(axis cs:61001,5.37580478869213)
--(axis cs:62001,5.37608625490963)
--(axis cs:63001,5.37659354606769)
--(axis cs:64001,5.37684412269381)
--(axis cs:65001,5.37724099410024)
--(axis cs:66001,5.37768991265678)
--(axis cs:67001,5.37807527906341)
--(axis cs:68001,5.37842308817238)
--(axis cs:69001,5.37894557259367)
--(axis cs:70001,5.37924647708565)
--(axis cs:71001,5.37965216439762)
--(axis cs:72001,5.37981831531635)
--(axis cs:73001,5.38016024753151)
--(axis cs:74001,5.38081252544783)
--(axis cs:75001,5.38124666169964)
--(axis cs:76001,5.3812519910325)
--(axis cs:77001,5.38158897229031)
--(axis cs:78001,5.38185934887449)
--(axis cs:79001,5.38221624463853)
--(axis cs:80001,5.38249795042802)
--(axis cs:81001,5.38276443887049)
--(axis cs:82001,5.38313101545885)
--(axis cs:83001,5.38336582552614)
--(axis cs:84001,5.38368022870152)
--(axis cs:85001,5.38396085809986)
--(axis cs:86001,5.38431991482937)
--(axis cs:87001,5.38438591055145)
--(axis cs:88001,5.38462279760302)
--(axis cs:89001,5.38482992950896)
--(axis cs:90001,5.38510932107986)
--(axis cs:91001,5.38529039629938)
--(axis cs:92001,5.38550199319732)
--(axis cs:93001,5.38574840299966)
--(axis cs:94001,5.38581878444216)
--(axis cs:95001,5.38594181034982)
--(axis cs:96001,5.38613064097377)
--(axis cs:97001,5.38627421948149)
--(axis cs:98001,5.38636187058652)
--(axis cs:99001,5.38647532565112)
--(axis cs:99001,5.38129025247435)
--(axis cs:99001,5.38129025247435)
--(axis cs:98001,5.38113927736094)
--(axis cs:97001,5.38097725215282)
--(axis cs:96001,5.38075762060684)
--(axis cs:95001,5.3805953839955)
--(axis cs:94001,5.38051263756398)
--(axis cs:93001,5.38028680092288)
--(axis cs:92001,5.38009185904343)
--(axis cs:91001,5.379843613217)
--(axis cs:90001,5.3795839601059)
--(axis cs:89001,5.37918081194339)
--(axis cs:88001,5.37882216324971)
--(axis cs:87001,5.37855934295138)
--(axis cs:86001,5.37844516929756)
--(axis cs:85001,5.37810547053157)
--(axis cs:84001,5.377744516242)
--(axis cs:83001,5.37739657492686)
--(axis cs:82001,5.37706312851501)
--(axis cs:81001,5.37673778949707)
--(axis cs:80001,5.37631256444054)
--(axis cs:79001,5.37603289144836)
--(axis cs:78001,5.37574888691736)
--(axis cs:77001,5.37543548193756)
--(axis cs:76001,5.37507542571202)
--(axis cs:75001,5.37493578920102)
--(axis cs:74001,5.37452051059222)
--(axis cs:73001,5.37386284804252)
--(axis cs:72001,5.3734478893197)
--(axis cs:71001,5.37319075330777)
--(axis cs:70001,5.37264906726371)
--(axis cs:69001,5.37230267018542)
--(axis cs:68001,5.37181646712827)
--(axis cs:67001,5.37138816178076)
--(axis cs:66001,5.37097571362162)
--(axis cs:65001,5.37044473381164)
--(axis cs:64001,5.37006233189283)
--(axis cs:63001,5.36966254510546)
--(axis cs:62001,5.36911269349443)
--(axis cs:61001,5.36863087629697)
--(axis cs:60001,5.36806353684379)
--(axis cs:59001,5.36775146868488)
--(axis cs:58001,5.36746768222681)
--(axis cs:57001,5.3667852309767)
--(axis cs:56001,5.3663043166647)
--(axis cs:55001,5.36578045051348)
--(axis cs:54001,5.36499470435966)
--(axis cs:53001,5.36425442694853)
--(axis cs:52001,5.36340469302507)
--(axis cs:51001,5.36274523052661)
--(axis cs:50001,5.36219530971383)
--(axis cs:49001,5.36136970637337)
--(axis cs:48001,5.36007310333394)
--(axis cs:47001,5.35959924866634)
--(axis cs:46001,5.3586872570441)
--(axis cs:45001,5.35787494396195)
--(axis cs:44001,5.35681748008973)
--(axis cs:43001,5.35571318604353)
--(axis cs:42001,5.3545420971426)
--(axis cs:41001,5.35312158723313)
--(axis cs:40001,5.35192750058447)
--(axis cs:39001,5.35112526979234)
--(axis cs:38001,5.34965268089328)
--(axis cs:37001,5.34800888193841)
--(axis cs:36001,5.34628561890937)
--(axis cs:35001,5.34437997518896)
--(axis cs:34001,5.34280165843372)
--(axis cs:33001,5.341134434655)
--(axis cs:32001,5.33971539391596)
--(axis cs:31001,5.33827454782399)
--(axis cs:30001,5.33636475039983)
--(axis cs:29001,5.33406272417648)
--(axis cs:28001,5.33210873001732)
--(axis cs:27001,5.33016425923149)
--(axis cs:26001,5.32775313019496)
--(axis cs:25001,5.32526148532688)
--(axis cs:24001,5.32259661466645)
--(axis cs:23001,5.3194672499372)
--(axis cs:22001,5.31650985938142)
--(axis cs:21001,5.31198576937874)
--(axis cs:20001,5.30769319961321)
--(axis cs:19001,5.30262049767089)
--(axis cs:18001,5.29664465259909)
--(axis cs:17001,5.29060817842211)
--(axis cs:16001,5.28439750780214)
--(axis cs:15001,5.27724291898475)
--(axis cs:14001,5.26971748046078)
--(axis cs:13001,5.26111134368292)
--(axis cs:12001,5.24981885848213)
--(axis cs:11001,5.23617077858857)
--(axis cs:10001,5.22090604698016)
--(axis cs:9001,5.20350049788856)
--(axis cs:8001,5.17903533608517)
--(axis cs:7001,5.15117087489264)
--(axis cs:6001,5.11087281410334)
--(axis cs:5001,5.05797200852156)
--(axis cs:4001,4.96969063237199)
--(axis cs:3001,4.82132424465519)
--(axis cs:2001,4.5334012236843)
--(axis cs:1001,4.33005352804997)
--(axis cs:1,3.32736867468859)
--cycle;

\path [draw=mediumpurple148103189, fill=mediumpurple148103189, opacity=0.2]
(axis cs:1,3.15024083614974)
--(axis cs:1,4.12975916385026)
--(axis cs:1001,4.34614079373146)
--(axis cs:2001,4.34503097225345)
--(axis cs:3001,4.33923679970305)
--(axis cs:4001,4.33757581348076)
--(axis cs:5001,4.33748459874029)
--(axis cs:6001,4.33847205784888)
--(axis cs:7001,4.3388698234995)
--(axis cs:8001,4.3382902118389)
--(axis cs:9001,4.33776178972871)
--(axis cs:10001,4.33680902315395)
--(axis cs:11001,4.3361936674831)
--(axis cs:12001,4.33698556532742)
--(axis cs:13001,4.33800069428063)
--(axis cs:14001,4.33860650748489)
--(axis cs:15001,4.33792364153824)
--(axis cs:16001,4.33730583209958)
--(axis cs:17001,4.33709366930838)
--(axis cs:18001,4.33665820277588)
--(axis cs:19001,4.3368519518212)
--(axis cs:20001,4.33596466577552)
--(axis cs:21001,4.33570367824516)
--(axis cs:22001,4.33535117243465)
--(axis cs:23001,4.33548857821103)
--(axis cs:24001,4.33560601272618)
--(axis cs:25001,4.33545030538294)
--(axis cs:26001,4.33497828076901)
--(axis cs:27001,4.33498022218349)
--(axis cs:28001,4.33420284881914)
--(axis cs:29001,4.33385547867427)
--(axis cs:30001,4.33402065688288)
--(axis cs:31001,4.33340370210694)
--(axis cs:32001,4.33330390076495)
--(axis cs:33001,4.33342213184128)
--(axis cs:34001,4.33324506318434)
--(axis cs:35001,4.33287309183757)
--(axis cs:36001,4.3327542391097)
--(axis cs:37001,4.33274088999673)
--(axis cs:38001,4.33278645979029)
--(axis cs:39001,4.33271861241671)
--(axis cs:40001,4.33306151012657)
--(axis cs:41001,4.33303521317824)
--(axis cs:42001,4.33327084191857)
--(axis cs:43001,4.33326276075275)
--(axis cs:44001,4.33323484954319)
--(axis cs:45001,4.33337209303223)
--(axis cs:46001,4.3330359894645)
--(axis cs:47001,4.33295054681475)
--(axis cs:48001,4.33260709179797)
--(axis cs:49001,4.33292641189672)
--(axis cs:50001,4.33310042409044)
--(axis cs:51001,4.33332084193061)
--(axis cs:52001,4.33356628288104)
--(axis cs:53001,4.33336207180888)
--(axis cs:54001,4.33348059161191)
--(axis cs:55001,4.33352551314676)
--(axis cs:56001,4.33335267929987)
--(axis cs:57001,4.33346416118076)
--(axis cs:58001,4.33325678383)
--(axis cs:59001,4.33312468513577)
--(axis cs:60001,4.33322637885799)
--(axis cs:61001,4.33341763557792)
--(axis cs:62001,4.33347858052457)
--(axis cs:63001,4.33317123050055)
--(axis cs:64001,4.33338849133435)
--(axis cs:65001,4.33345093298047)
--(axis cs:66001,4.33312421838584)
--(axis cs:67001,4.33301092325226)
--(axis cs:68001,4.33288072944006)
--(axis cs:69001,4.33295231240366)
--(axis cs:70001,4.33275558480336)
--(axis cs:71001,4.33269794089435)
--(axis cs:72001,4.33281257356489)
--(axis cs:73001,4.33292912472022)
--(axis cs:74001,4.3327937641595)
--(axis cs:75001,4.33263427766308)
--(axis cs:76001,4.3327200236085)
--(axis cs:77001,4.33276928464941)
--(axis cs:78001,4.33267952868206)
--(axis cs:79001,4.33254551553608)
--(axis cs:80001,4.33249528820941)
--(axis cs:81001,4.33238835725149)
--(axis cs:82001,4.33250662704851)
--(axis cs:83001,4.33256507773286)
--(axis cs:84001,4.33247870248081)
--(axis cs:85001,4.3324724241549)
--(axis cs:86001,4.33251020726296)
--(axis cs:87001,4.33268017358301)
--(axis cs:88001,4.33271054701476)
--(axis cs:89001,4.33241309338049)
--(axis cs:90001,4.33264438304089)
--(axis cs:91001,4.33241868437615)
--(axis cs:92001,4.33238312506496)
--(axis cs:93001,4.33235648973604)
--(axis cs:94001,4.33242077102645)
--(axis cs:95001,4.33234820987288)
--(axis cs:96001,4.33215712019382)
--(axis cs:97001,4.33208149091291)
--(axis cs:98001,4.33199751116501)
--(axis cs:99001,4.3320056798391)
--(axis cs:99001,4.32896057302704)
--(axis cs:99001,4.32896057302704)
--(axis cs:98001,4.32886350045732)
--(axis cs:97001,4.32884489128933)
--(axis cs:96001,4.32879182825464)
--(axis cs:95001,4.32891767153889)
--(axis cs:94001,4.32898751186415)
--(axis cs:93001,4.32892500186082)
--(axis cs:92001,4.32885012239974)
--(axis cs:91001,4.32887163111489)
--(axis cs:90001,4.32904426486302)
--(axis cs:89001,4.32888149881511)
--(axis cs:88001,4.32906375100459)
--(axis cs:87001,4.32910026572229)
--(axis cs:86001,4.32896396164205)
--(axis cs:85001,4.32891979476018)
--(axis cs:84001,4.32892675697801)
--(axis cs:83001,4.32901442130932)
--(axis cs:82001,4.32906091481073)
--(axis cs:81001,4.32891409580465)
--(axis cs:80001,4.3289959431502)
--(axis cs:79001,4.32901319891057)
--(axis cs:78001,4.32907301295202)
--(axis cs:77001,4.32917900173649)
--(axis cs:76001,4.32922968757952)
--(axis cs:75001,4.32903903335944)
--(axis cs:74001,4.32924323534051)
--(axis cs:73001,4.32926290004656)
--(axis cs:72001,4.32917878766621)
--(axis cs:71001,4.328960908981)
--(axis cs:70001,4.3290292468419)
--(axis cs:69001,4.32908331026848)
--(axis cs:68001,4.32892600869615)
--(axis cs:67001,4.32900964360495)
--(axis cs:66001,4.328997264622)
--(axis cs:65001,4.32927117898703)
--(axis cs:64001,4.32899365894455)
--(axis cs:63001,4.32871477130894)
--(axis cs:62001,4.3289687993564)
--(axis cs:61001,4.32884232737351)
--(axis cs:60001,4.32861792373696)
--(axis cs:59001,4.32864036969381)
--(axis cs:58001,4.32869801005286)
--(axis cs:57001,4.32887088557281)
--(axis cs:56001,4.32874764033727)
--(axis cs:55001,4.32896207798795)
--(axis cs:54001,4.32891454921881)
--(axis cs:53001,4.32889713084767)
--(axis cs:52001,4.32920020237886)
--(axis cs:51001,4.32896970139206)
--(axis cs:50001,4.32871833953429)
--(axis cs:49001,4.3285784553509)
--(axis cs:48001,4.328123309652)
--(axis cs:47001,4.32838602049233)
--(axis cs:46001,4.32843223948704)
--(axis cs:45001,4.32872341595646)
--(axis cs:44001,4.32847011170769)
--(axis cs:43001,4.3284520830881)
--(axis cs:42001,4.32847244990781)
--(axis cs:41001,4.32818085472254)
--(axis cs:40001,4.32806496171163)
--(axis cs:39001,4.32754445776098)
--(axis cs:38001,4.32766563357568)
--(axis cs:37001,4.32753639980624)
--(axis cs:36001,4.32742186710957)
--(axis cs:35001,4.32738118661162)
--(axis cs:34001,4.32737315392692)
--(axis cs:33001,4.32755541429369)
--(axis cs:32001,4.32773919163841)
--(axis cs:31001,4.32781625853949)
--(axis cs:30001,4.32846259367543)
--(axis cs:29001,4.32813686641728)
--(axis cs:28001,4.32799778687245)
--(axis cs:27001,4.32858483096269)
--(axis cs:26001,4.32851466181012)
--(axis cs:25001,4.32893271929608)
--(axis cs:24001,4.32870630759381)
--(axis cs:23001,4.32848603158854)
--(axis cs:22001,4.32822957389506)
--(axis cs:21001,4.32787996062918)
--(axis cs:20001,4.32751016048317)
--(axis cs:19001,4.32811726032553)
--(axis cs:18001,4.32745379100224)
--(axis cs:17001,4.32736371555133)
--(axis cs:16001,4.32737012565306)
--(axis cs:15001,4.32766931893107)
--(axis cs:14001,4.32785731652768)
--(axis cs:13001,4.32669586752231)
--(axis cs:12001,4.3248426156575)
--(axis cs:11001,4.32362998491213)
--(axis cs:10001,4.32374892105163)
--(axis cs:9001,4.32413355523296)
--(axis cs:8001,4.32348706600138)
--(axis cs:7001,4.32324701695186)
--(axis cs:6001,4.32129798047806)
--(axis cs:5001,4.32062377958404)
--(axis cs:4001,4.31865013003336)
--(axis cs:3001,4.3161847264549)
--(axis cs:2001,4.31759771340372)
--(axis cs:1001,4.31176130417063)
--(axis cs:1,3.15024083614974)
--cycle;

\path [draw=orchid227119194, fill=black, opacity=0.2]
(axis cs:1,3.71522495142581)
--(axis cs:1,4.72477504857419)
--(axis cs:1001,4.52938282817227)
--(axis cs:2001,4.51710051513528)
--(axis cs:3001,4.51794033010399)
--(axis cs:4001,4.51456294369908)
--(axis cs:5001,4.51200922457433)
--(axis cs:6001,4.51246758146867)
--(axis cs:7001,4.51398236650892)
--(axis cs:8001,4.51357849674265)
--(axis cs:9001,4.51314072018882)
--(axis cs:10001,4.51379805010857)
--(axis cs:11001,4.5148519639747)
--(axis cs:12001,4.51497949069581)
--(axis cs:13001,4.51431777351418)
--(axis cs:14001,4.51401717224414)
--(axis cs:15001,4.51331211243783)
--(axis cs:16001,4.51371779957241)
--(axis cs:17001,4.51264740246319)
--(axis cs:18001,4.5123806777584)
--(axis cs:19001,4.51246397402101)
--(axis cs:20001,4.51207646068586)
--(axis cs:21001,4.51252915652542)
--(axis cs:22001,4.51218541836223)
--(axis cs:23001,4.51296430732356)
--(axis cs:24001,4.51315361533314)
--(axis cs:25001,4.51275577492161)
--(axis cs:26001,4.51302744560489)
--(axis cs:27001,4.51288840476559)
--(axis cs:28001,4.5128152450468)
--(axis cs:29001,4.51306143150349)
--(axis cs:30001,4.51311938969137)
--(axis cs:31001,4.51259583768534)
--(axis cs:32001,4.51232606830383)
--(axis cs:33001,4.51189360476796)
--(axis cs:34001,4.51244162020492)
--(axis cs:35001,4.51227068067677)
--(axis cs:36001,4.51242249257883)
--(axis cs:37001,4.51232785147391)
--(axis cs:38001,4.51204143905581)
--(axis cs:39001,4.51190419856407)
--(axis cs:40001,4.51213354928717)
--(axis cs:41001,4.51173647218307)
--(axis cs:42001,4.51158999547873)
--(axis cs:43001,4.51182597793736)
--(axis cs:44001,4.51183051364336)
--(axis cs:45001,4.51179484190145)
--(axis cs:46001,4.51199105805295)
--(axis cs:47001,4.51213526857493)
--(axis cs:48001,4.51226158039961)
--(axis cs:49001,4.51231303021605)
--(axis cs:50001,4.51232481689941)
--(axis cs:51001,4.51227841153821)
--(axis cs:52001,4.51217907332208)
--(axis cs:53001,4.51230198694126)
--(axis cs:54001,4.51196393152678)
--(axis cs:55001,4.51178898249917)
--(axis cs:56001,4.51174507036936)
--(axis cs:57001,4.51161259896692)
--(axis cs:58001,4.5114641528893)
--(axis cs:59001,4.51156594495409)
--(axis cs:60001,4.51135810638932)
--(axis cs:61001,4.51176456680927)
--(axis cs:62001,4.51185265829218)
--(axis cs:63001,4.51205498575479)
--(axis cs:64001,4.51218725066646)
--(axis cs:65001,4.51203713692486)
--(axis cs:66001,4.51178152429897)
--(axis cs:67001,4.51165258994456)
--(axis cs:68001,4.51174290016386)
--(axis cs:69001,4.51199540996359)
--(axis cs:70001,4.51212106464956)
--(axis cs:71001,4.5120241725544)
--(axis cs:72001,4.51203368557708)
--(axis cs:73001,4.51195297715483)
--(axis cs:74001,4.51192624296909)
--(axis cs:75001,4.51162264406)
--(axis cs:76001,4.51179635985076)
--(axis cs:77001,4.51178371661675)
--(axis cs:78001,4.51182336254667)
--(axis cs:79001,4.51177283404621)
--(axis cs:80001,4.51197782990266)
--(axis cs:81001,4.51191712923344)
--(axis cs:82001,4.51197918324254)
--(axis cs:83001,4.51199742676625)
--(axis cs:84001,4.51222538525438)
--(axis cs:85001,4.51210301925582)
--(axis cs:86001,4.51215548923436)
--(axis cs:87001,4.51200965824715)
--(axis cs:88001,4.51200553335811)
--(axis cs:89001,4.5121795514153)
--(axis cs:90001,4.5120544758146)
--(axis cs:91001,4.51208409196584)
--(axis cs:92001,4.51197428934425)
--(axis cs:93001,4.51201314082798)
--(axis cs:94001,4.51191960643609)
--(axis cs:95001,4.51192290970396)
--(axis cs:96001,4.51196673573222)
--(axis cs:97001,4.51211636583298)
--(axis cs:98001,4.5120948075273)
--(axis cs:99001,4.51197191810141)
--(axis cs:99001,4.50833858381271)
--(axis cs:99001,4.50833858381271)
--(axis cs:98001,4.50840416697296)
--(axis cs:97001,4.50831806267807)
--(axis cs:96001,4.5082172206849)
--(axis cs:95001,4.5081901417376)
--(axis cs:94001,4.50821762614654)
--(axis cs:93001,4.50836470457153)
--(axis cs:92001,4.50820288264301)
--(axis cs:91001,4.50829238741351)
--(axis cs:90001,4.50820396575827)
--(axis cs:89001,4.50840268923368)
--(axis cs:88001,4.50822332767756)
--(axis cs:87001,4.50805723753566)
--(axis cs:86001,4.50815450716104)
--(axis cs:85001,4.50819650663211)
--(axis cs:84001,4.50820722864307)
--(axis cs:83001,4.50798642884995)
--(axis cs:82001,4.50794008603466)
--(axis cs:81001,4.50788707071471)
--(axis cs:80001,4.50781742265669)
--(axis cs:79001,4.50770995731086)
--(axis cs:78001,4.50781228315018)
--(axis cs:77001,4.50764668039108)
--(axis cs:76001,4.50757654313735)
--(axis cs:75001,4.50740590222605)
--(axis cs:74001,4.50749187300232)
--(axis cs:73001,4.50751963280942)
--(axis cs:72001,4.50763715232795)
--(axis cs:71001,4.50749217228582)
--(axis cs:70001,4.50754836864425)
--(axis cs:69001,4.50734170109278)
--(axis cs:68001,4.50709976391462)
--(axis cs:67001,4.50694235639952)
--(axis cs:66001,4.50714485560436)
--(axis cs:65001,4.50726902759568)
--(axis cs:64001,4.50733369431878)
--(axis cs:63001,4.50720756563328)
--(axis cs:62001,4.50700123116122)
--(axis cs:61001,4.50704102654166)
--(axis cs:60001,4.50672825883793)
--(axis cs:59001,4.50706695957295)
--(axis cs:58001,4.50687691019582)
--(axis cs:57001,4.50700391653281)
--(axis cs:56001,4.50717316323361)
--(axis cs:55001,4.50717140003933)
--(axis cs:54001,4.50742830195038)
--(axis cs:53001,4.50756782683584)
--(axis cs:52001,4.50741901133014)
--(axis cs:51001,4.50727100906139)
--(axis cs:50001,4.50735078958846)
--(axis cs:49001,4.50734779303245)
--(axis cs:48001,4.50735384427904)
--(axis cs:47001,4.50717666095848)
--(axis cs:46001,4.50717895999014)
--(axis cs:45001,4.50675230149536)
--(axis cs:44001,4.5067245191968)
--(axis cs:43001,4.50645452716718)
--(axis cs:42001,4.50621339015494)
--(axis cs:41001,4.50616065227731)
--(axis cs:40001,4.50631198957435)
--(axis cs:39001,4.5060102138869)
--(axis cs:38001,4.50604229558275)
--(axis cs:37001,4.50617759432485)
--(axis cs:36001,4.50596033012054)
--(axis cs:35001,4.50576423261142)
--(axis cs:34001,4.50568961122945)
--(axis cs:33001,4.50505376046339)
--(axis cs:32001,4.50545337608853)
--(axis cs:31001,4.50548486938862)
--(axis cs:30001,4.50577598046296)
--(axis cs:29001,4.50565378521318)
--(axis cs:28001,4.50525553813952)
--(axis cs:27001,4.50504426439481)
--(axis cs:26001,4.50549645732192)
--(axis cs:25001,4.50531070241929)
--(axis cs:24001,4.50530894872669)
--(axis cs:23001,4.50501578049871)
--(axis cs:22001,4.50381567249728)
--(axis cs:21001,4.5039691054621)
--(axis cs:20001,4.50314677815219)
--(axis cs:19001,4.50369096519272)
--(axis cs:18001,4.504058408959)
--(axis cs:17001,4.5039869131653)
--(axis cs:16001,4.50426607643534)
--(axis cs:15001,4.5037214186601)
--(axis cs:14001,4.50404439478678)
--(axis cs:13001,4.50420080198001)
--(axis cs:12001,4.50464554055158)
--(axis cs:11001,4.50420812147207)
--(axis cs:10001,4.50310825926049)
--(axis cs:9001,4.50259530914126)
--(axis cs:8001,4.50237450913162)
--(axis cs:7001,4.50172967462806)
--(axis cs:6001,4.49966372997942)
--(axis cs:5001,4.49724492459584)
--(axis cs:4001,4.49689419201699)
--(axis cs:3001,4.49624160924956)
--(axis cs:2001,4.48989598661384)
--(axis cs:1001,4.497350438561)
--(axis cs:1,3.71522495142581)
--cycle;

\path [draw=goldenrod18818934, fill=goldenrod18818934, opacity=0.2]
(axis cs:1,3.17148353639415)
--(axis cs:1,4.18851646360585)
--(axis cs:1001,4.35389067481833)
--(axis cs:2001,4.34112698793731)
--(axis cs:3001,4.33681653114022)
--(axis cs:4001,4.33647868991991)
--(axis cs:5001,4.33688829475906)
--(axis cs:6001,4.33864654769803)
--(axis cs:7001,4.33858311988066)
--(axis cs:8001,4.33667676336163)
--(axis cs:9001,4.33624681993513)
--(axis cs:10001,4.3342156401132)
--(axis cs:11001,4.33308246349848)
--(axis cs:12001,4.33397364813112)
--(axis cs:13001,4.33384564382046)
--(axis cs:14001,4.3337771673914)
--(axis cs:15001,4.33387354327976)
--(axis cs:16001,4.33160473728553)
--(axis cs:17001,4.3319245727416)
--(axis cs:18001,4.33211681548175)
--(axis cs:19001,4.33213482750441)
--(axis cs:20001,4.33198269078956)
--(axis cs:21001,4.33258921010725)
--(axis cs:22001,4.33212947732506)
--(axis cs:23001,4.3310949312179)
--(axis cs:24001,4.33139075845785)
--(axis cs:25001,4.33208107609992)
--(axis cs:26001,4.33230860474795)
--(axis cs:27001,4.33199692195294)
--(axis cs:28001,4.33250471902796)
--(axis cs:29001,4.33238390810142)
--(axis cs:30001,4.33204676918055)
--(axis cs:31001,4.33223894655632)
--(axis cs:32001,4.33247021637515)
--(axis cs:33001,4.33173197411579)
--(axis cs:34001,4.33158827194769)
--(axis cs:35001,4.33162611384907)
--(axis cs:36001,4.33153642012165)
--(axis cs:37001,4.33136854810278)
--(axis cs:38001,4.33157038144523)
--(axis cs:39001,4.33147166904369)
--(axis cs:40001,4.33169987746)
--(axis cs:41001,4.33173661743204)
--(axis cs:42001,4.33174795732732)
--(axis cs:43001,4.33177102941567)
--(axis cs:44001,4.33140781806133)
--(axis cs:45001,4.33170627354638)
--(axis cs:46001,4.33189492315756)
--(axis cs:47001,4.33177970837166)
--(axis cs:48001,4.33157867972321)
--(axis cs:49001,4.33161575249469)
--(axis cs:50001,4.33153078506626)
--(axis cs:51001,4.33161167244805)
--(axis cs:52001,4.33177640614876)
--(axis cs:53001,4.33173509210208)
--(axis cs:54001,4.33177111688417)
--(axis cs:55001,4.33178676017894)
--(axis cs:56001,4.33172125386481)
--(axis cs:57001,4.3317176009105)
--(axis cs:58001,4.33177471109719)
--(axis cs:59001,4.33180911658582)
--(axis cs:60001,4.33178234097574)
--(axis cs:61001,4.33202952154863)
--(axis cs:62001,4.33201339773627)
--(axis cs:63001,4.33190028111261)
--(axis cs:64001,4.33193370017055)
--(axis cs:65001,4.33190398392133)
--(axis cs:66001,4.33196030706379)
--(axis cs:67001,4.33227859606974)
--(axis cs:68001,4.33216449929486)
--(axis cs:69001,4.33211053594078)
--(axis cs:70001,4.3320832792923)
--(axis cs:71001,4.33195646828613)
--(axis cs:72001,4.33192552104482)
--(axis cs:73001,4.33203050188543)
--(axis cs:74001,4.3322724551313)
--(axis cs:75001,4.33224605413822)
--(axis cs:76001,4.3319904813466)
--(axis cs:77001,4.33209023360623)
--(axis cs:78001,4.33227318041787)
--(axis cs:79001,4.33216624908376)
--(axis cs:80001,4.33217013126847)
--(axis cs:81001,4.33237936144638)
--(axis cs:82001,4.33219868467826)
--(axis cs:83001,4.33225769717043)
--(axis cs:84001,4.3322362533018)
--(axis cs:85001,4.33210606873118)
--(axis cs:86001,4.33202551850779)
--(axis cs:87001,4.33186509286232)
--(axis cs:88001,4.33189042805187)
--(axis cs:89001,4.3317859520092)
--(axis cs:90001,4.33168028617648)
--(axis cs:91001,4.33168091753885)
--(axis cs:92001,4.33157985399078)
--(axis cs:93001,4.33163980842694)
--(axis cs:94001,4.33180986593044)
--(axis cs:95001,4.33184469436721)
--(axis cs:96001,4.33195395184631)
--(axis cs:97001,4.33197546328692)
--(axis cs:98001,4.33212942212329)
--(axis cs:99001,4.33225631814614)
--(axis cs:99001,4.3282117579339)
--(axis cs:99001,4.3282117579339)
--(axis cs:98001,4.32806261673345)
--(axis cs:97001,4.3278478375038)
--(axis cs:96001,4.32779500910202)
--(axis cs:95001,4.32768352112526)
--(axis cs:94001,4.32767971396764)
--(axis cs:93001,4.32751009318702)
--(axis cs:92001,4.32747776494814)
--(axis cs:91001,4.32756150836856)
--(axis cs:90001,4.32753327811725)
--(axis cs:89001,4.32762057151301)
--(axis cs:88001,4.32769707663557)
--(axis cs:87001,4.32774433691435)
--(axis cs:86001,4.32781052991025)
--(axis cs:85001,4.32783722605359)
--(axis cs:84001,4.32789779272146)
--(axis cs:83001,4.32793193909901)
--(axis cs:82001,4.32785180251093)
--(axis cs:81001,4.32799421418849)
--(axis cs:80001,4.32781861887216)
--(axis cs:79001,4.32780210574719)
--(axis cs:78001,4.32790450961175)
--(axis cs:77001,4.3277973003219)
--(axis cs:76001,4.32770136481332)
--(axis cs:75001,4.32782781154357)
--(axis cs:74001,4.32774240953269)
--(axis cs:73001,4.3275648461235)
--(axis cs:72001,4.32742032137404)
--(axis cs:71001,4.32749508873419)
--(axis cs:70001,4.32753301190354)
--(axis cs:69001,4.32755816451284)
--(axis cs:68001,4.32760109238762)
--(axis cs:67001,4.32772826949943)
--(axis cs:66001,4.32754303379468)
--(axis cs:65001,4.32744771836017)
--(axis cs:64001,4.32746849667013)
--(axis cs:63001,4.3273749684866)
--(axis cs:62001,4.32750757773187)
--(axis cs:61001,4.32736819324294)
--(axis cs:60001,4.32700267927393)
--(axis cs:59001,4.32697158204641)
--(axis cs:58001,4.32708013623303)
--(axis cs:57001,4.32701540377363)
--(axis cs:56001,4.32689984218704)
--(axis cs:55001,4.32683617394953)
--(axis cs:54001,4.32668039327305)
--(axis cs:53001,4.3266305991113)
--(axis cs:52001,4.326817082438)
--(axis cs:51001,4.32651424665157)
--(axis cs:50001,4.32640165628491)
--(axis cs:49001,4.32637490075729)
--(axis cs:48001,4.32638594601375)
--(axis cs:47001,4.3267603226915)
--(axis cs:46001,4.32690553770199)
--(axis cs:45001,4.32683287002821)
--(axis cs:44001,4.32645723046029)
--(axis cs:43001,4.32676249306055)
--(axis cs:42001,4.32669255599379)
--(axis cs:41001,4.32660439864073)
--(axis cs:40001,4.32662666437645)
--(axis cs:39001,4.32630223418956)
--(axis cs:38001,4.32623230795768)
--(axis cs:37001,4.32611097947755)
--(axis cs:36001,4.32589754004613)
--(axis cs:35001,4.32610023682662)
--(axis cs:34001,4.32573827727145)
--(axis cs:33001,4.32601779104284)
--(axis cs:32001,4.32637544469795)
--(axis cs:31001,4.32554112505427)
--(axis cs:30001,4.32535398412767)
--(axis cs:29001,4.32567477953004)
--(axis cs:28001,4.32568034579115)
--(axis cs:27001,4.32541280361278)
--(axis cs:26001,4.32577377669891)
--(axis cs:25001,4.32543502325611)
--(axis cs:24001,4.3249035626121)
--(axis cs:23001,4.32480698609004)
--(axis cs:22001,4.32566244122409)
--(axis cs:21001,4.32575086893661)
--(axis cs:20001,4.32518245095335)
--(axis cs:19001,4.32478849232086)
--(axis cs:18001,4.32412672654369)
--(axis cs:17001,4.32356392793483)
--(axis cs:16001,4.32284185980215)
--(axis cs:15001,4.32415725466705)
--(axis cs:14001,4.32398727800535)
--(axis cs:13001,4.32337149332284)
--(axis cs:12001,4.32280162059649)
--(axis cs:11001,4.32155256968032)
--(axis cs:10001,4.32212672565022)
--(axis cs:9001,4.32344432549316)
--(axis cs:8001,4.32352571132904)
--(axis cs:7001,4.32504778998936)
--(axis cs:6001,4.32455625183538)
--(axis cs:5001,4.32150002757648)
--(axis cs:4001,4.32001718611107)
--(axis cs:3001,4.3177786038148)
--(axis cs:2001,4.31794347683031)
--(axis cs:1001,4.3170783561507)
--(axis cs:1,3.17148353639415)
--cycle;

\addplot [very thick, black]
table {%
1 4.16
1001 5.39340659340659
2001 5.4031684157921
3001 5.40208597134289
4001 5.40262934266434
5001 5.39935212957408
6001 5.40063989335111
7001 5.40063133838023
8001 5.40069491313586
9001 5.40117986890345
10001 5.40024397560244
11001 5.40075811289883
12001 5.40024664611282
13001 5.40133066687178
14001 5.401095636026
15001 5.40105859609359
16001 5.40134116617711
17001 5.40222104582084
18001 5.40274762513194
19001 5.40339982106205
20001 5.40334483275836
21001 5.40389219560973
22001 5.40397527385119
23001 5.4038815703665
24001 5.40412816132661
25001 5.40452381904724
26001 5.40393292565671
27001 5.40410429243361
28001 5.40390271776008
29001 5.40339160718596
30001 5.4037178760708
31001 5.40408825521757
32001 5.4034636417612
33001 5.40323626556771
34001 5.40328402105821
35001 5.403143910174
36001 5.40283714341269
37001 5.40274965541472
38001 5.40276097997421
39001 5.40252865311146
40001 5.40202694932627
41001 5.40228287114948
42001 5.40222518511464
43001 5.40210041626939
44001 5.40198268221177
45001 5.40192351281083
46001 5.40219734353601
47001 5.40241271462309
48001 5.40230620195412
49001 5.40189873676047
50001 5.40195156096878
51001 5.4020234897355
52001 5.40189688659833
53001 5.4020588290787
54001 5.40191737190052
55001 5.40210759804367
56001 5.40240245709898
57001 5.4025034648515
58001 5.40226099550008
59001 5.40225826680904
60001 5.40229796170064
61001 5.40238914116162
62001 5.40233641392881
63001 5.40258884779607
64001 5.40258214715395
65001 5.40253719173551
66001 5.402508143816
67001 5.40242175489918
68001 5.40242849369862
69001 5.40212402718801
70001 5.40213654090656
71001 5.40229517894114
72001 5.40214774794794
73001 5.40210545061027
74001 5.40206483696166
75001 5.40207943894081
76001 5.40207655162432
77001 5.40195374086051
78001 5.40199407699901
79001 5.402072631992
80001 5.40228547143161
81001 5.40200960481969
82001 5.40206704796283
83001 5.4023640679028
84001 5.4023211628433
85001 5.40229338478371
86001 5.40232997290729
87001 5.4025379018632
88001 5.40257406165839
89001 5.40236401838182
90001 5.40250397217809
91001 5.40250502741728
92001 5.4024501907588
93001 5.40255760690745
94001 5.40278422569973
95001 5.4027330238629
96001 5.40268893032364
97001 5.40271646684055
98001 5.40265650350507
99001 5.40260583226432
};
\addlegendentry{Oracle}
\addplot [very thick, steelblue31119180,mark=diamond, mark options={fill=none},mark size=3pt,  mark repeat=10]
table {%
1 3.88
1001 4.48005994005994
2001 4.59653173413293
3001 4.86233922025991
4001 4.9990352411897
5001 5.08125574885023
6001 5.13070154974171
7001 5.1682016854735
8001 5.19431071116111
9001 5.21691812020887
10001 5.23298470152985
11001 5.24713026088537
12001 5.25979001749854
13001 5.27004845781094
14001 5.27838868652239
15001 5.28573161789214
16001 5.29258421348666
17001 5.29854949708841
18001 5.30428642853175
19001 5.30989526867007
20001 5.3145992700365
21001 5.31877339174325
22001 5.32308258715513
23001 5.3259501760793
24001 5.32889879588351
25001 5.33139154433823
26001 5.33373716395523
27001 5.33609570015925
28001 5.33772936680833
29001 5.33938484879832
30001 5.34145928469051
31001 5.34324312118964
32001 5.34459673135215
33001 5.34588648828823
34001 5.34750154407223
35001 5.34896374389303
36001 5.35091691897447
37001 5.35264019891354
38001 5.35417225862477
39001 5.35561088177226
40001 5.35633709157271
41001 5.35741128265164
42001 5.35868336468179
43001 5.35977116811237
44001 5.36086043499011
45001 5.36180129330459
46001 5.36267428968935
47001 5.36348120252761
48001 5.36400824982813
49001 5.36518683292178
50001 5.36605467890642
51001 5.36652065645772
52001 5.36717793888579
53001 5.36798852851833
54001 5.36879428158738
55001 5.36957800767259
56001 5.37009053409761
57001 5.37043700987702
58001 5.37106187824348
59001 5.37138251894036
60001 5.37172080465326
61001 5.37221783249455
62001 5.37259947420203
63001 5.37312804558658
64001 5.37345322729332
65001 5.37384286395594
66001 5.3743328131392
67001 5.37473172042209
68001 5.37511977765033
69001 5.37562412138954
70001 5.37594777217468
71001 5.37642145885269
72001 5.37663310231803
73001 5.37701154778702
74001 5.37766651802003
75001 5.37809122545033
76001 5.37816370837226
77001 5.37851222711393
78001 5.37880411789592
79001 5.37912456804344
80001 5.37940525743428
81001 5.37975111418378
82001 5.38009707198693
83001 5.3803812002265
84001 5.38071237247176
85001 5.38103316431571
86001 5.38138254206346
87001 5.38147262675142
88001 5.38172248042636
89001 5.38200537072617
90001 5.38234664059288
91001 5.38256700475819
92001 5.38279692612037
93001 5.38301760196127
94001 5.38316571100307
95001 5.38326859717266
96001 5.3834441307903
97001 5.38362573581716
98001 5.38375057397373
99001 5.38388278906273
};
\addlegendentry{UCB QR}
\addplot [very thick, mediumpurple148103189, mark=star, mark options={fill=none},mark size=3pt,  mark repeat=10]
table {%
1 3.64
1001 4.32895104895105
2001 4.33131434282859
3001 4.32771076307897
4001 4.32811297175706
5001 4.32905418916217
6001 4.32988501916347
7001 4.33105842022568
8001 4.33088863892014
9001 4.33094767248084
10001 4.33027897210279
11001 4.32991182619762
12001 4.33091409049246
13001 4.33234828090147
14001 4.33323191200629
15001 4.33279648023465
16001 4.33233797887632
17001 4.33222869242986
18001 4.33205599688906
19001 4.33248460607336
20001 4.33173741312934
21001 4.33179181943717
22001 4.33179037316486
23001 4.33198730489979
24001 4.33215616015999
25001 4.33219151233951
26001 4.33174647128957
27001 4.33178252657309
28001 4.33110031784579
29001 4.33099617254577
30001 4.33124162527916
31001 4.33060998032322
32001 4.33052154620168
33001 4.33048877306748
34001 4.33030910855563
35001 4.33012713922459
36001 4.33008805310964
37001 4.33013864490149
38001 4.33022604668298
39001 4.33013153508884
40001 4.3305632359191
41001 4.33060803395039
42001 4.33087164591319
43001 4.33085742192042
44001 4.33085248062544
45001 4.33104775449434
46001 4.33073411447577
47001 4.33066828365354
48001 4.33036520072498
49001 4.33075243362381
50001 4.33090938181236
51001 4.33114527166134
52001 4.33138324262995
53001 4.33112960132827
54001 4.33119757041536
55001 4.33124379556735
56001 4.33105015981857
57001 4.33116752337678
58001 4.33097739694143
59001 4.33088252741479
60001 4.33092215129748
61001 4.33112998147571
62001 4.33122368994049
63001 4.33094300090475
64001 4.33119107513945
65001 4.33136105598375
66001 4.33106074150392
67001 4.33101028342861
68001 4.3309033690681
69001 4.33101781133607
70001 4.33089241582263
71001 4.33082942493768
72001 4.33099568061555
73001 4.33109601238339
74001 4.33101849975
75001 4.33083665551126
76001 4.33097485559401
77001 4.33097414319295
78001 4.33087627081704
79001 4.33077935722333
80001 4.3307456156798
81001 4.33065122652807
82001 4.33078377092962
83001 4.33078974952109
84001 4.33070272972941
85001 4.33069610945754
86001 4.33073708445251
87001 4.33089021965265
88001 4.33088714900967
89001 4.3306472960978
90001 4.33084432395196
91001 4.33064515774552
92001 4.33061662373235
93001 4.33064074579843
94001 4.3307041414453
95001 4.33063294070589
96001 4.33047447422423
97001 4.33046319110112
98001 4.33043050581116
99001 4.33048312643307
};
\addlegendentry{FCFS ALIS}
\addplot [very thick, orchid227119194, mark=*, mark options={fill=none},mark size=2pt, mark repeat=10]
table {%
1 4.22
1001 4.51336663336663
2001 4.50349825087456
3001 4.50709096967677
4001 4.50572856785804
5001 4.50462707458508
6001 4.50606565572405
7001 4.50785602056849
8001 4.50797650293713
9001 4.50786801466504
10001 4.50845315468453
11001 4.50953004272339
12001 4.5098125156237
13001 4.5092592877471
14001 4.50903078351546
15001 4.50851676554896
16001 4.50899193800388
17001 4.50831715781425
18001 4.5082195433587
19001 4.50807746960686
20001 4.50761161941903
21001 4.50824913099376
22001 4.50800054542975
23001 4.50899004391113
24001 4.50923128202992
25001 4.50903323867045
26001 4.5092619514634
27001 4.5089663345802
28001 4.50903539159316
29001 4.50935760835833
30001 4.50944768507716
31001 4.50904035353698
32001 4.50888972219618
33001 4.50847368261568
34001 4.50906561571719
35001 4.5090174566441
36001 4.50919141134968
37001 4.50925272289938
38001 4.50904186731928
39001 4.50895720622548
40001 4.50922276943076
41001 4.50894856223019
42001 4.50890169281684
43001 4.50914025255227
44001 4.50927751642008
45001 4.50927357169841
46001 4.50958500902154
47001 4.50965596476671
48001 4.50980771233933
49001 4.50983041162425
50001 4.50983780324394
51001 4.5097747102998
52001 4.50979904232611
53001 4.50993490688855
54001 4.50969611673858
55001 4.50948019126925
56001 4.50945911680148
57001 4.50930825774986
58001 4.50917053154256
59001 4.50931645226352
60001 4.50904318261362
61001 4.50940279667546
62001 4.5094269447267
63001 4.50963127569404
64001 4.50976047249262
65001 4.50965308226027
66001 4.50946318995167
67001 4.50929747317204
68001 4.50942133203924
69001 4.50966855552818
70001 4.50983471664691
71001 4.50975817242011
72001 4.50983541895252
73001 4.50973630498212
74001 4.5097090579857
75001 4.50951427314302
76001 4.50968645149406
77001 4.50971519850391
78001 4.50981782284842
79001 4.50974139567854
80001 4.50989762627967
81001 4.50990209997407
82001 4.5099596346386
83001 4.5099919278081
84001 4.51021630694873
85001 4.51014976294397
86001 4.5101549981977
87001 4.5100334478914
88001 4.51011443051784
89001 4.51029112032449
90001 4.51012922078643
91001 4.51018823968967
92001 4.51008858599363
93001 4.51018892269976
94001 4.51006861629132
95001 4.51005652572078
96001 4.51009197820856
97001 4.51021721425552
98001 4.51024948725013
99001 4.51015525095706
};
\addlegendentry{Greedy}
\addplot [very thick, goldenrod18818934,mark=square, mark options={fill},mark size=2pt,  mark repeat=10, mark phase=5]
table {%
1 3.68
1001 4.33548451548452
2001 4.32953523238381
3001 4.32729756747751
4001 4.32824793801549
5001 4.32919416116777
6001 4.3316013997667
7001 4.33181545493501
8001 4.33010123734533
9001 4.32984557271415
10001 4.32817118288171
11001 4.3273175165894
12001 4.3283876343638
13001 4.32860856857165
14001 4.32888222269838
15001 4.3290153989734
16001 4.32722329854384
17001 4.32774425033822
18001 4.32812177101272
19001 4.32846165991264
20001 4.32858257087146
21001 4.32917003952193
22001 4.32889595927458
23001 4.32795095865397
24001 4.32814716053498
25001 4.32875804967801
26001 4.32904119072343
27001 4.32870486278286
28001 4.32909253240956
29001 4.32902934381573
30001 4.32870037665411
31001 4.3288900358053
32001 4.32942283053655
33001 4.32887488257931
34001 4.32866327460957
35001 4.32886317533785
36001 4.32871698008389
37001 4.32873976379017
38001 4.32890134470146
39001 4.32888695161662
40001 4.32916327091823
41001 4.32917050803639
42001 4.32922025666056
43001 4.32926676123811
44001 4.32893252426081
45001 4.32926957178729
46001 4.32940023042977
47001 4.32927001553158
48001 4.32898231286848
49001 4.32899532662599
50001 4.32896622067559
51001 4.32906295954981
52001 4.32929674429338
53001 4.32918284560669
54001 4.32922575507861
55001 4.32931146706424
56001 4.32931054802593
57001 4.32936650234206
58001 4.32942742366511
59001 4.32939034931611
60001 4.32939251012483
61001 4.32969885739578
62001 4.32976048773407
63001 4.3296376247996
64001 4.32970109842034
65001 4.32967585114075
66001 4.32975167042924
67001 4.33000343278459
68001 4.32988279584124
69001 4.32983435022681
70001 4.32980814559792
71001 4.32972577851016
72001 4.32967292120943
73001 4.32979767400446
74001 4.330007432332
75001 4.33003693284089
76001 4.32984592307996
77001 4.32994376696407
78001 4.33008884501481
79001 4.32998417741548
80001 4.32999437507031
81001 4.33018678781743
82001 4.33002524359459
83001 4.33009481813472
84001 4.33006702301163
85001 4.32997164739238
86001 4.32991802420902
87001 4.32980471488833
88001 4.32979375234372
89001 4.3297032617611
90001 4.32960678214687
91001 4.3296212129537
92001 4.32952880946946
93001 4.32957495080698
94001 4.32974478994904
95001 4.32976410774623
96001 4.32987448047416
97001 4.32991165039536
98001 4.33009601942837
99001 4.33023403804002
};
\addlegendentry{Random}
\end{axis}

\end{tikzpicture}

%% file: Figs/TikzPlot_Small_Regret.tex
\begin{tikzpicture}[scale=.65,transform shape]

\definecolor{darkgray176}{RGB}{176,176,176}
\definecolor{steelblue31119180}{RGB}{31,119,180}

\begin{axis}[
tick align=outside,
tick pos=left,
x grid style={darkgray176},
xlabel={Time},
xmajorgrids,
xmin=0, xmax=100000,
xtick style={color=black},
y grid style={darkgray176},
ytick={0,1000,2000},
ymajorgrids,
ymin=0, ymax=2500,
ytick style={color=black},
width=9cm,  
height=8cm,  
]
\path [draw=steelblue31119180, fill=steelblue31119180, opacity=0.2]
(axis cs:1,0.972368674688589)
--(axis cs:1,2.07763132531141)
--(axis cs:1001,1076.02141842198)
--(axis cs:2001,1744.06915140771)
--(axis cs:3001,1751.61094178979)
--(axis cs:4001,1741.67277987965)
--(axis cs:5001,1735.48698538369)
--(axis cs:6001,1765.05724256587)
--(axis cs:7001,1777.05770487664)
--(axis cs:8001,1807.94327598259)
--(axis cs:9001,1813.6970185051)
--(axis cs:10001,1841.12362415137)
--(axis cs:11001,1857.29026474711)
--(axis cs:12001,1862.328879356)
--(axis cs:13001,1870.6964207784)
--(axis cs:14001,1894.09055606857)
--(axis cs:15001,1916.48397230979)
--(axis cs:16001,1929.76047765791)
--(axis cs:17001,1944.77535864576)
--(axis cs:18001,1950.50460856369)
--(axis cs:19001,1945.31292375545)
--(axis cs:20001,1946.23331453617)
--(axis cs:21001,1953.39185727696)
--(axis cs:22001,1946.87158374931)
--(axis cs:23001,1967.33878419447)
--(axis cs:24001,1977.7636513906)
--(axis cs:25001,1993.54260534278)
--(axis cs:26001,2008.49586180094)
--(axis cs:27001,2020.63983649052)
--(axis cs:28001,2041.02845078502)
--(axis cs:29001,2057.25193615802)
--(axis cs:30001,2059.12612325467)
--(axis cs:31001,2068.55574290847)
--(axis cs:32001,2089.1726792955)
--(axis cs:33001,2107.62752195034)
--(axis cs:34001,2114.80581159525)
--(axis cs:35001,2121.76148841118)
--(axis cs:36001,2113.77643364364)
--(axis cs:37001,2108.72835939685)
--(axis cs:38001,2103.25347337437)
--(axis cs:39001,2101.16835282891)
--(axis cs:40001,2122.95304912047)
--(axis cs:41001,2127.06680185433)
--(axis cs:42001,2119.28237791376)
--(axis cs:43001,2119.38228694229)
--(axis cs:44001,2120.07905857172)
--(axis cs:45001,2120.67464676826)
--(axis cs:46001,2130.43248871437)
--(axis cs:47001,2133.88071343313)
--(axis cs:48001,2156.53596686734)
--(axis cs:49001,2137.92801799837)
--(axis cs:50001,2140.27731899875)
--(axis cs:51001,2155.03549791245)
--(axis cs:52001,2162.99755800355)
--(axis cs:53001,2159.5561173011)
--(axis cs:54001,2160.32596987425)
--(axis cs:55001,2157.11444130817)
--(axis cs:56001,2166.99696245996)
--(axis cs:57001,2178.28004909688)
--(axis cs:58001,2176.91196316307)
--(axis cs:59001,2197.70059612306)
--(axis cs:60001,2216.22472583581)
--(axis cs:61001,2218.55291500849)
--(axis cs:62001,2225.04889065211)
--(axis cs:63001,2226.2949958107)
--(axis cs:64001,2236.04569652669)
--(axis cs:65001,2246.12685750989)
--(axis cs:66001,2245.63692525987)
--(axis cs:67001,2252.02677252758)
--(axis cs:68001,2256.51341881016)
--(axis cs:69001,2256.14845453568)
--(axis cs:70001,2264.59764247304)
--(axis cs:71001,2258.4883243952)
--(axis cs:72001,2271.78352109256)
--(axis cs:73001,2273.04323004764)
--(axis cs:74001,2255.51269566498)
--(axis cs:75001,2254.84587413432)
--(axis cs:76001,2274.29757046116)
--(axis cs:77001,2276.49745532618)
--(axis cs:78001,2281.61607155904)
--(axis cs:79001,2288.43054268798)
--(axis cs:80001,2295.02353219216)
--(axis cs:81001,2289.26731294828)
--(axis cs:82001,2290.85139864087)
--(axis cs:83001,2291.11188449542)
--(axis cs:84001,2289.48789115605)
--(axis cs:85001,2286.06189934614)
--(axis cs:86001,2283.74199524056)
--(axis cs:87001,2300.3636038867)
--(axis cs:88001,2303.67581186285)
--(axis cs:89001,2297.93355622657)
--(axis cs:90001,2287.4690065089)
--(axis cs:91001,2289.25635363975)
--(axis cs:92001,2291.57387614602)
--(axis cs:93001,2298.35222737156)
--(axis cs:94001,2301.83655634786)
--(axis cs:95001,2318.46292504349)
--(axis cs:96001,2327.29266412294)
--(axis cs:97001,2330.23056392449)
--(axis cs:98001,2338.37467935009)
--(axis cs:99001,2347.28871478717)
--(axis cs:99001,1833.96128521288)
--(axis cs:99001,1833.96128521288)
--(axis cs:98001,1826.55532064996)
--(axis cs:97001,1816.41943607557)
--(axis cs:96001,1811.47733587712)
--(axis cs:95001,1810.54707495656)
--(axis cs:94001,1803.05344365219)
--(axis cs:93001,1790.41777262849)
--(axis cs:92001,1793.83612385403)
--(axis cs:91001,1793.59364636031)
--(axis cs:90001,1790.18099349116)
--(axis cs:89001,1795.15644377349)
--(axis cs:88001,1793.21418813721)
--(axis cs:87001,1793.44639611335)
--(axis cs:86001,1778.50800475949)
--(axis cs:85001,1788.34810065391)
--(axis cs:84001,1790.882108844)
--(axis cs:83001,1795.65811550464)
--(axis cs:82001,1793.27860135919)
--(axis cs:81001,1801.10268705178)
--(axis cs:80001,1800.1864678079)
--(axis cs:79001,1799.93945731208)
--(axis cs:78001,1804.99392844102)
--(axis cs:77001,1802.67254467388)
--(axis cs:76001,1804.87242953889)
--(axis cs:75001,1781.52412586574)
--(axis cs:74001,1789.89730433507)
--(axis cs:73001,1813.32676995241)
--(axis cs:72001,1813.1064789075)
--(axis cs:71001,1799.72167560486)
--(axis cs:70001,1802.77235752702)
--(axis cs:69001,1797.78154546437)
--(axis cs:68001,1807.25658118989)
--(axis cs:67001,1803.98322747247)
--(axis cs:66001,1802.49307474018)
--(axis cs:65001,1804.36314249017)
--(axis cs:64001,1802.00430347336)
--(axis cs:63001,1789.63500418936)
--(axis cs:62001,1792.68110934795)
--(axis cs:61001,1780.93708499157)
--(axis cs:60001,1777.34527416424)
--(axis cs:59001,1769.22940387699)
--(axis cs:58001,1759.97803683699)
--(axis cs:57001,1761.96995090318)
--(axis cs:56001,1742.9330375401)
--(axis cs:55001,1739.37555869189)
--(axis cs:54001,1749.96403012581)
--(axis cs:53001,1763.73388269896)
--(axis cs:52001,1770.5724419965)
--(axis cs:51001,1769.93450208761)
--(axis cs:50001,1754.3326810013)
--(axis cs:49001,1763.84198200169)
--(axis cs:48001,1778.75403313265)
--(axis cs:47001,1768.96928656687)
--(axis cs:46001,1763.61751128563)
--(axis cs:45001,1767.29535323174)
--(axis cs:44001,1764.29094142827)
--(axis cs:43001,1770.38771305771)
--(axis cs:42001,1771.40762208624)
--(axis cs:41001,1775.30319814567)
--(axis cs:40001,1770.17695087953)
--(axis cs:39001,1751.28164717109)
--(axis cs:38001,1759.75652662562)
--(axis cs:37001,1766.00164060315)
--(axis cs:36001,1780.31356635636)
--(axis cs:35001,1800.88851158882)
--(axis cs:34001,1795.20418840475)
--(axis cs:33001,1793.98247804965)
--(axis cs:32001,1776.7573207045)
--(axis cs:31001,1760.49425709153)
--(axis cs:30001,1753.44387674533)
--(axis cs:29001,1748.55806384198)
--(axis cs:28001,1726.26154921498)
--(axis cs:27001,1700.33016350948)
--(axis cs:26001,1697.31413819905)
--(axis cs:25001,1687.02739465722)
--(axis cs:24001,1675.2463486094)
--(axis cs:23001,1669.11121580552)
--(axis cs:22001,1657.65841625069)
--(axis cs:21001,1668.29814272304)
--(axis cs:20001,1669.97668546383)
--(axis cs:19001,1668.85707624455)
--(axis cs:18001,1675.3853914363)
--(axis cs:17001,1674.75464135424)
--(axis cs:16001,1667.76952234209)
--(axis cs:15001,1661.80602769021)
--(axis cs:14001,1651.27944393142)
--(axis cs:13001,1638.3135792216)
--(axis cs:12001,1623.00112064401)
--(axis cs:11001,1616.1597352529)
--(axis cs:10001,1599.52637584863)
--(axis cs:9001,1572.1529814949)
--(axis cs:8001,1563.50672401741)
--(axis cs:7001,1538.59229512336)
--(axis cs:6001,1527.07275743413)
--(axis cs:5001,1502.60301461631)
--(axis cs:4001,1506.85722012035)
--(axis cs:3001,1505.43905821021)
--(axis cs:2001,1491.42084859229)
--(axis cs:1001,775.708581578018)
--(axis cs:1,0.972368674688589)
--cycle;

\addplot [ultra thick, steelblue31119180]
table {%
1 1.525
1001 925.864999999999
2001 1617.745
3001 1628.525
4001 1624.265
5001 1619.045
6001 1646.065
7001 1657.825
8001 1685.725
9001 1692.925
10001 1720.325
11001 1736.725
12001 1742.665
13001 1754.505
14001 1772.685
15001 1789.145
16001 1798.765
17001 1809.765
18001 1812.945
19001 1807.085
20001 1808.105
21001 1810.845
22001 1802.265
23001 1818.225
24001 1826.505
25001 1840.285
26001 1852.905
27001 1860.485
28001 1883.645
29001 1902.905
30001 1906.285
31001 1914.525
32001 1932.965
33001 1950.805
34001 1955.005
35001 1961.325
36001 1947.045
37001 1937.365
38001 1931.505
39001 1926.225
40001 1946.565
41001 1951.185
42001 1945.345
43001 1944.885
44001 1942.185
45001 1943.985
46001 1947.025
47001 1951.425
48001 1967.645
49001 1950.88500000003
50001 1947.30500000003
51001 1962.48500000003
52001 1966.78500000003
53001 1961.64500000003
54001 1955.14500000003
55001 1948.24500000003
56001 1954.96500000003
57001 1970.12500000003
58001 1968.44500000003
59001 1983.46500000003
60001 1996.78500000003
61001 1999.74500000003
62001 2008.86500000003
63001 2007.96500000003
64001 2019.02500000003
65001 2025.24500000003
66001 2024.06500000003
67001 2028.00500000003
68001 2031.88500000003
69001 2026.96500000003
70001 2033.68500000003
71001 2029.10500000003
72001 2042.44500000003
73001 2043.18500000003
74001 2022.70500000003
75001 2018.18500000003
76001 2039.58500000003
77001 2039.58500000003
78001 2043.30500000003
79001 2044.18500000003
80001 2047.60500000003
81001 2045.18500000003
82001 2042.06500000003
83001 2043.38500000003
84001 2040.18500000003
85001 2037.20500000003
86001 2031.12500000003
87001 2046.90500000003
88001 2048.44500000003
89001 2046.54500000003
90001 2038.82500000003
91001 2041.42500000003
92001 2042.70500000003
93001 2044.38500000003
94001 2052.44500000003
95001 2064.50500000003
96001 2069.38500000003
97001 2073.32500000003
98001 2082.46500000003
99001 2090.62500000003
};
\end{axis}

\end{tikzpicture}

%% file: Figs/TikzPlot_Small_ActionFrequency.tex
\begin{tikzpicture}[scale=0.65, transform shape]

\definecolor{crimson2143940}{RGB}{214,39,40}
\definecolor{darkgray176}{RGB}{176,176,176}
\definecolor{darkorange25512714}{RGB}{255,127,14}
\definecolor{goldenrod18818934}{RGB}{188,189,34}
\definecolor{orchid227119194}{RGB}{227,119,194}
\definecolor{sienna1408675}{RGB}{140,86,75}
\definecolor{steelblue31119180}{RGB}{31,119,180}

\begin{axis}[
tick align=outside,
tick pos=left,
x grid style={darkgray176},
xlabel={Episode number},
xmin=1, xmax=20,
xtick style={color=black},
y grid style={darkgray176},
ymin=-1, ymax=17,
ytick style={color=black},
legend image post style={scale=1,xscale=0.8},
legend cell align=left,
xmajorgrids,
ymajorgrids,
legend style={at={(.3,.98)}, anchor=north east, font = \small, mark size=2pt},  
width=9cm,  
height=8cm,  
]
\path [draw=steelblue31119180, fill=steelblue31119180, opacity=0.2]
(axis cs:1,1)
--(axis cs:1,1)
--(axis cs:2,1.36260795446849)
--(axis cs:3,1.56169185039722)
--(axis cs:4,1.78773326585717)
--(axis cs:5,1.97053574520851)
--(axis cs:6,2.10529671466546)
--(axis cs:7,2.22942667684088)
--(axis cs:8,2.4050076894523)
--(axis cs:9,2.57348532380293)
--(axis cs:10,2.83107931686353)
--(axis cs:11,2.93968021192512)
--(axis cs:12,3.05987425947704)
--(axis cs:13,3.14925054109046)
--(axis cs:14,3.20929773392104)
--(axis cs:15,3.34727368104539)
--(axis cs:16,3.43107931686353)
--(axis cs:17,3.46988228016279)
--(axis cs:18,3.51817256806449)
--(axis cs:19,3.55621723182436)
--(axis cs:20,3.68262749607618)
--(axis cs:21,3.82084284012271)
--(axis cs:22,3.89557929444549)
--(axis cs:23,3.99281920611243)
--(axis cs:24,4.04512519259423)
--(axis cs:25,4.13014679648881)
--(axis cs:26,4.15777222016556)
--(axis cs:27,4.15777222016556)
--(axis cs:28,4.17650474457494)
--(axis cs:29,4.20476386129089)
--(axis cs:30,4.22768940028147)
--(axis cs:31,4.22768940028147)
--(axis cs:32,4.25050287814546)
--(axis cs:33,4.28692651182856)
--(axis cs:34,4.34081101920492)
--(axis cs:35,4.42498973527457)
--(axis cs:36,4.42498973527457)
--(axis cs:37,4.45971293595071)
--(axis cs:38,4.45971293595071)
--(axis cs:39,4.52520395018265)
--(axis cs:40,4.55513188150102)
--(axis cs:41,4.5765061125651)
--(axis cs:42,4.61483976305052)
--(axis cs:43,4.63596531816105)
--(axis cs:44,4.65700465355735)
--(axis cs:45,4.69047397552843)
--(axis cs:46,4.72348849320497)
--(axis cs:47,4.76469011631651)
--(axis cs:48,4.83009386989729)
--(axis cs:49,4.86618688747204)
--(axis cs:50,4.89467274114666)
--(axis cs:51,4.89467274114666)
--(axis cs:52,4.91463097462544)
--(axis cs:53,4.91463097462544)
--(axis cs:54,4.93866195424431)
--(axis cs:55,4.99039699954424)
--(axis cs:56,5.01403580284877)
--(axis cs:57,5.04906943629653)
--(axis cs:58,5.04906943629653)
--(axis cs:59,5.087777974711)
--(axis cs:60,5.10293903593807)
--(axis cs:60,4.29706096406193)
--(axis cs:60,4.29706096406193)
--(axis cs:59,4.272222025289)
--(axis cs:58,4.23093056370347)
--(axis cs:57,4.23093056370347)
--(axis cs:56,4.18596419715123)
--(axis cs:55,4.16960300045576)
--(axis cs:54,4.14133804575569)
--(axis cs:53,4.12536902537456)
--(axis cs:52,4.12536902537456)
--(axis cs:51,4.10532725885334)
--(axis cs:50,4.10532725885334)
--(axis cs:49,4.09381311252796)
--(axis cs:48,4.04990613010271)
--(axis cs:47,3.99530988368349)
--(axis cs:46,3.95651150679503)
--(axis cs:45,3.90952602447157)
--(axis cs:44,3.86299534644264)
--(axis cs:43,3.84403468183895)
--(axis cs:42,3.82516023694948)
--(axis cs:41,3.7834938874349)
--(axis cs:40,3.76486811849898)
--(axis cs:39,3.75479604981735)
--(axis cs:38,3.74028706404929)
--(axis cs:37,3.74028706404929)
--(axis cs:36,3.69501026472543)
--(axis cs:35,3.69501026472543)
--(axis cs:34,3.61918898079508)
--(axis cs:33,3.55307348817144)
--(axis cs:32,3.50949712185454)
--(axis cs:31,3.49231059971853)
--(axis cs:30,3.49231059971853)
--(axis cs:29,3.47523613870911)
--(axis cs:28,3.42349525542506)
--(axis cs:27,3.40222777983444)
--(axis cs:26,3.40222777983444)
--(axis cs:25,3.38985320351119)
--(axis cs:24,3.31487480740577)
--(axis cs:23,3.28718079388757)
--(axis cs:22,3.22442070555451)
--(axis cs:21,3.13915715987729)
--(axis cs:20,2.99737250392382)
--(axis cs:19,2.88378276817564)
--(axis cs:18,2.84182743193551)
--(axis cs:17,2.81011771983721)
--(axis cs:16,2.76892068313647)
--(axis cs:15,2.69272631895461)
--(axis cs:14,2.51070226607896)
--(axis cs:13,2.45074945890954)
--(axis cs:12,2.38012574052296)
--(axis cs:11,2.26031978807488)
--(axis cs:10,2.16892068313647)
--(axis cs:9,1.94651467619707)
--(axis cs:8,1.8349923105477)
--(axis cs:7,1.65057332315911)
--(axis cs:6,1.53470328533454)
--(axis cs:5,1.46946425479149)
--(axis cs:4,1.37226673414283)
--(axis cs:3,1.27830814960278)
--(axis cs:2,1.11739204553151)
--(axis cs:1,1)
--cycle;

\path [draw=darkorange25512714, fill=darkorange25512714, opacity=0.2]
(axis cs:1,0)
--(axis cs:1,0)
--(axis cs:2,0.290293370493451)
--(axis cs:3,0.290293370493451)
--(axis cs:4,0.290293370493451)
--(axis cs:5,0.290293370493451)
--(axis cs:6,0.290293370493451)
--(axis cs:7,0.290293370493451)
--(axis cs:8,0.290293370493451)
--(axis cs:9,0.290293370493451)
--(axis cs:10,0.290293370493451)
--(axis cs:11,0.290293370493451)
--(axis cs:12,0.290293370493451)
--(axis cs:13,0.290293370493451)
--(axis cs:14,0.290293370493451)
--(axis cs:15,0.290293370493451)
--(axis cs:16,0.290293370493451)
--(axis cs:17,0.290293370493451)
--(axis cs:18,0.290293370493451)
--(axis cs:19,0.290293370493451)
--(axis cs:20,0.290293370493451)
--(axis cs:21,0.290293370493451)
--(axis cs:22,0.290293370493451)
--(axis cs:23,0.290293370493451)
--(axis cs:24,0.290293370493451)
--(axis cs:25,0.290293370493451)
--(axis cs:26,0.290293370493451)
--(axis cs:27,0.290293370493451)
--(axis cs:28,0.290293370493451)
--(axis cs:29,0.290293370493451)
--(axis cs:30,0.290293370493451)
--(axis cs:31,0.290293370493451)
--(axis cs:32,0.290293370493451)
--(axis cs:33,0.290293370493451)
--(axis cs:34,0.290293370493451)
--(axis cs:35,0.290293370493451)
--(axis cs:36,0.290293370493451)
--(axis cs:37,0.290293370493451)
--(axis cs:38,0.290293370493451)
--(axis cs:39,0.290293370493451)
--(axis cs:40,0.290293370493451)
--(axis cs:41,0.290293370493451)
--(axis cs:42,0.290293370493451)
--(axis cs:43,0.290293370493451)
--(axis cs:44,0.290293370493451)
--(axis cs:45,0.290293370493451)
--(axis cs:46,0.290293370493451)
--(axis cs:47,0.290293370493451)
--(axis cs:48,0.290293370493451)
--(axis cs:49,0.290293370493451)
--(axis cs:50,0.290293370493451)
--(axis cs:51,0.290293370493451)
--(axis cs:52,0.290293370493451)
--(axis cs:53,0.290293370493451)
--(axis cs:54,0.290293370493451)
--(axis cs:55,0.290293370493451)
--(axis cs:56,0.290293370493451)
--(axis cs:57,0.290293370493451)
--(axis cs:58,0.290293370493451)
--(axis cs:59,0.290293370493451)
--(axis cs:60,0.290293370493451)
--(axis cs:60,0.0697066295065487)
--(axis cs:60,0.0697066295065487)
--(axis cs:59,0.0697066295065487)
--(axis cs:58,0.0697066295065487)
--(axis cs:57,0.0697066295065487)
--(axis cs:56,0.0697066295065487)
--(axis cs:55,0.0697066295065487)
--(axis cs:54,0.0697066295065487)
--(axis cs:53,0.0697066295065487)
--(axis cs:52,0.0697066295065487)
--(axis cs:51,0.0697066295065487)
--(axis cs:50,0.0697066295065487)
--(axis cs:49,0.0697066295065487)
--(axis cs:48,0.0697066295065487)
--(axis cs:47,0.0697066295065487)
--(axis cs:46,0.0697066295065487)
--(axis cs:45,0.0697066295065487)
--(axis cs:44,0.0697066295065487)
--(axis cs:43,0.0697066295065487)
--(axis cs:42,0.0697066295065487)
--(axis cs:41,0.0697066295065487)
--(axis cs:40,0.0697066295065487)
--(axis cs:39,0.0697066295065487)
--(axis cs:38,0.0697066295065487)
--(axis cs:37,0.0697066295065487)
--(axis cs:36,0.0697066295065487)
--(axis cs:35,0.0697066295065487)
--(axis cs:34,0.0697066295065487)
--(axis cs:33,0.0697066295065487)
--(axis cs:32,0.0697066295065487)
--(axis cs:31,0.0697066295065487)
--(axis cs:30,0.0697066295065487)
--(axis cs:29,0.0697066295065487)
--(axis cs:28,0.0697066295065487)
--(axis cs:27,0.0697066295065487)
--(axis cs:26,0.0697066295065487)
--(axis cs:25,0.0697066295065487)
--(axis cs:24,0.0697066295065487)
--(axis cs:23,0.0697066295065487)
--(axis cs:22,0.0697066295065487)
--(axis cs:21,0.0697066295065487)
--(axis cs:20,0.0697066295065487)
--(axis cs:19,0.0697066295065487)
--(axis cs:18,0.0697066295065487)
--(axis cs:17,0.0697066295065487)
--(axis cs:16,0.0697066295065487)
--(axis cs:15,0.0697066295065487)
--(axis cs:14,0.0697066295065487)
--(axis cs:13,0.0697066295065487)
--(axis cs:12,0.0697066295065487)
--(axis cs:11,0.0697066295065487)
--(axis cs:10,0.0697066295065487)
--(axis cs:9,0.0697066295065487)
--(axis cs:8,0.0697066295065487)
--(axis cs:7,0.0697066295065487)
--(axis cs:6,0.0697066295065487)
--(axis cs:5,0.0697066295065487)
--(axis cs:4,0.0697066295065487)
--(axis cs:3,0.0697066295065487)
--(axis cs:2,0.0697066295065487)
--(axis cs:1,0)
--cycle;

\path [draw=black, fill=black, opacity=0.2]
(axis cs:1,0)
--(axis cs:1,0)
--(axis cs:2,0.314832870542241)
--(axis cs:3,0.43155758038222)
--(axis cs:4,1.33051531146357)
--(axis cs:5,2.22237289760364)
--(axis cs:6,3.15191358737754)
--(axis cs:7,4.03862901698424)
--(axis cs:8,4.86311516076444)
--(axis cs:9,5.74828540001067)
--(axis cs:10,6.52538851442086)
--(axis cs:11,7.43259017617826)
--(axis cs:12,8.30381921779044)
--(axis cs:13,9.23364302582873)
--(axis cs:14,10.1745351771228)
--(axis cs:15,11.0051725394897)
--(axis cs:16,11.9339056864776)
--(axis cs:17,12.8949541385834)
--(axis cs:18,13.8658604012979)
--(axis cs:19,14.8310793168635)
--(axis cs:20,15.7092821299747)
--(axis cs:21,16.5676260484115)
--(axis cs:22,17.4763640373232)
--(axis cs:23,18.4189029106503)
--(axis cs:24,19.3890617059639)
--(axis cs:25,20.3183815964149)
--(axis cs:26,21.3094782184311)
--(axis cs:27,22.3094782184311)
--(axis cs:28,23.2869265118286)
--(axis cs:29,24.2322581362539)
--(axis cs:30,25.2186114903358)
--(axis cs:31,26.2186114903358)
--(axis cs:32,27.2002166396276)
--(axis cs:33,28.1539386775859)
--(axis cs:34,29.0883526625627)
--(axis cs:35,30.0121177290991)
--(axis cs:36,31.0121177290991)
--(axis cs:37,31.9687309984467)
--(axis cs:38,32.9687309984467)
--(axis cs:39,33.9524994951429)
--(axis cs:40,34.9415493850774)
--(axis cs:41,35.926058820187)
--(axis cs:42,36.8862295671304)
--(axis cs:43,37.866186887472)
--(axis cs:44,38.846058820187)
--(axis cs:45,39.801246859607)
--(axis cs:46,40.7516134341946)
--(axis cs:47,41.7103693876428)
--(axis cs:48,42.6522781776169)
--(axis cs:49,43.6101467964888)
--(axis cs:50,44.5977722201656)
--(axis cs:51,45.5977722201656)
--(axis cs:52,46.5765047445749)
--(axis cs:53,47.5765047445749)
--(axis cs:54,48.5595135164187)
--(axis cs:55,49.529459536121)
--(axis cs:56,50.5121168421661)
--(axis cs:57,51.4687393732038)
--(axis cs:58,52.4687393732038)
--(axis cs:59,53.4292478618099)
--(axis cs:60,54.4029723577178)
--(axis cs:60,53.6370276422822)
--(axis cs:60,53.6370276422822)
--(axis cs:59,52.6507521381901)
--(axis cs:58,51.6912606267962)
--(axis cs:57,50.6912606267962)
--(axis cs:56,49.7278831578339)
--(axis cs:55,48.750540463879)
--(axis cs:54,47.8004864835813)
--(axis cs:53,46.8234952554251)
--(axis cs:52,45.8234952554251)
--(axis cs:51,44.8422277798344)
--(axis cs:50,43.8422277798344)
--(axis cs:49,42.8698532035112)
--(axis cs:48,41.9077218223831)
--(axis cs:47,40.9696306123572)
--(axis cs:46,40.0083865658054)
--(axis cs:45,39.038753140393)
--(axis cs:44,38.073941179813)
--(axis cs:43,37.093813112528)
--(axis cs:42,36.1137704328696)
--(axis cs:41,35.153941179813)
--(axis cs:40,34.1784506149226)
--(axis cs:39,33.2075005048571)
--(axis cs:38,32.2712690015533)
--(axis cs:37,31.2712690015533)
--(axis cs:36,30.3078822709009)
--(axis cs:35,29.3078822709009)
--(axis cs:34,28.3916473374373)
--(axis cs:33,27.4460613224141)
--(axis cs:32,26.4797833603724)
--(axis cs:31,25.5013885096642)
--(axis cs:30,24.5013885096642)
--(axis cs:29,23.5277418637461)
--(axis cs:28,22.5530734881714)
--(axis cs:27,21.5705217815689)
--(axis cs:26,20.5705217815689)
--(axis cs:25,19.6016184035851)
--(axis cs:24,18.6909382940361)
--(axis cs:23,17.7410970893497)
--(axis cs:22,16.8436359626768)
--(axis cs:21,15.9123739515885)
--(axis cs:20,15.0507178700253)
--(axis cs:19,14.1689206831365)
--(axis cs:18,13.2141395987021)
--(axis cs:17,12.2650458614166)
--(axis cs:16,11.3060943135224)
--(axis cs:15,10.3948274605103)
--(axis cs:14,9.54546482287719)
--(axis cs:13,8.60635697417127)
--(axis cs:12,7.69618078220956)
--(axis cs:11,6.80740982382174)
--(axis cs:10,5.91461148557914)
--(axis cs:9,5.17171459998933)
--(axis cs:8,4.33688483923556)
--(axis cs:7,3.52137098301576)
--(axis cs:6,2.64808641262246)
--(axis cs:5,1.77762710239636)
--(axis cs:4,0.949484688536426)
--(axis cs:3,0.16844241961778)
--(axis cs:2,0.0851671294577595)
--(axis cs:1,0)
--cycle;

\path [draw=sienna1408675, fill=sienna1408675, opacity=0.2]
(axis cs:1,0)
--(axis cs:1,0)
--(axis cs:2,0.239613792158118)
--(axis cs:3,0.663426209337327)
--(axis cs:4,0.663426209337327)
--(axis cs:5,0.663426209337327)
--(axis cs:6,0.663426209337327)
--(axis cs:7,0.663426209337327)
--(axis cs:8,0.663426209337327)
--(axis cs:9,0.663426209337327)
--(axis cs:10,0.663426209337327)
--(axis cs:11,0.663426209337327)
--(axis cs:12,0.663426209337327)
--(axis cs:13,0.663426209337327)
--(axis cs:14,0.663426209337327)
--(axis cs:15,0.663426209337327)
--(axis cs:16,0.663426209337327)
--(axis cs:17,0.663426209337327)
--(axis cs:18,0.663426209337327)
--(axis cs:19,0.663426209337327)
--(axis cs:20,0.663426209337327)
--(axis cs:21,0.663426209337327)
--(axis cs:22,0.663426209337327)
--(axis cs:23,0.663426209337327)
--(axis cs:24,0.663426209337327)
--(axis cs:25,0.663426209337327)
--(axis cs:26,0.663426209337327)
--(axis cs:27,0.663426209337327)
--(axis cs:28,0.663426209337327)
--(axis cs:29,0.663426209337327)
--(axis cs:30,0.663426209337327)
--(axis cs:31,0.663426209337327)
--(axis cs:32,0.663426209337327)
--(axis cs:33,0.663426209337327)
--(axis cs:34,0.663426209337327)
--(axis cs:35,0.663426209337327)
--(axis cs:36,0.663426209337327)
--(axis cs:37,0.663426209337327)
--(axis cs:38,0.663426209337327)
--(axis cs:39,0.663426209337327)
--(axis cs:40,0.663426209337327)
--(axis cs:41,0.663426209337327)
--(axis cs:42,0.663426209337327)
--(axis cs:43,0.663426209337327)
--(axis cs:44,0.663426209337327)
--(axis cs:45,0.663426209337327)
--(axis cs:46,0.663426209337327)
--(axis cs:47,0.663426209337327)
--(axis cs:48,0.663426209337327)
--(axis cs:49,0.663426209337327)
--(axis cs:50,0.663426209337327)
--(axis cs:51,0.663426209337327)
--(axis cs:52,0.663426209337327)
--(axis cs:53,0.663426209337327)
--(axis cs:54,0.663426209337327)
--(axis cs:55,0.663426209337327)
--(axis cs:56,0.663426209337327)
--(axis cs:57,0.663426209337327)
--(axis cs:58,0.663426209337327)
--(axis cs:59,0.663426209337327)
--(axis cs:60,0.663426209337327)
--(axis cs:60,0.376573790662673)
--(axis cs:60,0.376573790662673)
--(axis cs:59,0.376573790662673)
--(axis cs:58,0.376573790662673)
--(axis cs:57,0.376573790662673)
--(axis cs:56,0.376573790662673)
--(axis cs:55,0.376573790662673)
--(axis cs:54,0.376573790662673)
--(axis cs:53,0.376573790662673)
--(axis cs:52,0.376573790662673)
--(axis cs:51,0.376573790662673)
--(axis cs:50,0.376573790662673)
--(axis cs:49,0.376573790662673)
--(axis cs:48,0.376573790662673)
--(axis cs:47,0.376573790662673)
--(axis cs:46,0.376573790662673)
--(axis cs:45,0.376573790662673)
--(axis cs:44,0.376573790662673)
--(axis cs:43,0.376573790662673)
--(axis cs:42,0.376573790662673)
--(axis cs:41,0.376573790662673)
--(axis cs:40,0.376573790662673)
--(axis cs:39,0.376573790662673)
--(axis cs:38,0.376573790662673)
--(axis cs:37,0.376573790662673)
--(axis cs:36,0.376573790662673)
--(axis cs:35,0.376573790662673)
--(axis cs:34,0.376573790662673)
--(axis cs:33,0.376573790662673)
--(axis cs:32,0.376573790662673)
--(axis cs:31,0.376573790662673)
--(axis cs:30,0.376573790662673)
--(axis cs:29,0.376573790662673)
--(axis cs:28,0.376573790662673)
--(axis cs:27,0.376573790662673)
--(axis cs:26,0.376573790662673)
--(axis cs:25,0.376573790662673)
--(axis cs:24,0.376573790662673)
--(axis cs:23,0.376573790662673)
--(axis cs:22,0.376573790662673)
--(axis cs:21,0.376573790662673)
--(axis cs:20,0.376573790662673)
--(axis cs:19,0.376573790662673)
--(axis cs:18,0.376573790662673)
--(axis cs:17,0.376573790662673)
--(axis cs:16,0.376573790662673)
--(axis cs:15,0.376573790662673)
--(axis cs:14,0.376573790662673)
--(axis cs:13,0.376573790662673)
--(axis cs:12,0.376573790662673)
--(axis cs:11,0.376573790662673)
--(axis cs:10,0.376573790662673)
--(axis cs:9,0.376573790662673)
--(axis cs:8,0.376573790662673)
--(axis cs:7,0.376573790662673)
--(axis cs:6,0.376573790662673)
--(axis cs:5,0.376573790662673)
--(axis cs:4,0.376573790662673)
--(axis cs:3,0.376573790662673)
--(axis cs:2,0.0403862078418825)
--(axis cs:1,0)
--cycle;

\path [draw=orchid227119194, fill=orchid227119194, opacity=0.2]
(axis cs:1,0)
--(axis cs:1,0)
--(axis cs:2,0.18612465290668)
--(axis cs:3,0.18612465290668)
--(axis cs:4,0.18612465290668)
--(axis cs:5,0.18612465290668)
--(axis cs:6,0.18612465290668)
--(axis cs:7,0.18612465290668)
--(axis cs:8,0.18612465290668)
--(axis cs:9,0.18612465290668)
--(axis cs:10,0.18612465290668)
--(axis cs:11,0.18612465290668)
--(axis cs:12,0.18612465290668)
--(axis cs:13,0.18612465290668)
--(axis cs:14,0.18612465290668)
--(axis cs:15,0.18612465290668)
--(axis cs:16,0.18612465290668)
--(axis cs:17,0.18612465290668)
--(axis cs:18,0.18612465290668)
--(axis cs:19,0.18612465290668)
--(axis cs:20,0.18612465290668)
--(axis cs:21,0.18612465290668)
--(axis cs:22,0.18612465290668)
--(axis cs:23,0.18612465290668)
--(axis cs:24,0.18612465290668)
--(axis cs:25,0.18612465290668)
--(axis cs:26,0.18612465290668)
--(axis cs:27,0.18612465290668)
--(axis cs:28,0.18612465290668)
--(axis cs:29,0.18612465290668)
--(axis cs:30,0.18612465290668)
--(axis cs:31,0.18612465290668)
--(axis cs:32,0.18612465290668)
--(axis cs:33,0.18612465290668)
--(axis cs:34,0.18612465290668)
--(axis cs:35,0.18612465290668)
--(axis cs:36,0.18612465290668)
--(axis cs:37,0.18612465290668)
--(axis cs:38,0.18612465290668)
--(axis cs:39,0.18612465290668)
--(axis cs:40,0.18612465290668)
--(axis cs:41,0.18612465290668)
--(axis cs:42,0.18612465290668)
--(axis cs:43,0.18612465290668)
--(axis cs:44,0.18612465290668)
--(axis cs:45,0.18612465290668)
--(axis cs:46,0.18612465290668)
--(axis cs:47,0.18612465290668)
--(axis cs:48,0.18612465290668)
--(axis cs:49,0.18612465290668)
--(axis cs:50,0.18612465290668)
--(axis cs:51,0.18612465290668)
--(axis cs:52,0.18612465290668)
--(axis cs:53,0.18612465290668)
--(axis cs:54,0.18612465290668)
--(axis cs:55,0.18612465290668)
--(axis cs:56,0.18612465290668)
--(axis cs:57,0.18612465290668)
--(axis cs:58,0.18612465290668)
--(axis cs:59,0.18612465290668)
--(axis cs:60,0.18612465290668)
--(axis cs:60,0.0138753470933197)
--(axis cs:60,0.0138753470933197)
--(axis cs:59,0.0138753470933197)
--(axis cs:58,0.0138753470933197)
--(axis cs:57,0.0138753470933197)
--(axis cs:56,0.0138753470933197)
--(axis cs:55,0.0138753470933197)
--(axis cs:54,0.0138753470933197)
--(axis cs:53,0.0138753470933197)
--(axis cs:52,0.0138753470933197)
--(axis cs:51,0.0138753470933197)
--(axis cs:50,0.0138753470933197)
--(axis cs:49,0.0138753470933197)
--(axis cs:48,0.0138753470933197)
--(axis cs:47,0.0138753470933197)
--(axis cs:46,0.0138753470933197)
--(axis cs:45,0.0138753470933197)
--(axis cs:44,0.0138753470933197)
--(axis cs:43,0.0138753470933197)
--(axis cs:42,0.0138753470933197)
--(axis cs:41,0.0138753470933197)
--(axis cs:40,0.0138753470933197)
--(axis cs:39,0.0138753470933197)
--(axis cs:38,0.0138753470933197)
--(axis cs:37,0.0138753470933197)
--(axis cs:36,0.0138753470933197)
--(axis cs:35,0.0138753470933197)
--(axis cs:34,0.0138753470933197)
--(axis cs:33,0.0138753470933197)
--(axis cs:32,0.0138753470933197)
--(axis cs:31,0.0138753470933197)
--(axis cs:30,0.0138753470933197)
--(axis cs:29,0.0138753470933197)
--(axis cs:28,0.0138753470933197)
--(axis cs:27,0.0138753470933197)
--(axis cs:26,0.0138753470933197)
--(axis cs:25,0.0138753470933197)
--(axis cs:24,0.0138753470933197)
--(axis cs:23,0.0138753470933197)
--(axis cs:22,0.0138753470933197)
--(axis cs:21,0.0138753470933197)
--(axis cs:20,0.0138753470933197)
--(axis cs:19,0.0138753470933197)
--(axis cs:18,0.0138753470933197)
--(axis cs:17,0.0138753470933197)
--(axis cs:16,0.0138753470933197)
--(axis cs:15,0.0138753470933197)
--(axis cs:14,0.0138753470933197)
--(axis cs:13,0.0138753470933197)
--(axis cs:12,0.0138753470933197)
--(axis cs:11,0.0138753470933197)
--(axis cs:10,0.0138753470933197)
--(axis cs:9,0.0138753470933197)
--(axis cs:8,0.0138753470933197)
--(axis cs:7,0.0138753470933197)
--(axis cs:6,0.0138753470933197)
--(axis cs:5,0.0138753470933197)
--(axis cs:4,0.0138753470933197)
--(axis cs:3,0.0138753470933197)
--(axis cs:2,0.0138753470933197)
--(axis cs:1,0)
--cycle;

\path [draw=goldenrod18818934, fill=goldenrod18818934, opacity=0.2]
(axis cs:1,0)
--(axis cs:1,0)
--(axis cs:2,0.239613792158118)
--(axis cs:3,0.623426209337327)
--(axis cs:4,0.623426209337327)
--(axis cs:5,0.623426209337327)
--(axis cs:6,0.623426209337327)
--(axis cs:7,0.623426209337327)
--(axis cs:8,0.623426209337327)
--(axis cs:9,0.623426209337327)
--(axis cs:10,0.623426209337327)
--(axis cs:11,0.623426209337327)
--(axis cs:12,0.623426209337327)
--(axis cs:13,0.623426209337327)
--(axis cs:14,0.623426209337327)
--(axis cs:15,0.623426209337327)
--(axis cs:16,0.623426209337327)
--(axis cs:17,0.623426209337327)
--(axis cs:18,0.623426209337327)
--(axis cs:19,0.623426209337327)
--(axis cs:20,0.623426209337327)
--(axis cs:21,0.623426209337327)
--(axis cs:22,0.623426209337327)
--(axis cs:23,0.623426209337327)
--(axis cs:24,0.623426209337327)
--(axis cs:25,0.623426209337327)
--(axis cs:26,0.623426209337327)
--(axis cs:27,0.623426209337327)
--(axis cs:28,0.623426209337327)
--(axis cs:29,0.623426209337327)
--(axis cs:30,0.623426209337327)
--(axis cs:31,0.623426209337327)
--(axis cs:32,0.623426209337327)
--(axis cs:33,0.623426209337327)
--(axis cs:34,0.623426209337327)
--(axis cs:35,0.623426209337327)
--(axis cs:36,0.623426209337327)
--(axis cs:37,0.623426209337327)
--(axis cs:38,0.623426209337327)
--(axis cs:39,0.623426209337327)
--(axis cs:40,0.623426209337327)
--(axis cs:41,0.623426209337327)
--(axis cs:42,0.623426209337327)
--(axis cs:43,0.623426209337327)
--(axis cs:44,0.623426209337327)
--(axis cs:45,0.623426209337327)
--(axis cs:46,0.623426209337327)
--(axis cs:47,0.623426209337327)
--(axis cs:48,0.623426209337327)
--(axis cs:49,0.623426209337327)
--(axis cs:50,0.623426209337327)
--(axis cs:51,0.623426209337327)
--(axis cs:52,0.623426209337327)
--(axis cs:53,0.623426209337327)
--(axis cs:54,0.623426209337327)
--(axis cs:55,0.623426209337327)
--(axis cs:56,0.623426209337327)
--(axis cs:57,0.623426209337327)
--(axis cs:58,0.623426209337327)
--(axis cs:59,0.623426209337327)
--(axis cs:60,0.623426209337327)
--(axis cs:60,0.336573790662673)
--(axis cs:60,0.336573790662673)
--(axis cs:59,0.336573790662673)
--(axis cs:58,0.336573790662673)
--(axis cs:57,0.336573790662673)
--(axis cs:56,0.336573790662673)
--(axis cs:55,0.336573790662673)
--(axis cs:54,0.336573790662673)
--(axis cs:53,0.336573790662673)
--(axis cs:52,0.336573790662673)
--(axis cs:51,0.336573790662673)
--(axis cs:50,0.336573790662673)
--(axis cs:49,0.336573790662673)
--(axis cs:48,0.336573790662673)
--(axis cs:47,0.336573790662673)
--(axis cs:46,0.336573790662673)
--(axis cs:45,0.336573790662673)
--(axis cs:44,0.336573790662673)
--(axis cs:43,0.336573790662673)
--(axis cs:42,0.336573790662673)
--(axis cs:41,0.336573790662673)
--(axis cs:40,0.336573790662673)
--(axis cs:39,0.336573790662673)
--(axis cs:38,0.336573790662673)
--(axis cs:37,0.336573790662673)
--(axis cs:36,0.336573790662673)
--(axis cs:35,0.336573790662673)
--(axis cs:34,0.336573790662673)
--(axis cs:33,0.336573790662673)
--(axis cs:32,0.336573790662673)
--(axis cs:31,0.336573790662673)
--(axis cs:30,0.336573790662673)
--(axis cs:29,0.336573790662673)
--(axis cs:28,0.336573790662673)
--(axis cs:27,0.336573790662673)
--(axis cs:26,0.336573790662673)
--(axis cs:25,0.336573790662673)
--(axis cs:24,0.336573790662673)
--(axis cs:23,0.336573790662673)
--(axis cs:22,0.336573790662673)
--(axis cs:21,0.336573790662673)
--(axis cs:20,0.336573790662673)
--(axis cs:19,0.336573790662673)
--(axis cs:18,0.336573790662673)
--(axis cs:17,0.336573790662673)
--(axis cs:16,0.336573790662673)
--(axis cs:15,0.336573790662673)
--(axis cs:14,0.336573790662673)
--(axis cs:13,0.336573790662673)
--(axis cs:12,0.336573790662673)
--(axis cs:11,0.336573790662673)
--(axis cs:10,0.336573790662673)
--(axis cs:9,0.336573790662673)
--(axis cs:8,0.336573790662673)
--(axis cs:7,0.336573790662673)
--(axis cs:6,0.336573790662673)
--(axis cs:5,0.336573790662673)
--(axis cs:4,0.336573790662673)
--(axis cs:3,0.336573790662673)
--(axis cs:2,0.0403862078418825)
--(axis cs:1,0)
--cycle;

\addplot [very thick, sienna1408675,mark=diamond, mark options={fill=none},mark size=3pt,  mark repeat=3]
table {%
1 0
2 0.14
3 0.52
4 0.52
5 0.52
6 0.52
7 0.52
8 0.52
9 0.52
10 0.52
11 0.52
12 0.52
13 0.52
14 0.52
15 0.52
16 0.52
17 0.52
18 0.52
19 0.52
20 0.52
21 0.52
22 0.52
23 0.52
24 0.52
25 0.52
26 0.52
27 0.52
28 0.52
29 0.52
30 0.52
31 0.52
32 0.52
33 0.52
34 0.52
35 0.52
36 0.52
37 0.52
38 0.52
39 0.52
40 0.52
41 0.52
42 0.52
43 0.52
44 0.52
45 0.52
46 0.52
47 0.52
48 0.52
49 0.52
50 0.52
51 0.52
52 0.52
53 0.52
54 0.52
55 0.52
56 0.52
57 0.52
58 0.52
59 0.52
60 0.52
};
\addlegendentry{Action 1}
\addplot [very thick, black]
table {%
1 0
2 0.2
3 0.3
4 1.14
5 2
6 2.9
7 3.78
8 4.6
9 5.46
10 6.22
11 7.12
12 8
13 8.92
14 9.86
15 10.7
16 11.62
17 12.58
18 13.54
19 14.5
20 15.38
21 16.24
22 17.16
23 18.08
24 19.04
25 19.96
26 20.94
27 21.94
28 22.92
29 23.88
30 24.86
31 25.86
32 26.84
33 27.8
34 28.74
35 29.66
36 30.66
37 31.62
38 32.62
39 33.58
40 34.56
41 35.54
42 36.5
43 37.48
44 38.46
45 39.42
46 40.38
47 41.34
48 42.28
49 43.24
50 44.22
51 45.22
52 46.2
53 47.2
54 48.18
55 49.14
56 50.12
57 51.08
58 52.08
59 53.04
60 54.02
};
\addlegendentry{Action 2}
\addplot [very thick, darkorange25512714, mark=star, mark options={fill=none},mark size=3pt,  mark repeat=3, mark phase=2]
table {%
1 0
2 0.18
3 0.18
4 0.18
5 0.18
6 0.18
7 0.18
8 0.18
9 0.18
10 0.18
11 0.18
12 0.18
13 0.18
14 0.18
15 0.18
16 0.18
17 0.18
18 0.18
19 0.18
20 0.18
21 0.18
22 0.18
23 0.18
24 0.18
25 0.18
26 0.18
27 0.18
28 0.18
29 0.18
30 0.18
31 0.18
32 0.18
33 0.18
34 0.18
35 0.18
36 0.18
37 0.18
38 0.18
39 0.18
40 0.18
41 0.18
42 0.18
43 0.18
44 0.18
45 0.18
46 0.18
47 0.18
48 0.18
49 0.18
50 0.18
51 0.18
52 0.18
53 0.18
54 0.18
55 0.18
56 0.18
57 0.18
58 0.18
59 0.18
60 0.18
};
\addlegendentry{Action 3}
\addplot [very thick, steelblue31119180, mark=*, mark options={fill=none},mark size=2pt, mark repeat=3, mark phase = 3]
table {%
1 1
2 1.24
3 1.42
4 1.58
5 1.72
6 1.82
7 1.94
8 2.12
9 2.26
10 2.5
11 2.6
12 2.72
13 2.8
14 2.86
15 3.02
16 3.1
17 3.14
18 3.18
19 3.22
20 3.34
21 3.48
22 3.56
23 3.64
24 3.68
25 3.76
26 3.78
27 3.78
28 3.8
29 3.84
30 3.86
31 3.86
32 3.88
33 3.92
34 3.98
35 4.06
36 4.06
37 4.1
38 4.1
39 4.14
40 4.16
41 4.18
42 4.22
43 4.24
44 4.26
45 4.3
46 4.34
47 4.38
48 4.44
49 4.48
50 4.5
51 4.5
52 4.52
53 4.52
54 4.54
55 4.58
56 4.6
57 4.64
58 4.64
59 4.68
60 4.7
};
\addlegendentry{Action 4}
\addplot [very thick, orchid227119194,mark=square, mark options={fill},mark size=2pt,  mark repeat=3, mark phase=3]
table {%
1 0
2 0.1
3 0.1
4 0.1
5 0.1
6 0.1
7 0.1
8 0.1
9 0.1
10 0.1
11 0.1
12 0.1
13 0.1
14 0.1
15 0.1
16 0.1
17 0.1
18 0.1
19 0.1
20 0.1
21 0.1
22 0.1
23 0.1
24 0.1
25 0.1
26 0.1
27 0.1
28 0.1
29 0.1
30 0.1
31 0.1
32 0.1
33 0.1
34 0.1
35 0.1
36 0.1
37 0.1
38 0.1
39 0.1
40 0.1
41 0.1
42 0.1
43 0.1
44 0.1
45 0.1
46 0.1
47 0.1
48 0.1
49 0.1
50 0.1
51 0.1
52 0.1
53 0.1
54 0.1
55 0.1
56 0.1
57 0.1
58 0.1
59 0.1
60 0.1
};
\addlegendentry{Action 5}
\addplot [very thick, goldenrod18818934,mark=o, mark options={fill},mark size=2pt,  mark repeat=3, mark phase=3]
table {%
1 0
2 0.14
3 0.48
4 0.48
5 0.48
6 0.48
7 0.48
8 0.48
9 0.48
10 0.48
11 0.48
12 0.48
13 0.48
14 0.48
15 0.48
16 0.48
17 0.48
18 0.48
19 0.48
20 0.48
21 0.48
22 0.48
23 0.48
24 0.48
25 0.48
26 0.48
27 0.48
28 0.48
29 0.48
30 0.48
31 0.48
32 0.48
33 0.48
34 0.48
35 0.48
36 0.48
37 0.48
38 0.48
39 0.48
40 0.48
41 0.48
42 0.48
43 0.48
44 0.48
45 0.48
46 0.48
47 0.48
48 0.48
49 0.48
50 0.48
51 0.48
52 0.48
53 0.48
54 0.48
55 0.48
56 0.48
57 0.48
58 0.48
59 0.48
60 0.48
};
\addlegendentry{Action 6}
\end{axis}

\end{tikzpicture}

%% file: Figs/TikzPlot_Big_RunningAvgTotalReward.tex
\begin{tikzpicture}[scale=.65,transform shape]

\definecolor{darkgray176}{RGB}{176,176,176}
\definecolor{forestgreen4416044}{RGB}{44,160,44}
\definecolor{goldenrod18818934}{RGB}{188,189,34}
\definecolor{mediumpurple148103189}{RGB}{148,103,189}
\definecolor{orchid227119194}{RGB}{227,119,194}
\definecolor{steelblue31119180}{RGB}{31,119,180}

\definecolor{black}{RGB}{0,0,0}
\definecolor{darkgray}{RGB}{160,160,160}
\definecolor{lightgray}{RGB}{200,200,200}

\begin{axis}[
tick align=outside,
tick pos=left,
x grid style={darkgray176},
xlabel={Time},
xmajorgrids,
xmin=-50, xmax=3000,
xtick={0,1000,2000,3000},
xtick style={color=black},
y grid style={darkgray176},
ymajorgrids,
ytick={110,115,120},
ymin=110, ymax=120,
ytick style={color=black},
legend image post style={scale=1,xscale=0.8},
legend cell align=left,
legend style={at={(.98,.85)}, anchor=north east, font = \small, mark size=1pt},  
width=9cm,  
height=8cm,  
]
\path [draw=black, fill=black, opacity=0.2]
(axis cs:1,100.851958319365)
--(axis cs:1,104.728041680635)
--(axis cs:101,117.792932212471)
--(axis cs:201,118.278650077864)
--(axis cs:301,118.53678777884)
--(axis cs:401,118.652942927003)
--(axis cs:501,118.756955455581)
--(axis cs:601,118.802620739404)
--(axis cs:701,118.836496620256)
--(axis cs:801,118.850806297981)
--(axis cs:901,118.85638755939)
--(axis cs:1001,118.911635340094)
--(axis cs:1101,118.957918779771)
--(axis cs:1201,118.978497972108)
--(axis cs:1301,119.005875984156)
--(axis cs:1401,119.013990047689)
--(axis cs:1501,119.037185633276)
--(axis cs:1601,119.064331165146)
--(axis cs:1701,119.079727979145)
--(axis cs:1801,119.092164594882)
--(axis cs:1901,119.103675764996)
--(axis cs:2001,119.115372930546)
--(axis cs:2101,119.119248157779)
--(axis cs:2201,119.132801962185)
--(axis cs:2301,119.136857831192)
--(axis cs:2401,119.148690296921)
--(axis cs:2501,119.147988523025)
--(axis cs:2601,119.163061256496)
--(axis cs:2701,119.173146678232)
--(axis cs:2801,119.172676203755)
--(axis cs:2901,119.17095361826)
--(axis cs:3001,119.17082303439)
--(axis cs:3101,119.173763322211)
--(axis cs:3201,119.182039142965)
--(axis cs:3301,119.189762085941)
--(axis cs:3401,119.199325703717)
--(axis cs:3501,119.20369452852)
--(axis cs:3601,119.211695190527)
--(axis cs:3701,119.217778785725)
--(axis cs:3801,119.222561865041)
--(axis cs:3901,119.224556914136)
--(axis cs:4001,119.230145866039)
--(axis cs:4101,119.237123192824)
--(axis cs:4201,119.236877702838)
--(axis cs:4301,119.239031836782)
--(axis cs:4401,119.240061696388)
--(axis cs:4501,119.244814613208)
--(axis cs:4601,119.249754404703)
--(axis cs:4701,119.255148631124)
--(axis cs:4801,119.256383203056)
--(axis cs:4901,119.256096145954)
--(axis cs:5001,119.261005037568)
--(axis cs:5101,119.264978404634)
--(axis cs:5201,119.265206444669)
--(axis cs:5301,119.26629343163)
--(axis cs:5401,119.27078113444)
--(axis cs:5501,119.275357710601)
--(axis cs:5601,119.281375908598)
--(axis cs:5701,119.281893144231)
--(axis cs:5801,119.284314364611)
--(axis cs:5901,119.283553780662)
--(axis cs:6001,119.281632482807)
--(axis cs:6101,119.284132753546)
--(axis cs:6201,119.285384233799)
--(axis cs:6301,119.28745579621)
--(axis cs:6401,119.290575664327)
--(axis cs:6501,119.29344106848)
--(axis cs:6601,119.293692523898)
--(axis cs:6701,119.29749535827)
--(axis cs:6801,119.298907960358)
--(axis cs:6901,119.301821181713)
--(axis cs:7001,119.301506411106)
--(axis cs:7101,119.301080854158)
--(axis cs:7201,119.300444079074)
--(axis cs:7301,119.303182938724)
--(axis cs:7401,119.302757466539)
--(axis cs:7501,119.306507181812)
--(axis cs:7601,119.308634309786)
--(axis cs:7701,119.308691934081)
--(axis cs:7801,119.311023032091)
--(axis cs:7901,119.31180875808)
--(axis cs:8001,119.314565212823)
--(axis cs:8101,119.315589392595)
--(axis cs:8201,119.317367501136)
--(axis cs:8301,119.319101288457)
--(axis cs:8401,119.321171188447)
--(axis cs:8501,119.322215998601)
--(axis cs:8601,119.323771387806)
--(axis cs:8701,119.323557402816)
--(axis cs:8801,119.325867000788)
--(axis cs:8901,119.327508798911)
--(axis cs:9001,119.329502552)
--(axis cs:9101,119.330945366661)
--(axis cs:9201,119.329381084408)
--(axis cs:9301,119.328647897729)
--(axis cs:9401,119.32932789448)
--(axis cs:9501,119.329342819929)
--(axis cs:9601,119.328632594672)
--(axis cs:9701,119.329863816518)
--(axis cs:9801,119.330760740676)
--(axis cs:9901,119.33320243372)
--(axis cs:9901,119.293821099256)
--(axis cs:9901,119.293821099256)
--(axis cs:9801,119.290890111278)
--(axis cs:9701,119.289043512623)
--(axis cs:9601,119.286984528545)
--(axis cs:9501,119.287510142917)
--(axis cs:9401,119.287353309648)
--(axis cs:9301,119.286283830043)
--(axis cs:9201,119.286732381519)
--(axis cs:9101,119.288639294365)
--(axis cs:9001,119.287408902283)
--(axis cs:8901,119.284629163116)
--(axis cs:8801,119.283145611415)
--(axis cs:8701,119.28047201909)
--(axis cs:8601,119.280642052492)
--(axis cs:8501,119.279320291247)
--(axis cs:8401,119.27789082798)
--(axis cs:8301,119.275402988136)
--(axis cs:8201,119.272806867843)
--(axis cs:8101,119.27054318363)
--(axis cs:8001,119.270271682565)
--(axis cs:7901,119.267229338363)
--(axis cs:7801,119.266925948809)
--(axis cs:7701,119.264908896979)
--(axis cs:7601,119.264124537734)
--(axis cs:7501,119.26120912268)
--(axis cs:7401,119.25682907582)
--(axis cs:7301,119.257258096751)
--(axis cs:7201,119.254573279626)
--(axis cs:7101,119.255038002341)
--(axis cs:7001,119.253685704306)
--(axis cs:6901,119.252976673671)
--(axis cs:6801,119.249626078754)
--(axis cs:6701,119.24790383588)
--(axis cs:6601,119.244256271739)
--(axis cs:6501,119.243710139026)
--(axis cs:6401,119.240944410661)
--(axis cs:6301,119.237997306472)
--(axis cs:6201,119.236863790713)
--(axis cs:6101,119.234916582628)
--(axis cs:6001,119.232208543689)
--(axis cs:5901,119.233328103765)
--(axis cs:5801,119.233110217357)
--(axis cs:5701,119.229680263943)
--(axis cs:5601,119.227743891438)
--(axis cs:5501,119.221290171602)
--(axis cs:5401,119.216413829455)
--(axis cs:5301,119.211869179198)
--(axis cs:5201,119.21002908696)
--(axis cs:5101,119.209567762784)
--(axis cs:5001,119.205721617101)
--(axis cs:4901,119.200255618992)
--(axis cs:4801,119.200575763826)
--(axis cs:4701,119.19818470221)
--(axis cs:4601,119.192397301448)
--(axis cs:4501,119.186640618963)
--(axis cs:4401,119.181615195228)
--(axis cs:4301,119.179498737503)
--(axis cs:4201,119.176100159576)
--(axis cs:4101,119.175034817418)
--(axis cs:4001,119.167399747558)
--(axis cs:3901,119.160759671355)
--(axis cs:3801,119.157653867661)
--(axis cs:3701,119.153137182932)
--(axis cs:3601,119.144927969707)
--(axis cs:3501,119.137048116438)
--(axis cs:3401,119.13123001519)
--(axis cs:3301,119.121252758046)
--(axis cs:3201,119.113387286276)
--(axis cs:3101,119.105056413359)
--(axis cs:3001,119.101099658046)
--(axis cs:2901,119.099373855025)
--(axis cs:2801,119.097777205742)
--(axis cs:2701,119.096642288817)
--(axis cs:2601,119.08482032751)
--(axis cs:2501,119.068684807643)
--(axis cs:2401,119.066794917573)
--(axis cs:2301,119.050573720307)
--(axis cs:2201,119.046034930136)
--(axis cs:2101,119.026463408142)
--(axis cs:2001,119.019939413282)
--(axis cs:1901,119.004951273405)
--(axis cs:1801,118.990933683852)
--(axis cs:1701,118.978731750426)
--(axis cs:1601,118.963251595629)
--(axis cs:1501,118.935912301434)
--(axis cs:1401,118.912019945173)
--(axis cs:1301,118.896891118073)
--(axis cs:1201,118.864999113653)
--(axis cs:1101,118.838720638939)
--(axis cs:1001,118.785027996569)
--(axis cs:901,118.72694207435)
--(axis cs:801,118.717083839347)
--(axis cs:701,118.677996960343)
--(axis cs:601,118.638111373741)
--(axis cs:501,118.571866899709)
--(axis cs:401,118.446957322373)
--(axis cs:301,118.294707237771)
--(axis cs:201,118.001548927111)
--(axis cs:101,117.381523233074)
--(axis cs:1,100.851958319365)
--cycle;

\path [draw=steelblue31119180, fill=steelblue31119180, opacity=0.2]
(axis cs:1,98.3400988554737)
--(axis cs:1,101.719901144526)
--(axis cs:101,114.090149063396)
--(axis cs:201,116.221358144659)
--(axis cs:301,117.095616147498)
--(axis cs:401,117.550778144447)
--(axis cs:501,117.830968386528)
--(axis cs:601,118.037658764489)
--(axis cs:701,118.191893116812)
--(axis cs:801,118.295688514804)
--(axis cs:901,118.411771813528)
--(axis cs:1001,118.485335244889)
--(axis cs:1101,118.568650311983)
--(axis cs:1201,118.640925473501)
--(axis cs:1301,118.678603195465)
--(axis cs:1401,118.725178304938)
--(axis cs:1501,118.776400917293)
--(axis cs:1601,118.811008219709)
--(axis cs:1701,118.843510269966)
--(axis cs:1801,118.874488930084)
--(axis cs:1901,118.90273088212)
--(axis cs:2001,118.923738375678)
--(axis cs:2101,118.954523880402)
--(axis cs:2201,118.975933179717)
--(axis cs:2301,118.991963274666)
--(axis cs:2401,119.005643082104)
--(axis cs:2501,119.026284921302)
--(axis cs:2601,119.034479607471)
--(axis cs:2701,119.03821818587)
--(axis cs:2801,119.036804005356)
--(axis cs:2901,119.049959897379)
--(axis cs:3001,119.057173299341)
--(axis cs:3101,119.070983256895)
--(axis cs:3201,119.072184945484)
--(axis cs:3301,119.078978730586)
--(axis cs:3401,119.090199331304)
--(axis cs:3501,119.092388748042)
--(axis cs:3601,119.099382334812)
--(axis cs:3701,119.111578985737)
--(axis cs:3801,119.119288287541)
--(axis cs:3901,119.125288557244)
--(axis cs:4001,119.137800558073)
--(axis cs:4101,119.145557804061)
--(axis cs:4201,119.146327472008)
--(axis cs:4301,119.151348377336)
--(axis cs:4401,119.158893002702)
--(axis cs:4501,119.161490106191)
--(axis cs:4601,119.163952175298)
--(axis cs:4701,119.168719957629)
--(axis cs:4801,119.174788631587)
--(axis cs:4901,119.179213725772)
--(axis cs:5001,119.183818650704)
--(axis cs:5101,119.186471181879)
--(axis cs:5201,119.189481599579)
--(axis cs:5301,119.189117116509)
--(axis cs:5401,119.192525555745)
--(axis cs:5501,119.198220127479)
--(axis cs:5601,119.203732280957)
--(axis cs:5701,119.208278829171)
--(axis cs:5801,119.212145645948)
--(axis cs:5901,119.216459596634)
--(axis cs:6001,119.215580426654)
--(axis cs:6101,119.217896859982)
--(axis cs:6201,119.223366979828)
--(axis cs:6301,119.226405435732)
--(axis cs:6401,119.229240900338)
--(axis cs:6501,119.230969048408)
--(axis cs:6601,119.233595233077)
--(axis cs:6701,119.236319144172)
--(axis cs:6801,119.239301836995)
--(axis cs:6901,119.24269854086)
--(axis cs:7001,119.243721526389)
--(axis cs:7101,119.242625378439)
--(axis cs:7201,119.244751505053)
--(axis cs:7301,119.247873313193)
--(axis cs:7401,119.249085627959)
--(axis cs:7501,119.252085037691)
--(axis cs:7601,119.253754539059)
--(axis cs:7701,119.257887836791)
--(axis cs:7801,119.259820648572)
--(axis cs:7901,119.263530482683)
--(axis cs:8001,119.265964085304)
--(axis cs:8101,119.268315562553)
--(axis cs:8201,119.268603664999)
--(axis cs:8301,119.271775510444)
--(axis cs:8401,119.274905836959)
--(axis cs:8501,119.275252321303)
--(axis cs:8601,119.275935056058)
--(axis cs:8701,119.27657151567)
--(axis cs:8801,119.275871471098)
--(axis cs:8901,119.276628101269)
--(axis cs:9001,119.278377635578)
--(axis cs:9101,119.278406109962)
--(axis cs:9201,119.280433428766)
--(axis cs:9301,119.281239441107)
--(axis cs:9401,119.280950996516)
--(axis cs:9501,119.281259457279)
--(axis cs:9601,119.28494590802)
--(axis cs:9701,119.285735486526)
--(axis cs:9801,119.288262779848)
--(axis cs:9901,119.29039271346)
--(axis cs:9901,119.249019467128)
--(axis cs:9901,119.249019467128)
--(axis cs:9801,119.246478573076)
--(axis cs:9701,119.243867647172)
--(axis cs:9601,119.242878276961)
--(axis cs:9501,119.238782643553)
--(axis cs:9401,119.236749248139)
--(axis cs:9301,119.237132776934)
--(axis cs:9201,119.235060539281)
--(axis cs:9101,119.23221909606)
--(axis cs:9001,119.232350061345)
--(axis cs:8901,119.230953069386)
--(axis cs:8801,119.230139209507)
--(axis cs:8701,119.23062076108)
--(axis cs:8601,119.230006113573)
--(axis cs:8501,119.230415247218)
--(axis cs:8401,119.229741228866)
--(axis cs:8301,119.22735471483)
--(axis cs:8201,119.223087592165)
--(axis cs:8101,119.222334974418)
--(axis cs:8001,119.218960299148)
--(axis cs:7901,119.216449266716)
--(axis cs:7801,119.212623909818)
--(axis cs:7701,119.21135511867)
--(axis cs:7601,119.207124292674)
--(axis cs:7501,119.20533930573)
--(axis cs:7401,119.20094815126)
--(axis cs:7301,119.198632644895)
--(axis cs:7201,119.195167950578)
--(axis cs:7101,119.192420389763)
--(axis cs:7001,119.192536151086)
--(axis cs:6901,119.191279143533)
--(axis cs:6801,119.187100162711)
--(axis cs:6701,119.184322551098)
--(axis cs:6601,119.182026642396)
--(axis cs:6501,119.180755301692)
--(axis cs:6401,119.178307920158)
--(axis cs:6301,119.174686454444)
--(axis cs:6201,119.170762999208)
--(axis cs:6101,119.164928906286)
--(axis cs:6001,119.161716690493)
--(axis cs:5901,119.162171143918)
--(axis cs:5801,119.157976746743)
--(axis cs:5701,119.153840097333)
--(axis cs:5601,119.148469111651)
--(axis cs:5501,119.141939843435)
--(axis cs:5401,119.136850485729)
--(axis cs:5301,119.133886090433)
--(axis cs:5201,119.132518015879)
--(axis cs:5101,119.129212017493)
--(axis cs:5001,119.125747436079)
--(axis cs:4901,119.121790150988)
--(axis cs:4801,119.117004744792)
--(axis cs:4701,119.110233456539)
--(axis cs:4601,119.10635427982)
--(axis cs:4501,119.103317714294)
--(axis cs:4401,119.09876207569)
--(axis cs:4301,119.089656040241)
--(axis cs:4201,119.083165505854)
--(axis cs:4101,119.083810642659)
--(axis cs:4001,119.075121211484)
--(axis cs:3901,119.063898830606)
--(axis cs:3801,119.055176327035)
--(axis cs:3701,119.047496939689)
--(axis cs:3601,119.035635715729)
--(axis cs:3501,119.029245070867)
--(axis cs:3401,119.025196140616)
--(axis cs:3301,119.012599579017)
--(axis cs:3201,119.003516397847)
--(axis cs:3101,118.999703618307)
--(axis cs:3001,118.986798709989)
--(axis cs:2901,118.97990566622)
--(axis cs:2801,118.96667332417)
--(axis cs:2701,118.96595064049)
--(axis cs:2601,118.959445805831)
--(axis cs:2501,118.949132911565)
--(axis cs:2401,118.927576409774)
--(axis cs:2301,118.914633857016)
--(axis cs:2201,118.893435289161)
--(axis cs:2101,118.867989208603)
--(axis cs:2001,118.836631439415)
--(axis cs:1901,118.814049759647)
--(axis cs:1801,118.783067982742)
--(axis cs:1701,118.74950560305)
--(axis cs:1601,118.718023635381)
--(axis cs:1501,118.680361241268)
--(axis cs:1401,118.620760310337)
--(axis cs:1301,118.5770770505)
--(axis cs:1201,118.530614909513)
--(axis cs:1101,118.454837426436)
--(axis cs:1001,118.359439979887)
--(axis cs:901,118.28183528969)
--(axis cs:801,118.163387640003)
--(axis cs:701,118.045624714857)
--(axis cs:601,117.872224762965)
--(axis cs:501,117.654580515667)
--(axis cs:401,117.365531082486)
--(axis cs:301,116.857141327585)
--(axis cs:201,115.922323447381)
--(axis cs:101,113.703910342545)
--(axis cs:1,98.3400988554737)
--cycle;

\path [draw=mediumpurple148103189, fill=mediumpurple148103189, opacity=0.2]
(axis cs:1,103.367058894265)
--(axis cs:1,107.032941105735)
--(axis cs:101,114.430311293808)
--(axis cs:201,114.44758403244)
--(axis cs:301,114.487455363597)
--(axis cs:401,114.554570897522)
--(axis cs:501,114.547892410918)
--(axis cs:601,114.5851500585)
--(axis cs:701,114.595321508981)
--(axis cs:801,114.597533155492)
--(axis cs:901,114.58078718498)
--(axis cs:1001,114.561697863037)
--(axis cs:1101,114.558802667751)
--(axis cs:1201,114.560586998697)
--(axis cs:1301,114.555736917183)
--(axis cs:1401,114.552241341029)
--(axis cs:1501,114.547361252212)
--(axis cs:1601,114.55705476167)
--(axis cs:1701,114.554444449227)
--(axis cs:1801,114.555386338488)
--(axis cs:1901,114.553175518344)
--(axis cs:2001,114.555588240705)
--(axis cs:2101,114.559085871818)
--(axis cs:2201,114.562097054699)
--(axis cs:2301,114.556813127004)
--(axis cs:2401,114.558671491499)
--(axis cs:2501,114.552898631915)
--(axis cs:2601,114.554301216555)
--(axis cs:2701,114.549689995857)
--(axis cs:2801,114.554202487368)
--(axis cs:2901,114.549906935466)
--(axis cs:3001,114.553392504385)
--(axis cs:3101,114.551643104503)
--(axis cs:3201,114.549603515487)
--(axis cs:3301,114.543882916076)
--(axis cs:3401,114.543129858996)
--(axis cs:3501,114.541531118926)
--(axis cs:3601,114.534126104142)
--(axis cs:3701,114.535946754236)
--(axis cs:3801,114.537472006864)
--(axis cs:3901,114.535976688268)
--(axis cs:4001,114.536002566937)
--(axis cs:4101,114.533176708446)
--(axis cs:4201,114.528839539783)
--(axis cs:4301,114.529153020975)
--(axis cs:4401,114.527407651176)
--(axis cs:4501,114.521626723104)
--(axis cs:4601,114.519341264671)
--(axis cs:4701,114.516052216462)
--(axis cs:4801,114.517419999661)
--(axis cs:4901,114.518604389268)
--(axis cs:5001,114.518679883944)
--(axis cs:5101,114.519916463555)
--(axis cs:5201,114.519207377244)
--(axis cs:5301,114.51723717504)
--(axis cs:5401,114.520231283448)
--(axis cs:5501,114.524039447214)
--(axis cs:5601,114.524884347738)
--(axis cs:5701,114.523403591175)
--(axis cs:5801,114.521716011436)
--(axis cs:5901,114.521397879553)
--(axis cs:6001,114.520819421466)
--(axis cs:6101,114.521966069634)
--(axis cs:6201,114.519291355822)
--(axis cs:6301,114.519180979146)
--(axis cs:6401,114.520854823352)
--(axis cs:6501,114.520251403287)
--(axis cs:6601,114.519577998804)
--(axis cs:6701,114.520537826449)
--(axis cs:6801,114.524360677206)
--(axis cs:6901,114.52720442396)
--(axis cs:7001,114.527958300583)
--(axis cs:7101,114.529312201339)
--(axis cs:7201,114.530009466495)
--(axis cs:7301,114.531445756596)
--(axis cs:7401,114.530990104928)
--(axis cs:7501,114.533100386669)
--(axis cs:7601,114.534201048005)
--(axis cs:7701,114.533538947776)
--(axis cs:7801,114.532239427524)
--(axis cs:7901,114.531500042543)
--(axis cs:8001,114.530801511605)
--(axis cs:8101,114.530478767734)
--(axis cs:8201,114.530842030574)
--(axis cs:8301,114.531019412247)
--(axis cs:8401,114.529861856159)
--(axis cs:8501,114.528929070076)
--(axis cs:8601,114.528751290016)
--(axis cs:8701,114.52873482506)
--(axis cs:8801,114.529690983754)
--(axis cs:8901,114.531455124012)
--(axis cs:9001,114.529841746371)
--(axis cs:9101,114.529902158522)
--(axis cs:9201,114.529922257571)
--(axis cs:9301,114.529407738269)
--(axis cs:9401,114.528896694089)
--(axis cs:9501,114.528421280851)
--(axis cs:9601,114.527894779657)
--(axis cs:9701,114.525889280194)
--(axis cs:9801,114.523534260502)
--(axis cs:9901,114.522371065734)
--(axis cs:9901,114.478873253021)
--(axis cs:9901,114.478873253021)
--(axis cs:9801,114.479736834284)
--(axis cs:9701,114.482346984108)
--(axis cs:9601,114.484758069004)
--(axis cs:9501,114.484846796193)
--(axis cs:9401,114.4844359301)
--(axis cs:9301,114.485126182814)
--(axis cs:9201,114.485530410617)
--(axis cs:9101,114.484905005526)
--(axis cs:9001,114.484274462939)
--(axis cs:8901,114.485466570179)
--(axis cs:8801,114.483139376432)
--(axis cs:8701,114.482171967263)
--(axis cs:8601,114.482605528959)
--(axis cs:8501,114.482015524678)
--(axis cs:8401,114.482591423212)
--(axis cs:8301,114.483080093836)
--(axis cs:8201,114.482185648978)
--(axis cs:8101,114.48089019906)
--(axis cs:8001,114.481789414529)
--(axis cs:7901,114.482546280707)
--(axis cs:7801,114.48258944057)
--(axis cs:7701,114.484547015086)
--(axis cs:7601,114.485456891739)
--(axis cs:7501,114.484611918357)
--(axis cs:7401,114.482202706854)
--(axis cs:7301,114.482650942487)
--(axis cs:7201,114.480602948448)
--(axis cs:7101,114.480193502083)
--(axis cs:7001,114.479172109358)
--(axis cs:6901,114.477809342161)
--(axis cs:6801,114.47559227089)
--(axis cs:6701,114.472355771522)
--(axis cs:6601,114.471714229647)
--(axis cs:6501,114.47220206541)
--(axis cs:6401,114.472240005581)
--(axis cs:6301,114.469874726298)
--(axis cs:6201,114.469497549193)
--(axis cs:6101,114.471811999536)
--(axis cs:6001,114.469415539374)
--(axis cs:5901,114.467580259746)
--(axis cs:5801,114.46765133902)
--(axis cs:5701,114.468531157115)
--(axis cs:5601,114.469991567277)
--(axis cs:5501,114.469063624955)
--(axis cs:5401,114.46547136421)
--(axis cs:5301,114.461638508793)
--(axis cs:5201,114.463069108047)
--(axis cs:5101,114.463784771497)
--(axis cs:5001,114.461764027274)
--(axis cs:4901,114.461473145929)
--(axis cs:4801,114.460076355265)
--(axis cs:4701,114.458531914574)
--(axis cs:4601,114.46132815502)
--(axis cs:4501,114.462456813888)
--(axis cs:4401,114.466716411537)
--(axis cs:4301,114.467336167586)
--(axis cs:4201,114.46654727288)
--(axis cs:4101,114.469334874095)
--(axis cs:4001,114.470730749733)
--(axis cs:3901,114.470416544237)
--(axis cs:3801,114.470478532468)
--(axis cs:3701,114.46847907662)
--(axis cs:3601,114.464763093304)
--(axis cs:3501,114.469528578303)
--(axis cs:3401,114.468484372113)
--(axis cs:3301,114.467198574381)
--(axis cs:3201,114.472402107755)
--(axis cs:3101,114.473768053188)
--(axis cs:3001,114.474564843166)
--(axis cs:2901,114.468479827719)
--(axis cs:2801,114.473309115631)
--(axis cs:2701,114.466304080411)
--(axis cs:2601,114.468559221738)
--(axis cs:2501,114.465286094195)
--(axis cs:2401,114.469083610542)
--(axis cs:2301,114.466428941661)
--(axis cs:2201,114.470469960295)
--(axis cs:2101,114.468539068687)
--(axis cs:2001,114.46399196919)
--(axis cs:1901,114.46060670154)
--(axis cs:1801,114.453020102378)
--(axis cs:1701,114.4517048747)
--(axis cs:1601,114.45133999161)
--(axis cs:1501,114.440393577901)
--(axis cs:1401,114.441206196444)
--(axis cs:1301,114.434762698498)
--(axis cs:1201,114.431136565)
--(axis cs:1101,114.421687795464)
--(axis cs:1001,114.414925513586)
--(axis cs:901,114.427226133555)
--(axis cs:801,114.436199678466)
--(axis cs:701,114.42590530985)
--(axis cs:601,114.405066247656)
--(axis cs:501,114.360969864532)
--(axis cs:401,114.352012643625)
--(axis cs:301,114.256531347366)
--(axis cs:201,114.16505278348)
--(axis cs:101,114.05840157748)
--(axis cs:1,103.367058894265)
--cycle;

\path [draw=orchid227119194, fill=orchid227119194, opacity=0.2]
(axis cs:1,104.891947483345)
--(axis cs:1,108.308052516655)
--(axis cs:101,112.789587490722)
--(axis cs:201,112.723258808189)
--(axis cs:301,112.697645711131)
--(axis cs:401,112.687835122221)
--(axis cs:501,112.659543373765)
--(axis cs:601,112.667728828086)
--(axis cs:701,112.648893770017)
--(axis cs:801,112.658891862035)
--(axis cs:901,112.654777789784)
--(axis cs:1001,112.675130173764)
--(axis cs:1101,112.678417453715)
--(axis cs:1201,112.671180833547)
--(axis cs:1301,112.665512325456)
--(axis cs:1401,112.65159418842)
--(axis cs:1501,112.657346008525)
--(axis cs:1601,112.656308331822)
--(axis cs:1701,112.64632868738)
--(axis cs:1801,112.648037972223)
--(axis cs:1901,112.645679893903)
--(axis cs:2001,112.638652371722)
--(axis cs:2101,112.633982799782)
--(axis cs:2201,112.635304424046)
--(axis cs:2301,112.630690143569)
--(axis cs:2401,112.638074550363)
--(axis cs:2501,112.640423903049)
--(axis cs:2601,112.637798414001)
--(axis cs:2701,112.640360917024)
--(axis cs:2801,112.645665171403)
--(axis cs:2901,112.644889665905)
--(axis cs:3001,112.651001866937)
--(axis cs:3101,112.651983497272)
--(axis cs:3201,112.647063023289)
--(axis cs:3301,112.649605769175)
--(axis cs:3401,112.650623743811)
--(axis cs:3501,112.648534176959)
--(axis cs:3601,112.651903310368)
--(axis cs:3701,112.656078037301)
--(axis cs:3801,112.659629523521)
--(axis cs:3901,112.658362165647)
--(axis cs:4001,112.656324207605)
--(axis cs:4101,112.658462757882)
--(axis cs:4201,112.654336914697)
--(axis cs:4301,112.651496493759)
--(axis cs:4401,112.650748969592)
--(axis cs:4501,112.650966124733)
--(axis cs:4601,112.650814133433)
--(axis cs:4701,112.651609732849)
--(axis cs:4801,112.652934584678)
--(axis cs:4901,112.654872952575)
--(axis cs:5001,112.656287657557)
--(axis cs:5101,112.657902465844)
--(axis cs:5201,112.655810907953)
--(axis cs:5301,112.655614508331)
--(axis cs:5401,112.650946269361)
--(axis cs:5501,112.652479754675)
--(axis cs:5601,112.650779433219)
--(axis cs:5701,112.650590822627)
--(axis cs:5801,112.648769564599)
--(axis cs:5901,112.648880532837)
--(axis cs:6001,112.647467393824)
--(axis cs:6101,112.646813337955)
--(axis cs:6201,112.646150013872)
--(axis cs:6301,112.645995793679)
--(axis cs:6401,112.646883207512)
--(axis cs:6501,112.645679150355)
--(axis cs:6601,112.644526340008)
--(axis cs:6701,112.645453350178)
--(axis cs:6801,112.642017886526)
--(axis cs:6901,112.641298786218)
--(axis cs:7001,112.642760497786)
--(axis cs:7101,112.645777012093)
--(axis cs:7201,112.645890670963)
--(axis cs:7301,112.645388937605)
--(axis cs:7401,112.642709325705)
--(axis cs:7501,112.642060568594)
--(axis cs:7601,112.640194828726)
--(axis cs:7701,112.641960895311)
--(axis cs:7801,112.641901217334)
--(axis cs:7901,112.643818137652)
--(axis cs:8001,112.645446821764)
--(axis cs:8101,112.645809298632)
--(axis cs:8201,112.644871249143)
--(axis cs:8301,112.644205898987)
--(axis cs:8401,112.645537639671)
--(axis cs:8501,112.646288654426)
--(axis cs:8601,112.647280170478)
--(axis cs:8701,112.646822022621)
--(axis cs:8801,112.646679140787)
--(axis cs:8901,112.647852873152)
--(axis cs:9001,112.645475443514)
--(axis cs:9101,112.645708281382)
--(axis cs:9201,112.645385026362)
--(axis cs:9301,112.644622195316)
--(axis cs:9401,112.643652277611)
--(axis cs:9501,112.645380900014)
--(axis cs:9601,112.64471840449)
--(axis cs:9701,112.644383820456)
--(axis cs:9801,112.644046716579)
--(axis cs:9901,112.64293071996)
--(axis cs:9901,112.603622153487)
--(axis cs:9901,112.603622153487)
--(axis cs:9801,112.604634030283)
--(axis cs:9701,112.605258484461)
--(axis cs:9601,112.605530527913)
--(axis cs:9501,112.605112732235)
--(axis cs:9401,112.60336825212)
--(axis cs:9301,112.604305876934)
--(axis cs:9201,112.604605192092)
--(axis cs:9101,112.604103827177)
--(axis cs:9001,112.603483561041)
--(axis cs:8901,112.605334409176)
--(axis cs:8801,112.602926585835)
--(axis cs:8701,112.60201604197)
--(axis cs:8601,112.602181520024)
--(axis cs:8501,112.600882266642)
--(axis cs:8401,112.600287857293)
--(axis cs:8301,112.599029855742)
--(axis cs:8201,112.59952577561)
--(axis cs:8101,112.600454125636)
--(axis cs:8001,112.59996250207)
--(axis cs:7901,112.599054916391)
--(axis cs:7801,112.596819459502)
--(axis cs:7701,112.59611987342)
--(axis cs:7601,112.594655848816)
--(axis cs:7501,112.596374306756)
--(axis cs:7401,112.596612387577)
--(axis cs:7301,112.599120033769)
--(axis cs:7201,112.598839227662)
--(axis cs:7101,112.598861771177)
--(axis cs:7001,112.59578256749)
--(axis cs:6901,112.594356915854)
--(axis cs:6801,112.593773908799)
--(axis cs:6701,112.59785660356)
--(axis cs:6601,112.596806791335)
--(axis cs:6501,112.597344999776)
--(axis cs:6401,112.598959629545)
--(axis cs:6301,112.597886129826)
--(axis cs:6201,112.597355872275)
--(axis cs:6101,112.59725812574)
--(axis cs:6001,112.597298478531)
--(axis cs:5901,112.599202842862)
--(axis cs:5801,112.599118730522)
--(axis cs:5701,112.598705792002)
--(axis cs:5601,112.596826351463)
--(axis cs:5501,112.598165582536)
--(axis cs:5401,112.596663432546)
--(axis cs:5301,112.600593754261)
--(axis cs:5201,112.60018986113)
--(axis cs:5101,112.603175754113)
--(axis cs:5001,112.601412802351)
--(axis cs:4901,112.600519824409)
--(axis cs:4801,112.598133942712)
--(axis cs:4701,112.59722243052)
--(axis cs:4601,112.596049591844)
--(axis cs:4501,112.596632186754)
--(axis cs:4401,112.597267390326)
--(axis cs:4301,112.597743217935)
--(axis cs:4201,112.598978962475)
--(axis cs:4101,112.603936656895)
--(axis cs:4001,112.600026704667)
--(axis cs:3901,112.601453266294)
--(axis cs:3801,112.602085814548)
--(axis cs:3701,112.598188377182)
--(axis cs:3601,112.595794551337)
--(axis cs:3501,112.591728605103)
--(axis cs:3401,112.592869346457)
--(axis cs:3301,112.592181567996)
--(axis cs:3201,112.588969466558)
--(axis cs:3101,112.590576967094)
--(axis cs:3001,112.590610928798)
--(axis cs:2901,112.5840038191)
--(axis cs:2801,112.583017441949)
--(axis cs:2701,112.575655373239)
--(axis cs:2601,112.573558756318)
--(axis cs:2501,112.573426556767)
--(axis cs:2401,112.567939610403)
--(axis cs:2301,112.555950447478)
--(axis cs:2201,112.558707388766)
--(axis cs:2101,112.553546947957)
--(axis cs:2001,112.554940831676)
--(axis cs:1901,112.561453193945)
--(axis cs:1801,112.561334598571)
--(axis cs:1701,112.556540213267)
--(axis cs:1601,112.561268182856)
--(axis cs:1501,112.5581103539)
--(axis cs:1401,112.552003241987)
--(axis cs:1301,112.559022647642)
--(axis cs:1201,112.557495269701)
--(axis cs:1101,112.560256479074)
--(axis cs:1001,112.552382313749)
--(axis cs:901,112.526531866154)
--(axis cs:801,112.52048391824)
--(axis cs:701,112.501776700739)
--(axis cs:601,112.51323623015)
--(axis cs:501,112.501414710067)
--(axis cs:401,112.514159890248)
--(axis cs:301,112.505012096178)
--(axis cs:201,112.485596913204)
--(axis cs:101,112.445858053832)
--(axis cs:1,104.891947483345)
--cycle;

\path [draw=goldenrod18818934, fill=goldenrod18818934, opacity=0.2]
(axis cs:1,105.175540247608)
--(axis cs:1,108.824459752392)
--(axis cs:101,112.962364770684)
--(axis cs:201,112.884634181169)
--(axis cs:301,112.919409225678)
--(axis cs:401,112.852650934891)
--(axis cs:501,112.843286301699)
--(axis cs:601,112.854151978945)
--(axis cs:701,112.842244841152)
--(axis cs:801,112.843641982766)
--(axis cs:901,112.835413032244)
--(axis cs:1001,112.828401411466)
--(axis cs:1101,112.821465506808)
--(axis cs:1201,112.822365020064)
--(axis cs:1301,112.827316408074)
--(axis cs:1401,112.843342054304)
--(axis cs:1501,112.839788242401)
--(axis cs:1601,112.84775591489)
--(axis cs:1701,112.854067273346)
--(axis cs:1801,112.852768404234)
--(axis cs:1901,112.850338846566)
--(axis cs:2001,112.842481900082)
--(axis cs:2101,112.838327969382)
--(axis cs:2201,112.837476690611)
--(axis cs:2301,112.841748223876)
--(axis cs:2401,112.842361107371)
--(axis cs:2501,112.846135913425)
--(axis cs:2601,112.845448375189)
--(axis cs:2701,112.841478918279)
--(axis cs:2801,112.840272466728)
--(axis cs:2901,112.845803219894)
--(axis cs:3001,112.849045649041)
--(axis cs:3101,112.852886835942)
--(axis cs:3201,112.851772203007)
--(axis cs:3301,112.849167873039)
--(axis cs:3401,112.847963437355)
--(axis cs:3501,112.848992635776)
--(axis cs:3601,112.850604729016)
--(axis cs:3701,112.854088562537)
--(axis cs:3801,112.852666278781)
--(axis cs:3901,112.849506476522)
--(axis cs:4001,112.845689284204)
--(axis cs:4101,112.844430915361)
--(axis cs:4201,112.84157795556)
--(axis cs:4301,112.841539526269)
--(axis cs:4401,112.843999775048)
--(axis cs:4501,112.841834398116)
--(axis cs:4601,112.841548478311)
--(axis cs:4701,112.844459031974)
--(axis cs:4801,112.843235561475)
--(axis cs:4901,112.841550657317)
--(axis cs:5001,112.841490371827)
--(axis cs:5101,112.843430364522)
--(axis cs:5201,112.842868893327)
--(axis cs:5301,112.840436265015)
--(axis cs:5401,112.839760107794)
--(axis cs:5501,112.842169727342)
--(axis cs:5601,112.841108681807)
--(axis cs:5701,112.840816456829)
--(axis cs:5801,112.842733212293)
--(axis cs:5901,112.843529781397)
--(axis cs:6001,112.844547132132)
--(axis cs:6101,112.844675129181)
--(axis cs:6201,112.84495085196)
--(axis cs:6301,112.845200524331)
--(axis cs:6401,112.844670406866)
--(axis cs:6501,112.843441423238)
--(axis cs:6601,112.843289388744)
--(axis cs:6701,112.841859324482)
--(axis cs:6801,112.838625385551)
--(axis cs:6901,112.83945027881)
--(axis cs:7001,112.83903186424)
--(axis cs:7101,112.841665001307)
--(axis cs:7201,112.841599845356)
--(axis cs:7301,112.842644106747)
--(axis cs:7401,112.842248173127)
--(axis cs:7501,112.842093156635)
--(axis cs:7601,112.841806409408)
--(axis cs:7701,112.840182525708)
--(axis cs:7801,112.840757133288)
--(axis cs:7901,112.839503878906)
--(axis cs:8001,112.839136817154)
--(axis cs:8101,112.840691553582)
--(axis cs:8201,112.840378179932)
--(axis cs:8301,112.841268409304)
--(axis cs:8401,112.84034189741)
--(axis cs:8501,112.83931993758)
--(axis cs:8601,112.838485690242)
--(axis cs:8701,112.838249242572)
--(axis cs:8801,112.837993859813)
--(axis cs:8901,112.838535252661)
--(axis cs:9001,112.838023496023)
--(axis cs:9101,112.836053411069)
--(axis cs:9201,112.834016853167)
--(axis cs:9301,112.833931322961)
--(axis cs:9401,112.833494245423)
--(axis cs:9501,112.8332627646)
--(axis cs:9601,112.832824352588)
--(axis cs:9701,112.831691184207)
--(axis cs:9801,112.831512705882)
--(axis cs:9901,112.831143169264)
--(axis cs:9901,112.793779565813)
--(axis cs:9901,112.793779565813)
--(axis cs:9801,112.793490865182)
--(axis cs:9701,112.793951533039)
--(axis cs:9601,112.794981084346)
--(axis cs:9501,112.794927952167)
--(axis cs:9401,112.794681480563)
--(axis cs:9301,112.794422617476)
--(axis cs:9201,112.794445270515)
--(axis cs:9101,112.796347424005)
--(axis cs:9001,112.797581436762)
--(axis cs:8901,112.797521370191)
--(axis cs:8801,112.796138625132)
--(axis cs:8701,112.796250240246)
--(axis cs:8601,112.795752189074)
--(axis cs:8501,112.795643007956)
--(axis cs:8401,112.796265649311)
--(axis cs:8301,112.797334168699)
--(axis cs:8201,112.796173460112)
--(axis cs:8101,112.796713705028)
--(axis cs:8001,112.794596466185)
--(axis cs:7901,112.795084147926)
--(axis cs:7801,112.796197103348)
--(axis cs:7701,112.795186906833)
--(axis cs:7601,112.7950913672)
--(axis cs:7501,112.794693938419)
--(axis cs:7401,112.793844246816)
--(axis cs:7301,112.795293162121)
--(axis cs:7201,112.794045203943)
--(axis cs:7101,112.793794793088)
--(axis cs:7001,112.791655180468)
--(axis cs:6901,112.791307582369)
--(axis cs:6801,112.790470335667)
--(axis cs:6701,112.792532557327)
--(axis cs:6601,112.793226290699)
--(axis cs:6501,112.793442133138)
--(axis cs:6401,112.795179616568)
--(axis cs:6301,112.794780431073)
--(axis cs:6201,112.793962226576)
--(axis cs:6101,112.79433486918)
--(axis cs:6001,112.793479863369)
--(axis cs:5901,112.792714922891)
--(axis cs:5801,112.792374527752)
--(axis cs:5701,112.790223711563)
--(axis cs:5601,112.78985721714)
--(axis cs:5501,112.791035144499)
--(axis cs:5401,112.788160647622)
--(axis cs:5301,112.789248700086)
--(axis cs:5201,112.790482385274)
--(axis cs:5101,112.789139719775)
--(axis cs:5001,112.786875954908)
--(axis cs:4901,112.785998822381)
--(axis cs:4801,112.786716531839)
--(axis cs:4701,112.787576705103)
--(axis cs:4601,112.785045740337)
--(axis cs:4501,112.785399549895)
--(axis cs:4401,112.787424901162)
--(axis cs:4301,112.785031038716)
--(axis cs:4201,112.783915974456)
--(axis cs:4101,112.787019950282)
--(axis cs:4001,112.787677374131)
--(axis cs:3901,112.791524028477)
--(axis cs:3801,112.793784655184)
--(axis cs:3701,112.794276744136)
--(axis cs:3601,112.788717681426)
--(axis cs:3501,112.785437527034)
--(axis cs:3401,112.783721361233)
--(axis cs:3301,112.78289513817)
--(axis cs:3201,112.784972564254)
--(axis cs:3101,112.785352441711)
--(axis cs:3001,112.782317229999)
--(axis cs:2901,112.777802433329)
--(axis cs:2801,112.772715751766)
--(axis cs:2701,112.772856513042)
--(axis cs:2601,112.776620060028)
--(axis cs:2501,112.775239536395)
--(axis cs:2401,112.770491870555)
--(axis cs:2301,112.76508358838)
--(axis cs:2201,112.757670969543)
--(axis cs:2101,112.754979979214)
--(axis cs:2001,112.757398159888)
--(axis cs:1901,112.761739007195)
--(axis cs:1801,112.76117940254)
--(axis cs:1701,112.759806918306)
--(axis cs:1601,112.749708170056)
--(axis cs:1501,112.737533543075)
--(axis cs:1401,112.739356018501)
--(axis cs:1301,112.719985667252)
--(axis cs:1201,112.710008002417)
--(axis cs:1101,112.701277454137)
--(axis cs:1001,112.701169018105)
--(axis cs:901,112.699947678078)
--(axis cs:801,112.698680114613)
--(axis cs:701,112.687141749433)
--(axis cs:601,112.684683295597)
--(axis cs:501,112.658510105487)
--(axis cs:401,112.644356546405)
--(axis cs:301,112.682916355717)
--(axis cs:201,112.59745537107)
--(axis cs:101,112.593872853078)
--(axis cs:1,105.175540247608)
--cycle;

\addplot [very thick, black]
table {%
1 102.79
101 117.587227722772
201 118.140099502488
301 118.415747508306
401 118.549950124688
501 118.664411177645
601 118.720366056572
701 118.7572467903
801 118.783945068664
901 118.79166481687
1001 118.848331668332
1101 118.898319709355
1201 118.921748542881
1301 118.951383551115
1401 118.963004996431
1501 118.986548967355
1601 119.013791380387
1701 119.029229864785
1801 119.041549139367
1901 119.0543135192
2001 119.067656171914
2101 119.07285578296
2201 119.089418446161
2301 119.09371577575
2401 119.107742607247
2501 119.108336665334
2601 119.123940792003
2701 119.134894483525
2801 119.135226704748
2901 119.135163736643
3001 119.135961346218
3101 119.139409867785
3201 119.147713214621
3301 119.155507421993
3401 119.165277859453
3501 119.170371322479
3601 119.178311580117
3701 119.185457984329
3801 119.190107866351
3901 119.192658292745
4001 119.198772806798
4101 119.206079005121
4201 119.206488931207
4301 119.209265287142
4401 119.210838445808
4501 119.215727616085
4601 119.221075853075
4701 119.226666666667
4801 119.228479483441
4901 119.228175882473
5001 119.233363327335
5101 119.237273083709
5201 119.237617765814
5301 119.239081305414
5401 119.243597481948
5501 119.248323941102
5601 119.254559900018
5701 119.255786704087
5801 119.258712290984
5901 119.258440942213
6001 119.256920513248
6101 119.259524668087
6201 119.261124012256
6301 119.262726551341
6401 119.265760037494
6501 119.268575603753
6601 119.268974397819
6701 119.272699597075
6801 119.274267019556
6901 119.277398927692
7001 119.277596057706
7101 119.27805942825
7201 119.27750867935
7301 119.280220517737
7401 119.27979327118
7501 119.283858152246
7601 119.28637942376
7701 119.28680041553
7801 119.28897449045
7901 119.289519048222
8001 119.292418447694
8101 119.293066288113
8201 119.29508718449
8301 119.297252138297
8401 119.299531008213
8501 119.300768144924
8601 119.302206720149
8701 119.302014710953
8801 119.304506306102
8901 119.306068981013
9001 119.308455727141
9101 119.309792330513
9201 119.308056732964
9301 119.307465863886
9401 119.308340602064
9501 119.308426481423
9601 119.307808561608
9701 119.309453664571
9801 119.310825425977
9901 119.313511766488
};
\addlegendentry{Oracle}
\addplot [very thick, steelblue31119180,mark=diamond, mark options={fill=none},mark size=3pt,  mark repeat=5]
table {%
1 100.03
101 113.89702970297
201 116.07184079602
301 116.976378737542
401 117.458154613466
501 117.742774451098
601 117.954941763727
701 118.118758915835
801 118.229538077403
901 118.346803551609
1001 118.422387612388
1101 118.51174386921
1201 118.585770191507
1301 118.627840122982
1401 118.672969307637
1501 118.72838107928
1601 118.764515927545
1701 118.796507936508
1801 118.828778456413
1901 118.858390320884
2001 118.880184907546
2101 118.911256544503
2201 118.934684234439
2301 118.953298565841
2401 118.966609745939
2501 118.987708916433
2601 118.996962706651
2701 119.00208441318
2801 119.001738664763
2901 119.014932781799
3001 119.021986004665
3101 119.035343437601
3201 119.037850671665
3301 119.045789154802
3401 119.05769773596
3501 119.060816909454
3601 119.067509025271
3701 119.079537962713
3801 119.087232307288
3901 119.094593693925
4001 119.106460884779
4101 119.11468422336
4201 119.114746488931
4301 119.120502208789
4401 119.128827539196
4501 119.132403910242
4601 119.135153227559
4701 119.139476707084
4801 119.14589668819
4901 119.15050193838
5001 119.154783043391
5101 119.157841599686
5201 119.160999807729
5301 119.161501603471
5401 119.164688020737
5501 119.170079985457
5601 119.176100696304
5701 119.181059463252
5801 119.185061196345
5901 119.189315370276
6001 119.188648558574
6101 119.191412883134
6201 119.197064989518
6301 119.200545945088
6401 119.203774410248
6501 119.20586217505
6601 119.207810937737
6701 119.210320847635
6801 119.213200999853
6901 119.216988842197
7001 119.218128838737
7101 119.217522884101
7201 119.219959727816
7301 119.223252979044
7401 119.22501688961
7501 119.22871217171
7601 119.230439415866
7701 119.23462147773
7801 119.236222279195
7901 119.239989874699
8001 119.242462192226
8101 119.245325268485
8201 119.245845628582
8301 119.249565112637
8401 119.252323532913
8501 119.252833784261
8601 119.252970584816
8701 119.253596138375
8801 119.253005340302
8901 119.253790585327
9001 119.255363848461
9101 119.255312603011
9201 119.257746984023
9301 119.259186109021
9401 119.258850122327
9501 119.260021050416
9601 119.26391209249
9701 119.264801566849
9801 119.267370676462
9901 119.269706090294
};
\addlegendentry{UCB QR}
\addplot [very thick, mediumpurple148103189, mark=star, mark options={fill=none},mark size=3pt,  mark repeat=5]
table {%
1 105.2
101 114.244356435644
201 114.30631840796
301 114.371993355482
401 114.453291770574
501 114.454431137725
601 114.495108153078
701 114.510613409415
801 114.516866416979
901 114.504006659267
1001 114.488311688312
1101 114.490245231608
1201 114.495861781849
1301 114.49524980784
1401 114.496723768737
1501 114.493877415057
1601 114.50419737664
1701 114.503074661964
1801 114.504203220433
1901 114.506891109942
2001 114.509790104948
2101 114.513812470252
2201 114.516283507497
2301 114.511621034333
2401 114.51387755102
2501 114.509092363055
2601 114.511430219147
2701 114.507997038134
2801 114.513755801499
2901 114.509193381593
3001 114.513978673775
3101 114.512705578846
3201 114.511002811621
3301 114.505540745229
3401 114.505807115554
3501 114.505529848615
3601 114.499444598723
3701 114.502212915428
3801 114.503975269666
3901 114.503196616252
4001 114.503366658335
4101 114.50125579127
4201 114.497693406332
4301 114.49824459428
4401 114.497062031357
4501 114.492041768496
4601 114.490334709846
4701 114.487292065518
4801 114.488748177463
4901 114.490038767598
5001 114.490221955609
5101 114.491850617526
5201 114.491138242646
5301 114.489437841917
5401 114.492851323829
5501 114.496551536084
5601 114.497437957508
5701 114.495967374145
5801 114.494683675228
5901 114.494489069649
6001 114.49511748042
6101 114.496889034585
6201 114.494394452508
6301 114.494527852722
6401 114.496547414467
6501 114.496226734349
6601 114.495646114225
6701 114.496446798985
6801 114.499976474048
6901 114.50250688306
7001 114.503565204971
7101 114.504752851711
7201 114.505306207471
7301 114.507048349541
7401 114.506596405891
7501 114.508856152513
7601 114.509828969872
7701 114.509042981431
7801 114.507414434047
7901 114.507023161625
8001 114.506295463067
8101 114.505684483397
8201 114.506513839776
8301 114.507049753042
8401 114.506226639686
8501 114.505472297377
8601 114.505678409487
8701 114.505453396161
8801 114.506415180093
8901 114.508460847096
9001 114.507058104655
9101 114.507403582024
9201 114.507726334094
9301 114.507266960542
9401 114.506666312094
9501 114.506634038522
9601 114.506326424331
9701 114.504118132151
9801 114.501635547393
9901 114.500622159378
};
\addlegendentry{FCFS ALIS}
\addplot [very thick, orchid227119194, mark=*, mark options={fill=none},mark size=2pt, mark repeat=5]
table {%
1 106.6
101 112.617722772277
201 112.604427860697
301 112.601328903655
401 112.600997506234
501 112.580479041916
601 112.590482529118
701 112.575335235378
801 112.589687890137
901 112.590654827969
1001 112.613756243756
1101 112.619336966394
1201 112.614338051624
1301 112.612267486549
1401 112.601798715203
1501 112.607728181213
1601 112.608788257339
1701 112.601434450323
1801 112.604686285397
1901 112.603566543924
2001 112.596796601699
2101 112.59376487387
2201 112.597005906406
2301 112.593320295524
2401 112.603007080383
2501 112.606925229908
2601 112.60567858516
2701 112.608008145131
2801 112.614341306676
2901 112.614446742503
3001 112.620806397867
3101 112.621280232183
3201 112.618016244923
3301 112.620893668585
3401 112.621746545134
3501 112.620131391031
3601 112.623848930853
3701 112.627133207241
3801 112.630857669034
3901 112.62990771597
4001 112.628175456136
4101 112.631199707388
4201 112.626657938586
4301 112.624619855847
4401 112.624008179959
4501 112.623799155743
4601 112.623431862639
4701 112.624416081685
4801 112.625534263695
4901 112.627696388492
5001 112.628850229954
5101 112.630539109978
5201 112.628000384541
5301 112.628104131296
5401 112.623804850954
5501 112.625322668606
5601 112.623802892341
5701 112.624648307314
5801 112.623944147561
5901 112.624041687849
6001 112.622382936177
6101 112.622035731847
6201 112.621752943074
6301 112.621940961752
6401 112.622921418528
6501 112.621512075065
6601 112.620666565672
6701 112.621654976869
6801 112.617895897662
6901 112.617827851036
7001 112.619271532638
7101 112.622319391635
7201 112.622364949313
7301 112.622254485687
7401 112.619660856641
7501 112.619217437675
7601 112.617425338771
7701 112.619040384366
7801 112.619360338418
7901 112.621436527022
8001 112.622704661917
8101 112.623131712134
8201 112.622198512377
8301 112.621617877364
8401 112.622912748482
8501 112.623585460534
8601 112.624730845251
8701 112.624419032295
8801 112.624802863311
8901 112.626593641164
9001 112.624479502278
9101 112.62490605428
9201 112.624995109227
9301 112.624464036125
9401 112.623510264865
9501 112.625246816125
9601 112.625124466201
9701 112.624821152459
9801 112.624340373431
9901 112.623276436724
};
\addlegendentry{Greedy}
\addplot [very thick, goldenrod18818934,mark=square, mark options={fill},mark size=2pt,  mark repeat=5, mark phase=5]
table {%
1 107
101 112.778118811881
201 112.741044776119
301 112.801162790698
401 112.748503740648
501 112.750898203593
601 112.769417637271
701 112.764693295292
801 112.771161048689
901 112.767680355161
1001 112.764785214785
1101 112.761371480472
1201 112.766186511241
1301 112.773651037663
1401 112.791349036403
1501 112.788660892738
1601 112.798732042473
1701 112.806937095826
1801 112.806973903387
1901 112.806038926881
2001 112.799940029985
2101 112.796653974298
2201 112.797573830077
2301 112.803415906128
2401 112.806426488963
2501 112.81068772491
2601 112.811034217609
2701 112.807167715661
2801 112.806494109247
2901 112.811802826612
3001 112.81568143952
3101 112.819119638826
3201 112.81837238363
3301 112.816031505604
3401 112.815842399294
3501 112.817215081405
3601 112.819661205221
3701 112.824182653337
3801 112.823225466982
3901 112.820515252499
4001 112.816683329168
4101 112.815725432821
4201 112.812746965008
4301 112.813285282492
4401 112.815712338105
4501 112.813616974006
4601 112.813297109324
4701 112.816017868539
4801 112.814976046657
4901 112.813774739849
5001 112.814183163367
5101 112.816285042149
5201 112.8166756393
5301 112.814842482551
5401 112.813960377708
5501 112.816602435921
5601 112.815482949473
5701 112.815520084196
5801 112.817553870022
5901 112.818122352144
6001 112.81901349775
6101 112.81950499918
6201 112.819456539268
6301 112.819990477702
6401 112.819925011717
6501 112.818441778188
6601 112.818257839721
6701 112.817195940904
6801 112.814547860609
6901 112.81537893059
7001 112.815343522354
7101 112.817729897198
7201 112.817822524649
7301 112.818968634434
7401 112.818046209972
7501 112.818393547527
7601 112.818448888304
7701 112.817684716271
7801 112.818477118318
7901 112.817294013416
8001 112.81686664167
8101 112.818702629305
8201 112.818275820022
8301 112.819301289001
8401 112.81830377336
8501 112.817481472768
8601 112.817118939658
8701 112.817249741409
8801 112.817066242472
8901 112.818028311426
9001 112.817802466393
9101 112.816200417537
9201 112.814231061841
9301 112.814176970218
9401 112.814087862993
9501 112.814095358383
9601 112.813902718467
9701 112.812821358623
9801 112.812501785532
9901 112.812461367539
};
\addlegendentry{Random}
\end{axis}

\end{tikzpicture}

%% file: Figs/TikzPlot_Minimal_RunningAvgTotalReward.tex
\begin{tikzpicture}[scale=.65,transform shape]

\definecolor{darkgray176}{RGB}{176,176,176}
\definecolor{forestgreen4416044}{RGB}{44,160,44}
\definecolor{goldenrod18818934}{RGB}{188,189,34}
\definecolor{mediumpurple148103189}{RGB}{148,103,189}
\definecolor{orchid227119194}{RGB}{227,119,194}
\definecolor{steelblue31119180}{RGB}{31,119,180}

\definecolor{black}{RGB}{0,0,0}
\definecolor{darkgray}{RGB}{160,160,160}
\definecolor{lightgray}{RGB}{200,200,200}

\begin{axis}[
tick align=outside,
tick pos=left,
x grid style={darkgray176},
xlabel={Time},
xmajorgrids,
xmin=-50, xmax=3000,
xtick={0,1000,2000,3000},
xtick style={color=black},
y grid style={darkgray176},
ymajorgrids,
ymin=95, ymax=108,
ytick={95,100,105},
ytick style={color=black},
width=9cm,  
height=8cm,  
]
\path [draw=black, fill=black, opacity=0.2]
(axis cs:1,89.359542045971)
--(axis cs:1,92.840457954029)
--(axis cs:101,106.124750951726)
--(axis cs:201,106.375883359422)
--(axis cs:301,106.582561783019)
--(axis cs:401,106.691103998356)
--(axis cs:501,106.759503755024)
--(axis cs:601,106.775781600293)
--(axis cs:701,106.805767887371)
--(axis cs:801,106.818549041115)
--(axis cs:901,106.856342672715)
--(axis cs:1001,106.876152357606)
--(axis cs:1101,106.912665222579)
--(axis cs:1201,106.924702653812)
--(axis cs:1301,106.939357514909)
--(axis cs:1401,106.938744599337)
--(axis cs:1501,106.947453446479)
--(axis cs:1601,106.950101463929)
--(axis cs:1701,106.950538234994)
--(axis cs:1801,106.948452387597)
--(axis cs:1901,106.965199522414)
--(axis cs:2001,106.967684442378)
--(axis cs:2101,106.981003839439)
--(axis cs:2201,106.982909647021)
--(axis cs:2301,106.986025982099)
--(axis cs:2401,106.993193691053)
--(axis cs:2501,106.999965549449)
--(axis cs:2601,107.008322437991)
--(axis cs:2701,107.012453983931)
--(axis cs:2801,107.020677961564)
--(axis cs:2901,107.028113757724)
--(axis cs:3001,107.034229324785)
--(axis cs:3101,107.038293020216)
--(axis cs:3201,107.036097161992)
--(axis cs:3301,107.038372701591)
--(axis cs:3401,107.038511868937)
--(axis cs:3501,107.046587424715)
--(axis cs:3601,107.05014626954)
--(axis cs:3701,107.053086655412)
--(axis cs:3801,107.054556503361)
--(axis cs:3901,107.052263571835)
--(axis cs:4001,107.056282672754)
--(axis cs:4101,107.054027492914)
--(axis cs:4201,107.057888302982)
--(axis cs:4301,107.058446876259)
--(axis cs:4401,107.055841503399)
--(axis cs:4501,107.056786549923)
--(axis cs:4601,107.060907439095)
--(axis cs:4701,107.062115281954)
--(axis cs:4801,107.061626946362)
--(axis cs:4901,107.06681151817)
--(axis cs:5001,107.069223217948)
--(axis cs:5101,107.07111064249)
--(axis cs:5201,107.073176625461)
--(axis cs:5301,107.078030003027)
--(axis cs:5401,107.079527808362)
--(axis cs:5501,107.080706939073)
--(axis cs:5601,107.083153769028)
--(axis cs:5701,107.084966174115)
--(axis cs:5801,107.084927227343)
--(axis cs:5901,107.088530930163)
--(axis cs:6001,107.088262032247)
--(axis cs:6101,107.087297125695)
--(axis cs:6201,107.088951743022)
--(axis cs:6301,107.089613633514)
--(axis cs:6401,107.090298170082)
--(axis cs:6501,107.090161730042)
--(axis cs:6601,107.089298253916)
--(axis cs:6701,107.089177809237)
--(axis cs:6801,107.090016899619)
--(axis cs:6901,107.091576026653)
--(axis cs:7001,107.092382476249)
--(axis cs:7101,107.094079152114)
--(axis cs:7201,107.093111184627)
--(axis cs:7301,107.092167144481)
--(axis cs:7401,107.09251049788)
--(axis cs:7501,107.093161738991)
--(axis cs:7601,107.095801787048)
--(axis cs:7701,107.0979702515)
--(axis cs:7801,107.100381637513)
--(axis cs:7901,107.100860353175)
--(axis cs:8001,107.102124623121)
--(axis cs:8101,107.103723903806)
--(axis cs:8201,107.103849157961)
--(axis cs:8301,107.104852469211)
--(axis cs:8401,107.10563843029)
--(axis cs:8501,107.106715797593)
--(axis cs:8601,107.107887700172)
--(axis cs:8701,107.108935250708)
--(axis cs:8801,107.109744192905)
--(axis cs:8901,107.110532161241)
--(axis cs:9001,107.110734210872)
--(axis cs:9101,107.110712934829)
--(axis cs:9201,107.110977291921)
--(axis cs:9301,107.109429208829)
--(axis cs:9401,107.109717709943)
--(axis cs:9501,107.109627756804)
--(axis cs:9601,107.111567645038)
--(axis cs:9701,107.113364991712)
--(axis cs:9801,107.113110109673)
--(axis cs:9901,107.112956468355)
--(axis cs:9901,107.070285628403)
--(axis cs:9901,107.070285628403)
--(axis cs:9801,107.070309949505)
--(axis cs:9701,107.070080013958)
--(axis cs:9601,107.068646915946)
--(axis cs:9501,107.067082063215)
--(axis cs:9401,107.06733366757)
--(axis cs:9301,107.066130408416)
--(axis cs:9201,107.06615128106)
--(axis cs:9101,107.066003909474)
--(axis cs:9001,107.065355112537)
--(axis cs:8901,107.06425718827)
--(axis cs:8801,107.063397495539)
--(axis cs:8701,107.063290930191)
--(axis cs:8601,107.062020450043)
--(axis cs:8501,107.06103623158)
--(axis cs:8401,107.059503814681)
--(axis cs:8301,107.0583375079)
--(axis cs:8201,107.058348135052)
--(axis cs:8101,107.057819115575)
--(axis cs:8001,107.056120596227)
--(axis cs:7901,107.054271908564)
--(axis cs:7801,107.054165215454)
--(axis cs:7701,107.051542798753)
--(axis cs:7601,107.049605396218)
--(axis cs:7501,107.046782268475)
--(axis cs:7401,107.045157384839)
--(axis cs:7301,107.044592203554)
--(axis cs:7201,107.044969637481)
--(axis cs:7101,107.045816637212)
--(axis cs:7001,107.043906625308)
--(axis cs:6901,107.042781312863)
--(axis cs:6801,107.041760780134)
--(axis cs:6701,107.040734144203)
--(axis cs:6601,107.040303321603)
--(axis cs:6501,107.040027471619)
--(axis cs:6401,107.039706516685)
--(axis cs:6301,107.038258132872)
--(axis cs:6201,107.037102119259)
--(axis cs:6101,107.034823838082)
--(axis cs:6001,107.035434018411)
--(axis cs:5901,107.034587185411)
--(axis cs:5801,107.029742657159)
--(axis cs:5701,107.028424459107)
--(axis cs:5601,107.027447909243)
--(axis cs:5501,107.02449938705)
--(axis cs:5401,107.021738623781)
--(axis cs:5301,107.019215799651)
--(axis cs:5201,107.013272134393)
--(axis cs:5101,107.009924448669)
--(axis cs:5001,107.007061525103)
--(axis cs:4901,107.003619006213)
--(axis cs:4801,106.998293903461)
--(axis cs:4701,106.999777932256)
--(axis cs:4601,106.997153851929)
--(axis cs:4501,106.992940177471)
--(axis cs:4401,106.992838342091)
--(axis cs:4301,106.993494532716)
--(axis cs:4201,106.992556829129)
--(axis cs:4101,106.988123202038)
--(axis cs:4001,106.989685835119)
--(axis cs:3901,106.984511614015)
--(axis cs:3801,106.985548732103)
--(axis cs:3701,106.984978732321)
--(axis cs:3601,106.98098397206)
--(axis cs:3501,106.97795413484)
--(axis cs:3401,106.967768636796)
--(axis cs:3301,106.969067468054)
--(axis cs:3201,106.96747672117)
--(axis cs:3101,106.964983342248)
--(axis cs:3001,106.958073241027)
--(axis cs:2901,106.950976211252)
--(axis cs:2801,106.943777590025)
--(axis cs:2701,106.934906253019)
--(axis cs:2601,106.932231195227)
--(axis cs:2501,106.924368716844)
--(axis cs:2401,106.919659286872)
--(axis cs:2301,106.914495530287)
--(axis cs:2201,106.910802302093)
--(axis cs:2101,106.907715817867)
--(axis cs:2001,106.890996217292)
--(axis cs:1901,106.888509052021)
--(axis cs:1801,106.87152540252)
--(axis cs:1701,106.870719848486)
--(axis cs:1601,106.868249566677)
--(axis cs:1501,106.86231337564)
--(axis cs:1401,106.848521639064)
--(axis cs:1301,106.84145724297)
--(axis cs:1201,106.827254048936)
--(axis cs:1101,106.814183097131)
--(axis cs:1001,106.769621868168)
--(axis cs:901,106.749517482668)
--(axis cs:801,106.706269935165)
--(axis cs:701,106.683247804497)
--(axis cs:601,106.63400209355)
--(axis cs:501,106.60871979787)
--(axis cs:401,106.525454605135)
--(axis cs:301,106.386541207014)
--(axis cs:201,106.125310670429)
--(axis cs:101,105.746932216591)
--(axis cs:1,89.359542045971)
--cycle;

\path [draw=steelblue31119180, fill=steelblue31119180, opacity=0.2]
(axis cs:1,87.8889953911853)
--(axis cs:1,91.5110046088147)
--(axis cs:101,104.749308988177)
--(axis cs:201,105.624028304377)
--(axis cs:301,106.063929716464)
--(axis cs:401,106.274064757777)
--(axis cs:501,106.418958490619)
--(axis cs:601,106.522613241195)
--(axis cs:701,106.617061327091)
--(axis cs:801,106.681708559737)
--(axis cs:901,106.727876687256)
--(axis cs:1001,106.769470610022)
--(axis cs:1101,106.812556497946)
--(axis cs:1201,106.818666892702)
--(axis cs:1301,106.839585627518)
--(axis cs:1401,106.85656615122)
--(axis cs:1501,106.870802555251)
--(axis cs:1601,106.87285530147)
--(axis cs:1701,106.886247302921)
--(axis cs:1801,106.903365869721)
--(axis cs:1901,106.908753142554)
--(axis cs:2001,106.914291127783)
--(axis cs:2101,106.930950897632)
--(axis cs:2201,106.943111473429)
--(axis cs:2301,106.949512703667)
--(axis cs:2401,106.953313161614)
--(axis cs:2501,106.957090566722)
--(axis cs:2601,106.961935575429)
--(axis cs:2701,106.967569365373)
--(axis cs:2801,106.971544316425)
--(axis cs:2901,106.977203894805)
--(axis cs:3001,106.982250013578)
--(axis cs:3101,106.989151654093)
--(axis cs:3201,106.987842339256)
--(axis cs:3301,106.993336655319)
--(axis cs:3401,106.999303627503)
--(axis cs:3501,107.00195217632)
--(axis cs:3601,107.000181530638)
--(axis cs:3701,107.004568747238)
--(axis cs:3801,107.010315765182)
--(axis cs:3901,107.013321205292)
--(axis cs:4001,107.015509467745)
--(axis cs:4101,107.019333256193)
--(axis cs:4201,107.02480714166)
--(axis cs:4301,107.02337545586)
--(axis cs:4401,107.028103158109)
--(axis cs:4501,107.037522127215)
--(axis cs:4601,107.03876170368)
--(axis cs:4701,107.036536729963)
--(axis cs:4801,107.039319473796)
--(axis cs:4901,107.040132242332)
--(axis cs:5001,107.043358428947)
--(axis cs:5101,107.04477937104)
--(axis cs:5201,107.050805868889)
--(axis cs:5301,107.05220572951)
--(axis cs:5401,107.05410586405)
--(axis cs:5501,107.055874851297)
--(axis cs:5601,107.060586995326)
--(axis cs:5701,107.060560195808)
--(axis cs:5801,107.05788244741)
--(axis cs:5901,107.059658307946)
--(axis cs:6001,107.061064236824)
--(axis cs:6101,107.062130234187)
--(axis cs:6201,107.059903336446)
--(axis cs:6301,107.062978562186)
--(axis cs:6401,107.064168039086)
--(axis cs:6501,107.067888213178)
--(axis cs:6601,107.070566165631)
--(axis cs:6701,107.070742704469)
--(axis cs:6801,107.072498814866)
--(axis cs:6901,107.072553236161)
--(axis cs:7001,107.071781279894)
--(axis cs:7101,107.071861803067)
--(axis cs:7201,107.073486042599)
--(axis cs:7301,107.076433851249)
--(axis cs:7401,107.077509926731)
--(axis cs:7501,107.080063885586)
--(axis cs:7601,107.081460924799)
--(axis cs:7701,107.081696520447)
--(axis cs:7801,107.082028762555)
--(axis cs:7901,107.083152235281)
--(axis cs:8001,107.086053788128)
--(axis cs:8101,107.089218426934)
--(axis cs:8201,107.089190185361)
--(axis cs:8301,107.091004670926)
--(axis cs:8401,107.091084007306)
--(axis cs:8501,107.091655005957)
--(axis cs:8601,107.093528830498)
--(axis cs:8701,107.09521893454)
--(axis cs:8801,107.095227835697)
--(axis cs:8901,107.096789807186)
--(axis cs:9001,107.098625478558)
--(axis cs:9101,107.100086304541)
--(axis cs:9201,107.101482005495)
--(axis cs:9301,107.100929261251)
--(axis cs:9401,107.101303251342)
--(axis cs:9501,107.102410373139)
--(axis cs:9601,107.103797439071)
--(axis cs:9701,107.102203691795)
--(axis cs:9801,107.103473515101)
--(axis cs:9901,107.103802125339)
--(axis cs:9901,107.058957191902)
--(axis cs:9901,107.058957191902)
--(axis cs:9801,107.058222230232)
--(axis cs:9701,107.056769609926)
--(axis cs:9601,107.057862804654)
--(axis cs:9501,107.05631818175)
--(axis cs:9401,107.055241797058)
--(axis cs:9301,107.054473383626)
--(axis cs:9201,107.055127058737)
--(axis cs:9101,107.053296840168)
--(axis cs:9001,107.052222205032)
--(axis cs:8901,107.050816079793)
--(axis cs:8801,107.048487651179)
--(axis cs:8701,107.048125508628)
--(axis cs:8601,107.045987504812)
--(axis cs:8501,107.043973743602)
--(axis cs:8401,107.042845286825)
--(axis cs:8301,107.042160008028)
--(axis cs:8201,107.041064661609)
--(axis cs:8101,107.040572956845)
--(axis cs:8001,107.036828351604)
--(axis cs:7901,107.033782329964)
--(axis cs:7801,107.032987261032)
--(axis cs:7701,107.030992740688)
--(axis cs:7601,107.030287529352)
--(axis cs:7501,107.028884254662)
--(axis cs:7401,107.025468049218)
--(axis cs:7301,107.02437973593)
--(axis cs:7201,107.021239689938)
--(axis cs:7101,107.019632352686)
--(axis cs:7001,107.018077311736)
--(axis cs:6901,107.019250850204)
--(axis cs:6801,107.01972879872)
--(axis cs:6701,107.017536656821)
--(axis cs:6601,107.017850740898)
--(axis cs:6501,107.015160548551)
--(axis cs:6401,107.011845083864)
--(axis cs:6301,107.010025722848)
--(axis cs:6201,107.007379359893)
--(axis cs:6101,107.010379190497)
--(axis cs:6001,107.008910767342)
--(axis cs:5901,107.006774500052)
--(axis cs:5801,107.005416983722)
--(axis cs:5701,107.007196338133)
--(axis cs:5601,107.007450855058)
--(axis cs:5501,107.002612696422)
--(axis cs:5401,107.000643256483)
--(axis cs:5301,106.997343412162)
--(axis cs:5201,106.996235084774)
--(axis cs:5101,106.99022552996)
--(axis cs:5001,106.9878073379)
--(axis cs:4901,106.985548231041)
--(axis cs:4801,106.98469635624)
--(axis cs:4701,106.982757037321)
--(axis cs:4601,106.985150489321)
--(axis cs:4501,106.983628728151)
--(axis cs:4401,106.974991593084)
--(axis cs:4301,106.968738006126)
--(axis cs:4201,106.969551344415)
--(axis cs:4101,106.962066402427)
--(axis cs:4001,106.955857690466)
--(axis cs:3901,106.951985126418)
--(axis cs:3801,106.947211201406)
--(axis cs:3701,106.942213203586)
--(axis cs:3601,106.937596864252)
--(axis cs:3501,106.940321459784)
--(axis cs:3401,106.935791932627)
--(axis cs:3301,106.929886610358)
--(axis cs:3201,106.922954286798)
--(axis cs:3101,106.922618742553)
--(axis cs:3001,106.916816964096)
--(axis cs:2901,106.911434505747)
--(axis cs:2801,106.902693455799)
--(axis cs:2701,106.897051145549)
--(axis cs:2601,106.889975228108)
--(axis cs:2501,106.883517190176)
--(axis cs:2401,106.874591877953)
--(axis cs:2301,106.866506418454)
--(axis cs:2201,106.860868535657)
--(axis cs:2101,106.845850625453)
--(axis cs:2001,106.826368542382)
--(axis cs:1901,106.817412033669)
--(axis cs:1801,106.808982825448)
--(axis cs:1701,106.78815598926)
--(axis cs:1601,106.766819901528)
--(axis cs:1501,106.761216098979)
--(axis cs:1401,106.742591593248)
--(axis cs:1301,106.721828669177)
--(axis cs:1201,106.697286479488)
--(axis cs:1101,106.689405354915)
--(axis cs:1001,106.640539379988)
--(axis cs:901,106.597162158471)
--(axis cs:801,106.546406296692)
--(axis cs:701,106.478059928258)
--(axis cs:601,106.367037341167)
--(axis cs:501,106.251859872655)
--(axis cs:401,106.06788037938)
--(axis cs:301,105.813080250313)
--(axis cs:201,105.291792591146)
--(axis cs:101,104.3584137841)
--(axis cs:1,87.8889953911853)
--cycle;

\path [draw=mediumpurple148103189, fill=mediumpurple148103189, opacity=0.2]
(axis cs:1,93.3133740295136)
--(axis cs:1,96.6066259704864)
--(axis cs:101,105.320066716298)
--(axis cs:201,105.521137851311)
--(axis cs:301,105.427153141463)
--(axis cs:401,105.489809675686)
--(axis cs:501,105.427190498188)
--(axis cs:601,105.432970648861)
--(axis cs:701,105.445407794387)
--(axis cs:801,105.435197189771)
--(axis cs:901,105.445676118242)
--(axis cs:1001,105.423749189867)
--(axis cs:1101,105.436930013525)
--(axis cs:1201,105.445426176203)
--(axis cs:1301,105.45781656116)
--(axis cs:1401,105.455484283529)
--(axis cs:1501,105.459415287358)
--(axis cs:1601,105.459660249886)
--(axis cs:1701,105.46071046806)
--(axis cs:1801,105.44862097445)
--(axis cs:1901,105.451103575714)
--(axis cs:2001,105.44890387648)
--(axis cs:2101,105.453320768958)
--(axis cs:2201,105.455917382199)
--(axis cs:2301,105.460511731703)
--(axis cs:2401,105.460709950486)
--(axis cs:2501,105.45637807878)
--(axis cs:2601,105.449412406286)
--(axis cs:2701,105.451957524685)
--(axis cs:2801,105.45245506079)
--(axis cs:2901,105.455451672621)
--(axis cs:3001,105.45434476577)
--(axis cs:3101,105.459329876053)
--(axis cs:3201,105.456180563717)
--(axis cs:3301,105.455531678008)
--(axis cs:3401,105.448849879793)
--(axis cs:3501,105.449543698067)
--(axis cs:3601,105.450150104779)
--(axis cs:3701,105.454585482993)
--(axis cs:3801,105.452038867479)
--(axis cs:3901,105.452756790086)
--(axis cs:4001,105.452608263899)
--(axis cs:4101,105.457296515569)
--(axis cs:4201,105.459417612323)
--(axis cs:4301,105.462633422884)
--(axis cs:4401,105.463636409716)
--(axis cs:4501,105.464164816584)
--(axis cs:4601,105.460760841669)
--(axis cs:4701,105.460595878367)
--(axis cs:4801,105.460065359022)
--(axis cs:4901,105.459921564479)
--(axis cs:5001,105.459010965677)
--(axis cs:5101,105.460276476283)
--(axis cs:5201,105.461567310598)
--(axis cs:5301,105.460015411935)
--(axis cs:5401,105.461469635232)
--(axis cs:5501,105.458777203099)
--(axis cs:5601,105.457511666277)
--(axis cs:5701,105.457439840286)
--(axis cs:5801,105.456913890343)
--(axis cs:5901,105.45507745123)
--(axis cs:6001,105.453954834756)
--(axis cs:6101,105.45458392502)
--(axis cs:6201,105.456210978211)
--(axis cs:6301,105.457675640772)
--(axis cs:6401,105.458875654485)
--(axis cs:6501,105.459086735244)
--(axis cs:6601,105.460698473655)
--(axis cs:6701,105.459821608795)
--(axis cs:6801,105.459148084725)
--(axis cs:6901,105.458645194761)
--(axis cs:7001,105.455986980615)
--(axis cs:7101,105.458816012428)
--(axis cs:7201,105.455520698919)
--(axis cs:7301,105.454618390213)
--(axis cs:7401,105.45465900491)
--(axis cs:7501,105.456267256917)
--(axis cs:7601,105.457033931192)
--(axis cs:7701,105.457555861907)
--(axis cs:7801,105.456261755012)
--(axis cs:7901,105.457152607597)
--(axis cs:8001,105.457971433372)
--(axis cs:8101,105.459494091992)
--(axis cs:8201,105.459458097613)
--(axis cs:8301,105.458009371449)
--(axis cs:8401,105.458198565491)
--(axis cs:8501,105.456189326345)
--(axis cs:8601,105.456529519677)
--(axis cs:8701,105.455619608706)
--(axis cs:8801,105.45542540478)
--(axis cs:8901,105.456284821638)
--(axis cs:9001,105.45476835616)
--(axis cs:9101,105.453655429448)
--(axis cs:9201,105.452244122606)
--(axis cs:9301,105.452879252973)
--(axis cs:9401,105.45420258035)
--(axis cs:9501,105.455369150893)
--(axis cs:9601,105.455105833828)
--(axis cs:9701,105.455913257372)
--(axis cs:9801,105.453354065639)
--(axis cs:9901,105.454604101179)
--(axis cs:9901,105.416004928212)
--(axis cs:9901,105.416004928212)
--(axis cs:9801,105.4144390167)
--(axis cs:9701,105.416466909621)
--(axis cs:9601,105.41615965935)
--(axis cs:9501,105.415454972883)
--(axis cs:9401,105.414770933106)
--(axis cs:9301,105.413971623277)
--(axis cs:9201,105.412979222682)
--(axis cs:9101,105.414468952488)
--(axis cs:9001,105.415494947917)
--(axis cs:8901,105.416026154657)
--(axis cs:8801,105.415080219581)
--(axis cs:8701,105.414388436346)
--(axis cs:8601,105.415469085137)
--(axis cs:8501,105.415526942329)
--(axis cs:8401,105.41757812776)
--(axis cs:8301,105.417350223781)
--(axis cs:8201,105.418544584986)
--(axis cs:8101,105.418286428931)
--(axis cs:8001,105.417254163428)
--(axis cs:7901,105.41607103498)
--(axis cs:7801,105.414442000917)
--(axis cs:7701,105.414707480516)
--(axis cs:7601,105.412451662808)
--(axis cs:7501,105.411446381264)
--(axis cs:7401,105.409218849434)
--(axis cs:7301,105.409876884407)
--(axis cs:7201,105.410908963628)
--(axis cs:7101,105.414007533552)
--(axis cs:7001,105.410672068092)
--(axis cs:6901,105.412698088821)
--(axis cs:6801,105.41425288572)
--(axis cs:6701,105.41539999992)
--(axis cs:6601,105.416504980367)
--(axis cs:6501,105.413742060326)
--(axis cs:6401,105.412491319425)
--(axis cs:6301,105.410596062132)
--(axis cs:6201,105.409323612984)
--(axis cs:6101,105.407451806828)
--(axis cs:6001,105.406375110253)
--(axis cs:5901,105.406135902439)
--(axis cs:5801,105.408112829188)
--(axis cs:5701,105.407920622791)
--(axis cs:5601,105.407344609388)
--(axis cs:5501,105.408646901609)
--(axis cs:5401,105.410135623053)
--(axis cs:5301,105.409118713702)
--(axis cs:5201,105.411453262369)
--(axis cs:5101,105.407463182608)
--(axis cs:5001,105.405352161698)
--(axis cs:4901,105.405918060088)
--(axis cs:4801,105.404637827814)
--(axis cs:4701,105.40360748262)
--(axis cs:4601,105.403603426968)
--(axis cs:4501,105.405930717742)
--(axis cs:4401,105.403775542113)
--(axis cs:4301,105.402658369722)
--(axis cs:4201,105.40151073807)
--(axis cs:4101,105.398211897013)
--(axis cs:4001,105.396079564144)
--(axis cs:3901,105.395271920501)
--(axis cs:3801,105.394822484797)
--(axis cs:3701,105.395768475396)
--(axis cs:3601,105.390816293443)
--(axis cs:3501,105.388274068286)
--(axis cs:3401,105.387127773838)
--(axis cs:3301,105.391484377733)
--(axis cs:3201,105.390404878333)
--(axis cs:3101,105.392640456098)
--(axis cs:3001,105.38457559411)
--(axis cs:2901,105.384844776879)
--(axis cs:2801,105.381818412969)
--(axis cs:2701,105.380756285015)
--(axis cs:2601,105.377062026624)
--(axis cs:2501,105.383637914823)
--(axis cs:2401,105.383713206532)
--(axis cs:2301,105.381609085334)
--(axis cs:2201,105.373650995811)
--(axis cs:2101,105.369068569452)
--(axis cs:2001,105.361031156003)
--(axis cs:1901,105.366897476364)
--(axis cs:1801,105.36275048585)
--(axis cs:1701,105.371376539583)
--(axis cs:1601,105.368284784467)
--(axis cs:1501,105.364062394188)
--(axis cs:1401,105.358876887063)
--(axis cs:1301,105.361337935381)
--(axis cs:1201,105.342683732207)
--(axis cs:1101,105.331734836611)
--(axis cs:1001,105.314632428515)
--(axis cs:901,105.326354958339)
--(axis cs:801,105.307674220965)
--(axis cs:701,105.305633575085)
--(axis cs:601,105.27954183034)
--(axis cs:501,105.261192735345)
--(axis cs:401,105.304005785661)
--(axis cs:301,105.210853503055)
--(axis cs:201,105.2555785666)
--(axis cs:101,104.96468575895)
--(axis cs:1,93.3133740295136)
--cycle;

\path [draw=orchid227119194, fill=orchid227119194, opacity=0.2]
(axis cs:1,91.1341210696893)
--(axis cs:1,94.6858789303107)
--(axis cs:101,96.4399026454563)
--(axis cs:201,96.4946449677081)
--(axis cs:301,96.4380730385239)
--(axis cs:401,96.4317576307503)
--(axis cs:501,96.4377580743183)
--(axis cs:601,96.4346878455216)
--(axis cs:701,96.4276472837081)
--(axis cs:801,96.4317317457119)
--(axis cs:901,96.446633473995)
--(axis cs:1001,96.449001454427)
--(axis cs:1101,96.4505178417876)
--(axis cs:1201,96.4468804708101)
--(axis cs:1301,96.4444845215962)
--(axis cs:1401,96.4372020630523)
--(axis cs:1501,96.4237526684128)
--(axis cs:1601,96.4265223886714)
--(axis cs:1701,96.421455520536)
--(axis cs:1801,96.4156197725648)
--(axis cs:1901,96.4170182448367)
--(axis cs:2001,96.421553822017)
--(axis cs:2101,96.4240426480266)
--(axis cs:2201,96.4254964330628)
--(axis cs:2301,96.426992976925)
--(axis cs:2401,96.428061983584)
--(axis cs:2501,96.4291485407545)
--(axis cs:2601,96.430428215046)
--(axis cs:2701,96.430546218545)
--(axis cs:2801,96.4298597593855)
--(axis cs:2901,96.433654777499)
--(axis cs:3001,96.4334166380929)
--(axis cs:3101,96.43254605855)
--(axis cs:3201,96.4354514362342)
--(axis cs:3301,96.4377026912926)
--(axis cs:3401,96.4357903291477)
--(axis cs:3501,96.4418439449519)
--(axis cs:3601,96.4399748953863)
--(axis cs:3701,96.4394211554163)
--(axis cs:3801,96.4379013710731)
--(axis cs:3901,96.4346551648911)
--(axis cs:4001,96.4345609413955)
--(axis cs:4101,96.4342418017495)
--(axis cs:4201,96.4360680856044)
--(axis cs:4301,96.4355566728351)
--(axis cs:4401,96.4352322869214)
--(axis cs:4501,96.4347181621657)
--(axis cs:4601,96.4348756600382)
--(axis cs:4701,96.434548251215)
--(axis cs:4801,96.4358875149351)
--(axis cs:4901,96.4352290154334)
--(axis cs:5001,96.4352729571452)
--(axis cs:5101,96.4346461809863)
--(axis cs:5201,96.4367566526698)
--(axis cs:5301,96.4381436756626)
--(axis cs:5401,96.4360226698641)
--(axis cs:5501,96.4338661273884)
--(axis cs:5601,96.4306042095096)
--(axis cs:5701,96.4292899191585)
--(axis cs:5801,96.4288440434222)
--(axis cs:5901,96.4277107953254)
--(axis cs:6001,96.4289327451358)
--(axis cs:6101,96.4301471025455)
--(axis cs:6201,96.4275367659359)
--(axis cs:6301,96.4272235168295)
--(axis cs:6401,96.4281082563232)
--(axis cs:6501,96.42935407659)
--(axis cs:6601,96.4303632478304)
--(axis cs:6701,96.4323425288096)
--(axis cs:6801,96.4299985161035)
--(axis cs:6901,96.4305653387333)
--(axis cs:7001,96.4304468329493)
--(axis cs:7101,96.4302139462398)
--(axis cs:7201,96.4299412579673)
--(axis cs:7301,96.4311691298494)
--(axis cs:7401,96.4316076998734)
--(axis cs:7501,96.4300364899325)
--(axis cs:7601,96.4290209584498)
--(axis cs:7701,96.4277255409948)
--(axis cs:7801,96.4283579829023)
--(axis cs:7901,96.4289685225748)
--(axis cs:8001,96.4281586519942)
--(axis cs:8101,96.4273984035801)
--(axis cs:8201,96.4280938234247)
--(axis cs:8301,96.4277040440982)
--(axis cs:8401,96.4292619497835)
--(axis cs:8501,96.4293668041802)
--(axis cs:8601,96.4297742537707)
--(axis cs:8701,96.4284068942182)
--(axis cs:8801,96.4282098006356)
--(axis cs:8901,96.4295811334812)
--(axis cs:9001,96.4295140755632)
--(axis cs:9101,96.4276095195327)
--(axis cs:9201,96.4275939820119)
--(axis cs:9301,96.4267449923568)
--(axis cs:9401,96.4257499237814)
--(axis cs:9501,96.4240230541441)
--(axis cs:9601,96.4249030725736)
--(axis cs:9701,96.4243578448943)
--(axis cs:9801,96.4231926669456)
--(axis cs:9901,96.4239226605891)
--(axis cs:9901,96.3907829247053)
--(axis cs:9901,96.3907829247053)
--(axis cs:9801,96.3902488186171)
--(axis cs:9701,96.3912013758047)
--(axis cs:9601,96.3910431830248)
--(axis cs:9501,96.3895797245108)
--(axis cs:9401,96.3909865936103)
--(axis cs:9301,96.3912982288023)
--(axis cs:9201,96.3920538823508)
--(axis cs:9101,96.3917861512726)
--(axis cs:9001,96.3930745257034)
--(axis cs:8901,96.3927422009756)
--(axis cs:8801,96.391401607159)
--(axis cs:8701,96.3914482948406)
--(axis cs:8601,96.3928626489151)
--(axis cs:8501,96.3916330781866)
--(axis cs:8401,96.3914808189345)
--(axis cs:8301,96.3905009914397)
--(axis cs:8201,96.3903380751243)
--(axis cs:8101,96.389155108332)
--(axis cs:8001,96.3902815429814)
--(axis cs:7901,96.3909808509221)
--(axis cs:7801,96.3901345180591)
--(axis cs:7701,96.3888567210491)
--(axis cs:7601,96.3897976180534)
--(axis cs:7501,96.3904061177198)
--(axis cs:7401,96.3914918812642)
--(axis cs:7301,96.3903429917778)
--(axis cs:7201,96.389419941866)
--(axis cs:7101,96.3895917149347)
--(axis cs:7001,96.3901016601232)
--(axis cs:6901,96.3902664247792)
--(axis cs:6801,96.3891810163183)
--(axis cs:6701,96.3912211184072)
--(axis cs:6601,96.3891580368235)
--(axis cs:6501,96.388452414719)
--(axis cs:6401,96.3873018358498)
--(axis cs:6301,96.3864949405583)
--(axis cs:6201,96.3870802313225)
--(axis cs:6101,96.3884990210408)
--(axis cs:6001,96.3869612725279)
--(axis cs:5901,96.3861749867455)
--(axis cs:5801,96.3863016211185)
--(axis cs:5701,96.3858512841393)
--(axis cs:5601,96.3867069849199)
--(axis cs:5501,96.3891042416354)
--(axis cs:5401,96.3910945306544)
--(axis cs:5301,96.3931485333545)
--(axis cs:5201,96.3917071043)
--(axis cs:5101,96.3886041620836)
--(axis cs:5001,96.3876624557722)
--(axis cs:4901,96.3875867364541)
--(axis cs:4801,96.3876700772332)
--(axis cs:4701,96.3864728081341)
--(axis cs:4601,96.3866240139457)
--(axis cs:4501,96.386988125326)
--(axis cs:4401,96.3862946387773)
--(axis cs:4301,96.3868474192364)
--(axis cs:4201,96.3871978034696)
--(axis cs:4101,96.3854850941295)
--(axis cs:4001,96.3853690761001)
--(axis cs:3901,96.3858729048346)
--(axis cs:3801,96.3873077844123)
--(axis cs:3701,96.3884740080532)
--(axis cs:3601,96.3881672873407)
--(axis cs:3501,96.3876790484786)
--(axis cs:3401,96.3812223141925)
--(axis cs:3301,96.3823760727183)
--(axis cs:3201,96.3808997040345)
--(axis cs:3101,96.376947653156)
--(axis cs:3001,96.3766133519104)
--(axis cs:2901,96.3766037540419)
--(axis cs:2801,96.3701545212285)
--(axis cs:2701,96.3703941739021)
--(axis cs:2601,96.3682107699598)
--(axis cs:2501,96.364997800709)
--(axis cs:2401,96.3627085287025)
--(axis cs:2301,96.361638053062)
--(axis cs:2201,96.3593286464465)
--(axis cs:2101,96.358584672297)
--(axis cs:2001,96.3535386317562)
--(axis cs:1901,96.3480001665257)
--(axis cs:1801,96.3428366405389)
--(axis cs:1701,96.3482328980413)
--(axis cs:1601,96.3494051566128)
--(axis cs:1501,96.3430827746251)
--(axis cs:1401,96.3520627477971)
--(axis cs:1301,96.3564071002331)
--(axis cs:1201,96.354435099548)
--(axis cs:1101,96.3533513680216)
--(axis cs:1001,96.3401094346838)
--(axis cs:901,96.3295041508662)
--(axis cs:801,96.3092170682706)
--(axis cs:701,96.3106123453932)
--(axis cs:601,96.3008196419992)
--(axis cs:501,96.2900662769791)
--(axis cs:401,96.2698882545366)
--(axis cs:301,96.2386047023399)
--(axis cs:201,96.2556037885107)
--(axis cs:101,96.0798993347417)
--(axis cs:1,91.1341210696893)
--cycle;

\path [draw=goldenrod18818934, fill=goldenrod18818934, opacity=0.2]
(axis cs:1,89.6951411141489)
--(axis cs:1,92.7648588858511)
--(axis cs:101,99.1464007604053)
--(axis cs:201,99.0987915807475)
--(axis cs:301,99.0571408351604)
--(axis cs:401,99.0589121179306)
--(axis cs:501,99.0388704726764)
--(axis cs:601,99.0539376093778)
--(axis cs:701,99.0575518908447)
--(axis cs:801,99.063681077912)
--(axis cs:901,99.0638028113896)
--(axis cs:1001,99.0564875952079)
--(axis cs:1101,99.0541145214261)
--(axis cs:1201,99.0503938725556)
--(axis cs:1301,99.0635313985762)
--(axis cs:1401,99.0649406630518)
--(axis cs:1501,99.0664739289026)
--(axis cs:1601,99.0671121048818)
--(axis cs:1701,99.0673652791275)
--(axis cs:1801,99.0667234063355)
--(axis cs:1901,99.0528844643998)
--(axis cs:2001,99.0559240963387)
--(axis cs:2101,99.0571201427402)
--(axis cs:2201,99.0591810338621)
--(axis cs:2301,99.0537995912595)
--(axis cs:2401,99.0583600351186)
--(axis cs:2501,99.0520396627954)
--(axis cs:2601,99.0520085665913)
--(axis cs:2701,99.0451080697966)
--(axis cs:2801,99.0420536448009)
--(axis cs:2901,99.0405896143526)
--(axis cs:3001,99.0378176258796)
--(axis cs:3101,99.0352697924621)
--(axis cs:3201,99.0315600398768)
--(axis cs:3301,99.0305809329815)
--(axis cs:3401,99.036195730152)
--(axis cs:3501,99.0350654476376)
--(axis cs:3601,99.0344011246613)
--(axis cs:3701,99.0350641421099)
--(axis cs:3801,99.0353579907924)
--(axis cs:3901,99.0352890457731)
--(axis cs:4001,99.0369505285253)
--(axis cs:4101,99.0367176941208)
--(axis cs:4201,99.0375687463053)
--(axis cs:4301,99.0368083516884)
--(axis cs:4401,99.0356821446353)
--(axis cs:4501,99.036327239837)
--(axis cs:4601,99.0390187011561)
--(axis cs:4701,99.0370051561369)
--(axis cs:4801,99.0389377643008)
--(axis cs:4901,99.0379093005617)
--(axis cs:5001,99.0359578337497)
--(axis cs:5101,99.0373810085286)
--(axis cs:5201,99.0374449414651)
--(axis cs:5301,99.0395791187102)
--(axis cs:5401,99.0382973611566)
--(axis cs:5501,99.0418077968666)
--(axis cs:5601,99.0409642229556)
--(axis cs:5701,99.0413940943434)
--(axis cs:5801,99.0421705107456)
--(axis cs:5901,99.0429997097466)
--(axis cs:6001,99.0414903311132)
--(axis cs:6101,99.0401675347682)
--(axis cs:6201,99.0423457028134)
--(axis cs:6301,99.0405723925468)
--(axis cs:6401,99.0434349764433)
--(axis cs:6501,99.0410172550701)
--(axis cs:6601,99.0425167952498)
--(axis cs:6701,99.0436792269606)
--(axis cs:6801,99.0417648874489)
--(axis cs:6901,99.0433257546872)
--(axis cs:7001,99.0433125266629)
--(axis cs:7101,99.0420144211832)
--(axis cs:7201,99.0407384358869)
--(axis cs:7301,99.0416892536491)
--(axis cs:7401,99.0421028041611)
--(axis cs:7501,99.0400536876792)
--(axis cs:7601,99.0396516459908)
--(axis cs:7701,99.0399485818606)
--(axis cs:7801,99.0399435095281)
--(axis cs:7901,99.0397667203077)
--(axis cs:8001,99.0388141688632)
--(axis cs:8101,99.0383330225719)
--(axis cs:8201,99.0362770896241)
--(axis cs:8301,99.0354599377689)
--(axis cs:8401,99.0345670024471)
--(axis cs:8501,99.0342486652101)
--(axis cs:8601,99.0339728209987)
--(axis cs:8701,99.0341987870593)
--(axis cs:8801,99.0344270199753)
--(axis cs:8901,99.0329559829591)
--(axis cs:9001,99.0327556776152)
--(axis cs:9101,99.0341982660007)
--(axis cs:9201,99.0340913906552)
--(axis cs:9301,99.0331695056373)
--(axis cs:9401,99.0321690487817)
--(axis cs:9501,99.0323274219376)
--(axis cs:9601,99.0336052752012)
--(axis cs:9701,99.0345246710778)
--(axis cs:9801,99.0347891713201)
--(axis cs:9901,99.0345896244577)
--(axis cs:9901,99.0056446953079)
--(axis cs:9901,99.0056446953079)
--(axis cs:9801,99.00525368145)
--(axis cs:9701,99.0041682471779)
--(axis cs:9601,99.0029513334854)
--(axis cs:9501,99.0013595583802)
--(axis cs:9401,99.0004423755349)
--(axis cs:9301,99.0009795105974)
--(axis cs:9201,99.0016416818368)
--(axis cs:9101,99.0017933832686)
--(axis cs:9001,99.0004828514372)
--(axis cs:8901,99.0004020666983)
--(axis cs:8801,99.0014416313143)
--(axis cs:8701,99.0017005348578)
--(axis cs:8601,99.0015881602826)
--(axis cs:8501,99.0018412065697)
--(axis cs:8401,99.0013311049212)
--(axis cs:8301,99.0018006332467)
--(axis cs:8201,99.0033497851472)
--(axis cs:8101,99.0048172057949)
--(axis cs:8001,99.0050703455725)
--(axis cs:7901,99.0056908167129)
--(axis cs:7801,99.0054506707051)
--(axis cs:7701,99.0044975939607)
--(axis cs:7601,99.0035795078048)
--(axis cs:7501,99.0045616969362)
--(axis cs:7401,99.0055556203762)
--(axis cs:7301,99.0052056922487)
--(axis cs:7201,99.003777603552)
--(axis cs:7101,99.004666328007)
--(axis cs:7001,99.0055204971907)
--(axis cs:6901,99.005742496291)
--(axis cs:6801,99.0041313042876)
--(axis cs:6701,99.0050925981401)
--(axis cs:6601,99.0037398325339)
--(axis cs:6501,99.0017238616812)
--(axis cs:6401,99.0032858484277)
--(axis cs:6301,98.9996244016129)
--(axis cs:6201,99.0005086755127)
--(axis cs:6101,98.9976328258285)
--(axis cs:6001,98.9976498121962)
--(axis cs:5901,98.9983356571404)
--(axis cs:5801,98.9988982704991)
--(axis cs:5701,98.9976130973773)
--(axis cs:5601,98.9972789479067)
--(axis cs:5501,98.9974614269109)
--(axis cs:5401,98.9937448532481)
--(axis cs:5301,98.9937806247345)
--(axis cs:5201,98.9922608843376)
--(axis cs:5101,98.9922641590857)
--(axis cs:5001,98.9903609025031)
--(axis cs:4901,98.9916642558554)
--(axis cs:4801,98.9919682969364)
--(axis cs:4701,98.9892786132739)
--(axis cs:4601,98.9904096839775)
--(axis cs:4501,98.9865432333911)
--(axis cs:4401,98.9861583461622)
--(axis cs:4301,98.9878603300135)
--(axis cs:4201,98.9882393946135)
--(axis cs:4101,98.9866522156574)
--(axis cs:4001,98.9842241777982)
--(axis cs:3901,98.9824499954984)
--(axis cs:3801,98.9813060449876)
--(axis cs:3701,98.9804343718052)
--(axis cs:3601,98.9778954596209)
--(axis cs:3501,98.9772453207142)
--(axis cs:3401,98.9770121498833)
--(axis cs:3301,98.9715093426925)
--(axis cs:3201,98.9711641088267)
--(axis cs:3101,98.9751139547162)
--(axis cs:3001,98.9764576157066)
--(axis cs:2901,98.978197010949)
--(axis cs:2801,98.9789174369556)
--(axis cs:2701,98.9795938924397)
--(axis cs:2601,98.9859614449427)
--(axis cs:2501,98.983538106097)
--(axis cs:2401,98.9907694942442)
--(axis cs:2301,98.9864785486797)
--(axis cs:2201,98.987170624475)
--(axis cs:2101,98.9826989910056)
--(axis cs:2001,98.9766396218022)
--(axis cs:1901,98.9733122741589)
--(axis cs:1801,98.9827380039921)
--(axis cs:1701,98.9794777543822)
--(axis cs:1601,98.9758735290969)
--(axis cs:1501,98.968382833256)
--(axis cs:1401,98.9634533412308)
--(axis cs:1301,98.9606961187182)
--(axis cs:1201,98.9439275262788)
--(axis cs:1101,98.9401997383378)
--(axis cs:1001,98.9375983188781)
--(axis cs:901,98.9355756569787)
--(axis cs:801,98.9299019433114)
--(axis cs:701,98.9150016041625)
--(axis cs:601,98.8974434222362)
--(axis cs:501,98.8746624614554)
--(axis cs:401,98.8649282810719)
--(axis cs:301,98.836015311019)
--(axis cs:201,98.8316561804465)
--(axis cs:101,98.8023121108819)
--(axis cs:1,89.6951411141489)
--cycle;

\addplot [very thick, black]
table {%
1 91.1
101 105.935841584158
201 106.250597014925
301 106.484551495017
401 106.608279301746
501 106.684111776447
601 106.704891846922
701 106.744507845934
801 106.76240948814
901 106.802930077691
1001 106.822887112887
1101 106.863424159855
1201 106.875978351374
1301 106.890407378939
1401 106.893633119201
1501 106.904883411059
1601 106.909175515303
1701 106.91062904174
1801 106.909988895058
1901 106.926854287217
2001 106.929340329835
2101 106.944359828653
2201 106.946855974557
2301 106.950260756193
2401 106.956426488963
2501 106.962167133147
2601 106.970276816609
2701 106.973680118475
2801 106.982227775794
2901 106.989544984488
3001 106.996151282906
3101 107.001638181232
3201 107.001786941581
3301 107.003720084823
3401 107.003140252867
3501 107.012270779777
3601 107.0155651208
3701 107.019032693867
3801 107.020052617732
3901 107.018387592925
4001 107.022984253937
4101 107.021075347476
4201 107.025222566056
4301 107.025970704487
4401 107.024339922745
4501 107.024863363697
4601 107.029030645512
4701 107.030946607105
4801 107.029960424911
4901 107.035215262191
5001 107.038142371526
5101 107.040517545579
5201 107.043224379927
5301 107.048622901339
5401 107.050633216071
5501 107.052603163061
5601 107.055300839136
5701 107.056695316611
5801 107.057334942251
5901 107.061559057787
6001 107.061848025329
6101 107.061060481888
6201 107.06302693114
6301 107.063935883193
6401 107.065002343384
6501 107.065094600831
6601 107.064800787759
6701 107.06495597672
6801 107.065888839876
6901 107.067178669758
7001 107.068144550778
7101 107.069947894663
7201 107.069040411054
7301 107.068379674017
7401 107.068833941359
7501 107.069972003733
7601 107.072703591633
7701 107.074756525127
7801 107.077273426484
7901 107.077566130869
8001 107.079122609674
8101 107.08077150969
8201 107.081098646507
8301 107.081594988556
8401 107.082571122485
8501 107.083876014587
8601 107.084954075108
8701 107.086113090449
8801 107.086570844222
8901 107.087394674756
9001 107.088044661704
9101 107.088358422151
9201 107.088564286491
9301 107.087779808623
9401 107.088525688757
9501 107.088354910009
9601 107.090107280492
9701 107.091722502835
9801 107.091710029589
9901 107.091621048379
};
\addplot [very thick, steelblue31119180,mark=diamond, mark options={fill=none},mark size=3pt,  mark repeat=5]
table {%
1 89.7
101 104.553861386139
201 105.457910447761
301 105.938504983389
401 106.170972568579
501 106.335409181637
601 106.444825291181
701 106.547560627675
801 106.614057428215
901 106.662519422863
1001 106.705004995005
1101 106.75098092643
1201 106.757976686095
1301 106.780707148347
1401 106.799578872234
1501 106.816009327115
1601 106.819837601499
1701 106.837201646091
1801 106.856174347585
1901 106.863082588112
2001 106.870329835082
2101 106.888400761542
2201 106.901990004543
2301 106.90800956106
2401 106.913952519783
2501 106.920303878449
2601 106.925955401769
2701 106.932310255461
2801 106.937118886112
2901 106.944319200276
3001 106.949533488837
3101 106.955885198323
3201 106.955398313027
3301 106.961611632838
3401 106.967547780065
3501 106.971136818052
3601 106.968889197445
3701 106.973390975412
3801 106.978763483294
3901 106.982653165855
4001 106.985683579105
4101 106.99069982931
4201 106.997179243037
4301 106.996056730993
4401 107.001547375596
4501 107.010575427683
4601 107.011956096501
4701 107.009646883642
4801 107.012007915018
4901 107.012840236686
5001 107.015582883423
5101 107.0175024505
5201 107.023520476831
5301 107.024774570836
5401 107.027374560267
5501 107.029243773859
5601 107.034018925192
5701 107.033878266971
5801 107.031649715566
5901 107.033216403999
6001 107.034987502083
6101 107.036254712342
6201 107.03364134817
6301 107.036502142517
6401 107.038006561475
6501 107.041524380865
6601 107.044208453265
6701 107.044139680645
6801 107.046113806793
6901 107.045902043182
7001 107.044929295815
7101 107.045747077876
7201 107.047362866269
7301 107.05040679359
7401 107.051488987975
7501 107.054474070124
7601 107.055874227075
7701 107.056344630567
7801 107.057508011793
7901 107.058467282622
8001 107.061441069866
8101 107.06489569189
8201 107.065127423485
8301 107.066582339477
8401 107.066964647066
8501 107.067814374779
8601 107.069758167655
8701 107.071672221584
8801 107.071857743438
8901 107.07380294349
9001 107.075423841795
9101 107.076691572355
9201 107.078304532116
9301 107.077701322438
9401 107.0782725242
9501 107.079364277445
9601 107.080830121862
9701 107.079486650861
9801 107.080847872666
9901 107.08137965862
};
\addplot [very thick, mediumpurple148103189, mark=star, mark options={fill=none},mark size=3pt,  mark repeat=5]
table {%
1 94.96
101 105.142376237624
201 105.388358208955
301 105.319003322259
401 105.396907730673
501 105.344191616766
601 105.356256239601
701 105.375520684736
801 105.371435705368
901 105.386015538291
1001 105.369190809191
1101 105.384332425068
1201 105.394054954205
1301 105.409577248271
1401 105.407180585296
1501 105.411738840773
1601 105.413972517177
1701 105.416043503821
1801 105.40568573015
1901 105.409000526039
2001 105.404967516242
2101 105.411194669205
2201 105.414784189005
2301 105.421060408518
2401 105.422211578509
2501 105.420007996801
2601 105.413237216455
2701 105.41635690485
2801 105.41713673688
2901 105.42014822475
3001 105.41946017994
3101 105.425985166075
3201 105.423292721025
3301 105.42350802787
3401 105.417988826816
3501 105.418908883176
3601 105.420483199111
3701 105.425176979195
3801 105.423430676138
3901 105.424014355294
4001 105.424343914021
4101 105.427754206291
4201 105.430464175196
4301 105.432645896303
4401 105.433705975915
4501 105.435047767163
4601 105.432182134319
4701 105.432101680494
4801 105.432351593418
4901 105.432919812283
5001 105.432181563687
5101 105.433869829445
5201 105.436510286483
5301 105.434567062818
5401 105.435802629143
5501 105.433712052354
5601 105.432428137833
5701 105.432680231538
5801 105.432513359766
5901 105.430606676834
6001 105.430164972505
6101 105.431017865924
6201 105.432767295597
6301 105.434135851452
6401 105.435683486955
6501 105.436414397785
6601 105.438601727011
6701 105.437610804358
6801 105.436700485223
6901 105.435671641791
7001 105.433329524354
7101 105.43641177299
7201 105.433214831273
7301 105.43224763731
7401 105.431938927172
7501 105.433856819091
7601 105.434742797
7701 105.436131671211
7801 105.435351877964
7901 105.436611821288
8001 105.4376127984
8101 105.438890260462
8201 105.4390013413
8301 105.437679797615
8401 105.437888346625
8501 105.435858134337
8601 105.435999302407
8701 105.435004022526
8801 105.43525281218
8901 105.436155488147
9001 105.435131652039
9101 105.434062190968
9201 105.432611672644
9301 105.433425438125
9401 105.434486756728
9501 105.435412061888
9601 105.435632746589
9701 105.436190083497
9801 105.433896541169
9901 105.435304514695
};
\addplot [very thick, orchid227119194, mark=*, mark options={fill=none},mark size=2pt, mark repeat=5]
table {%
1 92.91
101 96.259900990099
201 96.3751243781094
301 96.3383388704319
401 96.3508229426434
501 96.3639121756487
601 96.3677537437604
701 96.3691298145506
801 96.3704744069912
901 96.3880688124306
1001 96.3945554445554
1101 96.4019346049046
1201 96.400657785179
1301 96.4004458109146
1401 96.3946324054247
1501 96.3834177215189
1601 96.3879637726421
1701 96.3848442092887
1801 96.3792282065519
1901 96.3825092056812
2001 96.3875462268866
2101 96.3913136601618
2201 96.3924125397547
2301 96.3943155149935
2401 96.3953852561433
2501 96.3970731707317
2601 96.3993194925029
2701 96.4004701962235
2801 96.400007140307
2901 96.4051292657704
3001 96.4050149950016
3101 96.404746855853
3201 96.4081755701343
3301 96.4100393820055
3401 96.4085063216701
3501 96.4147614967152
3601 96.4140710913635
3701 96.4139475817347
3801 96.4126045777427
3901 96.4102640348629
4001 96.4099650087478
4101 96.4098634479395
4201 96.411632944537
4301 96.4112020460358
4401 96.4107634628494
4501 96.4108531437458
4601 96.410749836992
4701 96.4105105296745
4801 96.4117787960841
4901 96.4114078759437
5001 96.4114677064587
5101 96.411625171535
5201 96.4142318784849
5301 96.4156461045086
5401 96.4135586002592
5501 96.4114851845119
5601 96.4086555972148
5701 96.4075706016489
5801 96.4075728322703
5901 96.4069428910355
6001 96.4079470088319
6101 96.4093230617931
6201 96.4073084986292
6301 96.4068592286939
6401 96.4077050460865
6501 96.4089032456545
6601 96.4097606423269
6701 96.4117818236084
6801 96.4095897662109
6901 96.4104158817563
7001 96.4102742465362
7101 96.4099028305872
7201 96.4096805999167
7301 96.4107560608136
7401 96.4115497905688
7501 96.4102213038262
7601 96.4094092882516
7701 96.4082911310219
7801 96.4092462504807
7901 96.4099746867484
8001 96.4092200974878
8101 96.4082767559561
8201 96.4092159492745
8301 96.409102517769
8401 96.410371384359
8501 96.4104999411834
8601 96.4113184513429
8701 96.4099275945294
8801 96.4098057038973
8901 96.4111616672284
9001 96.4112943006333
9101 96.4096978354027
9201 96.4098239321813
9301 96.4090216105795
9401 96.4083682586959
9501 96.4068013893274
9601 96.4079731277992
9701 96.4077796103495
9801 96.4067207427813
9901 96.4073527926472
};
\addplot [very thick, goldenrod18818934,mark=square, mark options={fill},mark size=2pt,  mark repeat=5, mark phase=5]
table {%
1 91.23
101 98.9743564356436
201 98.965223880597
301 98.9465780730897
401 98.9619201995012
501 98.9567664670659
601 98.975690515807
701 98.9862767475036
801 98.9967915106117
901 98.9996892341842
1001 98.997042957043
1101 98.997157129882
1201 98.9971606994172
1301 99.0121137586472
1401 99.0141970021413
1501 99.0174283810793
1601 99.0214928169894
1701 99.0234215167549
1801 99.0247307051638
1901 99.0130983692793
2001 99.0162818590704
2101 99.0199095668729
2201 99.0231758291685
2301 99.0201390699696
2401 99.0245647646814
2501 99.0177888844462
2601 99.018985005767
2701 99.0123509811181
2801 99.0104855408783
2901 99.0093933126508
3001 99.0071376207931
3101 99.0051918735892
3201 99.0013620743517
3301 99.001045137837
3401 99.0066039400176
3501 99.0061553841759
3601 99.0061482921411
3701 99.0077492569576
3801 99.00833201789
3901 99.0088695206358
4001 99.0105873531618
4101 99.0116849548891
4201 99.0129040704594
4301 99.012334340851
4401 99.0109202453988
4501 99.011435236614
4601 99.0147141925668
4701 99.0131418847054
4801 99.0154530306186
4901 99.0147867782085
5001 99.0131593681264
5101 99.0148225838071
5201 99.0148529129013
5301 99.0166798717223
5401 99.0160211072024
5501 99.0196346118887
5601 99.0191215854312
5701 99.0195035958604
5801 99.0205343906223
5901 99.0206676834435
6001 99.0195700716547
6101 99.0189001802983
6201 99.021427189163
6301 99.0200983970799
6401 99.0233604124355
6501 99.0213705583756
6601 99.0231283138918
6701 99.0243859125503
6801 99.0229480958682
6901 99.0245341254891
7001 99.0244165119268
7101 99.0233403745951
7201 99.0222580197194
7301 99.0234474729489
7401 99.0238292122687
7501 99.0223076923077
7601 99.0216155768978
7701 99.0222230879106
7801 99.0226970901166
7901 99.0227287685103
8001 99.0219422572179
8101 99.0215751141834
8201 99.0198134373856
8301 99.0186302855078
8401 99.0179490536841
8501 99.0180449358899
8601 99.0177804906406
8701 99.0179496609585
8801 99.0179343256448
8901 99.0166790248287
9001 99.0166192645262
9101 99.0179958246347
9201 99.017866536246
9301 99.0170745081174
9401 99.0163057121583
9501 99.0168434901589
9601 99.0182783043433
9701 99.0193464591279
9801 99.020021426385
9901 99.0201171598828
};
\end{axis}

\end{tikzpicture}

%% file: Figs/TikzPlot_Balanced_RunningAvgTotalReward.tex
\begin{tikzpicture}[scale=.65,transform shape]

\definecolor{darkgray176}{RGB}{176,176,176}
\definecolor{forestgreen4416044}{RGB}{44,160,44}
\definecolor{goldenrod18818934}{RGB}{188,189,34}
\definecolor{mediumpurple148103189}{RGB}{148,103,189}
\definecolor{orchid227119194}{RGB}{227,119,194}
\definecolor{steelblue31119180}{RGB}{31,119,180}

\definecolor{black}{RGB}{0,0,0}
\definecolor{darkgray}{RGB}{160,160,160}
\definecolor{lightgray}{RGB}{200,200,200}

\begin{axis}[
tick align=outside,
tick pos=left,
x grid style={darkgray176},
xlabel={Time},
xmajorgrids,
xmin=-50, xmax=3000,
xtick={0,1000,2000,3000},
xtick style={color=black},
y grid style={darkgray176},
ytick={90,95,100},
ymajorgrids,
ymin=88, ymax=100,
ytick style={color=black},
width=9cm,  
height=8cm,  
]
\path [draw=black, fill=black, opacity=0.2]
(axis cs:1,84.6748858682211)
--(axis cs:1,88.2651141317789)
--(axis cs:101,97.8176798884734)
--(axis cs:201,98.0086405532551)
--(axis cs:301,98.147030868722)
--(axis cs:401,98.2778941240064)
--(axis cs:501,98.3251824502401)
--(axis cs:601,98.3439865284821)
--(axis cs:701,98.3438822051678)
--(axis cs:801,98.3609337303779)
--(axis cs:901,98.396251329645)
--(axis cs:1001,98.3901177971497)
--(axis cs:1101,98.4017146229128)
--(axis cs:1201,98.4105229373283)
--(axis cs:1301,98.4027791326191)
--(axis cs:1401,98.3938529845476)
--(axis cs:1501,98.3990788409749)
--(axis cs:1601,98.3999305983393)
--(axis cs:1701,98.4070521063672)
--(axis cs:1801,98.4179638801581)
--(axis cs:1901,98.4186456383945)
--(axis cs:2001,98.4284891898082)
--(axis cs:2101,98.426365830425)
--(axis cs:2201,98.420425275836)
--(axis cs:2301,98.4164369617139)
--(axis cs:2401,98.4238444278562)
--(axis cs:2501,98.4230968185089)
--(axis cs:2601,98.4293559185136)
--(axis cs:2701,98.4331923689802)
--(axis cs:2801,98.4295033309294)
--(axis cs:2901,98.4404510866808)
--(axis cs:3001,98.4406288744561)
--(axis cs:3101,98.4431322434879)
--(axis cs:3201,98.4426585327632)
--(axis cs:3301,98.4473282934453)
--(axis cs:3401,98.4498594924457)
--(axis cs:3501,98.4514840890638)
--(axis cs:3601,98.4455050837156)
--(axis cs:3701,98.4430217439956)
--(axis cs:3801,98.4453267593802)
--(axis cs:3901,98.4464686452685)
--(axis cs:4001,98.4449544516042)
--(axis cs:4101,98.4479750489878)
--(axis cs:4201,98.4486012073578)
--(axis cs:4301,98.4529528741256)
--(axis cs:4401,98.4516753167658)
--(axis cs:4501,98.4564381750569)
--(axis cs:4601,98.458233835917)
--(axis cs:4701,98.460962036902)
--(axis cs:4801,98.4603056864236)
--(axis cs:4901,98.4627266550659)
--(axis cs:5001,98.4622516111891)
--(axis cs:5101,98.4647162634173)
--(axis cs:5201,98.4665006514253)
--(axis cs:5301,98.4683962567637)
--(axis cs:5401,98.4675314468423)
--(axis cs:5501,98.4666156643577)
--(axis cs:5601,98.4661537775876)
--(axis cs:5701,98.4691945922572)
--(axis cs:5801,98.4682299815775)
--(axis cs:5901,98.4687978860615)
--(axis cs:6001,98.4718754930922)
--(axis cs:6101,98.4731714923992)
--(axis cs:6201,98.4722101465691)
--(axis cs:6301,98.473882483792)
--(axis cs:6401,98.4756622517063)
--(axis cs:6501,98.4790398923047)
--(axis cs:6601,98.4813673993633)
--(axis cs:6701,98.4794947712942)
--(axis cs:6801,98.4801897127827)
--(axis cs:6901,98.4791014922144)
--(axis cs:7001,98.4803303921646)
--(axis cs:7101,98.4809935240309)
--(axis cs:7201,98.4789785219557)
--(axis cs:7301,98.4792160332038)
--(axis cs:7401,98.4806172493967)
--(axis cs:7501,98.4821412940358)
--(axis cs:7601,98.4824377775161)
--(axis cs:7701,98.4831227656199)
--(axis cs:7801,98.4851261914879)
--(axis cs:7901,98.4840219246341)
--(axis cs:8001,98.4856503373909)
--(axis cs:8101,98.4840642148887)
--(axis cs:8201,98.4835482192681)
--(axis cs:8301,98.4854790331298)
--(axis cs:8401,98.4856324025085)
--(axis cs:8501,98.4855382078632)
--(axis cs:8601,98.4864533668053)
--(axis cs:8701,98.4859566413152)
--(axis cs:8801,98.4884372345021)
--(axis cs:8901,98.4889189178491)
--(axis cs:9001,98.4902511641083)
--(axis cs:9101,98.4896015775189)
--(axis cs:9201,98.4886895486023)
--(axis cs:9301,98.4886725612651)
--(axis cs:9401,98.489146988739)
--(axis cs:9501,98.4905482759901)
--(axis cs:9601,98.4916555711846)
--(axis cs:9701,98.4914544640068)
--(axis cs:9801,98.491592941935)
--(axis cs:9901,98.4923153381327)
--(axis cs:9901,98.4510762384757)
--(axis cs:9901,98.4510762384757)
--(axis cs:9801,98.4497844685334)
--(axis cs:9701,98.4490959947089)
--(axis cs:9601,98.4492026727483)
--(axis cs:9501,98.4486918040014)
--(axis cs:9401,98.4465641058254)
--(axis cs:9301,98.4459731757524)
--(axis cs:9201,98.4453219718845)
--(axis cs:9101,98.445407762114)
--(axis cs:9001,98.445082687686)
--(axis cs:8901,98.4436527033171)
--(axis cs:8801,98.4427660378533)
--(axis cs:8701,98.4395438758668)
--(axis cs:8601,98.4395668634006)
--(axis cs:8501,98.4384189736448)
--(axis cs:8401,98.4380219243573)
--(axis cs:8301,98.4369592273208)
--(axis cs:8201,98.435364108497)
--(axis cs:8101,98.4357135903204)
--(axis cs:8001,98.4357469879435)
--(axis cs:7901,98.4342390549887)
--(axis cs:7801,98.4343046507119)
--(axis cs:7701,98.4319220337568)
--(axis cs:7601,98.4310025592818)
--(axis cs:7501,98.4309822894864)
--(axis cs:7401,98.4297056799372)
--(axis cs:7301,98.4282322615502)
--(axis cs:7201,98.428114937286)
--(axis cs:7101,98.4295233045848)
--(axis cs:7001,98.4275142014649)
--(axis cs:6901,98.4257354879335)
--(axis cs:6801,98.4267239763807)
--(axis cs:6701,98.4270714128575)
--(axis cs:6601,98.4281523703687)
--(axis cs:6501,98.4253102076799)
--(axis cs:6401,98.4220193605418)
--(axis cs:6301,98.4195915679457)
--(axis cs:6201,98.4183978198881)
--(axis cs:6101,98.4177972012576)
--(axis cs:6001,98.4147634004256)
--(axis cs:5901,98.4123544609983)
--(axis cs:5801,98.4114666224562)
--(axis cs:5701,98.41166139792)
--(axis cs:5601,98.4083864830802)
--(axis cs:5501,98.4078907890145)
--(axis cs:5401,98.4072732189603)
--(axis cs:5301,98.4065330018667)
--(axis cs:5201,98.4037973681864)
--(axis cs:5101,98.4016197491294)
--(axis cs:5001,98.3981443096268)
--(axis cs:4901,98.3982445752954)
--(axis cs:4801,98.3944662360926)
--(axis cs:4701,98.3931328365292)
--(axis cs:4601,98.3904208043785)
--(axis cs:4501,98.3896138133901)
--(axis cs:4401,98.3846437016392)
--(axis cs:4301,98.3861868608197)
--(axis cs:4201,98.3822961980219)
--(axis cs:4101,98.3815445803709)
--(axis cs:4001,98.3758203546943)
--(axis cs:3901,98.3780430184076)
--(axis cs:3801,98.3760044692438)
--(axis cs:3701,98.3708555864557)
--(axis cs:3601,98.3737340165343)
--(axis cs:3501,98.3781474447836)
--(axis cs:3401,98.3775206898536)
--(axis cs:3301,98.3740894587509)
--(axis cs:3201,98.3678881713918)
--(axis cs:3101,98.3682189335518)
--(axis cs:3001,98.3644227749941)
--(axis cs:2901,98.3644782480314)
--(axis cs:2801,98.3542024884209)
--(axis cs:2701,98.3556339916271)
--(axis cs:2601,98.3522280876378)
--(axis cs:2501,98.3428847888481)
--(axis cs:2401,98.3431443268293)
--(axis cs:2301,98.3353926775734)
--(axis cs:2201,98.334531561965)
--(axis cs:2101,98.3408593004651)
--(axis cs:2001,98.3393568871534)
--(axis cs:1901,98.3286768234677)
--(axis cs:1801,98.3207146317797)
--(axis cs:1701,98.3085975115047)
--(axis cs:1601,98.2967214941029)
--(axis cs:1501,98.2926866486986)
--(axis cs:1401,98.2845053309414)
--(axis cs:1301,98.2940233270272)
--(axis cs:1201,98.2988359302821)
--(axis cs:1101,98.2862417803571)
--(axis cs:1001,98.2743577273258)
--(axis cs:901,98.272094952264)
--(axis cs:801,98.2280800024561)
--(axis cs:701,98.2063888362017)
--(axis cs:601,98.1943163001369)
--(axis cs:501,98.1623624599395)
--(axis cs:401,98.086844030607)
--(axis cs:301,97.9292482010454)
--(axis cs:201,97.7580261134116)
--(axis cs:101,97.4908349630117)
--(axis cs:1,84.6748858682211)
--cycle;

\path [draw=steelblue31119180, fill=steelblue31119180, opacity=0.2]
(axis cs:1,77.9611805184702)
--(axis cs:1,81.3188194815298)
--(axis cs:101,92.2188445518848)
--(axis cs:201,95.0609936720588)
--(axis cs:301,96.1346588925427)
--(axis cs:401,96.6470521276787)
--(axis cs:501,97.0113948728328)
--(axis cs:601,97.240602241393)
--(axis cs:701,97.4131812146497)
--(axis cs:801,97.5501392671254)
--(axis cs:901,97.6314340744511)
--(axis cs:1001,97.7117962710782)
--(axis cs:1101,97.7886895494247)
--(axis cs:1201,97.8581495194635)
--(axis cs:1301,97.9181274415571)
--(axis cs:1401,97.967360873105)
--(axis cs:1501,98.0057803612566)
--(axis cs:1601,98.0352207012069)
--(axis cs:1701,98.0574661519601)
--(axis cs:1801,98.0838489987849)
--(axis cs:1901,98.1016781157281)
--(axis cs:2001,98.1245382999149)
--(axis cs:2101,98.1517293604)
--(axis cs:2201,98.1714278083937)
--(axis cs:2301,98.1829081969135)
--(axis cs:2401,98.1939042625372)
--(axis cs:2501,98.2080066283703)
--(axis cs:2601,98.2268748941691)
--(axis cs:2701,98.2327364454568)
--(axis cs:2801,98.2362982498961)
--(axis cs:2901,98.2404481408023)
--(axis cs:3001,98.2507937845175)
--(axis cs:3101,98.2571877475168)
--(axis cs:3201,98.2686712797166)
--(axis cs:3301,98.2743617482881)
--(axis cs:3401,98.2785668988933)
--(axis cs:3501,98.2835784734131)
--(axis cs:3601,98.2919713245974)
--(axis cs:3701,98.2977196789505)
--(axis cs:3801,98.2999269347846)
--(axis cs:3901,98.3026102669261)
--(axis cs:4001,98.309148723632)
--(axis cs:4101,98.3153378875526)
--(axis cs:4201,98.3192408745389)
--(axis cs:4301,98.3248565070515)
--(axis cs:4401,98.3256765647481)
--(axis cs:4501,98.3274950468066)
--(axis cs:4601,98.3292900010022)
--(axis cs:4701,98.3318843198354)
--(axis cs:4801,98.3341076527101)
--(axis cs:4901,98.3349444864147)
--(axis cs:5001,98.3370644076458)
--(axis cs:5101,98.3410732437092)
--(axis cs:5201,98.3450979392881)
--(axis cs:5301,98.3495946875304)
--(axis cs:5401,98.355188214746)
--(axis cs:5501,98.3535326280262)
--(axis cs:5601,98.3540864215676)
--(axis cs:5701,98.358015729811)
--(axis cs:5801,98.362753377617)
--(axis cs:5901,98.3673113973885)
--(axis cs:6001,98.3705304810623)
--(axis cs:6101,98.3739960526727)
--(axis cs:6201,98.3776997154369)
--(axis cs:6301,98.3789401422986)
--(axis cs:6401,98.3810233219368)
--(axis cs:6501,98.3801808951799)
--(axis cs:6601,98.3831176702719)
--(axis cs:6701,98.385329642197)
--(axis cs:6801,98.3884898490302)
--(axis cs:6901,98.3889452919334)
--(axis cs:7001,98.39143744821)
--(axis cs:7101,98.3933729060488)
--(axis cs:7201,98.3964186301179)
--(axis cs:7301,98.3999024645308)
--(axis cs:7401,98.4034489402674)
--(axis cs:7501,98.4058220944342)
--(axis cs:7601,98.4081973070688)
--(axis cs:7701,98.4072296799992)
--(axis cs:7801,98.4096065692757)
--(axis cs:7901,98.4102467648373)
--(axis cs:8001,98.4123469696566)
--(axis cs:8101,98.4125701486045)
--(axis cs:8201,98.4131915129426)
--(axis cs:8301,98.4165394417229)
--(axis cs:8401,98.4186799286097)
--(axis cs:8501,98.4208641012407)
--(axis cs:8601,98.4240803443327)
--(axis cs:8701,98.4241031610226)
--(axis cs:8801,98.4242285328337)
--(axis cs:8901,98.4262180955003)
--(axis cs:9001,98.4247728224656)
--(axis cs:9101,98.424632911195)
--(axis cs:9201,98.4252836249735)
--(axis cs:9301,98.4272849814)
--(axis cs:9401,98.428426026088)
--(axis cs:9501,98.4305567810359)
--(axis cs:9601,98.4318466727926)
--(axis cs:9701,98.4318393592077)
--(axis cs:9801,98.4318802766565)
--(axis cs:9901,98.4353918634953)
--(axis cs:9901,98.3944313866815)
--(axis cs:9901,98.3944313866815)
--(axis cs:9801,98.3906500773889)
--(axis cs:9701,98.3899192223818)
--(axis cs:9601,98.3899635553086)
--(axis cs:9501,98.3881086226059)
--(axis cs:9401,98.3858533058981)
--(axis cs:9301,98.3842105567141)
--(axis cs:9201,98.3819655870686)
--(axis cs:9101,98.3814389490402)
--(axis cs:9001,98.3806665731571)
--(axis cs:8901,98.3817854996015)
--(axis cs:8801,98.3792278925726)
--(axis cs:8701,98.3784160896383)
--(axis cs:8601,98.3779682546674)
--(axis cs:8501,98.374432922639)
--(axis cs:8401,98.3722830519879)
--(axis cs:8301,98.3707030591806)
--(axis cs:8201,98.3678595783877)
--(axis cs:8101,98.3667163592341)
--(axis cs:8001,98.3660982246941)
--(axis cs:7901,98.365161411343)
--(axis cs:7801,98.3639122103679)
--(axis cs:7701,98.3615380125082)
--(axis cs:7601,98.3620460819589)
--(axis cs:7501,98.3598838114449)
--(axis cs:7401,98.3571644903501)
--(axis cs:7301,98.3534053015286)
--(axis cs:7201,98.3498388341232)
--(axis cs:7101,98.3469397259749)
--(axis cs:7001,98.3446516819144)
--(axis cs:6901,98.342079197271)
--(axis cs:6801,98.3411940209889)
--(axis cs:6701,98.3377266180627)
--(axis cs:6601,98.3350068563151)
--(axis cs:6501,98.3327678819313)
--(axis cs:6401,98.334346151583)
--(axis cs:6301,98.3318803623831)
--(axis cs:6201,98.3309408264112)
--(axis cs:6101,98.3269349422462)
--(axis cs:6001,98.3230005970912)
--(axis cs:5901,98.3209651659059)
--(axis cs:5801,98.3166019059548)
--(axis cs:5701,98.3129718162336)
--(axis cs:5601,98.3078632302803)
--(axis cs:5501,98.3060492661748)
--(axis cs:5401,98.3080074897532)
--(axis cs:5301,98.3016371555181)
--(axis cs:5201,98.296078757501)
--(axis cs:5101,98.2915399693862)
--(axis cs:5001,98.2865348724981)
--(axis cs:4901,98.2827702656767)
--(axis cs:4801,98.2808600623491)
--(axis cs:4701,98.2776072776971)
--(axis cs:4601,98.2744048479435)
--(axis cs:4501,98.2705542755662)
--(axis cs:4401,98.2674659028729)
--(axis cs:4301,98.2656945275914)
--(axis cs:4201,98.2610114463371)
--(axis cs:4101,98.2561958846982)
--(axis cs:4001,98.2511811938887)
--(axis cs:3901,98.245551742815)
--(axis cs:3801,98.2393995582435)
--(axis cs:3701,98.2357469516899)
--(axis cs:3601,98.2279731352748)
--(axis cs:3501,98.2179867936534)
--(axis cs:3401,98.2111596521211)
--(axis cs:3301,98.2055837227813)
--(axis cs:3201,98.1991387796399)
--(axis cs:3101,98.187068943873)
--(axis cs:3001,98.1798626633332)
--(axis cs:2901,98.1674939481326)
--(axis cs:2801,98.1608027854484)
--(axis cs:2701,98.1560677011556)
--(axis cs:2601,98.1473119570421)
--(axis cs:2501,98.1268754188109)
--(axis cs:2401,98.11041893613)
--(axis cs:2301,98.0952925853551)
--(axis cs:2201,98.080530392424)
--(axis cs:2101,98.0586180931935)
--(axis cs:2001,98.0272558030336)
--(axis cs:1901,98.0018463450821)
--(axis cs:1801,97.9772170756183)
--(axis cs:1701,97.9527513671463)
--(axis cs:1601,97.9263782994177)
--(axis cs:1501,97.8912083129605)
--(axis cs:1401,97.853138770007)
--(axis cs:1301,97.7981369704336)
--(axis cs:1201,97.7306931116772)
--(axis cs:1101,97.658322258023)
--(axis cs:1001,97.5726193133373)
--(axis cs:901,97.4829943384235)
--(axis cs:801,97.3890617316263)
--(axis cs:701,97.2376889707995)
--(axis cs:601,97.0542064108533)
--(axis cs:501,96.8056510353509)
--(axis cs:401,96.4017758024958)
--(axis cs:301,95.8479324695835)
--(axis cs:201,94.6784093130157)
--(axis cs:101,91.6682841609866)
--(axis cs:1,77.9611805184702)
--cycle;

\path [draw=mediumpurple148103189, fill=mediumpurple148103189, opacity=0.2]
(axis cs:1,82.488098316915)
--(axis cs:1,86.451901683085)
--(axis cs:101,92.1042032496267)
--(axis cs:201,92.1818306643527)
--(axis cs:301,92.196686280924)
--(axis cs:401,92.2416874733755)
--(axis cs:501,92.2499458962022)
--(axis cs:601,92.2600321063077)
--(axis cs:701,92.257286769649)
--(axis cs:801,92.2607838933606)
--(axis cs:901,92.2642814542804)
--(axis cs:1001,92.2683352124742)
--(axis cs:1101,92.2750416699954)
--(axis cs:1201,92.2670453350194)
--(axis cs:1301,92.2782204029765)
--(axis cs:1401,92.2667794054558)
--(axis cs:1501,92.2711312248331)
--(axis cs:1601,92.2789344380704)
--(axis cs:1701,92.2740219507224)
--(axis cs:1801,92.2771296125489)
--(axis cs:1901,92.2785119705786)
--(axis cs:2001,92.278341244846)
--(axis cs:2101,92.2786584808604)
--(axis cs:2201,92.2829416817683)
--(axis cs:2301,92.2892593768466)
--(axis cs:2401,92.2893622449916)
--(axis cs:2501,92.2888784445484)
--(axis cs:2601,92.2829014154044)
--(axis cs:2701,92.2865336973051)
--(axis cs:2801,92.2919043353144)
--(axis cs:2901,92.2946315927253)
--(axis cs:3001,92.2969815043041)
--(axis cs:3101,92.2940060284742)
--(axis cs:3201,92.3032511779538)
--(axis cs:3301,92.3068904386936)
--(axis cs:3401,92.3056941854295)
--(axis cs:3501,92.303079293519)
--(axis cs:3601,92.2995912874923)
--(axis cs:3701,92.2985320814081)
--(axis cs:3801,92.2962182503677)
--(axis cs:3901,92.2985951646951)
--(axis cs:4001,92.3006196870963)
--(axis cs:4101,92.300598183623)
--(axis cs:4201,92.2985568697522)
--(axis cs:4301,92.2973842038189)
--(axis cs:4401,92.295206561798)
--(axis cs:4501,92.2950597875722)
--(axis cs:4601,92.2939471722597)
--(axis cs:4701,92.2959098009655)
--(axis cs:4801,92.2997722904573)
--(axis cs:4901,92.2970066736172)
--(axis cs:5001,92.2943458041623)
--(axis cs:5101,92.2940015479458)
--(axis cs:5201,92.2943703503355)
--(axis cs:5301,92.2939568632587)
--(axis cs:5401,92.2972196023257)
--(axis cs:5501,92.2998806586589)
--(axis cs:5601,92.3007709271281)
--(axis cs:5701,92.2999142673215)
--(axis cs:5801,92.2991929270751)
--(axis cs:5901,92.2962453030463)
--(axis cs:6001,92.295856047872)
--(axis cs:6101,92.2942778166431)
--(axis cs:6201,92.2922608960946)
--(axis cs:6301,92.2924594629239)
--(axis cs:6401,92.2929650331755)
--(axis cs:6501,92.2950039937152)
--(axis cs:6601,92.2949441839612)
--(axis cs:6701,92.295374685463)
--(axis cs:6801,92.2956193189086)
--(axis cs:6901,92.292101841029)
--(axis cs:7001,92.292111613634)
--(axis cs:7101,92.292292344701)
--(axis cs:7201,92.2925916814171)
--(axis cs:7301,92.2933977924117)
--(axis cs:7401,92.2942471972665)
--(axis cs:7501,92.2950529465436)
--(axis cs:7601,92.2955000650632)
--(axis cs:7701,92.294001109012)
--(axis cs:7801,92.2943427153064)
--(axis cs:7901,92.2941956724786)
--(axis cs:8001,92.2946312141376)
--(axis cs:8101,92.2939007535448)
--(axis cs:8201,92.2926640816415)
--(axis cs:8301,92.2927638083069)
--(axis cs:8401,92.2937843263659)
--(axis cs:8501,92.2921037096861)
--(axis cs:8601,92.2918280691697)
--(axis cs:8701,92.2924229959696)
--(axis cs:8801,92.2918944321581)
--(axis cs:8901,92.2943846812651)
--(axis cs:9001,92.2944608494664)
--(axis cs:9101,92.2947174967682)
--(axis cs:9201,92.2950301355632)
--(axis cs:9301,92.2956824989232)
--(axis cs:9401,92.2960049744047)
--(axis cs:9501,92.2979066287062)
--(axis cs:9601,92.2992034590085)
--(axis cs:9701,92.3011764818745)
--(axis cs:9801,92.3006055774065)
--(axis cs:9901,92.3010972026574)
--(axis cs:9901,92.2631750930702)
--(axis cs:9901,92.2631750930702)
--(axis cs:9801,92.2624512535292)
--(axis cs:9701,92.2633364549362)
--(axis cs:9601,92.2616318706447)
--(axis cs:9501,92.260207254043)
--(axis cs:9401,92.2579105664951)
--(axis cs:9301,92.2571204254934)
--(axis cs:9201,92.2563381939662)
--(axis cs:9101,92.2563186531055)
--(axis cs:9001,92.2551669696649)
--(axis cs:8901,92.2543019831546)
--(axis cs:8801,92.2517506081782)
--(axis cs:8701,92.2521626838373)
--(axis cs:8601,92.2519319587341)
--(axis cs:8501,92.2513005956897)
--(axis cs:8401,92.252910114772)
--(axis cs:8301,92.2517489010052)
--(axis cs:8201,92.2520792423434)
--(axis cs:8101,92.2533773602683)
--(axis cs:8001,92.2544376522542)
--(axis cs:7901,92.2538539414943)
--(axis cs:7801,92.2541228660293)
--(axis cs:7701,92.2539069548758)
--(axis cs:7601,92.2551143277798)
--(axis cs:7501,92.2541831553095)
--(axis cs:7401,92.2526356564019)
--(axis cs:7301,92.2513275876732)
--(axis cs:7201,92.2504717819909)
--(axis cs:7101,92.2503129221628)
--(axis cs:7001,92.2488082549562)
--(axis cs:6901,92.2489646710707)
--(axis cs:6801,92.2514855186153)
--(axis cs:6701,92.2512452220135)
--(axis cs:6601,92.2499065962236)
--(axis cs:6501,92.248620064122)
--(axis cs:6401,92.2459851308614)
--(axis cs:6301,92.2448552490265)
--(axis cs:6201,92.2443073993415)
--(axis cs:6101,92.2451976791772)
--(axis cs:6001,92.2447905110348)
--(axis cs:5901,92.2453340902768)
--(axis cs:5801,92.2482195880084)
--(axis cs:5701,92.2485859957903)
--(axis cs:5601,92.2480524972604)
--(axis cs:5501,92.2460273580652)
--(axis cs:5401,92.2431506994702)
--(axis cs:5301,92.239583978092)
--(axis cs:5201,92.2382118453961)
--(axis cs:5101,92.2375256035932)
--(axis cs:5001,92.2373798507067)
--(axis cs:4901,92.2418711064277)
--(axis cs:4801,92.2440727418276)
--(axis cs:4701,92.2410185121595)
--(axis cs:4601,92.2373938405635)
--(axis cs:4501,92.2359733161826)
--(axis cs:4401,92.2338822816467)
--(axis cs:4301,92.2357057752558)
--(axis cs:4201,92.2373964746895)
--(axis cs:4101,92.2377192999176)
--(axis cs:4001,92.2363710652156)
--(axis cs:3901,92.2325455684502)
--(axis cs:3801,92.2301905894113)
--(axis cs:3701,92.2313787535014)
--(axis cs:3601,92.2293229030103)
--(axis cs:3501,92.2331560678063)
--(axis cs:3401,92.2348585931649)
--(axis cs:3301,92.2349453686374)
--(axis cs:3201,92.2323689407591)
--(axis cs:3101,92.222117802548)
--(axis cs:3001,92.2238182291181)
--(axis cs:2901,92.221424939505)
--(axis cs:2801,92.2170781709334)
--(axis cs:2701,92.2101267988075)
--(axis cs:2601,92.2047187306933)
--(axis cs:2501,92.2106977249839)
--(axis cs:2401,92.208913473459)
--(axis cs:2301,92.2060991629187)
--(axis cs:2201,92.1969220165506)
--(axis cs:2101,92.1918032040514)
--(axis cs:2001,92.1907941874378)
--(axis cs:1901,92.1908515223198)
--(axis cs:1801,92.1869347961129)
--(axis cs:1701,92.1862484784369)
--(axis cs:1601,92.1863372671139)
--(axis cs:1501,92.1758241382582)
--(axis cs:1401,92.1693662048225)
--(axis cs:1301,92.1770755232341)
--(axis cs:1201,92.1656440904594)
--(axis cs:1101,92.1624696833197)
--(axis cs:1001,92.150465986327)
--(axis cs:901,92.1386708209692)
--(axis cs:801,92.1260325860401)
--(axis cs:701,92.1156946854153)
--(axis cs:601,92.0921808720617)
--(axis cs:501,92.0667008103847)
--(axis cs:401,92.0340232498165)
--(axis cs:301,91.9817854798733)
--(axis cs:201,91.9147862510702)
--(axis cs:101,91.7140145721555)
--(axis cs:1,82.488098316915)
--cycle;

\path [draw=orchid227119194, fill=orchid227119194, opacity=0.2]
(axis cs:1,87.6839887262217)
--(axis cs:1,91.3760112737783)
--(axis cs:101,92.5471044716401)
--(axis cs:201,92.5786567147931)
--(axis cs:301,92.5717433013818)
--(axis cs:401,92.5743541336406)
--(axis cs:501,92.5599841565309)
--(axis cs:601,92.5509773466559)
--(axis cs:701,92.5552533436709)
--(axis cs:801,92.5454714821191)
--(axis cs:901,92.5424765598922)
--(axis cs:1001,92.5309452449102)
--(axis cs:1101,92.5440061658064)
--(axis cs:1201,92.5456186870034)
--(axis cs:1301,92.5448907181258)
--(axis cs:1401,92.5518699646468)
--(axis cs:1501,92.5600219843316)
--(axis cs:1601,92.5491023499698)
--(axis cs:1701,92.5559890384935)
--(axis cs:1801,92.5565202097402)
--(axis cs:1901,92.5596518447763)
--(axis cs:2001,92.5564418029054)
--(axis cs:2101,92.553462799058)
--(axis cs:2201,92.5562471486713)
--(axis cs:2301,92.5530773869702)
--(axis cs:2401,92.5552366966029)
--(axis cs:2501,92.5538554122996)
--(axis cs:2601,92.5553092576105)
--(axis cs:2701,92.55356823784)
--(axis cs:2801,92.5502182968924)
--(axis cs:2901,92.5490356328303)
--(axis cs:3001,92.5481935350197)
--(axis cs:3101,92.5477232051366)
--(axis cs:3201,92.5503601203983)
--(axis cs:3301,92.5461016112893)
--(axis cs:3401,92.5435686476907)
--(axis cs:3501,92.5459933161606)
--(axis cs:3601,92.546628275749)
--(axis cs:3701,92.5444263333746)
--(axis cs:3801,92.5482805248208)
--(axis cs:3901,92.5462034068331)
--(axis cs:4001,92.5495635782437)
--(axis cs:4101,92.5513041013925)
--(axis cs:4201,92.5543198290945)
--(axis cs:4301,92.5567811773757)
--(axis cs:4401,92.556443927028)
--(axis cs:4501,92.5583469734356)
--(axis cs:4601,92.557574566541)
--(axis cs:4701,92.5572142333484)
--(axis cs:4801,92.5569629501405)
--(axis cs:4901,92.5586636829008)
--(axis cs:5001,92.5590643622603)
--(axis cs:5101,92.5593190392591)
--(axis cs:5201,92.555709698423)
--(axis cs:5301,92.5550819939534)
--(axis cs:5401,92.5534341320104)
--(axis cs:5501,92.5511156732029)
--(axis cs:5601,92.5496903647421)
--(axis cs:5701,92.5500867658365)
--(axis cs:5801,92.5485206311355)
--(axis cs:5901,92.5473910236868)
--(axis cs:6001,92.5463171403733)
--(axis cs:6101,92.5457968719068)
--(axis cs:6201,92.544761063334)
--(axis cs:6301,92.5452535201567)
--(axis cs:6401,92.5457210830825)
--(axis cs:6501,92.5454027305155)
--(axis cs:6601,92.5460373842978)
--(axis cs:6701,92.5454453365671)
--(axis cs:6801,92.5433871068173)
--(axis cs:6901,92.5425760164564)
--(axis cs:7001,92.5430978560394)
--(axis cs:7101,92.541915843177)
--(axis cs:7201,92.5415305043931)
--(axis cs:7301,92.5419697619596)
--(axis cs:7401,92.5404327729037)
--(axis cs:7501,92.5411155812281)
--(axis cs:7601,92.5406252363628)
--(axis cs:7701,92.5416688725873)
--(axis cs:7801,92.5420661899401)
--(axis cs:7901,92.5426999689557)
--(axis cs:8001,92.54406030952)
--(axis cs:8101,92.5433737854034)
--(axis cs:8201,92.5432025322403)
--(axis cs:8301,92.544883225484)
--(axis cs:8401,92.5453086842177)
--(axis cs:8501,92.5454716142799)
--(axis cs:8601,92.5464500417875)
--(axis cs:8701,92.5476666224173)
--(axis cs:8801,92.5484849186508)
--(axis cs:8901,92.5486007324608)
--(axis cs:9001,92.5487929913339)
--(axis cs:9101,92.5475817848462)
--(axis cs:9201,92.5471088885753)
--(axis cs:9301,92.5462733366285)
--(axis cs:9401,92.5462137468652)
--(axis cs:9501,92.5464414793428)
--(axis cs:9601,92.546717223175)
--(axis cs:9701,92.5474482992635)
--(axis cs:9801,92.5482800018503)
--(axis cs:9901,92.5488646050639)
--(axis cs:9901,92.5160500500215)
--(axis cs:9901,92.5160500500215)
--(axis cs:9801,92.5150421081386)
--(axis cs:9701,92.5143267754711)
--(axis cs:9601,92.5130411353294)
--(axis cs:9501,92.5128554367714)
--(axis cs:9401,92.5127225365089)
--(axis cs:9301,92.5128428874335)
--(axis cs:9201,92.5135692985782)
--(axis cs:9101,92.51389717351)
--(axis cs:9001,92.5144688684595)
--(axis cs:8901,92.5143494978504)
--(axis cs:8801,92.5137216487846)
--(axis cs:8701,92.5132528121304)
--(axis cs:8601,92.511889686151)
--(axis cs:8501,92.5110346790974)
--(axis cs:8401,92.5103584982606)
--(axis cs:8301,92.5096210511092)
--(axis cs:8201,92.5083204527615)
--(axis cs:8101,92.5085062294095)
--(axis cs:8001,92.5082831475479)
--(axis cs:7901,92.5067899690268)
--(axis cs:7801,92.5062558200586)
--(axis cs:7701,92.5053691744197)
--(axis cs:7601,92.5034794866999)
--(axis cs:7501,92.5030278663123)
--(axis cs:7401,92.5020831033293)
--(axis cs:7301,92.5036020775144)
--(axis cs:7201,92.5023106287828)
--(axis cs:7101,92.5031594983241)
--(axis cs:7001,92.5039754192071)
--(axis cs:6901,92.5027884234798)
--(axis cs:6801,92.5043617536443)
--(axis cs:6701,92.5060962243938)
--(axis cs:6601,92.5068152137934)
--(axis cs:6501,92.5062877786368)
--(axis cs:6401,92.5053928053725)
--(axis cs:6301,92.5045988842236)
--(axis cs:6201,92.5043600461644)
--(axis cs:6101,92.5053652326663)
--(axis cs:6001,92.5056841927378)
--(axis cs:5901,92.5066404963944)
--(axis cs:5801,92.5081764900504)
--(axis cs:5701,92.5091063581768)
--(axis cs:5601,92.5093919419889)
--(axis cs:5501,92.5092188114362)
--(axis cs:5401,92.5109539442717)
--(axis cs:5301,92.5125241181009)
--(axis cs:5201,92.5127732856186)
--(axis cs:5101,92.514627245783)
--(axis cs:5001,92.5133891470379)
--(axis cs:4901,92.5138643725981)
--(axis cs:4801,92.512355941757)
--(axis cs:4701,92.5103373514208)
--(axis cs:4601,92.5095760528895)
--(axis cs:4501,92.5096201449825)
--(axis cs:4401,92.5069416671552)
--(axis cs:4301,92.5073062441542)
--(axis cs:4201,92.5040757909959)
--(axis cs:4101,92.5008929237233)
--(axis cs:4001,92.498399431004)
--(axis cs:3901,92.4959498871941)
--(axis cs:3801,92.4974074520274)
--(axis cs:3701,92.4936498622482)
--(axis cs:3601,92.4953045206964)
--(axis cs:3501,92.4940809483353)
--(axis cs:3401,92.4918503467227)
--(axis cs:3301,92.4935227449663)
--(axis cs:3201,92.4954630598578)
--(axis cs:3101,92.4925799228866)
--(axis cs:3001,92.4897338225278)
--(axis cs:2901,92.4900681244948)
--(axis cs:2801,92.4886606749034)
--(axis cs:2701,92.4912670083651)
--(axis cs:2601,92.4912651368532)
--(axis cs:2501,92.490710761231)
--(axis cs:2401,92.4873205712021)
--(axis cs:2301,92.4850364765674)
--(axis cs:2201,92.4848977854494)
--(axis cs:2101,92.4795214941358)
--(axis cs:2001,92.4824287618123)
--(axis cs:1901,92.4837042835772)
--(axis cs:1801,92.4757396458954)
--(axis cs:1701,92.4739933248222)
--(axis cs:1601,92.4658757886935)
--(axis cs:1501,92.4709706872207)
--(axis cs:1401,92.4591792858885)
--(axis cs:1301,92.4474075140033)
--(axis cs:1201,92.4497185319807)
--(axis cs:1101,92.4374107279265)
--(axis cs:1001,92.4287750348101)
--(axis cs:901,92.4353258818392)
--(axis cs:801,92.4258893168821)
--(axis cs:701,92.4291974409225)
--(axis cs:601,92.4183071791345)
--(axis cs:501,92.4162633484591)
--(axis cs:401,92.4202593326936)
--(axis cs:301,92.3821437418076)
--(axis cs:201,92.3512935339631)
--(axis cs:101,92.2241826570727)
--(axis cs:1,87.6839887262217)
--cycle;

\path [draw=goldenrod18818934, fill=goldenrod18818934, opacity=0.2]
(axis cs:1,83.609799163237)
--(axis cs:1,87.150200836763)
--(axis cs:101,90.4521710412673)
--(axis cs:201,90.4443699982471)
--(axis cs:301,90.4498652126798)
--(axis cs:401,90.4590216040174)
--(axis cs:501,90.4394794248716)
--(axis cs:601,90.4288204460176)
--(axis cs:701,90.4132085320711)
--(axis cs:801,90.389325261559)
--(axis cs:901,90.3847181368914)
--(axis cs:1001,90.4009102974461)
--(axis cs:1101,90.3962601650636)
--(axis cs:1201,90.3809164435893)
--(axis cs:1301,90.379142285863)
--(axis cs:1401,90.3772580907624)
--(axis cs:1501,90.3740818340638)
--(axis cs:1601,90.3762689444671)
--(axis cs:1701,90.3744258261361)
--(axis cs:1801,90.3772075637277)
--(axis cs:1901,90.3712859913033)
--(axis cs:2001,90.3725999516346)
--(axis cs:2101,90.3682414814355)
--(axis cs:2201,90.3707597799373)
--(axis cs:2301,90.3741210049181)
--(axis cs:2401,90.3702722630278)
--(axis cs:2501,90.366634908908)
--(axis cs:2601,90.3630980617792)
--(axis cs:2701,90.3633718253541)
--(axis cs:2801,90.3579077532522)
--(axis cs:2901,90.3584915249572)
--(axis cs:3001,90.3575205282775)
--(axis cs:3101,90.3600414497455)
--(axis cs:3201,90.3632119445285)
--(axis cs:3301,90.3642135135663)
--(axis cs:3401,90.364738581585)
--(axis cs:3501,90.365156691673)
--(axis cs:3601,90.3682009677378)
--(axis cs:3701,90.3675295455035)
--(axis cs:3801,90.3658278750815)
--(axis cs:3901,90.3641935389738)
--(axis cs:4001,90.3680646075785)
--(axis cs:4101,90.3710201112317)
--(axis cs:4201,90.3716147576589)
--(axis cs:4301,90.3722623088091)
--(axis cs:4401,90.3728584455824)
--(axis cs:4501,90.3716605244839)
--(axis cs:4601,90.3667220953359)
--(axis cs:4701,90.367179457238)
--(axis cs:4801,90.3656623768587)
--(axis cs:4901,90.3631486712016)
--(axis cs:5001,90.3639576330415)
--(axis cs:5101,90.3640778863189)
--(axis cs:5201,90.3617372328387)
--(axis cs:5301,90.3609821915298)
--(axis cs:5401,90.3619727704257)
--(axis cs:5501,90.3637041212628)
--(axis cs:5601,90.3639752151694)
--(axis cs:5701,90.3651860513028)
--(axis cs:5801,90.3665382344229)
--(axis cs:5901,90.365541376926)
--(axis cs:6001,90.3650719570826)
--(axis cs:6101,90.3667454293009)
--(axis cs:6201,90.3677302571336)
--(axis cs:6301,90.3643413151028)
--(axis cs:6401,90.3656723023532)
--(axis cs:6501,90.3653360380458)
--(axis cs:6601,90.3676027660973)
--(axis cs:6701,90.3670472001225)
--(axis cs:6801,90.3646242892114)
--(axis cs:6901,90.3641414029803)
--(axis cs:7001,90.365612177015)
--(axis cs:7101,90.3661258287273)
--(axis cs:7201,90.3661109823039)
--(axis cs:7301,90.3684193820356)
--(axis cs:7401,90.368718354148)
--(axis cs:7501,90.3684945449905)
--(axis cs:7601,90.3704666736676)
--(axis cs:7701,90.3700611877381)
--(axis cs:7801,90.3694870553626)
--(axis cs:7901,90.3706408471327)
--(axis cs:8001,90.3716717857399)
--(axis cs:8101,90.3735810764557)
--(axis cs:8201,90.3752182610564)
--(axis cs:8301,90.3757993369242)
--(axis cs:8401,90.376221017094)
--(axis cs:8501,90.3778242137319)
--(axis cs:8601,90.3772521216843)
--(axis cs:8701,90.377432066973)
--(axis cs:8801,90.3791016399719)
--(axis cs:8901,90.379599591545)
--(axis cs:9001,90.3789897949278)
--(axis cs:9101,90.3795675948937)
--(axis cs:9201,90.3786074853991)
--(axis cs:9301,90.37909120133)
--(axis cs:9401,90.3792886000227)
--(axis cs:9501,90.3802316316533)
--(axis cs:9601,90.3788816812242)
--(axis cs:9701,90.378311123582)
--(axis cs:9801,90.3773485122458)
--(axis cs:9901,90.3770972915496)
--(axis cs:9901,90.3417856495674)
--(axis cs:9901,90.3417856495674)
--(axis cs:9801,90.3418413663381)
--(axis cs:9701,90.3422269652748)
--(axis cs:9601,90.3426014976114)
--(axis cs:9501,90.3432648424021)
--(axis cs:9401,90.342145289989)
--(axis cs:9301,90.3421667279248)
--(axis cs:9201,90.3420294018958)
--(axis cs:9101,90.3422673682971)
--(axis cs:9001,90.3423922737313)
--(axis cs:8901,90.3426900388336)
--(axis cs:8801,90.3419073362808)
--(axis cs:8701,90.3398785869748)
--(axis cs:8601,90.3398249623756)
--(axis cs:8501,90.3395408021485)
--(axis cs:8401,90.3376987543617)
--(axis cs:8301,90.3372521026613)
--(axis cs:8201,90.3365656677328)
--(axis cs:8101,90.3344845944492)
--(axis cs:8001,90.3326626724529)
--(axis cs:7901,90.331191832275)
--(axis cs:7801,90.3306975363563)
--(axis cs:7701,90.3303777163003)
--(axis cs:7601,90.3301832408173)
--(axis cs:7501,90.3282712195744)
--(axis cs:7401,90.3285225592421)
--(axis cs:7301,90.3276222560962)
--(axis cs:7201,90.3247736170573)
--(axis cs:7101,90.3246698338556)
--(axis cs:7001,90.3244406725779)
--(axis cs:6901,90.3227909256678)
--(axis cs:6801,90.3227628597374)
--(axis cs:6701,90.3245361456468)
--(axis cs:6601,90.3247317286762)
--(axis cs:6501,90.3215275214066)
--(axis cs:6401,90.3209579116758)
--(axis cs:6301,90.3189946633134)
--(axis cs:6201,90.3213424730713)
--(axis cs:6101,90.320620576272)
--(axis cs:6001,90.3182608207877)
--(axis cs:5901,90.3186884146348)
--(axis cs:5801,90.3197055166543)
--(axis cs:5701,90.318027420018)
--(axis cs:5601,90.3167532261805)
--(axis cs:5501,90.3160304724474)
--(axis cs:5401,90.3136539653639)
--(axis cs:5301,90.313015167459)
--(axis cs:5201,90.3132368106144)
--(axis cs:5101,90.3155692416951)
--(axis cs:5001,90.3140227708777)
--(axis cs:4901,90.3148313328792)
--(axis cs:4801,90.3162622221832)
--(axis cs:4701,90.3175429422514)
--(axis cs:4601,90.3182072678461)
--(axis cs:4501,90.3213277003551)
--(axis cs:4401,90.3215973599163)
--(axis cs:4301,90.3213717297865)
--(axis cs:4201,90.3206823144667)
--(axis cs:4101,90.3201040048376)
--(axis cs:4001,90.3153545376351)
--(axis cs:3901,90.3120023082449)
--(axis cs:3801,90.3132513146054)
--(axis cs:3701,90.3127298438506)
--(axis cs:3601,90.3132264135451)
--(axis cs:3501,90.3096333683099)
--(axis cs:3401,90.3083575666066)
--(axis cs:3301,90.3074011486573)
--(axis cs:3201,90.304366937071)
--(axis cs:3101,90.300848585727)
--(axis cs:3001,90.2991072624588)
--(axis cs:2901,90.2971996160287)
--(axis cs:2801,90.2944449779152)
--(axis cs:2701,90.2998195852345)
--(axis cs:2601,90.2970249678248)
--(axis cs:2501,90.3008820842947)
--(axis cs:2401,90.3037802151064)
--(axis cs:2301,90.305974605686)
--(axis cs:2201,90.3006259538202)
--(axis cs:2101,90.2971749869129)
--(axis cs:2001,90.2973750608591)
--(axis cs:1901,90.2933221096961)
--(axis cs:1801,90.294685828832)
--(axis cs:1701,90.2885724101956)
--(axis cs:1601,90.2881532916354)
--(axis cs:1501,90.2862779260962)
--(axis cs:1401,90.2873386258685)
--(axis cs:1301,90.2886363459587)
--(axis cs:1201,90.2849453382592)
--(axis cs:1101,90.2981449212216)
--(axis cs:1001,90.3017870052512)
--(axis cs:901,90.2772796433527)
--(axis cs:801,90.2718482715246)
--(axis cs:701,90.2868770599404)
--(axis cs:601,90.2908800531504)
--(axis cs:501,90.287307002274)
--(axis cs:401,90.3007788947356)
--(axis cs:301,90.2568457507754)
--(axis cs:201,90.2248837330962)
--(axis cs:101,90.0783240082375)
--(axis cs:1,83.609799163237)
--cycle;

\addplot [very thick, black]
table {%
1 86.47
101 97.6542574257426
201 97.8833333333333
301 98.0381395348837
401 98.1823690773067
501 98.2437724550898
601 98.2691514143095
701 98.2751355206847
801 98.294506866417
901 98.3341731409545
1001 98.3322377622378
1101 98.3439782016349
1201 98.3546794338052
1301 98.3484012298232
1401 98.3391791577445
1501 98.3458827448367
1601 98.3483260462211
1701 98.357824808936
1801 98.3693392559689
1901 98.3736612309311
2001 98.3839230384808
2101 98.383612565445
2201 98.3774784189005
2301 98.3759148196436
2401 98.3834943773427
2501 98.3829908036785
2601 98.3907920030757
2701 98.3944131803037
2801 98.3918529096752
2901 98.4024646673561
3001 98.4025258247251
3101 98.4056755885199
3201 98.4052733520775
3301 98.4107088760981
3401 98.4136900911497
3501 98.4148157669237
3601 98.409619550125
3701 98.4069386652256
3801 98.410665614312
3901 98.4122558318381
4001 98.4103874031492
4101 98.4147598146793
4201 98.4154487026899
4301 98.4195698674727
4401 98.4181595092025
4501 98.4230259942235
4601 98.4243273201478
4701 98.4270474367156
4801 98.4273859612581
4901 98.4304856151806
5001 98.430197960408
5101 98.4331680062733
5201 98.4351490098058
5301 98.4374646293152
5401 98.4374023329013
5501 98.4372532266861
5601 98.4372701303339
5701 98.4404279950886
5801 98.4398483020169
5901 98.4405761735299
6001 98.4433194467589
6101 98.4454843468284
6201 98.4453039832286
6301 98.4467370258689
6401 98.448840806124
6501 98.4521750499923
6601 98.454759884866
6701 98.4532830920758
6801 98.4534568445817
6901 98.4524184900739
7001 98.4539222968147
7101 98.4552584143079
7201 98.4535467296209
7301 98.453724147377
7401 98.455161464667
7501 98.4565617917611
7601 98.4567201683989
7701 98.4575223996884
7801 98.4597154210999
7901 98.4591304898114
8001 98.4606986626672
8101 98.4598889026046
8201 98.4594561638825
8301 98.4612191302253
8401 98.4618271634329
8501 98.461978590754
8601 98.4630101151029
8701 98.462750258591
8801 98.4656016361777
8901 98.4662858105831
9001 98.4676669258971
9101 98.4675046698164
9201 98.4670057602434
9301 98.4673228685087
9401 98.4678555472822
9501 98.4696200399957
9601 98.4704291219665
9701 98.4702752293579
9801 98.4706887052342
9901 98.4716957883042
};
\addplot [very thick, steelblue31119180,mark=diamond, mark options={fill=none},mark size=3pt,  mark repeat=5]
table {%
1 79.64
101 91.9435643564357
201 94.8697014925373
301 95.9912956810631
401 96.5244139650873
501 96.9085229540918
601 97.1474043261232
701 97.3254350927246
801 97.4696004993758
901 97.5572142064373
1001 97.6422077922078
1101 97.7235059037239
1201 97.7944213155703
1301 97.8581322059953
1401 97.910249821556
1501 97.9484943371086
1601 97.9807995003123
1701 98.0051087595532
1801 98.0305330372016
1901 98.0517622304051
2001 98.0758970514743
2101 98.1051737267967
2201 98.1259791004089
2301 98.1391003911343
2401 98.1521615993336
2501 98.1674410235906
2601 98.1870934256056
2701 98.1944020733062
2801 98.1985505176722
2901 98.2039710444675
3001 98.2153282239253
3101 98.2221283456949
3201 98.2339050296782
3301 98.2399727355347
3401 98.2448632755072
3501 98.2507826335332
3601 98.2599722299361
3701 98.2667333153202
3801 98.2696632465141
3901 98.2740810048705
4001 98.2801649587603
4101 98.2857668861254
4201 98.290126160438
4301 98.2952755173215
4401 98.2965712338105
4501 98.2990246611864
4601 98.3018474244729
4701 98.3047457987662
4801 98.3074838575296
4901 98.3088573760457
5001 98.311799640072
5101 98.3163066065477
5201 98.3205883483945
5301 98.3256159215242
5401 98.3315978522496
5501 98.3297909471005
5601 98.330974825924
5701 98.3354937730223
5801 98.3396776417859
5901 98.3441382816472
6001 98.3467655390768
6101 98.3504654974594
6201 98.354320270924
6301 98.3554102523409
6401 98.3576847367599
6501 98.3564743885556
6601 98.3590622632935
6701 98.3615281301298
6801 98.3648419350096
6901 98.3655122446022
7001 98.3680445650622
7101 98.3701563160118
7201 98.3731287321205
7301 98.3766538830297
7401 98.3803067153087
7501 98.3828529529396
7601 98.3851216945139
7701 98.3843838462537
7801 98.3867593898218
7901 98.3877040880901
8001 98.3892225971754
8101 98.3896432539193
8201 98.3905255456652
8301 98.3936212504518
8401 98.3954814902988
8501 98.3976485119398
8601 98.4010242995
8701 98.4012596253305
8801 98.4017282127031
8901 98.4040017975509
9001 98.4027196978114
9101 98.4030359301176
9201 98.403624606021
9301 98.4057477690571
9401 98.407139665993
9501 98.4093327018209
9601 98.4109051140506
9701 98.4108792907948
9801 98.4112651770227
9901 98.4149116250884
};
\addplot [very thick, mediumpurple148103189, mark=star, mark options={fill=none},mark size=3pt,  mark repeat=5]
table {%
1 84.47
101 91.9091089108911
201 92.0483084577115
301 92.0892358803987
401 92.137855361596
501 92.1583233532935
601 92.1761064891847
701 92.1864907275321
801 92.1934082397004
901 92.2014761376248
1001 92.2094005994006
1101 92.2187556766576
1201 92.2163447127394
1301 92.2276479631053
1401 92.2180728051392
1501 92.2234776815457
1601 92.2326358525922
1701 92.2301352145796
1801 92.2320322043309
1901 92.2346817464492
2001 92.2345677161419
2101 92.2352308424559
2201 92.2399318491595
2301 92.2476792698826
2401 92.2491378592253
2501 92.2497880847661
2601 92.2438100730488
2701 92.2483302480563
2801 92.2544912531239
2901 92.2580282661151
3001 92.2603998667111
3101 92.2580619155111
3201 92.2678100593564
3301 92.2709179036655
3401 92.2702763892972
3501 92.2681176806627
3601 92.2644570952513
3701 92.2649554174547
3801 92.2632044198895
3901 92.2655703665727
4001 92.268495376156
4101 92.2691587417703
4201 92.2679766722209
4301 92.2665449895373
4401 92.2645444217224
4501 92.2655165518774
4601 92.2656705064116
4701 92.2684641565625
4801 92.2719225161424
4901 92.2694388900225
5001 92.2658628274345
5101 92.2657635757695
5201 92.2662910978658
5301 92.2667704206753
5401 92.270185150898
5501 92.2729540083621
5601 92.2744117121943
5701 92.2742501315559
5801 92.2737062575418
5901 92.2707896966616
6001 92.2703232794534
6101 92.2697377479102
6201 92.2682841477181
6301 92.2686573559752
6401 92.2694750820185
6501 92.2718120289186
6601 92.2724253900924
6701 92.2733099537382
6801 92.273552418762
6901 92.2705332560499
7001 92.2704599342951
7101 92.2713026334319
7201 92.271531731704
7301 92.2723626900424
7401 92.2734414268342
7501 92.2746180509265
7601 92.2753071964215
7701 92.2739540319439
7801 92.2742327906679
7901 92.2740248069865
8001 92.2745344331959
8101 92.2736390569065
8201 92.2723716619925
8301 92.2722563546561
8401 92.273347220569
8501 92.2717021526879
8601 92.2718800139519
8701 92.2722928399035
8801 92.2718225201682
8901 92.2743433322099
9001 92.2748139095656
9101 92.2755180749369
9201 92.2756841647647
9301 92.2764014622083
9401 92.2769577704499
9501 92.2790569413746
9601 92.2804176648266
9701 92.2822564684053
9801 92.2815284154678
9901 92.2821361478638
};
\addplot [very thick, orchid227119194, mark=*, mark options={fill=none},mark size=2pt, mark repeat=5]
table {%
1 89.53
101 92.3856435643564
201 92.4649751243781
301 92.4769435215947
401 92.4973067331671
501 92.488123752495
601 92.4846422628952
701 92.4922253922967
801 92.4856803995006
901 92.4889012208657
1001 92.4798601398602
1101 92.4907084468664
1201 92.4976686094921
1301 92.4961491160646
1401 92.5055246252676
1501 92.5154963357761
1601 92.5074890693316
1701 92.5149911816578
1801 92.5161299278178
1901 92.5216780641767
2001 92.5194352823588
2101 92.5164921465969
2201 92.5205724670604
2301 92.5190569317688
2401 92.5212786339025
2501 92.5222830867653
2601 92.5232871972318
2701 92.5224176231026
2801 92.5194394858979
2901 92.5195518786625
3001 92.5189636787737
3101 92.5201515640116
3201 92.5229115901281
3301 92.5198121781278
3401 92.5177094972067
3501 92.520037132248
3601 92.5209663982227
3701 92.5190380978114
3801 92.5228439884241
3901 92.5210766470136
4001 92.5239815046238
4101 92.5260985125579
4201 92.5291978100452
4301 92.5320437107649
4401 92.5316927970916
4501 92.533983559209
4601 92.5335753097152
4701 92.5337757923846
4801 92.5346594459488
4901 92.5362640277494
5001 92.5362267546491
5101 92.536973142521
5201 92.5342414920208
5301 92.5338030560272
5401 92.5321940381411
5501 92.5301672423195
5601 92.5295411533655
5701 92.5295965620067
5801 92.528348560593
5901 92.5270157600406
6001 92.5260006665555
6101 92.5255810522865
6201 92.5245605547492
6301 92.5249262021901
6401 92.5255569442275
6501 92.5258452545762
6601 92.5264262990456
6701 92.5257707804805
6801 92.5238744302308
6901 92.5226822199681
7001 92.5235366376232
7101 92.5225376707506
7201 92.521920566588
7301 92.522785919737
7401 92.5212579381165
7501 92.5220717237702
7601 92.5220523615314
7701 92.5235190235035
7801 92.5241610049994
7901 92.5247449689913
8001 92.526171728534
8101 92.5259400074065
8201 92.5257614925009
8301 92.5272521382966
8401 92.5278335912391
8501 92.5282531466886
8601 92.5291698639693
8701 92.5304597172739
8801 92.5311032837177
8901 92.5314751151556
9001 92.5316309298967
9101 92.5307394791781
9201 92.5303390935767
9301 92.529558112031
9401 92.5294681416871
9501 92.5296484580571
9601 92.5298791792522
9701 92.5308875373673
9801 92.5316610549944
9901 92.5324573275427
};
\addplot [very thick, goldenrod18818934,mark=square, mark options={fill},mark size=2pt,  mark repeat=5, mark phase=5]
table {%
1 85.38
101 90.2652475247524
201 90.3346268656716
301 90.3533554817276
401 90.3799002493765
501 90.3633932135728
601 90.359850249584
701 90.3500427960057
801 90.3305867665418
901 90.3309988901221
1001 90.3513486513486
1101 90.3472025431426
1201 90.3329308909242
1301 90.3338893159109
1401 90.3322983583155
1501 90.33017988008
1601 90.3322111180512
1701 90.3314991181658
1801 90.3359466962798
1901 90.3323040504997
2001 90.3349875062469
2101 90.3327082341742
2201 90.3356928668788
2301 90.340047805302
2401 90.3370262390671
2501 90.3337584966013
2601 90.330061514802
2701 90.3315957052943
2801 90.3261763655837
2901 90.3278455704929
3001 90.3283138953682
3101 90.3304450177363
3201 90.3337894407997
3301 90.3358073311118
3401 90.3365480740958
3501 90.3373950299914
3601 90.3407136906415
3701 90.3401296946771
3801 90.3395395948434
3901 90.3380979236094
4001 90.3417095726068
4101 90.3455620580347
4201 90.3461485360628
4301 90.3468170192978
4401 90.3472279027494
4501 90.3464941124195
4601 90.342464681591
4701 90.3423611997447
4801 90.3409622995209
4901 90.3389900020404
5001 90.3389902019596
5101 90.339823564007
5201 90.3374870217266
5301 90.3369986794944
5401 90.3378133678948
5501 90.3398672968551
5601 90.3403642206749
5701 90.3416067356604
5801 90.3431218755386
5901 90.3421148957804
6001 90.3416663889352
6101 90.3436830027865
6201 90.3445363651024
6301 90.3416679892081
6401 90.3433151070145
6501 90.3434317797262
6601 90.3461672473867
6701 90.3457916728846
6801 90.3436935744744
6901 90.343466164324
7001 90.3450264247965
7101 90.3453978312914
7201 90.3454422996806
7301 90.3480208190659
7401 90.348620456695
7501 90.3483828822824
7601 90.3503249572424
7701 90.3502194520192
7801 90.3500922958595
7901 90.3509163397039
8001 90.3521672290964
8101 90.3540328354524
8201 90.3558919643946
8301 90.3565257197928
8401 90.3569598857279
8501 90.3586825079402
8601 90.35853854203
8701 90.3586553269739
8801 90.3605044881264
8901 90.3611448151893
9001 90.3606910343295
9101 90.3609174815954
9201 90.3603184436474
9301 90.3606289646274
9401 90.3607169450059
9501 90.3617482370277
9601 90.3607415894178
9701 90.3602690444284
9801 90.3595949392919
9901 90.3594414705585
};
\end{axis}

\end{tikzpicture}

%% file: AppendixLB.tex
\begin{proof}{}
    Let $\pi$ satisfy~\eqref{eq:def_stability} and~\eqref{eq:def_eps_restrictive}, and let $\theta\in\R_{\geq 0}^L$. 

    We first prove that the regret lower bound is strictly positive, i.e., the second inequality in \eqref{eq:regret_lb}.
    We have $K(\theta,k\ell) < \infty$ for any $(k\ell)\in O^c(\theta)$ as a consequence of Lemma~\ref{lem:suboptgap_nonempty} and assumption \eqref{eq:ass_KL_finite}. 
    Moreover, $\phi_{k\ell} > 0$ for any $(k\ell)\in O^c(\theta)$ by strict complementarity. 
    It remains to be shown that $O^c(\theta)$ is nonempty. 
    Let $B=\scrL\cup\scrJ$ be the optimal basis of \LP{\theta,\eps}, then $|B|=I+J$ 
    and $\scrL = O(\theta)$ by construction. 
    It follows from Lemma~\ref{lem:LP_basis} that $|\scrJ| \geq 1$ (since the corresponding spanning forest consists of at least one tree), hence $|O(\theta)| = |\scrL| \leq I+J-1$.
    Thus, since $O(\theta)\cup O^c(\theta)$, is a partition of all $L$ lines, 
    \begin{align}
        |O^c(\theta)| = L - |O(\theta)| \geq L - (I+J-1) > 0, 
    \end{align}
    where the last inequality holds by assumption (see Section~\ref{sec:arr_serv_process}). 
    This means that $O^c(\theta)\neq \emptyset$ and thus the regret lower bound is strictly positive. 

    Next, we prove the first inequality in \eqref{eq:regret_lb}.
    Using the regret decomposition in~\eqref{eq:regret_decomp}, it is sufficient to show that simultaneously
    \begin{align}\label{eq:tbp_regret_rho}
        \lim_{t\to\infty} \frac{1}{\ln(t)} \Biggl(
            \sum_{j\in\calJ}  w_j^\theta \Bigl(t (\mu_j-\eps) - \sum_{i\in\calC_j}  \E_\pi^\theta(D_{ij}(t))\Bigr)
            + \sum_{i\in\calI}  v_i^\theta \Bigl(t\lambda_i - \sum_{j\in\calS_i} \E_\pi^\theta(D_{ij}(t))\Bigr)
            \Biggr)
        \geq  0,
    \end{align}
    and
    \begin{align}\label{eq:tbp_regret_E}
        \lim_{t\to\infty}  \frac{\E_\pi^\theta(D_{k\ell}(t))}{\ln(t)} \geq \frac{1}{K(\theta,k\ell)},  \ \ \ \forall (k\ell) \in O^c(\theta).
    \end{align}

    \noindent {\it Proof of \eqref{eq:tbp_regret_rho}.}\\
    Since $v^\theta\in\R^I$, $w^\theta\in\R_{\geq 0}^J$ is the optimal solution of \dual{\theta,\eps} in \eqref{eq:LP_eps_dual}, we have for any $j\in\calJ$ that $w_j^\theta \geq 0$ by \eqref{eq:LP_eps_dual_constr_w}.
    Moreover, $\pi$ satisfies \eqref{eq:def_eps_restrictive} by assumption. 
    Therefore
    \begin{align}\label{eq:sumw}
        \lim_{t\to\infty} \frac{w_j^\theta  \bigl(t (\mu_j-\eps) - \sum_{i\in\calC_j}  \E_\pi^\theta(D_{ij}(t))\bigr)}{\ln(t)}
        \geq 0, \ \ \ \forall j\in\calJ.
    \end{align}
    On the other hand, we have for any $i\in\calI$ that $v_i^\theta \in\R$. Hence by \eqref{eq:def_stability}
    \begin{align}\label{eq:sumv}
        \lim_{t\to\infty} \frac{v_i^\theta  \bigl(t \lambda_i - \sum_{j\in\calS_i}  \E_\pi^\theta(D_{ij}(t))\bigr)}{\ln(t)}
        = 0, \ \ \ \forall i\in\calI.
    \end{align}
    Combining \eqref{eq:sumw} with \eqref{eq:sumv} and summing over all servers and customer types completes the proof of \eqref{eq:tbp_regret_rho}.\\

    \noindent {\it Proof of \eqref{eq:tbp_regret_E}.}\\
    Let $(k\ell) \in O^c(\theta)$. 
    By Markov's inequality, for any $t\geq 0$ and for any $0 < \eta < 1$,
    \begin{align}
        \frac{\E_\pi^\theta(D_{k\ell}(t))}{\ln(t)} \geq
        \frac{1-\eta}{K(\theta,k\ell)}
        \P_\pi^\theta\Bigl(D_{k\ell}(t) \geq \frac{(1-\eta)\ln(t)}{K(\theta,k\ell)}\Bigr).
    \end{align}
    Hence, to prove~\eqref{eq:tbp_regret_E} it suffices to show that for $0 < \eta < 1$
     \begin{align}
        \lim_{t\to\infty} \P_\pi^\theta\Bigl(D_{k\ell}(t) < \frac{(1-\eta)\ln(t)}{K(\theta,k\ell)}\Bigr) &= 0.
    \end{align}

    Let $0 <\eta < 1$, and let $a,\delta\in\R$ be such that $0 < \delta < \eta/(2-\eta)$ and  $0 < a < \delta$.
    We note that $I(\cdot,\cdot)$ is continuous and by \eqref{eq:ass_KL_finite} finite, and that $0 < K(\theta,k\ell) < \infty$ by Lemma~\ref{lem:suboptgap_nonempty}.
    This means that we can find a $\theta_{k\ell}'\in \Delta(\theta,k\ell)$ such that 
    \begin{align}\label{eq:ass_K2}
        0 < K(\theta,k\ell) \leq I(\theta_{k\ell}, \theta_{k\ell}') < (1+\delta)K(\theta,k\ell) < \infty.
    \end{align}
    Note that the choice of $\delta$ implies $1- \eta < (1-\delta)/(1+\delta)$, and therefore,
    \begin{align}
        \P_\pi^\theta\Bigl(D_{k\ell}(t) < \frac{(1-\eta)\ln(t)}{K(\theta,k\ell)}\Bigr)
        &\leq \P_\pi^\theta\Bigl(D_{k\ell}(t) < \frac{1-\delta}{1+\delta} \frac{\ln(t)}{K(\theta,k\ell)}\Bigr) 
        \stackrel{\mathrm{\eqref{eq:ass_K2}}}{\leq} \P_\pi^\theta\Bigl(D_{k\ell}(t) <  \frac{(1-\delta)\ln(t)}{I(\theta_{k\ell}, \theta_{k\ell}')}\Bigr).
    \end{align}
    We define the following functions,
    \begin{align}
        g(t) := \frac{(1-\delta)\ln(t)}{I(\theta_{k\ell}, \theta_{k\ell}')}, \ \ \
        c(t) := (1-a)\ln(t).
        \label{eq:def_g_c_t}
    \end{align}
    We next prove that
    \begin{align}
        \lim_{t\to\infty} \P_\pi^\theta(D_{k\ell}(t) < g(t)) = 0.
    \end{align}

    Consider the log-likelihood ratio of the parameter values $\theta_{k\ell}$ and $\theta_{k\ell}'$ given a sequence $X_1,X_2,\dots,X_n$ of observations from the reward distribution belonging to type-$k$ customers being served by server $\ell$, which is defined by
    \begin{align}\label{eq:def_loglike}
        L_n = \sum_{i=1}^n l_i = \sum_{i=1}^{n} \ln\Bigl(\frac{P_{\theta_{k\ell}}(X_i)}{P_{\theta_{k\ell'}}(X_i)}\Bigr).
    \end{align}

    By the law of total probability, we have
    \begin{align}\label{eq:prob-split}
        \P_\pi^\theta(D_{k\ell}(t) < g(t)) &=
        \P_\pi^\theta(D_{k\ell}(t) < g(t), \ L_{D_{k\ell}(t)} > c(t))
        + \P_\pi^\theta(D_{k\ell}(t) < g(t), \ L_{D_{k\ell}(t)} \leq c(t)).
    \end{align}
    To complete the proof, we will show that both terms of \eqref{eq:prob-split} are $o(1)$ as $t\to\infty$, i.e., we will prove separately that
    \begin{align}\label{eq:tbp_L_geq_c}
        \lim_{t\to\infty} \P_\pi^\theta(D_{k\ell}(t) < g(t), L_{D_{k\ell}(t)} > c(t)) = 0
    \end{align}
    and
    \begin{align}\label{eq:tbp_L_leq_c}
        \lim_{t\to\infty} \P_\pi^\theta(D_{k\ell}(t) < g(t), L_{D_{k\ell}(t)} \leq c(t)) = 0.
    \end{align}

    \noindent {\it Proof of \eqref{eq:tbp_L_geq_c}.}
    In order to prove \eqref{eq:tbp_L_geq_c}, we will use the following result, proven in \citep[Lemma 2(ii)]{Burnetas1997}:
    \emph{
        Let $\{Z_i\}_{i\in\N}$ be an i.i.d.\ sequence of random variables that satisfies the strong law of large numbers, i.e.,
        $
            \P\Bigl( \lim_{n\to\infty}  \sum_{i=1}^n Z_i/n = \mu \Bigr) = 1, 
        $
        for some constant $\mu$.
        Let $h_n$ be an increasing sequence of positive constants such that $h_n \to \infty$ as $n\to\infty$. Then, for any $\zeta > 0$,
        $
            \lim_{n\to\infty} \P\Bigl(\max_{s\leq \lfloor h_n\rfloor } sum_{i=1}^s Z_i/h_n > (1+\zeta) \mu  \Bigr) = 0.
        $
    }

    We leverage the definition of the log-likelihood ratio \eqref{eq:def_loglike} and the properties of the maximum to write
    \begin{align}\label{eq:L_geq_c_ub}
        \P_\pi^\theta(D_{k\ell}(t) < g(t), \ L_{D_{k\ell}(t)} > c(t))
        &\leq \P_\pi^\theta \Bigl(\max_{s\leq \lfloor g(t)\rfloor } L_s > c(t) \Bigr).
    \end{align}
    Next, we divide by $g(t) > 0$:
    \begin{align}
        \P_\pi^\theta \Bigl(\max_{s\leq \lfloor g(t)\rfloor } L_s > c(t) \Bigr)  &= \P_\pi^\theta \Bigl(\max_{s\leq  \lfloor g(t)\rfloor } \frac{L_s}{g(t)} > \frac{c(t)}{g(t)}  \Bigr) \\
        &\stackrel{\mathrm{\eqref{eq:def_g_c_t}}}{=} \P_\pi^\theta \Bigl(\max_{s\leq  \lfloor g(t)\rfloor } \frac{L_s}{g(t)} > \frac{1-a}{1-\delta}I(\theta_{k\ell},\theta_{k\ell}')  \Bigr) \\
        &= \P_\pi^\theta \Bigl(\max_{s\leq  \lfloor g(t)\rfloor } \frac{L_s}{g(t)} > \Bigl(1+ \frac{\delta - a}{1-\delta}\Bigr)I(\theta_{k\ell},\theta_{k\ell}')  \Bigr).
    \end{align}
    We make a few observations.
    Firstly, note that $\ln(t)\to\infty$ and therefore $g(t) \to \infty$ as $t\to\infty$.
    Secondly, $(\delta - a)/(1-\delta) > 0$, since $0<a<\delta$ by definition and $\delta <1$.
    Thirdly, by definition, for every $i\in[n]$
    \begin{align}
        \E_\pi^\theta(l_i)
        &=\E_\pi^\theta\Bigl(\ln\Bigl(\frac{P_{\theta_{k\ell}}(X_i)}{P_{\theta_{k\ell'}}(X_i)}\Bigr)\Bigr)
        = I(\theta_{k\ell},\theta_{k\ell}'),
        \label{eq:LB_proof_KL_exp}
    \end{align}
    hence by the strong law of large numbers,
    \begin{align}
        \P_\pi^\theta\Bigl(\lim_{n\to\infty} \frac{L_n}{n} = I(\theta_{k\ell},\theta_{k\ell}') \Bigr) = 1.
    \end{align}
    Therefore, we can apply \citep[Lemma 2(ii)]{Burnetas1997}  to conclude that
    \begin{align}
        \lim_{t\to\infty} \P_\pi^\theta \Bigl(\max_{s\leq  \lfloor g(t)\rfloor } \frac{L_s}{g(t)} > \Bigl(1+ \frac{\delta - a}{1-\delta}\Bigr)I(\theta_{k\ell},\theta_{k\ell}')  \Bigr)  = 0.
    \end{align}
    Combining with \eqref{eq:L_geq_c_ub}, this completes the proof of \eqref{eq:tbp_L_geq_c}.\\

    \noindent {\it Proof of \eqref{eq:tbp_L_leq_c}.}
    Consider the same queueing network structure but with a slightly different payoff vector $\theta' = A(\theta,k\ell,\theta_{k\ell}')$, which will be our `confusing' system. Let $x'$ be the optimal solution to \LP{\theta',\eps}. 
    In order to prove \eqref{eq:tbp_L_leq_c}, we describe a change--of--measure transformation to relate the measure of an event under probability law $\P_\pi^\theta$ by the measure of that same event under the law $\P_\pi^{\theta'}$. \\

    Denote the event $\{D_{k\ell}(t) = m, \ L_m \leq c(t)\}$ by $B$.
    Given  event $B$ and noting that $\theta$ and $\theta'$ differ only  at index $k\ell$, we find that the \gls{RN} derivative of $\P_\pi^\theta$ with respect to $\P_\pi^{\theta'}$ has the form 
    $\textrm{d}\P_\pi^\theta / \textrm{d}\P_\pi^{\theta'} = \prod_{i=1}^m P_{\theta_{k\ell}}(X_i)/P_{\theta_{k\ell'}}(X_i)$.
    Moreover, the rewards are sampled independently from the rest of the system.
    Hence, we can write
    \begin{align}
        \P_\pi^\theta(B)
        &= \int_B \prod_{i=1}^m \frac{P_{\theta_{k\ell}}(X_i)}{P_{\theta_{k\ell'}}(X_i)} \textrm{d}\P_\pi^{\theta'}.
        \label{eq:LB_proof_B_int_prod}
    \end{align}
    We exponentiate and use the definition of $L_m$ to write
    \begin{align}
        \P_\pi^\theta(B)
        &= \int_B \exp\Bigl(\sum_{i=1}^m \ln\Bigl(\frac{P_{\theta_{k\ell}}(X_i)}{P_{\theta_{k\ell'}}(X_i)}\Bigr)\Bigr) \textrm{d}\P_\pi^{\theta'}
        = \int_B e^{L_m} \textrm{d}\P_\pi^{\theta'}.
        \label{eq:LB_proof_meas_change}
    \end{align}
    Lastly, the definition of $B$ implies that $L_m \leq c(t)$, hence
    \begin{align}
        \P_\pi^\theta( B) &\leq \int_B e^{c(t)} \textrm{d}\P_\pi^{\theta'} = e^{c(t)} \P_\pi^{\theta'}(B).
        \label{eq:change_measure}
    \end{align}

    We next prove \eqref{eq:tbp_L_leq_c}.
    Note that $\{D_{k\ell}(t) < g(t), L_{D_{k\ell}(t)} \leq c(t)\}$ is a disjoint union of events of the form $\{D_{k\ell}(t) = m, L_{D_{k\ell}(t)} \leq c(t)\}$ for $m<g(t)$, so we can apply the change--of--measure transformation \eqref{eq:change_measure} to write
    \begin{align}
        \P_\pi^\theta(D_{k\ell}(t) < g(t), L_{D_{k\ell}(t)} \leq c(t))
        &\leq e^{c(t)} \P_\pi^{\theta'}(D_{k\ell}(t) < g(t), L_{D_{k\ell}(t)} \leq c(t)).
    \end{align}
    Recall \eqref{eq:def_A_vector}, \eqref{eq:def_Delta_set}, and that $\theta_{k\ell}'\in \Delta(\theta,k\ell)$ by assumption, implying $(k\ell)\in O(\theta')$.
    Therefore, by the definition of $g(t)$ in~\eqref{eq:def_g_c_t} and assumption~\eqref{eq:consistency},
    \begin{align}
        \P_\pi^{\theta'}(D_{k\ell}(t) < g(t), L_{D_{k\ell}(t)} \leq c(t))
        \leq \P_\pi^{\theta'}(D_{k\ell}(t) < g(t)) 
        = \P_\pi^{\theta'}\Bigl(D_{k\ell}(t) < \frac{(1-\delta)\ln(t)}{I(\theta_{k\ell}, \theta_{k\ell}')}\Bigr) = o(1) \ \mathrm{as} \ t\to\infty.
    \end{align}
    This completes the proof. 

    \Halmos \\
\end{proof}

%% file: AppendixUB.tex
\label{app:regret_ub}
\begin{proof}{}
    Let
    \begin{align}
        \calT_k := \sum_{m=1}^k H_m
        \label{eq:def_duration_k_episodes}
    \end{align}
    denote the total duration of the first $k\in\N_{\geq 1}$ episodes. 
    Then,~$\calD_{ij}(\calT_k)$ is the number of type-$(ij)$ departures up to and including episode~$k$. 
    Moreover, let~$\calD_{ij}^{m}$ denote the number of type-$(ij)$ departures within episode $m\in\N_{\geq 1}$.
    Note that this quantity is similar to~$D_{ij}^m$ as introduced in Section~\ref{sec:ucbqr_proof_outline}, although~$\calD_{ij}^m$ includes the departures within the warmup time of the episode. Note that $D_{ij}(\calT_k) = \sum_{m=1}^k \calD_{ij}^m$ for any $(ij)\in\calL$ and $k\in\N_{\geq 1}$. 
    
    First, we show that Algorithm~\ref{alg:learning_alg} satisfies the consistency assumption~\eqref{eq:consistency}.
    Let $\theta\in\R_{\geq 0}^L$, $(ij)\in O(\theta)$, $b \in\R_{> 0}$, $k\in\N_{\geq 1}$ and  $z_k := \lceil k/2 \rceil$.
    We use a similar analysis as in the proof of Lemma~\ref{lem:prob_insuf_samples} in Appendix~\ref{app:prob_insuf_samples}.
    In particular, consider the stationary queue length process as introduced in Section~\ref{sec:step1} (dropping the action $a$ from notation)
    and the event $\underline{\Omega}_{ij}^{m-1}$ as introduced in~\eqref{eq:def_omega}. 
    
    Recall also that $A_m$ is the action chosen by Algorithm~\ref{alg:learning_alg} in Line~\ref{line:maxucb} in episode number~$m$. Recall that 1 is the optimal action by assumption.
    Event $\underline{\Omega}_{ij}^{m-1}$ implies that the number of departures in episode $m-1$ after the warmup time is at least as large as the number of departures of the stationary queue length process. 
    Consider the event $A_{m-1}=A_m=1$. In this case, the stationary measure is the same for episodes $m-1$ and $m$, so $\underline{\Omega}_{ij}^{m-1}$ implies that $\calD_{ij}^m \geq \hat{\calD}_{ij}^m$, where $\hat{\calD}_{ij}^m$ is the number of type-$(ij)$ departures of the stationary queue length process during episode $m$. 
    Lastly, from Burke's theorem \mbox{\citep[II.2.4 Theorem 2.1]{Cohen1981}} and the Poisson split and merge properties \citep{Cinlar1968}, it follows that on the event $A_m=1$, we have $\hat{\calD}_{ij}^m \sim \poi{x_{ij}^1 H_m}$.

    By the law of total probability, 
    \begin{align}
        \P^\theta(D_{ij}(\calT_k) < b \ln(\calT_k)) 
        &\leq \P^\theta\Bigl(\sum_{m=2}^k \calD_{ij}^m < b\ln(\calT_k)\Bigr) \\
        &\leq \P^\theta\Bigl(\sum_{m=2}^k \calD_{ij}^m < b \ln(\calT_k) \Bigm| \sum_{m=2}^k \ind{A_m=A_{m-1}=1, \ \underline{\Omega}_{ij}^{m-1}} > z_k\Bigr) \\
        &\qquad + \P^\theta\Bigl(\sum_{m=2}^k \ind{A_m=A_{m-1}=1, \ \underline{\Omega}_{ij}^{m-1}} \leq z_k\Bigr).
        \label{eq:dep_leq_log_split}
    \end{align}
    We analyze these probabilities individually. Let $X_{ij}^m \sim \poi{x_{ij}^1 H_m}$ be independent across $m\in\N_{\geq 1}$.
    We note that $\calD_{ij}^m \geq X_{ij}^m$ whenever $A_m=A_{m-1}=1$ and $\underline{\Omega}_{ij}^{m-1}$, hence
    \begin{align}
        \P^\theta\Bigl(\sum_{m=2}^k \calD_{ij}^m < b\ln(\calT_k) \bigm| \sum_{m=2}^k \ind{A_m=A_{m-1}=1, \ \underline{\Omega}_{ij}^{m-1}} > z_k\Bigr) 
        &\leq \P^\theta\Bigl(\sum_{m=2}^{z_k} X_{ij}^m < b\ln(\calT_k) \Bigr).
        \label{eq:dep_leq_log_good}
    \end{align}
    Note that $\sum_{m=2}^{z_k} X_{ij}^m$ is again Poisson distributed. 
    Moreover,  note that in light of the definition of $\calT_k$ in \eqref{eq:def_duration_k_episodes}, there exists a $k_0\in\N$ such that $\sum_{m=2}^{z_k} x_{ij}^1 H_m \geq b\ln(\calT_k)$ for any $k\geq k_0$. 
    Hence, it follows from Lemma~\ref{lem:pois_tail} in Appendix~\ref{app:prob_insuf_samples} that for $k\geq k_0$, 
    \begin{align}
        \P^\theta\Bigl(\sum_{m=2}^{z_k} X_{ij}^m < b\ln(\calT_k) \Bigr) 
        &=
        \P^\theta\Bigl(\sum_{m=2}^{z_k} X_{ij}^m - \sum_{m=2}^{z_k} x_{ij}^1 H_m \leq  b\ln(\calT_k) - \sum_{m=2}^{z_k} x_{ij}^1 H_m \Bigr) \\
        &\leq 
        \exp\Bigl(\frac{-(b\ln(\calT_k) - \sum_{m=2}^{z_k} x_{ij}^1 H_m)^2}{2 \sum_{m=2}^{z_k} x_{ij}^1 H_m}\Bigr).
    \end{align}
    It is readily verified that
    \begin{align}
        \lim_{k\to\infty} \frac{(b\ln(\calT_k) - \sum_{m=2}^{z_k} x_{ij}^1 H_m)^2}{2 \sum_{m=2}^{z_k} x_{ij}^1 H_m}  = \infty.
    \end{align}
    Hence, \eqref{eq:dep_leq_log_good} is $o(1)$ as $k\to\infty$.

    For the other probability in~\eqref{eq:dep_leq_log_split}, note
    \begin{align}
        \P^\theta\Bigl(\sum_{m=2}^k \ind{A_m=A_{m-1}=1, \ \underline{\Omega}_{ij}^{m-1}} \leq z_k\Bigr) 
        &= 
        \P^\theta\Bigl(k - \sum_{m=2}^k \ind{A_m=A_{m-1}=1, \ \underline{\Omega}_{ij}^{m-1}} \geq k - z_k\Bigr) \\
        &\leq 
        \P^\theta\Bigl(\sum_{m=2}^k \Bigl(\ind{A_m\neq 1} + \ind{A_{m-1}\neq 1} + \ind{(\underline{\Omega}_{ij}^{m-1})^c} \Bigr) \geq k - z_k\Bigr).
    \end{align} 
    From Boole's inequality, we obtain
    \begin{align}
        &\P^\theta\Bigl(\sum_{m=2}^k \Bigl(\ind{A_m\neq 1} + \ind{A_{m-1}\neq 1} + \ind{(\underline{\Omega}_{ij}^{m-1})^c} \Bigr) \geq k - z_k\Bigr) \\
        &\qquad \leq 2 \P^\theta\Bigl(\sum_{m=1}^k \ind{A_m\neq 1} \geq \frac{k - z_k}{3}\Bigr) 
        + 
        \P^\theta\Bigl(\sum_{m=1}^k \ind{(\underline{\Omega}_{ij}^{m-1})^c} \geq \frac{k - z_k}{3}\Bigr).
    \end{align}
    In view of the definition of $S^a(k)$ in~\eqref{eq:Sak_def}, $k-z_k = \Theta(k)$, Lemma~\ref{lem:n_bad_episodes}, and Markov's inequality, we have
    \begin{align}
        \P^\theta\Bigl(\sum_{m=1}^k \ind{A_m\neq 1} \geq \frac{k - z_k}{3}\Bigr) 
        &= \P^\theta\Bigl(\sum_{a\neq 1} S^a(k) \geq \frac{k - z_k}{3}\Bigr) 
        \leq \frac{3 \sum_{a\neq 1} \E(S^a(k))}{k-z_k} = o(1) \ \mathrm{as} \ k\to\infty.
    \end{align}
    Lastly, note that $\P^\theta((\underline{\Omega}_{ij}^m)^c) \leq 1/m^\beta$ (see~\eqref{eq:P_omega_bnd}) and that the $\underline{\Omega}_{ij}^m$ are mutually independent over all episodes. 
    Since $k-z_k = \Theta(k)$ is increasing in $k$ and since $\beta > 1$, we conclude that
    \begin{align}
        \P^\theta\Bigl(\sum_{m=1}^k \ind{(\underline{\Omega}_{ij}^{m-1})^c} \geq \frac{k - z_k}{3}\Bigr) = o(1) \ \mathrm{as} \ k\to\infty.
    \end{align}
    Since $\theta\in\R_{\geq 0}^L$ was chosen arbitrarily,  Algorithm~\ref{alg:learning_alg} satisfies~\eqref{eq:consistency} for any $\xi\in\R_{\geq 0}^L$.

    Next, we prove the regret upper bound by analyzing the terms in the regret decomposition~\eqref{eq:regret_decomp} individually. \\

    \noindent {\it Analysis of term~\textrm{I} in~\eqref{eq:regret_decomp}.}\\
    Let $(ij)\in O^c(\theta)$ be a suboptimal line. 
    Since $(ij)\in O^c(\theta)$, there are no departures of type~$(ij)$ in episodes where Algorithm~\ref{alg:learning_alg} chooses the optimal action~1. 
    Therefore, using \eqref{eq:def_duration_k_episodes} and the law of total probability,
    \begin{align}
        \E^\theta(D_{ij}(\calT_k)) 
        &= \sum_{m=1}^k \E^\theta(\calD_{ij}^{m}) 
        = \sum_{m=1}^k \sum_{a\in\calA\setminus\{1\}}  \E^\theta\bigl(\calD_{ij}^m  \bigm| A_m = a \bigr) \P(A_m=a).
        \label{eq:ucbqr_proof_bnd_total_dep}
    \end{align}
    Note that the expected workload for server $j$ under the $\FCFSRR{a}$ routing algorithm in Algorithm \ref{alg:routing} does not exceed  capacity $\mu_j -\eps$ , since the routing rates~$x_{k\ell}^a$ for all lines~$(k\ell)\in\calL$ satisfy constraint~\eqref{eq:LP_eps_constr_mu}. 
    Hence, for any $a\in\calA$ 
    \begin{align}
        \E^\theta\bigl(\calD_{ij}^{m} \bigm| A_m = a \bigr) \leq  H_m(\mu_j-\eps).
        \label{eq:ucbqr_proof_bnd_per_episode}
    \end{align}
    Applying~\eqref{eq:ucbqr_proof_bnd_per_episode} in~\eqref{eq:ucbqr_proof_bnd_total_dep} gives
    \begin{align}
        \E^\theta(D_{ij}(\calT_k)) 
        &\leq \sum_{m=1}^k \sum_{a\in\calA\setminus\{1\}}  H_m (\mu_j-\eps) \E^\theta(\ind{A_m=a})
        \leq H_k (\mu_j-\eps) \sum_{a\in\calA\setminus\{1\}}  \E^\theta(S^a(k)).
    \end{align}
    where we used in last inequality that the episode length $H_m$ is nondecreasing in $m$ by~\eqref{eq:ucbqr_episode_length}, and definition~\eqref{eq:Sak_def}.
    We use $\calT_k\geq k$ and Lemma~\ref{lem:n_bad_episodes} to obtain that for any suboptimal action $a\in\calA\setminus\{1\}$, 
    \begin{align}
        \lim_{k\to\infty} \frac{\E^\theta(S^a(k))}{\ln^{\beta}(\calT_k)} 
        &\leq \lim_{k\to\infty} \frac{\E^\theta(S^a(k))}{\ln^{\beta}(k)} 
        \leq 1.
        \label{eq:ucbqr_proof_bnd_bad_episode}
    \end{align}
    Moreover, $\lim_{k\to\infty} \ln(2Jk)/\ln(k) = 1$ and $\alpha\geq 1$ (recall~\eqref{eq:c0_choice}), so $\lim_{k\to\infty} H_k / (\alpha\ln^\beta(k)) \leq 1$. Hence, 
    \begin{align}
        \lim_{k\to\infty} \frac{\E^\theta(D_{ij}(\calT_k))}{\alpha\ln^{2\beta}(\calT_k)} 
        \leq \lim_{k\to\infty} \frac{H_k (\mu_j-\eps) \sum_{a\in\calA\setminus\{1\}}  \E^\theta(S^a(k))}{\alpha\ln^{2\beta}(k)}
        \leq (\mu_j-\eps)|\calA|.
        \label{eq:ucbqr_proof_bnd_part_I}
    \end{align}

    \noindent {\it Analysis of term~\textrm{II} in~\eqref{eq:regret_decomp}.}\\

    Recall that $(v_i^\theta)_{i\in\calI}$ and $(w_j^\theta)_{j\in\calJ}$ in \eqref{eq:regret_decomp} are the optimal dual variables, i.e., the minimizers of \eqref{eq:LP_eps_dual}.
    By \eqref{eq:LP_eps_dual_constr_w}, we have $w_j^\theta\geq 0$ for all $j\in\calJ$.
    The contribution of term~\textrm{II} in~\eqref{eq:regret_decomp} is clearly zero for $j\in\calJ$ with $w_j^\theta = 0$. 

    Let $j\in\calJ$ be such that $w_j>0$. 
    Recall that $x^1$ and $v^\theta,w^\theta$ satisfy the complementary slackness conditions \eqref{eq:slack_cond_w_eps}, thus $w_j^\theta > 0$ implies $\mu_j - \eps = \sum_{i\in\calC_j} x_{ij}^1$.
    Hence, 
    \begin{align}
        \lim_{k\to\infty} \frac{\calT_k (\mu_j-\eps) - \sum_{i\in\calC_j}  \E^\theta(D_{ij}(\calT_k))}{\alpha\ln^{2\beta}(\calT_k)}
        &= \lim_{k\to\infty} \sum_{i\in\calC_j} \frac{ \calT_k x_{ij}^1 -  \E^\theta(D_{ij}(\calT_k))}{\alpha\ln^{2\beta}(\calT_k)}.
        \label{eq:ucbqr_proof_bnd_tbp_II}
    \end{align}
    Let $i\in\calC_j$. 
    By~\eqref{eq:def_duration_k_episodes} and the law of total probability, 
    \begin{align}
        \calT_k x_{ij}^1 -  \E^\theta(D_{ij}(\calT_k))  
        &= \sum_{m=1}^k \sum_{a\in\calA} \E^\theta\bigl((H_m x_{ij}^1 - \calD_{ij}^m)\ind{A_m=a}\bigr).
        \label{eq:ucbqr_proof_II_dep_actions}
    \end{align}
    We consider the optimal action $a=1$ and all suboptimal actions $a\neq 1$ separately.

    For suboptimal actions, note that $\calD_{ij}^m$ is nonnegative, the episode length $H_m$ is nondecreasing in $m$, and recall \eqref{eq:Sak_def}, hence
    \begin{align}
        \sum_{m=1}^k \sum_{a\in\calA\setminus\{1\}} \E^\theta\bigl((H_m x_{ij}^1 - \calD_{ij}^m)\ind{A_m=a}\bigr) 
        &\leq \sum_{m=1}^k \sum_{a\in\calA\setminus\{1\}} H_m x_{ij}^1 \E^\theta(\ind{A_m=a}) \\
        &\leq H_k x_{ij}^1 \sum_{a\in\calA\setminus\{1\}}  \E^\theta(S^a(k)).
        \label{eq:ucbqr_proof_II_subopti}
    \end{align}
    
    On the other hand, for the optimal action, we consider for each episode $m$ also the preceding episode~$m-1$. 
    We let~$A_0 := 2$. 
    Then
    \begin{align}
        \sum_{m=1}^k \E^\theta\bigl((H_m x_{ij}^1 - \calD_{ij}^m) \ind{A_m=1}\bigr) 
        & = \sum_{m=1}^k \E^\theta\bigl((H_m x_{ij}^1 - \calD_{ij}^m) (\ind{A_m=1, A_{m-1} \neq 1}+\ind{A_m=A_{m-1}=1})\bigr).
        \label{eq:ucbqr_proof_II_opt_split}
    \end{align}
    Since $\calD_{ij}^m$ is nonnegative we have, similar to~\eqref{eq:ucbqr_proof_II_subopti},
    \begin{align}
        \sum_{m=1}^k \E^\theta\bigl((H_m x_{ij}^1 - \calD_{ij}^m) \ind{A_m=1,A_{m-1}\neq 1}\bigr)
        \leq H_k x_{ij}^1 \sum_{m=1}^k \E^\theta(\ind{A_{m-1}\neq 1}) 
        = H_k x_{ij}^1 \sum_{a\in\calA\setminus\{1\}} \E^\theta(S^a(k)).
        \label{eq:ucbqr_proof_II_opt_split_dep_geq}
    \end{align}
    
    For the other term in~\eqref{eq:ucbqr_proof_II_opt_split},  consider the event $\underline{\Omega}_{ij}^{m-1}$ as introduced in~\eqref{eq:def_omega}. 
    By the law of total probability, 
    \begin{align}
        &\E^\theta\bigl((H_m x_{ij}^1 - \calD_{ij}^m) \ind{A_m=1=A_{m-1}=1}\bigr)
        =\E^\theta\bigl((H_m x_{ij}^1 - \calD_{ij}^m) \ind{A_m=1=A_{m-1}=1} \bigm| \underline{\Omega}_{ij}^{m-1} \bigr)\P^\theta( \underline{\Omega}_{ij}^{m-1} ) \\
        &\qquad + \E^\theta\bigl((H_m x_{ij}^1 - \calD_{ij}^m) \ind{A_m=1=A_{m-1}=1} \bigm| (\underline{\Omega}_{ij}^{m-1})^c \bigr)\P^\theta( (\underline{\Omega}_{ij}^{m-1})^c ).
        \label{eq:ucbqr_proof_II_cond_dep_split}
    \end{align}
    Since $\underline{\Omega}_{ij}^{m-1}$ implies that $\calD_{ij}^m \geq \hat{\calD}_{ij}^m$, and $\hat{D}_{ij}^m\ind{A_m=1} \sim \poi{x_{ij}^1 H_m}$, we have
    \begin{align}
        \E^\theta\bigl((H_m x_{ij}^1 - \calD_{ij}^m) \ind{A_m=1=A_{m-1}=1} \bigm| \underline{\Omega}_{ij}^{m-1} \bigr)
        &\leq \E^\theta\bigl((H_m x_{ij}^1 - \hat{\calD}_{ij}^m) \ind{A_m=1=A_{m-1}=1} \bigm| \underline{\Omega}_{ij}^{m-1} \bigr) = 0.
        \label{eq:ucbqr_proof_II_cond_dep_1}
    \end{align}
    On the other hand, by~\eqref{eq:P_omega_bnd}, $\P^\theta((\underline{\Omega}_{ij}^{m-1})^c) \leq 1/m^\beta$.
    This, together with $\calD_{ij}^m \geq 0$ gives
    \begin{align}
        \E^\theta\bigl((H_m x_{ij}^1 - \calD_{ij}^m) \ind{A_m=1=A_{m-1}=1} \bigm| (\underline{\Omega}_{ij}^{m-1})^c \bigr)\P^\theta( (\underline{\Omega}_{ij}^{m-1})^c ) 
        \leq \frac{H_m x_{ij}^1}{m^\beta}.
        \label{eq:ucbqr_proof_II_cond_dep_2}
    \end{align}
    Hence, by~\eqref{eq:ucbqr_proof_II_cond_dep_split},~\eqref{eq:ucbqr_proof_II_cond_dep_1}, and~\eqref{eq:ucbqr_proof_II_cond_dep_2} we have
    \begin{align}
        \sum_{m=2}^k \E^\theta\bigl((H_m x_{ij}^1 - \calD_{ij}^m) \ind{A_m=A_{m-1}=1} \bigr) 
        &\leq \sum_{m=2}^k \frac{H_m x_{ij}^1}{m^\beta}.
        \label{eq:ucbqr_proof_II_opt_split_dep2}
    \end{align}

    Substitution of~\eqref{eq:ucbqr_proof_II_subopti},~\eqref{eq:ucbqr_proof_II_opt_split_dep_geq}, and~\eqref{eq:ucbqr_proof_II_opt_split_dep2} into~\eqref{eq:ucbqr_proof_II_dep_actions} yields
    \begin{align}
        \calT_k x_{ij}^1 - \E^\theta(D_{ij}(\calT_k))  
        &\leq 2 H_k x_{ij}^1 \sum_{a\in\calA\setminus\{1\}} \E^\theta(S^a(k)) + \sum_{m=2}^k \frac{H_mx_{ij}^1}{m^\beta}.
    \end{align}

    For the first term, we invoke Lemma~\ref{lem:n_bad_episodes} similarly as in~\eqref{eq:ucbqr_proof_bnd_part_I}. 
    For the other term, we note that $m^{\beta-1}\geq H_m$ for $m\geq 1$ large enough since $\beta > 1$, and $\sum_{m=1}^k 1/m = \calO(\ln(k))$. 
    Hence, 
    \begin{align}
        \lim_{k\to\infty} \frac{\calT_k x_{ij}^1 -  \E^\theta(D_{ij}(\calT_k))}{\alpha \ln^{2\beta}(k)} 
        &\leq x_{ij}^1 |\calA|.
        \label{eq:ubqr_II_lim}
    \end{align}

    Concluding from~\eqref{eq:ucbqr_proof_bnd_tbp_II} and \eqref{eq:ubqr_II_lim} and using $\sum_{i\in\calJ} x_{ij}^1 \leq \mu_j-\eps$, we have for any $j\in\calJ$ with $w_j^\theta > 0$ that
    \begin{align}
        \lim_{k\to\infty} \frac{\calT_k (\mu_j-\eps) - \sum_{i\in\calC_j}  \E^\theta(D_{ij}(\calT_k))}{\alpha\ln^{2\beta}(\calT_k)}
        &\leq  (\mu_j-\eps) |\calA|.
        \label{eq:ucbqr_proof_part_II}
    \end{align}

    \noindent {\it Analysis of term~\textrm{III} in~\eqref{eq:regret_decomp}.}\\
    Consider an episode $m\in\N_{\geq 1}$ and a customer type $i\in\calI$. 
    Algorithm~\ref{alg:learning_alg} routes type-$i$ customers to servers according to the routing rates $x_{ij}^{A_m}$ using the \FCFSRR{A_m} policy as described in~Algorithm~\ref{alg:routing}.
    Since $x^{A_m}$ is a basic feasible solution of \LP{\theta,\eps} in \eqref{eq:LP_eps}, it satisfies~\eqref{eq:LP_eps_constr_x}. 
    This means that any type-$i$ customer is placed in a virtual queue of a compatible server $j\in\calS_i$.
    We next show that all virtual queues are positive recurrent. 

    Let $j\in\calJ$. 
    Recall that for any $a\in\calA$, $x^a$ satisfies~\eqref{eq:LP_eps_constr_mu}.
    This implies that the queue length process $Q_j^m(t)$, $t\in [0,H_m)$ of the virtual queue of server $j$ during episode $m$ (as introduced in Section~\ref{sec:step1}) has negative drift:
    \begin{align}
        \lim_{h\to 0}\frac{\E^\theta\bigl(Q_j^m(t+h)\bigr) - \E^\theta\bigl(Q_j^m(t)\bigr)}{h} &= \sum_{(ij)\in\calL} x_{ij}^{A_m} - \mu_j \leq \mu_j-\eps - \mu_j < 0.
    \end{align}

    From the Foster-Lyapunov criterion \citep[Proposition 4.5]{Bramson2006} follows that the virtual queue length process is positive recurrent.
    Since all type-$i$ customers are placed in a virtual queue, and all virtual queues are positive recurrent,  \eqref{eq:def_stability} is satisfied. 

    From \eqref{eq:def_stability} it follows that
    \begin{align}
        \lim_{k\to\infty} \frac{\calT_k\lambda_i - \sum_{j\in\calS_i} \E^\theta(D_{ij}(\calT_k))}{\alpha\ln^{2\beta}(\calT_k)} &= 0.
        \label{eq:ucbqr_proof_part_III}
    \end{align}

    To conclude the proof, we recall that the dual variables $(v_i^\theta)_{i\in\calI}$ and $(w_j^\theta)_{j\in\calJ}$ are finite by strong duality \citep[Theorem 4.4]{Bertsimas1997}. 
    Hence, applying~\eqref{eq:ucbqr_proof_bnd_part_I},  \eqref{eq:ucbqr_proof_part_II}, and \eqref{eq:ucbqr_proof_part_III} to~\eqref{eq:regret_decomp} gives 
    \begin{align}
        \lim_{t\to\infty} \frac{R^\theta(t)}{\alpha\ln^{2\beta}(t)} 
        &= \lim_{k\to\infty} \frac{R^\theta(\calT_k)}{\alpha\ln^{2\beta}(\calT_k)} 
        \leq \sum_{(k\ell)\in O^c(\theta)}\phi_{k\ell}^\theta (\mu_\ell-\eps)|\calA| + \sum_{j\in\calJ} w_j^\theta(\mu_j-\eps)|\calA|. 
    \end{align}
    \Halmos
\end{proof}

%% file: AppendixOther.tex
\renewcommand{\thelemma}{C.\arabic{lemma}} 
\setcounter{lemma}{0} 

\subsection{Proof of Lemma~\ref{lem:LP_basis}}
\label{app:lem_LP_basis}
\begin{proof}{}
    Observe that \LP{\theta,\eps} in \eqref{eq:LP_eps} is of the form
    \begin{subequations}
        \label{eq:LP2}
        \begin{align}
            \LP{\theta,\eps}: \ \ \max_{x,\sigma} \ \ &
            \begin{bmatrix} \theta \\ 0 \end{bmatrix}^\top
            \begin{bmatrix} x \\ \sigma \end{bmatrix}, \\
            \textrm{s.t.} \ \
            &\begin{bmatrix} A' & 0 \\ A'' & I \end{bmatrix}
            \begin{bmatrix} x \\ \sigma \end{bmatrix} = \begin{bmatrix}  \lambda \\ \mu - \eps \end{bmatrix}, \
            x \in \R_{\geq 0}^L, \ \sigma \in \R_{\geq 0}^J,
        \end{align}
    \end{subequations}
    where 
    \begin{align}
        A' &\in \R^{I\times L}, \ A'_{i,k\ell} = \begin{cases} 1 & \textrm{if} \ i = k \ \textrm{and} \ \ell\in\calS_i, \\ 0 &\textrm{otherwise}, \end{cases} \\
        A'' &\in \R^{J\times L}, \ A'_{j,k\ell} = \begin{cases} 1 & \textrm{if} \ j = \ell\ \textrm{and} \ k\in\calC_j, \\ 0 &\textrm{otherwise}. \end{cases}
    \end{align}

    ``$\Longrightarrow$"
    Let $B=\scrL\cup\scrJ$ be a basis of \LP{\theta,\eps}.
    Since $B$ is a basis of \LP{\theta,\eps}, $|B|=I+J$ and the matrix
    \begin{align}
        \mathbf{B}
        &= \begin{bmatrix} \mathbf{C}_\scrL & \mathbf{D}_\scrJ \end{bmatrix} 
        = \begin{bmatrix} A_\scrL' & 0_\scrJ \\ A_\scrL'' & I_\scrJ \end{bmatrix} \in \R^{(I+J)\times (I+J)} 
        \label{eq:Bmat}
    \end{align}
    formed by the columns of $\scrL\cup\scrJ$ has full rank. 
    Here,  $\mathbf{C}_\scrL$ is the node edge incidence matrix of $\calG(\scrL,\scrJ)$. 
    We will prove that 
    (i)  $\calG(\scrL,\scrJ)$ contains no cycles, 
    (ii) each $v\in\calI\cup\calJ$ is contained in $\calG(\scrL,\scrJ)$,
    and (iii) each tree of  $\calG(\scrL,\scrJ)$ contains a unique node $j\in\scrJ$. 

    For (i), note that the sum of the entries of a row of the incidence matrix is equal to the degree of the corresponding node in the graph. 
    Suppose there is a cycle $\ell_1,\dots,\ell_n\in\scrL$.
    This implies that each row of the matrix $\mathbf{E} = [\mathbf{C}_{\ell_1},\dots,\mathbf{C}_{\ell_n}]$ has sum 2, so that $\mathbf{E} v^T = 0$ for any vector $v\in\R^n$ satisfying $\sum_i v_i = 0$. 
    Then, $\mathbf{B}$ has linearly dependent columns, contradicting the assumption that $\mathbf{B}$ has full rank.

    For (ii), let $i_0\in\calI$ and suppose $i_0\notin\calG(\scrL,\scrJ)$, then row $i_0$ of $A_\scrL'$ contains only zeros. Therefore, row $i_0$ row of $\mathbf{B}$ contains only zeros, which contradicts $\mathrm{rank}(\mathbf{B}) = I+J$.
    Similarly, let  $j_0\in\calJ\setminus\calG(\scrL,\scrJ)$, then row $j_0$ of $A_\scrL''$ and row $j_0$ of $I_\scrJ$ contain only zeros. Therefore, row $j_0$ of $\mathbf{B}$ contains only zeros, which contradicts $\mathrm{rank}(\mathbf{B}) = I+J$.

    Lastly, for (iii), note that by (i) we have that $\calG(\scrL,\scrJ)$ is a union of $K\geq 1$ trees $\calT_1,\dots,\calT_K$ and by (ii), $\calG(\scrL,\scrJ)$ has $I+J$ nodes. 
    By \cite[Theorem 6]{Harary1967}, the rank of the incidence matrix of a graph with $n$ nodes that consists of $K$ connected components is $n-K$, hence $\mathrm{rank}(\mathbf{C}_\scrL) = I+J-K$. Since $\mathbf{B}$ has full rank by assumption, this implies $\mathrm{rank}(\mathbf{D}_\scrJ)=K$. On the other hand, $\mathrm{rank}(\mathbf{C}_\calJ)=|\calJ|$ by construction, thus $|\scrJ| = K$, i.e., the number of servers $\scrJ$ equals the number of trees. It remains to be shown that each tree contains exactly one server $j\in\scrJ$. Suppose there is a tree that contains two servers $j_1,j_2\in\scrJ$, then there is a path $\ell_1,\dots,\ell_n\in\scrL$ between $j_1$ and $j_2$. It can be verified that the row sum of the matrix $\mathbf{E} = [\mathbf{C}_{\ell_1},\dots,\mathbf{C}_{\ell_n},\mathbf{D}_{j_1},\mathbf{D}_{j_2}]$ equals 2 for each row, hence $\mathbf{E} v' = 0$ for any vector $v\in\R^{n+2}$ satisfying $\sum_i v_i = 0$. 
    Then, $\mathbf{B}$ has linearly dependent columns, contradicting the assumption that $\mathbf{B}$ has full rank.

    ``$\Longleftarrow$"
    Let $\scrL\subseteq\calL$ and $\scrJ\subseteq\calJ$ be such that $\calG(\scrL,\scrJ)$ is a spanning forest of $(\calI\cup\calJ,\calL)$.
    Suppose $\calG(\scrL,\scrJ)$ is the union of trees $\calT_1,\dots,\calT_K$.
    Let $B=\scrL\cup\scrJ$ and $\mathbf{B}$ as in~\eqref{eq:Bmat}.
    We will prove that $\mathbf{B}$ has full rank.     
    Since $\calG(\scrL,\scrJ)$ is a graph with $K$ connected components and $I+J$ nodes, we have by \cite[Theorem 6]{Harary1967} that $\mathrm{rank}(\mathbf{C}_\scrL) = I+J-K$. 
    Moreover, $\mathrm{rank}(\mathbf{D}_\scrJ) = |\scrJ|$ by construction.
    Since $\calG(\scrL,\scrJ)$ is a spanning forest, each tree contains exactly one server $j\in\scrJ$, hence $\mathrm{rank}(\mathbf{D}_\scrJ) = |\calJ| = K$. 

    Next, we show that $\mathbf{C}_\scrL$ and $\mathbf{D}_\scrJ$ are linearly independent. 
    Consider a linear combination  $\sum_{\ell\in\scrL'} \alpha_\ell \mathbf{C}_\ell$ of the columns of $\mathbf{C}_\scrL$ with $\alpha_\ell \neq 0$ and $\scrL' \subseteq\scrL$. 
    Since $\calG(\scrL,\scrJ)$ is a forest, the subgraph induced by $\scrL'$ contains no cycles and therefore must contains at least two nodes with degree 1. This means that the vector $\sum_{\ell\in\scrL'} \alpha_\ell \mathbf{C}_\ell$ must have at least two nonzero entries corresponding to the nodes with degree 1. However, any column $b\in\mathbf{D}_\scrJ$ has only one nonzero entry and therefore $b\notin\mathrm{span}(\mathbf{C}_\scrL)$. Since any column of $\mathbf{D}_\scrJ$ is linearly independent of $\mathbf{C}_\scrL$ and the columns of $\mathbf{D}_\scrJ$ are linearly independent by construction, we have that $\mathbf{C}_\scrL$ and $\mathbf{D}_\scrJ$ are linearly independent.

    We have shown that $\mathbf{C}_\scrL$ and $\mathbf{D}_\scrJ$ are linearly independent and that $\mathrm{rank}(\mathbf{C}_\scrL) + \mathrm{rank}(\mathbf{D}_\scrJ) = I+J$.
    Together, this implies $\mathrm{rank}(\mathbf{B}) = I+J$ and completes the proof. 
    \Halmos
\end{proof}

\newpage
\subsection{Proof of Lemma~\ref{lem:basic_solutions}}
\label{app:lem_basic_solutions}

\begin{proof}{}
    Let $B=\scrL\cup\scrJ$ be a basis of \LP{\theta,\eps}.
    From Lemma~\ref{lem:LP_basis}, we obtain that $\calG(\scrL,\scrJ)$ is a spanning forest of $(\calI\cup\calJ,\calL)$.
    Let $x\in\R^L$ be defined as in~\eqref{eq:basic_sol} and define $\sigma\in\R^J$ by $\sigma_j:=\mu_j-\eps-\sum_{i\in\calC_j} x_{ij}$. 
    Recall that \LP{\theta,\eps} is of the form~\eqref{eq:LP2}.
    Hence, we must show that $[x,\sigma]^T\in\R^{L+J}$ satisfies 
    \begin{align}
        \mathbf{B}  \begin{bmatrix} x \\ \sigma \end{bmatrix} &= \begin{bmatrix}  \lambda \\ \mu - \eps \end{bmatrix},
        \label{eq:basic_sol_structure}
    \end{align} 
    where $\mathbf{B}$ is defined as in~\eqref{eq:Bmat}. 

    Let $i\in\calI$. 
    Recall that $i$ is contained in the spanning forest and that $i$ cannot be the root of a tree since $\scrJ\subseteq\calJ$ are the root nodes. 
    \begin{itemize}
        \item If $i$ is a leaf node, then $i$ has a unique parent node $j_0\in\calJ$ and 
        so $\mathbf{B}_{ij_0} =1$ and $\mathbf{B}_{ij}=0$ for all $j\neq j_0$. 
        Hence, $(\mathbf{B} [x, \sigma])_i = x_{ij_0}$. 
        In this case, $\subtree(i) = \{i\}$, so we obtain from \eqref{eq:basic_sol} that $x_{ij_0} = \lambda_i$.
        \item If $i$ is not a leaf node, then $i$ has a unique parent node $j_0\in\calJ$ and child nodes $j_1,\dots,j_m\in\calJ$ for some $m\geq 1$.
        Hence, $(\mathbf{B}  [x, \sigma])_i = x_{ij_0} + \sum_{k=1}^m x_{ij_k}$. 
        In this case, $\sfC(\subtree(i)) = \{i\} \cup \bigcup_{k=1}^m \sfC(\subtree(j_k))$ and $\sfS(\subtree(i)) = \bigcup_{k=1}^m \sfS(\subtree(j_k))$.
        Therefore by \eqref{eq:basic_sol},
        \begin{align}
            (\mathbf{B}  [x, \sigma])_i &= x_{ij_0} + \sum_{k=1}^m x_{ij_k}
            = \sum\limits_{k\in\sfC(\subtree(i))} \lambda_k - \sum\limits_{\ell\in\sfS(\subtree(i))} (\mu_\ell-\eps)
            + \sum_{k=1}^m \Bigl(\sum\limits_{\ell\in\sfS(\subtree(j_k))} (\mu_\ell-\eps) - \sum\limits_{k\in\sfC(\subtree(j_k))} \lambda_k\Bigr)
            = \lambda_i.
        \end{align}
    \end{itemize}
    Let now $j\in\calJ$.
    \begin{itemize}
        \item If $j$ is a root node, then $j$ has no parent node and possibly child nodes $i_1,\dots,i_m\in\calC$ for some $m\geq 1$. 
        Since $x_{\ell j} = 0$ for all $(\ell j)\notin\scrL$, we have
        $(\mathbf{B}[x,\sigma])_j = \sum_{k=1}^m x_{i_kj} + \sigma_j = \sum_{i\in\calC_j} x_{ij} + \sigma_j = \mu_j - \eps$. 
        \item If $j$ is a leaf node, then $j$ has a unique parent node $i_0\in\calC$ and so
        $(\mathbf{B} [x,\sigma])_j = x_{i_0 j} + \sigma_j$ by definition of $\sigma_j$. 
        In this case, $\subtree(j) = \{j\}$, so we obtain from \eqref{eq:basic_sol} that $x_{i_0j} = \mu_j - \eps$.
        Since $x_{\ell j} = 0$ for all $(\ell j)\notin\scrL$, it follows that $\sigma_j = \mu_j - \eps - x_{i_0j} = 0$. 
        Hence $(\mathbf{B} [x,\sigma])_j = \mu_j - \eps$.
        \item If $j$ is not a leaf or a root node, then $j$ has a unique parent node $i_0\in\calC$ and child nodes $i_1,\dots,i_m\in\calC$ for some $m\geq 1$
        and so $(\mathbf{B} [x,\sigma])_j = x_{i_0 j} + \sum_{k=1}^m x_{i_kj} + \sigma_j$.
        Since $x_{\ell j} = 0$ for all $(\ell j)\notin\scrL$, it follows that $\sigma_j = \mu_j - \eps - x_{i_0j} -\sum_{k=1}^m x_{i_km} = 0$. 
        In this case, $\sfS(\subtree(j)) = \{j\} \cup \bigcup_{k=1}^m \sfS(\subtree(i_k))$ and $\sfC(\subtree(j)) = \bigcup_{k=1}^m \sfC(\subtree(i_k))$.
        Therefore by \eqref{eq:basic_sol},
        \begin{align}
            (\mathbf{B} [x,\sigma])_j 
            &= x_{i_0j} + \sum_{k=1}^m x_{i_kj}  \\
            &= \sum\limits_{\ell\in\sfS(\subtree(j))} (\mu_\ell-\eps) - \sum\limits_{k\in\sfC(\subtree(j))} \lambda_k
            + \sum_{k=1}^m \Bigl(\sum\limits_{k\in\sfC(\subtree(i_k))} \lambda_k - \sum\limits_{\ell\in\sfS(\subtree(i_k))} (\mu_\ell-\eps)\Bigr) \\
            &= \mu_j - \eps.
        \end{align}
    \end{itemize}
    Hence \eqref{eq:basic_sol_structure} is satisfied and the proof is concluded.
    \Halmos
\end{proof}

\newpage
\subsection{Proof of Lemma~\ref{lem:insensitive}}
\label{app:proof_lem_insensitive}
Let $B=\scrL\cup\scrJ$ be a nondegenerate feasible basis of \LP{\theta,0} and let $\eps \geq 0$ satisfy~\eqref{eq:def_eps_insensitive}.
From Lemma~\ref{lem:basic_solutions}, we obtain that $[x^0,\sigma^0]^T\in\R^{L+J}$ defined by 
\begin{align}
    x_{ij}^0 =
    \begin{cases}
        \sum\limits_{k\in\sfC(\subtree(i))} \lambda_k - \sum\limits_{\ell\in\sfS(\subtree(i))} \mu_\ell, 
        &\textrm{if} \ i \ \textrm{is a child of} \ j \ \textrm{in} \ \calG(\scrL,\scrJ),
        \\
        \sum\limits_{\ell\in\sfS(\subtree(j))} \mu_\ell - \sum\limits_{k\in\sfC(\subtree(j))} \lambda_k, 
        &\textrm{if} \ i \ \textrm{is the parent of} \ j \ \textrm{in} \ \calG(\scrL,\scrJ), 
        \\
        0 
        &\textrm{if} \ (ij)\in\calL\setminus \scrL,
    \end{cases} 
    \qquad
    \sigma_j^0 = \mu_j - \sum_{i\in\calC_j} x_{ij}^0,
    \label{eq:basic_sol_0}
\end{align}
is the corresponding basic solution of \LP{\theta,0}. 
Similarly, we obtain that $[x^\eps,\sigma^\eps]^T\in\R^{L+J}$ defined by 
\begin{align}
    x_{ij}^\eps =
    \begin{cases}
        \sum\limits_{k\in\sfC(\subtree(i))} \lambda_k - \sum\limits_{\ell\in\sfS(\subtree(i))} (\mu_\ell-\eps), 
        &\textrm{if} \ i \ \textrm{is a child of} \ j \ \textrm{in} \ \calG(\scrL,\scrJ),
        \\
        \sum\limits_{\ell\in\sfS(\subtree(j))} (\mu_\ell-\eps) - \sum\limits_{k\in\sfC(\subtree(j))} \lambda_k, 
        &\textrm{if} \ i \ \textrm{is the parent of} \ j \ \textrm{in} \ \calG(\scrL,\scrJ), 
        \\
        0 
        &\textrm{if} \ (ij)\in\calL\setminus \scrL,
    \end{cases} 
    \qquad
    \sigma_j^\eps = \mu_j - \sum_{i\in\calC_j} x_{ij}^\eps,
    \label{eq:basic_sol_eps}
\end{align}
is the corresponding basic solution of \LP{\theta,\eps}.
It remains to be shown that $[x^\eps,\sigma^\eps]^T$ is nondegenerate and feasible for \LP{\theta,\eps}.
To this end, we show that (a) $x_{ij}^\eps > 0$ for $(ij)\in\scrL$ and $x_{ij}^\eps = 0$ otherwise, and (b) $\sigma_j^\eps > 0$ for $j\in\scrJ$ and $\sigma_j^\eps = 0$ otherwise.

{\it Proof of (a).}
Let $(ij)\in\scrL$.
If $i$ is a child of $j$ in $\calG(\scrL,\scrJ)$, then by \eqref{eq:basic_sol_eps},
\begin{align}
    x_{ij}^\eps
    &= \sum\limits_{k\in\sfC(\subtree(i))} \lambda_k - \sum\limits_{\ell\in\sfS(\subtree(i))} (\mu_\ell-\eps)
    \stackrel{\eps > 0}{>} \sum\limits_{k\in\sfC(\subtree(i))} \lambda_k - \sum\limits_{\ell\in\sfS(\subtree(i))} \mu_\ell
    \stackrel{\textrm{\eqref{eq:basic_sol_0}}}{=} x_{ij}^0
    > 0,
\end{align}
where the last step follows since $x^0$ is a nondegenerate basic feasible solution of \LP{\theta,0}.
On the other hand, if $i$ is the parent of $j$ in  $\calG(\scrL,\scrJ)$, then by \eqref{eq:basic_sol_eps},
\begin{align}
    x_{ij}^\eps
    &= \sum\limits_{\ell\in\sfS(\subtree(j))} (\mu_\ell-\eps) - \sum\limits_{k\in\sfC(\subtree(j))} \lambda_k
    = \sum\limits_{\ell\in\sfS(\subtree(j))} \mu_\ell - \sum\limits_{k\in\sfC(\subtree(j))} \lambda_k - |\sfS(\subtree(j))| \eps
    \stackrel{\textrm{\eqref{eq:def_eps_insensitive}}}{>} 0.
\end{align}
Lastly, let $(ij)\in\calL\setminus\scrL$, then $x_{ij}^\eps = 0$ follows directly from~\eqref{eq:basic_sol_eps}. 
Hence, (a) is satisfied. 

{\it Proof of (b).}
Let $j\in\scrJ$. By construction, $j$ is the root of a tree in the spanning forest $\calG(\scrL,\scrJ)$ so by \eqref{eq:basic_sol_eps},
\begin{align}
    \sum_{i\in\calC_j} x_{ij}^\eps
    &= \sum_{i: (ij)\in\scrL} \Bigl(\sum\limits_{k\in\sfC(\subtree(i))} \lambda_k - \sum\limits_{\ell\in\sfS(\subtree(i))} (\mu_\ell-\eps) \Bigr) \\
    &= (|\sfS(\subtree(j))|-1)\eps + \sum_{i: (ij)\in\scrL} \Bigl(\sum\limits_{k\in\sfC(\subtree(i))} \lambda_k - \sum\limits_{\ell\in\sfS(\subtree(i))} \mu_\ell \Bigr).
    \label{eq:ins_proof_sum_x_eps}
\end{align}
By \eqref{eq:def_eps_insensitive},
\begin{align}
    |\sfS(\subtree(j))|\eps
    &< \sum_{\ell\in\sfS(\subtree(j))} \mu_\ell  - \sum_{k\in\sfS(\subtree(j))} \lambda_k
    = \mu_j + \sum_{i: (ij)\in\scrL} \Bigl(\sum_{\ell\in\sfS(\subtree(i))} \mu_\ell  - \sum_{k\in\sfS(\subtree(i))} \lambda_k\Bigr).
    \label{eq:ins_proof_Seps}
\end{align}
Here, we used that $j$ is a root and therefore $\sfS(\subtree(j))$ is the union of $\{j\}$ and $\bigcup_{i:(ij)\in\scrL} \sfS(\subtree(i))$.
Bounding \eqref{eq:ins_proof_sum_x_eps} using \eqref{eq:ins_proof_Seps} gives
\begin{align}
    \sigma_j^\eps = \mu_j - \eps - \sum_{i\in\calC_j} x_{ij}^\eps > 0.
\end{align}
Lastly, let $j\in\calJ\setminus\scrJ$, then $j$ is not a root node in a tree in $\calG(\scrL,\scrJ)$. 
It can be verified that $\sigma_j^\eps = 0$ using a similar approach as in the proof of Lemma~\ref{lem:basic_solutions} (see the last two steps in Appendix~\ref{app:lem_basic_solutions}).
Hence, (b) is satisfied and the proof is concluded. 

\newpage
\subsection{Proof of Lemma~\ref{lem:oracle_reward}}\label{app:lem_oracle_reward}
Since $\pi$ satisfies~\eqref{eq:def_stability},
\begin{align}\label{eq:oracle_reward_proof_lam}
    \lim_{t\to\infty} \sum_{j\in\calS_i}  \frac{\E_\pi^\theta(D_{ij}(t))}{t} &= \lambda_i, \ \ \ \forall i\in\calI.
\end{align}
Moreover, by~\eqref{eq:def_eps_restrictive},
\begin{align}
    \label{eq:oracle_reward_proof_mu}
    \lim_{t\to\infty} \sum_{i\in\calC_j} \frac{ \E_\pi^\theta(D_{ij}(t))}{t} &\leq \mu_j - \eps, \ \ \ \forall j\in\calJ.
\end{align}
Now, since the summation over all lines $(ij)\in\calL$ in \eqref{eq:regret} is finite, we can interchange the summation and limit to obtain
\begin{align}
    \lim_{t\to\infty} \sum_{(ij)\in\calL}  \frac{\E_\pi^\theta(D_{ij}(t))}{t} \theta_{ij}
    =
    \sum_{(ij)\in\calL} \lim_{t\to\infty} \frac{\E_\pi^\theta(D_{ij}(t))}{t} \theta_{ij}.
\end{align}
In light of \eqref{eq:oracle_reward_proof_lam} and \eqref{eq:oracle_reward_proof_mu}, the set of values $\{\lim_{t\to\infty}\E_\pi^\theta(D_{ij}(t))/t\}_{(ij)\in\calL}$ satisfies the constraints of \LP{\theta,\eps} in \eqref{eq:LP_eps}.
Since $x^\theta$ is the optimal solution of \LP{\theta,\eps}, we have
\begin{align}
    \sum_{(ij)\in\calL} \lim_{t\to\infty} \frac{\E_\pi^\theta(D_{ij}(t))}{t} \theta_{ij} \leq \sum_{(ij)\in\calL} x_{ij} \theta_{ij}.
\end{align}
The result follows from \eqref{eq:regret},
\begin{align}
    \lim_{t\to\infty} \frac{R_\pi^\theta(t)}{t}
    &= \lim_{t\to\infty} \sum_{(ij)\in\calL} x_{ij} \theta_{ij} - \frac{\E_\pi^\theta(D_{ij}(t))}{t} \geq 0.
\end{align}

\newpage
\subsection{Proof of Lemma~\ref{lem:regret_decomp}}
\label{app:proof_regret_decomposition}

\begin{proof}{}
    Let $\pi$ satisfy~\eqref{eq:def_stability} and~\eqref{eq:def_eps_restrictive}. 
    By~\eqref{eq:regret} and~\eqref{eq:phi_def},
    \begin{align}
        R^\theta_\pi(t)
        &= \sum_{(ij)\in\calL} (v_i^\theta + w_j^\theta - \phi_{ij}^\theta) (t x_{ij}^\theta - \E_\pi^\theta(D_{ij}(t))).
    \end{align}
    By \eqref{eq:slack_cond_phi_eps}, we have $\phi_{ij}^\theta  x_{ij}^\theta = 0$ for all $(ij)\in\calL$.
    Moreover, $\phi_{ij}^\theta = 0$ for $(ij)\in O(\theta)$ and $O(\theta)\cup O^c(\theta)$ is a partition of $\calL$.
    Therefore,
    \begin{align}
        R^\theta_\pi(t)
        &= \sum_{(ij)\in\calL}  (v_i^\theta + w_j^\theta) t x_{ij}^\theta
        - \sum_{(k\ell)\in\calL}  (v_k^\theta + w_\ell^\theta) \E_\pi^\theta(D_{k\ell}(t))
        + \sum_{(mn)\in O^c(\theta)}  \phi_{mn}^\theta \E_\pi^\theta(D_{mn}(t)).
    \end{align}
    Reordering terms gives
    \begin{align}\label{eq:regret_decomp_intermediate}
        R^\theta_\pi(t) &=
        \sum_{(ij)\in\calL}  w_j^\theta (tx_{ij}^\theta - \E_\pi^\theta(D_{ij}(t)))
        + \sum_{(k\ell)\in\calL}  v_k^\theta (tx_{k\ell}^\theta - \E_\pi^\theta(D_{k\ell}(t)))
        + \sum_{(mn)\in O^c(\theta)}  \phi_{mn}^\theta \E_\pi^\theta(D_{mn}(t)).
    \end{align}

    As an intermediate step, note that by definition of the sets $\calS_i$ in~\eqref{eq:def_set_servers_customer_i}  and $\calC_j$ in~\eqref{eq:def_set_customers_server_j}, the following identities hold for any vector $(\gamma_{ij})_{(ij)\in\calL}$,
    \begin{align}\label{eq:summing_identities}
        \sum_{(ij)\in\calL} \gamma_{ij} &= \sum_{i\in\calI} \sum_{j\in\calS_i} \gamma_{ij} = \sum_{j\in\calJ}\sum_{i\in\calC_j} \gamma_{ij}.
    \end{align}
    Applying \eqref{eq:summing_identities} to \eqref{eq:regret_decomp_intermediate} gives
    \begin{align}
        R^\theta_\pi(t) &=
        \sum_{j\in\calJ}  w_j^\theta \Bigl(\sum_{i\in\calC_j} (tx_{ij}^\theta - \E_\pi^\theta(D_{ij}(t))) \Bigr)
        + \sum_{k\in\calI}  v_k^\theta \Bigl(\sum_{\ell\in\calS_k} (tx_{k\ell}^\theta - \E_\pi^\theta(D_{k\ell}(t))) \Bigr)
        + \sum_{(mn)\in O^c(\theta)}  \phi_{mn}^\theta \E_\pi^\theta(D_{mn}(t)).
    \end{align}
    Since $x^\theta$ is the optimal solution of \LP{\theta,\eps}, it satisfies  \eqref{eq:LP_eps_constr_x}, i.e., for all $k\in\calI$, $\sum_{\ell\in\calS_k} x_{k\ell}^\theta = \lambda_k$.
    Moreover, since the complementary slackness condition \eqref{eq:slack_cond_w_eps} are strict, we have for any $j\in\calJ$ that either $w_j^\theta = 0$, or  $\sum_{i\in\calC_j} x_{ij}^\theta = \mu_j -\eps$.
    Hence
    \begin{align}
        R^\theta_\pi(t) &=
        \sum_{j\in\calJ}  w_j^\theta \Bigl(t (\mu_j-\eps) - \sum_{i\in\calC_j}  \E_\pi^\theta(D_{ij}(t))\Bigr)
        + \sum_{k\in\calI}  v_k^\theta \Bigl(t\lambda_k - \sum_{\ell\in\calS_k} \E_\pi^\theta(D_{k\ell}(t))\Bigr)
        + \sum_{(mn)\in O^c(\theta)}  \phi_{mn}^\theta \E_\pi^\theta(D_{mn}(t)).
    \end{align}
    This concludes the proof of \eqref{eq:regret_decomp}.
    \Halmos
\end{proof}

\newpage
\subsection{Proof of Lemma~\ref{lem:suboptgap_nonempty}}
\label{app:suboptgap_nonempty}

\begin{proof}{}
    Let $x^\theta\in\R_{\geq 0}^L$ be the optimal solution of \LP{\theta,\eps} in \eqref{eq:LP_eps}, $v^\theta\in\R^I,w^\theta\in\R_{\geq 0}^J$ be the optimal solution of \dual{\theta,\eps} in \eqref{eq:LP_eps_dual}, and $(ij)\in O^c(\theta)$.
    Let $a_0\in\R_{\geq 0}$ be such that $a_0 > v_i^\theta + w_j^\theta \geq \theta_{ij}$.
    We consider a new program \LP{A(\theta,ij,a_0),\eps} with $A(\theta,ij,a_0)$ as in~\eqref{eq:def_A_vector} and denote its optimal solution by $x^A\in\R_{\geq 0}^L$. 
    Suppose that $(ij)\in O^c(A(\theta,ij,a_0))$, then $x_{ij}^A = x_{ij}^\theta=0$. 
    In this case, we must have $x^A = x^\theta$, since the programs only differ in the reward value for line $(ij)$. 
    By the nondegeneracy assumption, a primal solution implies a unique dual solution, hence $v^\theta,w^\theta$ must be optimal for \dual{A(\theta,ij,a_0)}.
    However, $v^\theta,w^\theta$ violates the feasibility constraint \eqref{eq:LP_eps_dual_constr_phi} since $a_0 > v_i^\theta + w_j^\theta$, leading to a contradiction. 
    This implies $(ij)\in O(A(\theta,ij,a_0))$ and concludes the proof that $\Delta(\theta,ij)$ is nonempty.
    \Halmos
\end{proof}

\newpage
\subsection{Proof of Lemma~\ref{lem:episode_not_mixed}}
\label{app:episodes_not_mixed}
In order to prove Lemma~\ref{lem:episode_not_mixed}, we present an intermediate result in Lemma~\ref{lem:q_hitting_time}.
Let
\begin{align}
    R_j^{mA_m} := \min\{t\geq 0: \ \hat{Q}_j^{mA_m}(t) = 0\}
    \label{eq:hittingtime_def}
\end{align}
denote the hitting time of  state $0$ by the process $\hat{Q}_j^{mA_m}(t)$, which is well-defined by positive recurrence of the process $\hat{Q}_j^a(t)$.
Lemma~\ref{lem:q_hitting_time} provides a tail bound on $R_j^{mA_m}$.

\begin{lemma}
    \label{lem:q_hitting_time}
    For any $n\in\N$ and $b \in\R_{>0}$,
    \begin{align}
        \P( R_j^{mA_m} > n b \bigm| \hat{Q}_j^{mA_m}(0) = n )
        &\leq (\rho_j^{A_m})^{-n/2}\exp\bigl(-nb(\tilde{\lambda}_j^{A_m}+\mu_j-\sqrt{b^{-2}+4\tilde{\lambda}_j^{A_m}\mu_j})\bigr).
     \end{align}
\end{lemma}

\begin{proof}{Proof of Lemma~\ref{lem:q_hitting_time}.}
    As discussed in Section~\ref{sec:step1}, $\hat{Q}_j^{mA_m}(t)$ behaves as an $M/M/1$ queueing system with arrival rate $\tilde{\lambda}_j^{A_m}$ independently from other servers.
    Therefore, from \cite[Section II.2.2]{Cohen1981} we obtain for the moment generating function of $R_j^{mA_m}$, 
    \begin{align}
        \E\bigl(\exp(sR_j^{mA_m}) \bigm| \hat{Q}_j^{mA_m}(0) = n\bigr) &= \Bigl(\frac{\tilde{\lambda}_j^{A_m}+\mu_j-s - \sqrt{(\tilde{\lambda}_j^{A_m}+\mu_j-s)^2 - 4\tilde{\lambda}_j^{A_m}\mu_j}}{2\tilde{\lambda}_j^{A_m}}\Bigr)^n.
    \end{align}
    The radius of convergence is determined by 
    \begin{align}
        (\tilde{\lambda}_j^{A_m}+\mu_j-s)^2 - 4\tilde{\lambda}_j^{A_m}\mu_j &>0
        \ \iff \
        s < \tilde{\lambda}_j^{A_m}+\mu_j-\sqrt{4\tilde{\lambda}_j^{A_m}\mu_j} \ =: \ \delta.
    \end{align}

    Let $b \in\R_{>0}$. By Markov's inequality,  we have
    \begin{align}\label{eq:app_tail_markov}
        \P( R_j^{mA_m} > n b \bigm| \hat{Q}_j^{mA_m}(0) = n ) &\leq \inf_{0 < s < \delta}\E(\exp(sR_j^a) \bigm| \hat{Q}_j^{mA_m}(0) = n)\exp(-nbs).
    \end{align}
    It can be verified that the infimum within the radius of convergence is attained at 
    \begin{align}
        s^* &= \tilde{\lambda}_j^{A_m} + \mu_j - \sqrt{b^{-2}+4\tilde{\lambda}_j^{A_m}\mu_j} \ < \ \delta.
    \end{align}
    Substitution of $s^*$ into \eqref{eq:app_tail_markov} gives
    \begin{align}\label{eq:app_tail_markov2}
        \P( R_j^{mA_m} > n b \bigm| \hat{Q}_j^{mA_m}(0) = n) &\leq \exp\bigl(-nb(\tilde{\lambda}_j^{A_m}+\mu_j-\sqrt{b^{-2}+4\tilde{\lambda}_j^{A_m}\mu_j})\bigr) \Bigl(\frac{-b^{-1} + \sqrt{b^{-2} + 4\tilde{\lambda}_j^{A_m}\mu_j}}{2\tilde{\lambda}_j^{A_m}}\Bigr)^n.
     \end{align}
     We note that the term
     \begin{align}
        \frac{-b^{-1} + \sqrt{b^{-2} + 4\tilde{\lambda}_j^{A_m}\mu_j}}{2\tilde{\lambda}_j^{A_m}}
     \end{align}
     is increasing in $b$, and that
     \begin{align}
        \lim_{b\to\infty} \frac{-b^{-1} + \sqrt{b^{-2} + 4\tilde{\lambda}_j^{A_m}\mu_j}}{2\tilde{\lambda}_j^{A_m}} &= \frac{\sqrt{4\tilde{\lambda}_j^{A_m}\mu_j}}{2\tilde{\lambda}_j^{A_m}} = (\rho_j^{A_m})^{-1/2}.
     \end{align}
     Therefore, we can bound \eqref{eq:app_tail_markov2} further and complete the proof, i.e., 
     \begin{align}
        \P( R_j^{mA_m} > n b \bigm| \hat{Q}_j^{mA_m}(0) = n)
        &\leq (\rho_j^{A_m})^{-n/2}\exp\bigl(-nb(\tilde{\lambda}_j^{A_m}+\mu_j-\sqrt{b^{-2}+4\tilde{\lambda}_j^{A_m}\mu_j})\bigr).
     \end{align}
     \Halmos
\end{proof}

\noindent We next present the proof of Lemma~\ref{lem:episode_not_mixed}.
\begin{proof}{Proof of Lemma~\ref{lem:episode_not_mixed}.}
    We define
    \begin{align}
        s^{mA_m} &:= \frac{-\ln^\beta(2Jm)}{\ln(\rho_j^{A_m})} > 0.
        \label{eq:sm_def}
    \end{align}
    Let $t \geq \tau_m$ and $j\in\calJ$. Note that by construction of the coupling
    \begin{align}
        \P( \underline{Q}_j^{mA_m}(t) = \hat{Q}_j^{mA_m}(t))
        &\geq
        \P( \underline{Q}_j^{mA_m}(\tau_m) = \hat{Q}_j^{mA_m}(\tau_m)).
        \label{eq:q_coupling_bound_tau}
    \end{align}
    By the law of total probability, we have
    \begin{align}
        \P( \underline{Q}_j^{mA_m}(\tau_m) = \hat{Q}_j^{mA_m}(\tau_m))
        &\geq
        \P(\underline{Q}_j^{mA_m}(\tau_m) = \hat{Q}_j^{mA_m}(\tau_m) \bigm| \hat{Q}_j^{mA_m}(0) \leq \lfloor s^{mA_m} \rfloor) \P(\hat{Q}_j^{mA_m}(0) \leq \lfloor s^{mA_m} \rfloor).
        \label{eq:q_coupling_bound_proof}
    \end{align}
    Since the initial queue length $\hat{Q}_j^{mA_m}(0)$ is sampled from the stationary measure  $p_j^{A_m}$ in \eqref{eq:queue_statdistr}, we have
    \begin{align}
        \P(\hat{Q}_j^{mA_m}(0) \leq \lfloor s^{mA_m} \rfloor) &= \sum_{i=0}^{\lfloor s^{mA_m} \rfloor} (1-\rho_j^{A_m})(\rho_j^{A_m})^i
        = 1-(\rho_j^{A_m})^{1+\lfloor s^{mA_m} \rfloor}.
    \end{align}
    We use that $b^{\ln(a)/\ln(b)} =a$ for any $a,b >0$ and~\eqref{eq:sm_def} to obtain
    \begin{align}
        (\rho_j^{A_m})^{\lfloor s^{mA_m} \rfloor}
        &\geq 
        (\rho_j^{A_m})^{s^{mA_m}}
        =
        (\rho_j^{A_m})^{-\ln^\beta(2Jm)/\ln(\rho_j^{A_m})}
        =
        (2Jm)^{-\ln^{\beta-1}(2Jm)},
    \end{align}
    and so
    \begin{align}
        \P(\hat{Q}_j^{mA_m}(0) \leq \lfloor s^{mA_m} \rfloor)
        &\leq 1 - \rho_j^{A_m} (2Jm)^{-\ln^{\beta-1}(2Jm)}
        \stackrel{\textrm{(i)}}{\geq} 1 - (2Jm)^{-\ln^{\beta-1}(2Jm)}
        \stackrel{\textrm{(ii)}}{\geq} 1 - (2Jm)^{-\beta}
        \stackrel{\textrm{(iii)}}{\geq} 1 - \frac{1}{2Jm^\beta},
        \label{eq:hittingtime_ql_bound}
    \end{align}
    where (i) follows from $\rho_j^{A_m} < 1 $ and $2Jm > 1$, (ii) by the assumption $m\geq C_\beta$ and $2Jm\geq 1$, and (iii) by $J\geq 1$, $m \geq 1$, and $\beta > 1$.

    We continue to bound the conditional probability in \eqref{eq:q_coupling_bound_proof}.
    Observe that if the process $\hat{Q}_j^{mA_m}(t)$ hits state 0 at time $x$, then  $\hat{Q}_j^{mA_m}(t)=\underline{Q}_j^{mA_m}(t)$ for all $t\geq x$, as illustrated in Figure~\ref{fig:coupling}.
    Therefore, 
    \begin{align}
        \label{eq:hittingtime_dep_ql}
        \P\bigl(\underline{Q}_j^{mA_m}(\tau_m) = \hat{Q}_j^{mA_m}(\tau_m) \bigm| \hat{Q}_j^{mA_m}(0) \leq \lfloor s^{mA_m} \rfloor \bigr)
        = \P\bigl( R_j^{mA_m} \leq \tau_m \bigm| \hat{Q}_j^{mA_m}(0) \leq \lfloor s^{mA_m} \rfloor\bigr).
    \end{align}
    As an intermediate step, we note that for any $A$ and any pairwise disjoint $B_1,\dots,B_n$,
    \begin{align}
        \P\Bigl(A \biggm| \dot{\bigcup}_{k=1}^n B_k\Bigr) &= \frac{\sum_{k=1}^n \P(A\cap B_k)}{\sum_{j=1}^n \P(B_j)} 
        = \sum_{k=1}^n \frac{\P(B_k)}{\sum_{j=1}^n \P(B_j)} \P(A\mid B_k) 
        \geq \min_{k\in\{1,\dots,n\}} \P(A\mid B_k),
    \end{align}
    where the inequality follows since the weighted average of a set of nonnegative numbers is at least as large as the minimum of that set.
    Therefore, since $\{\hat{Q}_j^{mA_m}(0) \leq \lfloor s^{mA_m} \rfloor\} = \dot{\bigcup}_{k=0}^{\lfloor s^{mA_m}\rfloor}\{\hat{Q}_j^{mA_m}(0) = k\}$, we have
    \begin{align}\label{eq:hittingtime_minimum}
        \P\bigl( R_j^{mA_m}\leq \tau_m \bigm| \hat{Q}_j^{mA_m}(0) \leq \lfloor s^{mA_m} \rfloor\bigr)
        &\geq \min_{k\in\{0,\dots,\lfloor s^{mA_m} \rfloor\}} \P\bigl(R_j^{mA_m}\leq \tau_m \bigm| \hat{Q}_j^{mA_m}(0) = k\bigr).
    \end{align}
    Moreover, since $\hat{Q}_j^{mA_m}(t)$ is a birth--death process, we have
    \begin{align}
        \P\bigl( R_j^{mA_m}\leq \tau_m \bigm| \hat{Q}_j^{mA_m}(0) = k_1 \bigr)
        &\geq \P\bigl( R_j^{mA_m}\leq \tau_m \bigm| \hat{Q}_j^{mA_m}(0) = k_2 \bigr), \ \forall k_1 \leq k_2 \in \N,
    \end{align}
    which implies that the minimum in \eqref{eq:hittingtime_minimum} is attained at $\lfloor s^{mA_m} \rfloor$, therefore
    \begin{align}
        \P\bigl( R_j^{mA_m}\leq \tau_m \bigm| \hat{Q}_j^{mA_m}(0) \leq \lfloor s^{mA_m} \rfloor\bigr)
        &\geq \P\bigl(R_j^{mA_m}\leq \tau_m \mid \hat{Q}_j^{mA_m}(0) = \lfloor s^{mA_m} \rfloor\bigr).
    \end{align}
    We obtain from Lemma~\ref{lem:q_hitting_time} that
    \begin{align}
        \P\bigl(R_j^{mA_m} > \tau_m \bigm| \hat{Q}_j^{mA_m}(0) = \lfloor s^{mA_m} \rfloor \bigr)
        &\leq  (\rho_j^{A_m})^{-\lfloor s^{mA_m} \rfloor /2} \exp\Bigl(-\tau_m\Bigl(\tilde{\lambda}_j^{A_m}+\mu_j-\sqrt{\Bigl(\frac{\tau_m}{\lfloor s^{mA_m} \rfloor}\Bigr)^{-2} + 4\tilde{\lambda}_j^{A_m}\mu_j}\Bigr)\Bigr)\\
        &\leq  (\rho_j^{A_m})^{-s^{mA_m}/2} \exp\Bigl(-\tau_m\Bigl(\tilde{\lambda}_j^{A_m}+\mu_j-\sqrt{\Bigl(\frac{\tau_m}{s^{mA_m}}\Bigr)^{-2} + 4\tilde{\lambda}_j^{A_m}\mu_j}\Bigr)\Bigr).
    \end{align}
    Now, by substitution of $\tau_m$ and $s^{mA_m}$ as defined in~\eqref{eq:ucbqr_episode_length} and~\eqref{eq:sm_def}, respectively, and noting that $\tau_m/s^{mA_m} = -\alpha\ln(\rho_j^{A_m})$, we have
    \begin{align}
        \P\bigl(R_j^{mA_m} > \tau_m \bigm| \hat{Q}_j^{mA_m}(0) = \lfloor s^{mA_m} \rfloor\bigr) &\leq
        (\rho_j^{A_m})^{\ln^\beta(2Jm)/(2\ln(\rho_j^{A_m}))}  \nonumber \\
        & \qquad \cdot \exp\Bigl(-\alpha \ln^\beta(2Jm)\Bigl(\tilde{\lambda}_j^{A_m}+\mu_j-\sqrt{\Bigl(\alpha\ln(\rho_j^{A_m})\Bigr)^{-2} + 4\tilde{\lambda}_j^{A_m}\mu_j}\Bigr)\Bigr) \\
        &=  (2Jm)^{-\ln^{\beta-1}(2Jm)\bigl(-1/2+\alpha (\tilde{\lambda}_j^{A_m}+\mu_j-\sqrt{(\alpha\ln(\rho_j^{A_m}))^{-2} + 4\tilde{\lambda}_j^{A_m}\mu_j})\bigr)}.
    \end{align}
    It can be verified that $f(\gamma):=-\frac 12 +\gamma \bigl(\tilde{\lambda}_j^{A_m}+\mu_j-\sqrt{(\gamma\ln(\rho_j^{A_m}))^{-2} + 4\tilde{\lambda}_j^{A_m}\mu_j}\bigr)$ is strictly increasing on $[0,\infty)$ and that
    \begin{align}
        \gamma^* = \frac{3(\tilde{\lambda}_j^{A_m}+\mu_j) \ln(\rho_j^{A_m}) - 2\sqrt{(\mu_j-\tilde{\lambda}_j^{A_m})^2 + 9\tilde{\lambda}_j^{A_m}\mu_j \ln^2(\rho_j^{A_m})}}{2(\mu_j-\tilde{\lambda}_j^{A_m})^2 \ln(\rho_j^{A_m})} > 0
    \end{align} 
    satisfies $f(\gamma^*) =1$. 
    By construction of $\alpha$ in~\eqref{eq:c0_choice}, we have $\alpha\geq\gamma^*$, and hence 
    \begin{align}
        \P\bigl(R_j^{mA_m} > \tau_m \bigm| \hat{Q}_j^{mA_m}(0) = \lfloor s^{mA_m} \rfloor\bigr) &\leq (2Jm)^{-\ln^{\beta-1}(2Jm)}.
    \end{align}
    Since $m\geq C_\beta$ and hence $\log^{\beta -1}(2Jm) \geq \beta$, we obtain
    \begin{align}
        \label{eq:q_coupling_subs}
        \P\bigl(R_j^{mA_m} > \tau_m \bigm| \hat{Q}_j^{mA_m}(0) = \lfloor s^{mA_m} \rfloor\bigr) &\leq (2Jm)^{-\beta}.
    \end{align}
    Substitution of \eqref{eq:hittingtime_ql_bound} and \eqref{eq:hittingtime_dep_ql} into \eqref{eq:q_coupling_bound_proof} and subsequently using \eqref{eq:q_coupling_bound_tau} and \eqref{eq:q_coupling_subs} and noting that $\beta >1$ gives
    \begin{align}
        \P(\underline{Q}_j^{mA_m}(t) = \hat{Q}_j^{mA_m}(t))
        &\geq\bigl(1-\frac{1}{2Jm^\beta}\bigr)^2
        \geq 1-\frac{1}{Jm^\beta}.
        \label{eq:dep_bnd_single_server}
    \end{align}
    To conclude the proof, we note that by independence of the customer arrivals,
    \begin{align}
        1 - \P(\underline{Q}_j^{mA_m}(t) = \hat{Q}_j^{mA_m}(t), \ \forall j\in\calJ)
        &= \P(\exists j\in\calJ: \ \underline{Q}_j^{mA_m}(t) \neq \hat{Q}_j^{mA_m}(t))
        \stackrel{\textrm{\eqref{eq:dep_bnd_single_server}}}{\leq} \sum_{j=1}^J \frac{1}{Jm^\beta}
        = \frac{1}{m^\beta}.
    \end{align}
    \Halmos
\end{proof}

\newpage
\subsection{Proof of Lemma~\ref{lem:prob_insuf_samples}}
\label{app:prob_insuf_samples}
In order to prove Lemma~\ref{lem:prob_insuf_samples}, we provide two intermediate results.

\begin{lemma}\label{lem:binomial_bound}
    Let $\gamma \in\R_{\geq 1}$, $u\in\N_{\geq 1}$, and $X_1,\dots,X_u$ be independent random variables with $X_m \sim\bern{1-1/m^\gamma}$ for all $m=1,\dots,u$. Then for any $\eta\in\R_{\geq 1}$ such that $\lceil \sum_{m=1}^u 1/m^\gamma\rceil \leq \eta \leq u$, we have
    \begin{align}
        \P\Bigl( \sum_{m=1}^u X_m < u - \eta \Bigr)
        &\leq \frac{u^{1+\eta}}{\eta^{(\gamma+1)\eta}} \exp((\gamma+1)\eta).
        \label{eq:binomial_bound_ineq}
    \end{align}
\end{lemma}

\begin{proof}{Proof of Lemma~\ref{lem:binomial_bound}.}
    Let $Y_m := 1-X_m$ for $m=1,\dots,u$, $\mu := \sum_{m=1}^u 1/m^\gamma$, and $\mu \leq \eta \leq u$.
    Then, $Y_m \sim \bern{1/m^\gamma}$ and
    \begin{align}
        \P\Bigl( \sum_{m=1}^u X_m < u - \eta \Bigr) &= \P\Bigl( \sum_{m=1}^u Y_m > \eta \Bigr).
        \label{eq:binomial_X_Y}
    \end{align}
    We note that $\sum_{m=1}^u Y_m$ follows a Poisson binomial distribution.
    Let $a_k := \P( \sum_{m=1}^u Y_m = k )$ and $m^* := \arg \max_{k \in [u]} a_k$ refer to the probability mass function and mode of the Poisson binomial distribution, respectively.
    It is known that the Poisson binomial distribution is strong Rayleigh \cite[Corollary~4.2]{Tang2023}.
    As a consequence, $\{ a_k \}_{k=1}^u$ is ultra-logconcave, and this in turn implies that the sequence is unimodal (see the discussion following \cite[Corollary~4.2]{Tang2023})
    This proves that if $m^* \leq k \leq l$, then $a_l \leq a_k$.
    What remains is to relate the mode $m^*$ to the mean $\mu$.
    By \cite[Theorem~4]{Darroch1964}, we have $m^* \in \{\lfloor \mu \rfloor, \lceil \mu \rceil \}$.
    Consequently, if $\lceil \mu \rceil \leq k \leq l$, then $a_l \leq a_k$. 

    Therefore, since $\lceil\mu\rceil \leq \eta$, we obtain 
    \begin{align}
        \P\Bigl( \sum_{m=1}^u Y_m > \eta \Bigr) &\leq \sum_{k=\eta}^u \P\Bigl( \sum_{m=1}^u Y_m = k \Bigr)
        \leq  u  \P\Bigl( \sum_{m=1}^u Y_m = \eta \Bigr).
    \end{align}
    Since the success probability of $Y_m$ is decreasing in $m$, the most likely way to get $\sum_{m=1}^u Y_m = \eta $ is for the first $\eta$ indices to be 1.
    Moreover, the number of combinations of cardinality $\eta$ from the set $\{1,\dots,u\}$ is $\binom{u}{\eta}$, and therefore 
    \begin{align}
        \P\Bigl( \sum_{m=1}^u Y_m > \eta \Bigr) &\leq u \binom{u}{\eta} \prod_{m=1}^\eta \P(Y_m = 1)\\
        &=   u \binom{u}{\eta} \prod_{m=1}^{\eta} \frac{1}{m^{\gamma}} \\
        &=   u \binom{u}{\eta} \frac{1}{(\eta!)^{\gamma}}.
    \end{align}
    By Stirling's formula \citep{Robbins1955}, 
    \begin{align}
        \eta! &\geq \sqrt{2\pi \eta} \ \eta^\eta\exp(-\eta)\geq \eta^\eta\exp(-\eta).
    \end{align}
    Moreover, since $\exp(\eta) \geq \eta^\eta/\eta!$, we have 
    \begin{align}
        \binom{u}{\eta} &\leq \frac{u^\eta}{\eta!} \leq \Bigl(\frac{u}{\eta}\Bigr)^{\eta} \exp(\eta).
    \end{align}
    Hence
    \begin{align}
        \P\Bigl( \sum_{m=1}^u Y_m > \eta \Bigr)
        &\leq u \Bigl(\frac{u}{\eta}\Bigr)^{\eta} \exp(\eta) \eta^{-\gamma\eta} \exp(\gamma\eta)
        = \frac{u^{1+\eta}}{\eta^{(\gamma+1)\eta}}  \exp((\gamma+1)\eta).
        \label{eq:binomial_lem_conclusion}
    \end{align}
    The result follows by substituting~\eqref{eq:binomial_lem_conclusion} into~\eqref{eq:binomial_X_Y}.
    \Halmos
\end{proof}

\begin{lemma}
    \label{lem:pois_tail}
    Let $\gamma > 0$ and $X\sim\poi{\gamma}$. For any $b\in (0,\gamma)$, 
    $
        \P(X - \gamma \leq - b) \leq \exp\Bigl(- \frac{b^2}{2\gamma}\Bigr).
    $
\end{lemma}

\begin{proof}{Proof of Lemma~\ref{lem:pois_tail}.}
    Let $\gamma>0$ and $0 < b < \gamma$.
    We have for any $t < 0$,
    \begin{align}
        \P(X \leq \gamma - b) &= \P\bigl(\ex^{t X} \geq \ex^{t(\gamma-b)}\bigr).
    \end{align}
    By applying Markov's inequality, we obtain
    \begin{align}
        \P(X \leq \gamma - b)
        &\leq
        \inf_{t<0} \E\bigl(\ex^{tX} \bigr) \ex^{-t(\gamma - b)}.
    \end{align}
    Since the moment generating function of $X$ is $\E(\ex^{tX}) = \ex^{\gamma \ex^t - 1}$ for $t\in\R$, we have
    \begin{align}
        \P(X \leq \gamma - b ) &\leq
        \inf_{t <0} \ex^{\gamma(\ex^t-1)} \ex^{-t(\gamma - b)}.
    \end{align}
    The infimum is attained for $t = \ln((\gamma-b)/\gamma) < 0$, hence we find after substituting that
    \begin{align}
        \P(X \leq \gamma - b) &\leq \exp\Bigl(- b - (\gamma - b)\ln\Bigl(\frac{\gamma-b}{\gamma}\Bigr)\Bigr) \\
        &= \exp\Bigl(-\frac{b^2}{2\gamma}\Bigl(\frac{2\gamma}{b} + \Bigl(\frac{2\gamma(\gamma-b)}{b^2}\Bigr)\ln\Bigl(\frac{\gamma - b}{\gamma}\Bigr)\Bigr)\Bigr).
    \end{align}
    Let $x = b/\gamma$. From the assumptions, it follows that $0 < x < 1$. 
    Now, it can be verified from a Taylor expansion that
    \begin{align}
        h(x) = \frac{2}{x} + \frac 2x \Bigl(\frac 1x -1\Bigr) \ln(1-x)
    \end{align}
    satisfies $h(x) \geq 1$ for any $0 < x < 1$. Therefore
    $
        \P(X \leq \gamma - b)
        \leq \exp\Bigl(-\frac{b^2}{2\gamma}\Bigr).
    $
    \Halmos
\end{proof}

\noindent We next present the proof of Lemma~\ref{lem:prob_insuf_samples}.
\begin{proof}{Proof of Lemma~\ref{lem:prob_insuf_samples}.}
    For $(ij)\in\calL^a$ and $m\in\N$,
    consider the processes $\hat{Q}_j^{ma}$ and $\underline{Q}_j^{ma}$ as defined in Section~\ref{sec:step1}.
    Let $\hat{D}_{ij}^{ma}$ denote the number of type-$(ij)$ departures within episode $m$ \emph{but after warm up time $\tau_m$} of the process $\hat{Q}_j^{ma}$. 
    We introduce the events 
    \begin{align}
        \Omega_{ij}^{ma} := \{D_{ij}^{ma} \geq \hat{D}_{ij}^{ma}\},
        \quad
        \underline{\Omega}_{ij}^{ma} := \{\hat{Q}_{ij}^{mA_m}(\tau_m) = \underline{Q}_{ij}^{mA_m}(\tau_m)\}.
        \label{eq:def_omega}
    \end{align}
    We note that $\underline{\Omega}_{ij}^{ma}$ implies $\Omega_{ij}^{ma}$, hence
    \begin{align}
        \P(\underline{\Omega}_{ij}^{ma})\leq \P(\Omega_{ij}^{ma}),   
        \label{eq:omega_underline_lb}
    \end{align}
    and so for any $k\geq 1$, 
    \begin{align}
        \P\Bigl(\sum_{m=1}^k \ind{\underline{\Omega}_{ij}^{ma}} \leq \sum_{m=1}^k \ind{\Omega_{ij}^{ma}}\Bigr) = 1.
        \label{eq:bnd_omega_underline}
    \end{align}

    Suppose $k > \lceil C_\beta\rceil$.
    It follows from \eqref{eq:def_overline_T} and \eqref{eq:def_event_E_lem} that
    \begin{align}
        E_k^a &= \Bigl\{\sum_{m=1}^{\lceil C_\beta + u_k\rceil} D_{ij}^{ma} \geq \lceil \sigma_{ijk}^a\rceil, \ \forall (ij)\in\calL^a \Bigr\}.
        \label{eq:event_Eka_departures}
    \end{align}
    By the law of total probability,
    \begin{align}
        \P((E_k^a)^c)
        &\leq
        \P\Bigl( \sum_{m=1}^{\lceil C_\beta + u_k\rceil} \ind{\Omega_{ij}^{ma}} < \lceil u_k-\eta_k\rceil \Bigr)
        +
        \P\Bigl((E_k^a)^c, \ \sum_{m=1}^{\lceil C_\beta + u_k\rceil} \ind{\Omega_{ij}^{ma}} \geq \lceil u_k-\eta_k\rceil\Bigr).
        \label{eq:prob_Eija_split}
    \end{align}
    We bound these terms individually.

    It follows from \eqref{eq:bnd_omega_underline} that
    \begin{align}
        \P\Bigl( \sum_{m=1}^{\lceil C_\beta + u_k\rceil} \ind{\Omega_{ij}^{ma}} < \lceil u_k-\eta_k\rceil \Bigr)
        &= \P\Bigl( \sum_{m=1}^{\lceil C_\beta + u_k\rceil} \ind{\Omega_{ij}^{ma}} < \lceil u_k-\eta_k\rceil, \ \sum_{m=1}^k \ind{\underline{\Omega}_{ij}^{ma}} \leq \sum_{m=1}^k \ind{\Omega_{ij}^{ma}}\Bigr)\\
        &\leq
        \P\Bigl( \sum_{m=1}^{\lceil C_\beta + u_k\rceil} \ind{\underline{\Omega}_{ij}^{ma}} < \lceil u_k-\eta_k\rceil \Bigr).
    \end{align}
    We observe that event $\underline{Q}_j^{ma}(\tau_m) = \hat{Q}_j^{ma}(\tau_m)$ implies that there exists an $0\leq s\leq \tau_m$ such that $\hat{Q}_j^{ma}(s) = \underline{Q}_j^{ma}(s) = 0$.
    In this case, we have $\hat{Q}_{ij}^{ma}(s) \leq \hat{Q}_j^{ma}(s)=0 = \underline{Q}_{ij}^{ma}(s)$ by construction, and so $\ind{\underline{\Omega}_{ij}^m} = 1$. 
    Thus, by invoking Lemma~\ref{lem:episode_not_mixed} we find that
    \begin{align}
        \P(\underline{\Omega}_{ij}^{ma})
        &\geq
        \P(\underline{\Omega}_{j}^{ma})
        \geq 
        1 - \frac{1}{m^\beta}.
        \label{eq:P_omega_bnd}
    \end{align}
    Moreover, the events $\underline{\Omega}_{ij}^{ma}$ are mutually independent over all episodes, since the initial values $\hat{Q}_j^{ma}(0)$ are sampled from $p_j^a$ in~\eqref{eq:queue_statdistr} and $\underline{Q}_j^{ma}(0) = 0$.
    Lastly, it can be verified that $\eta_k \geq \lceil\sum_{m=1}^{\lceil u_k \rceil } 1/m^\beta\rceil$ for $k$ sufficiently large. 
    Hence, it follows from \eqref{eq:P_omega_bnd} and Lemma~\ref{lem:binomial_bound} that for $k$ sufficiently large,
    \begin{align}
        \P\Bigl( \sum_{m=1}^{\lceil C_\beta + u_k\rceil} \ind{\underline{\Omega}_{ij}^{ma}} < \lceil u_k-\eta_k\rceil \Bigr)
        &\leq
        \P\Bigl( \sum_{m=\lceil C_\beta\rceil}^{\lceil C_\beta + u_k\rceil} \ind{\underline{\Omega}_{ij}^{ma}} < \lceil u_k-\eta_k\rceil \Bigr)
        \leq
        \frac{u_k^{1+\eta_k}}{\eta_k^{(\beta+1)\eta_k}} \exp((\beta+1)\eta_k).
        \label{eq:n_episodes_notmixed}
    \end{align}
    By using $a=\exp(\ln(a))$, we have
    \begin{align}
        u_k^{1+\eta_k} &=\exp\bigl((1+\eta_k)\ln(u_k)\bigr) \stackrel{\textrm{\eqref{eq:def_seq_u_eta_sig_lem}}}{=} \exp\bigl( (1+\eta_k)\beta\ln\ln(k)\bigr),
        \\
        \eta_k^{(\beta+1)\eta_k} &=\exp\bigl((\beta+1)\eta_k\ln(\eta_k)\bigr) \stackrel{\textrm{\eqref{eq:def_seq_u_eta_sig_lem}}}{=}  \exp\Bigl( \frac{(\beta+1)^2}{2}\eta_k\ln\ln(k)\Bigr),
    \end{align}
    and so
    \begin{align}
        \frac{u_k^{1+\eta_k}}{\eta_k^{(\beta+1)\eta_k}} &= \exp\Bigl((\beta - \frac{\beta^2+1}{2} \eta_k) \ln\ln(k)
        \Bigr).
        \label{eq:frac_u_eta_expand}
    \end{align}
    Substitution of \eqref{eq:frac_u_eta_expand} into \eqref{eq:n_episodes_notmixed} gives
    \begin{align}
        \P\Bigl( \sum_{m=\lceil C_\beta\rceil}^{\lceil C_\beta + u_k\rceil} \ind{\underline{\Omega}_{ij}^{ma}} < \lceil u_k-\eta_k\rceil \Bigr)
        &\leq \exp\Bigl(\beta \ln\ln(k) + \bigl(\beta +1 - \frac{\beta^2+1}{2} \ln\ln(k)\bigr)\eta_k  \Bigr)
        \\
        &\leq
        \exp\Bigl( \bigl[2\beta+1 - \frac{\beta^2+1}{2}\ln\ln(k)\bigr] \eta_k\Bigr).
        \label{eq:prob_Eija_part1}
    \end{align}
    where the last inequality follows from $\ln\ln(k) \leq \eta_k$ and $\beta > 1$.

    For the rightmost  probability in \eqref{eq:prob_Eija_split} we have by Boole's inequality that
    \begin{align}
        &
        \P
        \Bigl(
            (E_k^a)^c, \ \sum_{m=1}^{\lceil C_\beta + u_k\rceil} \ind{\Omega_{ij}^{ma}} \geq \lceil u_k-\eta_k\rceil
        \Bigr)
        \nonumber \\ 
        &\ \leq
        \sum_{(ij)\in\calL^a}
        \P
        \Bigl(
            \sum_{m=1}^{\lceil C_\beta + u_k\rceil} D_{ij}^{ma} < \lceil \sigma_{ijk}^a\rceil, \ \sum_{m=1}^{\lceil C_\beta + u_k\rceil} \ind{\Omega_{ij}^{ma}} \geq \lceil u_k-\eta_k\rceil
        \Bigr)
        \nonumber \\ 
        &\stackrel{\textrm{\eqref{eq:def_omega}}}{\leq}
        \sum_{(ij)\in\calL^a}
        \P
        \Bigl(
            \sum_{m=1}^{\lceil u_k - \eta_k \rceil} \hat{D}_{ij}^{ma} < \lceil \sigma_{ijk}^a\rceil
        \Bigr)
        .
    \end{align}
    Recall from~\eqref{eq:ucbqr_episode_length} that the length of an episode without the warmup period is $H_0$.
    From Burke's Theorem~\cite[II.2.4 Theorem 2.1]{Cohen1981} and Poisson split and merge properties \citep{Cinlar1968}, we have that
    $\hat{D}_{ij}^{ma}\stackrel{\textrm{i.i.d.}}{\sim}\poi{x_{ij}^a H_0}$,
    and hence
    $\sum_{m=1}^k \hat{D}_{ij}^{ma} \sim\poi{x_{ij}^a H_0 k}$.
    Moreover, $u_k -\eta_k > \eta_k$ for $k$ sufficiently large. 
    Therefore, we obtain from Lemma~\ref{lem:pois_tail} and \eqref{eq:def_seq_u_eta_sig_lem} that for $k$ sufficiently large, 
    \begin{align}
        \P\Bigl(\sum_{m=1}^{\lceil u_k - \eta_k \rceil} \hat{D}_{ij}^{ma} < \lceil \sigma_{ijk}^a\rceil \Bigr)
        &\leq   \exp\Bigl(\frac{-(x_{ij}^a H_0 \lceil u_k-\eta_k\rceil -\lceil \sigma_{ijk}^a\rceil )^2 }{2x_{ij}^aH_0\lceil u_k-\eta_k\rceil}\Bigr) 
        \leq  \exp\Bigl(\frac{-x_{ij}^aH_0\eta_k^2}{2u_k}\Bigr),
    \end{align}
    where we used $x_{ij}^a H_0 \lceil u_k-\eta_k\rceil -\lceil \sigma_{ijk}^a \rceil \geq x_{ij}^a H_0 \eta_k$ in the last equality. 
    We use \eqref{eq:ucbqr_episode_length}, the bounds $|\calL^a| \leq L$, $x_{ij}^a \geq \xmin$, together with the definitions of $\eta_k$ and $u_k$ in~\eqref{eq:def_seq_u_eta_sig_lem} to write
    \begin{align}
        \sum_{(ij)\in\calL^a} \P\Bigl(\sum_{m=1}^{\lceil u_k - \eta_k \rceil} \hat{D}_{ij}^{ma} < \lceil \sigma_{ijk}^a\rceil \Bigr)
        &
        \leq
        \sum_{(ij)\in\calL^a} \exp\Bigl(\frac{-x_{ij}^aH_0\eta_k^2}{2u_k}\Bigr)
        \leq
        L \exp\Bigl(- \frac 12  \xmin H_0  \ln(k)\Bigr)
         \\ &
        = L k^{-\xmin H_0/2}
        \leq
        L k^{-2}
        .
        \label{eq:prob_Eija_part2}
    \end{align}

    Finally, we obtain from \eqref{eq:prob_Eija_split}, \eqref{eq:prob_Eija_part1}, and \eqref{eq:prob_Eija_part2},
    \begin{align}
        \lim_{k\to\infty} k \P((E_k^a)^c)
        \leq
        \lim_{k\to\infty}
        \Bigl(
            L k^{-1}
            +
            k\exp
            {
                \Bigl(
                    \bigl[
                        2\beta+1 - \frac{\beta^2+1}{2}\ln\ln(k)
                    \bigr]
                    \eta_k
                \Bigr)
            }
        \Bigr)
        = 0.
        \label{eq:lim_Ekac_to_0}
    \end{align}
    Here, the last equality follows since
    \begin{align}
        k\exp
        {
            \Bigl(
                \bigl[
                    2\beta+1 - \frac{\beta^2+1}{2}\ln\ln(k)
                \bigr]
                \eta_k
            \Bigr)
        } \propto 
        k^{1-\log^{(\beta-1)/2}(k)\ln\ln(k)},
    \end{align}
    and $\beta > 1$. 
    \Halmos
\end{proof}

\newpage
\subsection{Proof of Lemma~\ref{lem:bnd_overestimation_suf_samples}.}
\label{app:bnd_overestimation_suf_samples}
In order to prove Lemma~\ref{lem:bnd_overestimation_suf_samples}, we provide an intermediate result in Lemma~\ref{lem:integral}.

\begin{lemma}
    \label{lem:integral}
        Let $(\sigma_k)_{k\in\R_{>0}} \in\R$ satisfy $\sigma_k = \omega(\ln(k))$
        and let $\gamma \in\R_{> 0}$.
        Then
        \begin{align}
            \lim_{k\to\infty} k \int_{\sigma_{k}}^\infty \exp\Bigl( -2m \Bigl(\gamma - \sqrt{\frac{\ln(k)}{m}}\Bigr)^2\Bigr) \mathrm{d}m &= 0.
        \end{align}
    \end{lemma}
    \begin{proof}{Proof of Lemma~\ref{lem:integral}.}
        Substitution of $\tau = \gamma \sqrt{m} - \sqrt{\ln(k)}$
        with $\textrm{d}m/\textrm{d}\tau = 2(\tau+\sqrt{\ln(k)})/\gamma^2$
        gives
        \begin{align}
            \int_{\sigma_{k}}^\infty \exp\Bigl( -2m \Bigl(\gamma - \sqrt{\frac{\ln(k)}{m}}\Bigr)^2\Bigr) \textrm{d}m
            &= \int_{\gamma \sqrt{\sigma_k}-\sqrt{\ln(k)}}^\infty \exp\bigl( -2\tau^2 \bigr) \frac{2(\tau+\sqrt{\ln(k)})}{\gamma^2}\textrm{d}\tau.
        \end{align}
        It can be verified that
        \begin{align}
            &\int_{\gamma \sqrt{\sigma_k}-\sqrt{\ln(k)}}^\infty \exp\bigl( -2\tau^2 \bigr) \frac{2(\tau+\sqrt{\ln(k)})}{\gamma^2}\textrm{d}\tau \\
            &\qquad = \frac{\exp\bigl(-2(\gamma\sqrt{\sigma_k}-\sqrt{\ln(k)})^2\bigr)  + \sqrt{2\ln(k)\pi}\textrm{erfc}\bigl(\sqrt{2}(\gamma\sqrt{\sigma_k}-\sqrt{\ln(k)})\bigr)}{2\gamma^2},
            \label{eq:proof_integral_subs}
        \end{align}
        where $\textrm{erfc}(x) := 1-2/\sqrt{\pi}\int_0^x \exp(-t^2)\textrm{d}t$ is the complementary error function.
        Now since $\sigma_k = \omega(\ln(k))$, we have
        \begin{align}
            \exp\bigl(-2(\gamma\sqrt{\sigma_k}-\sqrt{\ln(k)})^2\bigr) &= O\bigl(\exp(-\sigma_k)\bigr) = o(1/k).
            \label{eq:proof_integral_part1}
        \end{align}
        Moreover, we use the bound $\textrm{erfc}(x) \leq \exp(-x^2)/(x\sqrt{\pi})$ \citep{Karagiannidis2007} to obtain
        \begin{align}
            \sqrt{2\ln(k)\pi}\textrm{erfc}\bigl(\sqrt{2}(\gamma\sqrt{\sigma_k}-\sqrt{\ln(k)})\bigr)
            &= O\Bigl( \frac{\sqrt{\ln(k)}\exp(-\sigma_k)}{\sqrt{\sigma_k}-\sqrt{\ln(k)}}\Bigr)
            = o(1/k).
            \label{eq:proof_integral_part2}
        \end{align}
        Substitution of \eqref{eq:proof_integral_part1} and \eqref{eq:proof_integral_part2} into \eqref{eq:proof_integral_subs} gives the result.
        \Halmos
    \end{proof}

We next present the proof of Lemma~\ref{lem:bnd_overestimation_suf_samples}.
\begin{proof}{Proof of Lemma~\ref{lem:bnd_overestimation_suf_samples}.}
    Let $\lceil C_\beta + u_k \rceil \leq s \leq k$.
    Recall from \eqref{eq:def_reward_action_a} and \eqref{eq:Uestimator_episodes} that
    \begin{align}
        \P\bigl(\overline{U}^a[s,k] - r^a  \geq d^a -\xi_k, \ E_k^a \bigr)
        &= \P\Bigl(\sum_{(ij)\in\calL^a} x_{ij}^a \Bigl(\overline{\theta}_{ij}\bigl[\overline{T}_{ij}^a[s]\bigr] - \theta_{ij} + \sqrt{\frac{\ln(k)}{\overline{T}_{ij}^a[s]}}\Bigr) \geq d^a -\xi_k, \ E_k^a  \Bigr).
        \label{eq:est_per_line}
    \end{align}
    From Boole's inequality,\footnote{For any random variables $X,Y$ and any $z\geq 0$, $\{X+Y \geq z\} \subseteq \{X \geq z/2\} \cup \{Y \geq z/2\}$. Therefore $\P(X+Y\geq z) \leq \P(\{X \geq z/2\} \cup \{Y\geq z/2\})$.
    Applying Boole's inequality  gives $\P(X+Y \geq z) \leq \P(X \geq z/2) + \P(Y \geq z/2)$.}
    we obtain that for any $n,m,z \geq 0$,
    \begin{align}
        \P\Bigl(\sum_{(ij)\in\calL^a} x_{ij}^a \Bigl(\overline{\theta}_{ij}[n] - \theta_{ij} + \sqrt{\frac{\ln(m)}{n}}\Bigr) \geq z \Bigr)
        &\leq \sum_{(ij)\in\calL^a} \P\Bigl(\overline{\theta}_{ij}[n] - \theta_{ij} + \sqrt{\frac{\ln(m)}{n}} \geq  \frac{z}{x_{ij}^a |\calL^a|}\Bigr).
        \label{eq:est_boole}
    \end{align}
    Since $E_k^a$ can be expressed as a union of disjoint events of the form $\{\overline{T}_{ij}^a[s]=m\}$, it follows from \eqref{eq:est_per_line} and \eqref{eq:est_boole} that
    \begin{align}
        \P\bigl(\overline{U}^a[s,k] - r^a  \geq d^a -\xi_k, \ E_k^a \bigr)
        &\leq
        \sum_{(ij)\in\calL^a} \P\Bigl(\overline{\theta}_{ij}\bigl[\overline{T}_{ij}^a[s]\bigr] - \theta_{ij} + \sqrt{\frac{\ln(k)}{\overline{T}_{ij}^a[s]}} \geq  \frac{d^a -\xi_k}{x_{ij}^a |\calL^a|}, \ E_k^a\Bigr).
    \end{align}
    Note that $\overline{T}_{ij}^a[\cdot]$ is non decreasing and that $s\geq \lceil C_\beta + u_k\rceil$, so on \eqref{eq:def_event_E_lem}, $\overline{T}_{ij}^a[s] \geq \lceil\sigma_{ijk}^a\rceil$ for all $(ij)\in\calL^a$. Therefore
    \begin{align}
        \P\bigl(\overline{U}^a[s,k] - r^a  \geq d^a -\xi_k, \ E_k^a \bigr)
        &\leq
        \sum_{(ij)\in\calL^a} \P\Bigl(\overline{\theta}_{ij}\bigl[\overline{T}_{ij}^a[s]\bigr] - \theta_{ij} + \sqrt{\frac{\ln(k)}{\overline{T}_{ij}^a[s]}} \geq  \frac{d^a -\xi_k}{x_{ij}^a |\calL^a|}, \ \overline{T}_{ij}[s] \geq \lceil \sigma_{ijk}^a\rceil \Bigr).
    \end{align}
    By partitioning the event based on the value of $\overline{T}_{ij}^a[s]$, we find that
    \begin{align}
        \P\bigl(\overline{U}^a[s,k] - r^a  \geq d^a -\xi_k, \ E_k^a \bigr)
        &\leq  \sum_{(ij)\in\calL^a} \sum_{m=\lceil\sigma_{ijk}^a\rceil}^\infty \P\Bigl(\overline{\theta}_{ij}\bigl[m\bigr] - \theta_{ij} + \sqrt{\frac{\ln(k)}{m}} \geq  \frac{d^a -\xi_k}{x_{ij}^a |\calL^a|}\Bigr).
        \label{eq:est_sum_partition}
    \end{align}
    Since the right side of \eqref{eq:est_sum_partition} is independent of $s$, we have
    \begin{align}
        \sum_{s= \lceil C_\beta + u_k \rceil}^k \P\bigl(\overline{U}^a[s,k] - r^a  \geq d^a -\xi_k, \ E_k^a\bigr)
        &\leq k \sum_{(ij)\in\calL^a} \sum_{m=\lceil\sigma_{ijk}^a\rceil}^\infty \P\Bigl(\overline{\theta}_{ij}\bigl[m\bigr] - \theta_{ij} + \sqrt{\frac{\ln(k)}{m}} \geq  \frac{d^a -\xi_k}{x_{ij}^a |\calL^a|}\Bigr).
        \label{eq:est_sum_partition_times_k}
    \end{align}
    Note that by construction $\sigma_{ijk}^a = \omega(\ln(k))$ (recall \eqref{eq:def_seq_u_eta_sig_lem}) and $\xi_k = o(1)$, so that $(d^a -\xi_k)/(x_{ij}^a |\calL^a|) > \sqrt{\ln(k)/\sigma_{ijk}^a}$ for $k\in\N_{\geq 1}$ sufficiently large. For such $k$, we have by (i) Hoeffding's inequality \cite[Theorem 2]{Hoeffding1963} and (ii) Lemma~\ref{lem:integral} that
    \begin{align}
        \sum_{(ij)\in\calL^a} \sum_{m=\lceil\sigma_{ijk}^a\rceil}^\infty \P\Bigl(\overline{\theta}_{ij}\bigl[m\bigr] - \theta_{ij} + \sqrt{\frac{\ln(k)}{m}} \geq  \frac{d^a -\xi_k}{x_{ij}^a |\calL^a|}\Bigr)
        &\stackrel{\textrm{(i)}}{\leq}
        \sum_{(ij)\in\calL^a} \sum_{m=\lceil\sigma_{ijk}^a\rceil}^\infty \exp\Bigl( -2m \Bigl(\frac{d^a -\xi_k}{x_{ij}^a |\calL^a|} - \sqrt{\frac{\ln(k)}{m}}\Bigr)^2\Bigr) \\
        &\leq
        \sum_{(ij)\in\calL^a} \int_{\sigma_{ijk}^a}^\infty \exp\Bigl( -2m \Bigl(\frac{d^a -\xi_k}{x_{ij}^a |\calL^a|} - \sqrt{\frac{\ln(k)}{m}}\Bigr)^2\Bigr) \textrm{d}m\\
        &\stackrel{\textrm{(ii)}}{=} o(1/k).
        \label{eq:est_asymp_bnd}
    \end{align}
    The result follows from \eqref{eq:est_sum_partition_times_k} and \eqref{eq:est_asymp_bnd}.
    \Halmos
\end{proof}

\newpage
\subsection{Proof of Lemma~\ref{lem:prob_underestimation}}
\label{app:prob_underestimation}

\begin{proof}{}
    By definitions \eqref{eq:def_reward_action_a}, \eqref{eq:UCB_update_algo_U_per_line}, and \eqref{eq:UCB_def_action}, we have that for any $m=1,\dots,k$,
    \begin{align} 
        \P(U^a(m) < r^a - \xi) &=
        \P\Bigl( \sum_{(ij)\in\calL^a} x_{ij}^a \Bigl(\hat{\theta}_{ij}(m) - \theta_{ij} + \sqrt{\frac{ \ln(m)}{T_{ij}(m)}} \Bigr)  < - \xi \Bigr).
    \end{align}
    We use the inequality $\P(A+B < z) \leq \P(A<z)+\P(B<z)$ for any events $A,B$ and any $z\in\R$ to write
    \begin{align}
        \P\bigl(U^a(m) < r^a - \xi\bigr) &\leq \sum_{(ij)\in\calL^a} \P\Bigl( \hat{\theta}_{ij}(m) - \theta_{ij} + \sqrt{\frac{ \ln(m)}{T_{ij}(m)}}   < - \frac{\xi}{x_{ij}^a} \Bigr).
    \end{align}
    By partitioning the event based on the value of $T_{ij}(m)$, we find that
    \begin{align}
        \P(U^a(m) < r^a - \xi)&\leq \sum_{(ij)\in\calL^a} \sum_{s=1}^\infty \P\Bigl( \hat{\theta}_{ij}(m) - \theta_{ij} + \sqrt{\frac{ \ln(m)}{T_{ij}(m)}}   < - \frac{\xi}{x_{ij}^a}, \ T_{ij}(m) = s \Bigr) \\
        &\leq \sum_{(ij)\in\calL^a} \sum_{s=1}^\infty \P\Bigl( \overline{\theta}_{ij}[s] - \theta_{ij} + \sqrt{\frac{ \ln(m)}{s}}   < - \frac{\xi}{x_{ij}^a} \Bigr).
    \end{align}
    Since the distribution of the payoff parameters are $[0,1]$-bounded and independent, we can apply Hoeffding's inequality \cite[Theorem 2]{Hoeffding1963} to obtain
    \begin{align}
        \P(U^a(m) < r^a - \xi)&\leq \sum_{(ij)\in\calL^a} \sum_{s=1}^\infty  \exp\Bigl(- 2 s \Bigl(\frac{\xi}{x_{ij}^a} +  \sqrt{\frac{ \ln(m)}{s}}\Bigr)^2 \Bigr) \\
        &\leq \sum_{(ij)\in\calL^a} \sum_{s=1}^\infty  \exp\Bigl(- 2 s \Bigl(\frac{\xi}{x_{ij}^a}\Bigr)^2 -2  \ln(m) \Bigr) \\
        &= \frac{1}{m^{2}}\sum_{(ij)\in\calL^a}  \sum_{s=1}^\infty  \exp\Bigl(- 2 s \Bigl(\frac{\xi}{x_{ij}^a}\Bigr)^2 \Bigr).
    \end{align}
    We use that for any $y > 0$ we have
    \begin{align}
        \sum_{s=1}^\infty \exp(-sy) = \frac{1}{\exp(y)-1} \leq \frac 1y,
    \end{align}
    to write
    \begin{align}
        \P(U^a(m) < r^a - \xi)&\leq \frac{1}{m^{2}} \sum_{(ij)\in\calL^a}  \frac{(x_{ij}^a)^2}{2\xi^2}
        \leq \frac{\|\lambda\|_2^2}{2 m^{2} \xi^2},
    \end{align}
    where we used that $x_{ij}^a \leq \lambda_i$ for all $(ij) \in\calL^a$.
    Taking the sum over all episodes gives
    \begin{align}
        \sum_{m=1}^k \P(U^a(m) < r^a - \xi)
        &\leq
        \sum_{m=1}^\infty \P(U^a(m) < r^a - \xi)
        \leq \sum_{m=1}^\infty \frac{\|\lambda\|_2^2}{2 m^{2} \xi^2}
        = \frac{\pi^2 \|\lambda\|_2^2}{12\xi^2}.
    \end{align}
    \Halmos
\end{proof}

\newpage
\subsection{Proof of Lemma~\ref{lem:n_bad_episodes}}
\label{app:n_bad_episodes}
\begin{proof}{}
    Let $k\in\N_{\geq 1}$ be such that $k\geq C_\beta$.
    Let $a\in\calA\setminus\{1\}$ be a suboptimal action and $x^a$ its corresponding nondegenerate basic feasible solution of $\LP{\theta,\eps}$ in \eqref{eq:LP_eps}.
    Starting from \eqref{eq:Sak_def}, we split the summation based on whether the UCB index $U^a(m)$ is smaller or larger than the true optimal mean $r^1$ minus some small error $\xi_k := \ln^{-\frac{\beta}{4}}(k)$:
    \begin{align}
        S^a(k) &= \sum_{m=1}^k \ind{U^a(m) < r^1 - \xi_k, \ A_m = a} + \sum_{m=1}^k \ind{U^a(m) \geq r^1 - \xi_k, \ A_m = a}.
        \label{eq:Sak_split}
    \end{align}
    By construction, action $a$ can be chosen only if the UCB index $U^a$ is as large as the indices of the other actions, including the optimal action $U^1$ (Line \ref{line:maxucb} in Algorithm \ref{alg:learning_alg}). Therefore,
    \begin{align}
        \sum_{m=1}^k \ind{U^a(m) < r^1 - \xi_k, \ U^1(m) \leq U^a(m)} 
        &\leq \sum_{m=1}^k \ind{U^1(m) < r^1 - \xi_k}.
    \end{align}
    Taking the expectation of \eqref{eq:Sak_split} yields
    \begin{align}
        \label{eq:ESak_bnd}
        \E(S^a(k))
        &\leq
        \sum_{m=1}^k \P(U^1(m) < r^1 - \xi_k) + \sum_{m=1}^k \P(U^a(m) \geq r^1 - \xi_k, \ A_m = a).
    \end{align}
    We will next analyze the two terms on the right side of \eqref{eq:ESak_bnd} more carefully.

    Note that $\xi_k \to 0$ as $k\to\infty$, so $\xi_k < r^1$ for $k$ sufficiently large. 
    From Lemma~\ref{lem:prob_underestimation}, and using the definition of $\xi_k$, we find
    \begin{align}
        \lim_{k\to\infty} \frac{\sum_{m=1}^k \P(U^1(m) < r^1 - \xi_k) }{\ln^\beta (k)}
        &\leq
        \lim_{k\to\infty} \frac{\pi^2 \|\lambda\|_2^2 \ln^{\beta/2}(k)}{12\ln^\beta (k) }
        = 0.
        \label{eq:UCB_underestimation_asymp_bnd}
    \end{align}

    For the second term in \eqref{eq:ESak_bnd}, we note that the summation can be simplified as follows.
    Suppose that $m_1\in\N_{\geq 1}$ is the first episode in which Algorithm 1 chooses action $a$ in Line~\ref{line:maxucb}. Then  $U^a(m_1) = \overline{U}^a[1,m_1]$.
    Similarly, $U^a(m_2) = \overline{U}^a[2,m_2]$ where $m_2$ is the second episode in which Algorithm \ref{alg:learning_alg} chooses action $a$; so on and so on.
    Moreover, for fixed a sample path, $\overline{U}[s,k]$ is nondecreasing in $k$ (recall  \eqref{eq:Uestimator_episodes}).
    Therefore,
    \begin{align}
        \sum_{m=1}^k \P(U^a(m) \geq r^1 - \xi_k, A_m = a)
        &\leq \sum_{s=1}^k \P\bigl(\overline{U}^a[s,k] \geq r^1 - \xi_k\bigr).
    \end{align}
    Substitution of $r^1 = r^a + d^a$ gives
    \begin{align}\label{eq:prob_split_overestimate}
        \sum_{m=1}^k \P(U^a(m) \geq r^1 - \xi_k, A_m = a)
        &\leq \sum_{s=1}^k \P\bigl(\overline{U}^a[s,k] - r^a \geq d^a - \xi_k\bigr).
    \end{align}
    We bound the probability of the first $C_\beta + u_k$ episodes in \eqref{eq:prob_split_overestimate} by one,
    \begin{align}\label{eq:Uam_bnd}
        \sum_{m=1}^k \P(U^a(m) \geq r^1 - \xi_k, A_m = a)
        &\leq C_\beta + u_k + \sum_{s= \lceil C_\beta + u_k \rceil}^k \P\bigl(\overline{U}^a[s,k] - r^a  \geq d^a - \xi_k\bigr).
    \end{align}
    By  \eqref{eq:def_event_E_lem} and the law of total probability, we obtain
    \begin{align}
        \sum_{s= \lceil C_\beta + u_k \rceil}^k \P\bigl(\overline{U}^a[s,k] - r^a  \geq d^a - \xi_k\bigr)
        &\leq
        \sum_{s= \lceil C_\beta + u_k \rceil}^k \P\bigl(\overline{U}^a[s,k] - r^a  \geq d^a - \xi_k, \ E_k^a\bigr) +
        \sum_{s= \lceil C_\beta + u_k \rceil}^k \P((E_k^a)^c).
        \label{eq:Uam_split}
    \end{align}

    From Lemma~\ref{lem:prob_insuf_samples} it follows that
    \begin{align}
        \lim_{k\to\infty}\sum_{s= \lceil C_\beta + u_k \rceil}^k \P((E_k^a)^c)
        &\leq
        \lim_{k\to\infty} k \P((E_k^a)^c)
        =
        0.
        \label{eq:prob_Eijac_astymp}
    \end{align}
    Moreover, Lemma~\ref{lem:bnd_overestimation_suf_samples} implies immediately that the middle  term in~\eqref{eq:Uam_split} converges to zero, hence
    \begin{align}
        \lim_{k\to\infty} \sum_{s= \lceil C_\beta + u_k \rceil}^k \P\bigl(\overline{U}^a[s,k] - r^a  \geq d^a - \xi_k\bigr) &= 0.
        \label{eq:UCB_underestimation_asymp_bnd2}
    \end{align} 

    We continue to bound \eqref{eq:Sak_split}, 
    \begin{align}
        \lim_{k\to\infty} \frac{\E(S^a(k))}{\ln^\beta(k)}
        &\stackrel{\textrm{\eqref{eq:ESak_bnd}}}{\leq}
        \lim_{k\to\infty} \frac{\sum_{m=1}^k \P(U^1(m) < r^1 ) + \sum_{m=1}^k \P(U^a(m) \geq r^1 , \ A_m = a)}{\ln^\beta(k)} \\
        &\leq
        \lim_{k\to\infty} \frac{\sum_{m=1}^k \P(U^1(m) < r^1 )}{\ln^\beta(k)} + \frac{C_\beta + u_k + \sum_{s= \lceil C_\beta + u_k \rceil}^k \P\bigl(\overline{U}^a[s,k] - r^a  \geq d^a - \xi_k\bigr)}{\ln^\beta(k)}.
        \label{eq:ucbqr_limit_twoterms}
    \end{align}
    Now, by~\eqref{eq:UCB_underestimation_asymp_bnd},
    \begin{align}
        \lim_{k\to\infty} \frac{\sum_{m=1}^k \P(U^1(m) < r^1 )}{\ln^\beta(k)} &= 0.
    \end{align}
    Moreover, $C_\beta$ is a constant and $\lim_{k\to\infty} u_k /\log^\beta (k) = 1$ (recall~\eqref{eq:def_seq_u_eta_sig_lem}).
    Lastly, by~\eqref{eq:UCB_underestimation_asymp_bnd2},
    \begin{align}
        \lim_{k\to\infty} \frac{\sum_{s= \lceil C_\beta + u_k \rceil}^k \P\bigl(\overline{U}^a[s,k] - r^a  \geq d^a - \xi_k\bigr)}{\log^\beta(k)} &= 0.
    \end{align} 
    So we find that
    \begin{align}
        \lim_{k\to\infty} \frac{\E(S^a(k))}{\ln^\beta(k)}
        &= 1,
    \end{align}
    which concludes the proof. 
    \Halmos
\end{proof}